\documentclass[a4paper,12pt,times,custombib,print,index]{Classes/PhDThesisPSnPDF}

\input{Preamble/preamble}

\title{Scalable Gaussian Processes}
\subtitle{Advances in Iterative Methods and Pathwise Conditioning}

\author{Jihao Andreas Lin}

\dept{Department of Engineering}

\university{University of Cambridge}
\crest{\includegraphics[width=0.2\textwidth]{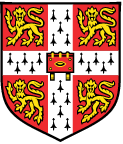}}
\degreetitle{Doctor of Philosophy}

\college{Queens' College}

\degreedate{May 2025} 

\ifdefineAbstract
\fi

\ifdefineChapter
\fi

\begin{document}

\frontmatter

\maketitle

% ******************************* Thesis Declaration ***************************

\begin{declaration}

This thesis is the result of my own work and includes nothing which is the outcome of work done in collaboration except as declared in the preface and specified in the text. It is not substantially the same as any work that has already been submitted, or, is being concurrently submitted, for any degree, diploma or other qualification at the University of Cambridge or any other University or similar institution except as declared in the preface and specified in the text. It does not exceed the prescribed word limit for the relevant Degree Committee.

% Author and date will be inserted automatically from thesis.tex \author \degreedate

\end{declaration}
% ************************** Thesis Abstract *****************************
% Use `abstract' as an option in the document class to print only the titlepage and the abstract.
\begin{abstract}
Driven by advances in hardware for massively-parallel computation, machine learning models trained on large amounts of data have become capable of accomplishing complex tasks, such as generating realistic images or maintaining conversations in natural language.
However, the inability to \emph{know when they don't know} often leads to overconfidence and hallucinations.

Gaussian processes are a powerful framework for uncertainty-aware function approximation and sequential decision-making.
Unfortunately, their classical formulation does not scale gracefully to large amounts of data and modern hardware for massively-parallel computation, prompting many researchers to develop techniques which improve their scalability.

This dissertation focuses on the powerful combination of iterative methods and pathwise conditioning to develop methodological contributions which facilitate the use of Gaussian processes in modern large-scale settings.
By combining these two techniques synergistically, expensive computations are expressed as solutions to systems of linear equations and obtained by leveraging iterative linear system solvers.
This drastically reduces memory requirements, facilitating application to significantly larger amounts of data, and introduces matrix multiplication as the main computational operation, which is ideal for modern hardware.

In particular, this dissertation introduces stochastic gradient algorithms as a computationally efficient method to solve linear systems iteratively.
To this end, custom optimisation objectives, stochastic gradient estimators, and variance reduction techniques are developed and analysed.
Empirically, the proposed methods achieve state-of-the-art performance on large- scale regression, Bayesian optimisation, and molecular binding affinity prediction tasks.

Additionally, generic improvements, which are applicable to any iterative linear system solver in the context of Gaussian processes, are contributed, leading to computational speed-ups of up to $72\times$ compared to established approaches.
Furthermore, iterative methods and pathwise conditioning are combined with structured linear algebra techniques to attain even greater scalability, which is demonstrated on real-world datasets with up to five million examples, including robotics, automated machine learning, and climate modelling applications.
\end{abstract}

% ************************** Thesis Acknowledgements **************************

\begin{acknowledgements}
My academic journey at the University of Cambridge has been a transformative experience, which would not have been possible without the support and guidance from many individuals.

First and foremost, I want to thank my supervisor José Miguel Hernández-Lobato for guiding me through my journey as a PhD student.
Miguel always provided me with ample academic freedom, unequivocal feedback, and efficacious advice, all of which I truly appreciate.
He also encouraged me to attend scientific conferences, expand my professional network, and collaborate with others, which led to numerous eye-opening and life-changing experiences.

This brings me to the next group of people who I want to thank: my outstanding collaborators.
In particular, I want to start by expressing my deepest gratitude to Javier Antorán for playing the role of a second, more hands-on supervisor for me.
His endless amounts of enthusiasm, cleverness, and optimism continue to inspire me.
Next, I would like to thank Shreyas Padhy for working closely together on a daily basis and sharing the joys and challenges of research.
Our pair programming days have become a nostalgic memory of mine.
Additionally, I am grateful to David Janz and Csaba Szepesvári for all of their bountiful theoretical insights, and I want to thank Alexander Terenin for teaching me about scientific writing and presentation.
Furthermore, I want to thank Sebastian Ament, Maximilian Balandat, David Eriksson, and Eytan Bakshy.
Working together provided me with many novel perspectives and insights about industry research in comparison to academia.
I also want to thank Austin Tripp, from whom I learned about computational chemistry, and Gergely Flamich, who taught me about neural compression.
In addition to being collaborators, I am grateful to Bruno Mlodozeniec and Kenza Tazi for being good friends and lab mates, which leads me to the next category.

I want to thank Runa Eschenhagen, for teaching me about second-order optimisation and how to improve my bench press,
Bruno Mlodozeniec, for proving complicated theorems and making great coffee for me, Kenza Tazi for sharing my passion for Gaussian processes and always bringing a smile into the office,
Juyeon Heo, for helping me organise my thoughts and motivating me when facing challenging situations, and Isaac Reid, for working on a coursework project and leading the MLG reading group together in our first PhD year.

Additionally, I also want to thank James Allingham, for calming me down during stressful times and providing the best food recommendations,
Jonathan So, for generously ordering dinner for me when working late,
Yichao Liang, for sharing various kinds of tasty snacks with me,
and Tony OuYang, for keeping me up to date on the latest football match.

Before moving on to individuals outside of the research lab, I want to thank Carl Rasmussen, Richard Turner, and Hong Ge for fostering a collaborative research environment.
I am also grateful to Hong Ge for being my advisor, and Kimberly Cole and Francesca McCaughan for their daily administrative support.

Furthermore, I am grateful to Noah Winkel, Daniel Fey, Lukas Schlitt, Sven Kantwill, Leonie Schenck, Christian Pöpel and Moritz Mathes, for being friends since middle school, Timm Behnecke and Jens Geisse, for our fascinating philosophical discussions, Veronika Kaletta and Marius van der Wijden, for being my personal counsellors, and Daniel Vasconcelos, Tom Kayser, and Andre Dinter, for our entertaining adventures together.

I am also deeply indebted to numerous individuals who influenced my academic journey before starting my PhD, including Peter Theissen, Frank Bühler, and Gerhard Röhner, who provided me with solid foundations to study computer science, Horst Fey, who vouched for me by recommending me to Stefan Roth, who offered me to become his research assistant during the first year of my undergraduate studies, Jochen Gast, who convinced me to learn the Python programming language and patiently explained machine learning and computer vision topics to me, Frank Wood, who introduced me to the world of Bayesian machine learning, and Joe Watson and Pascal Klink, who supervised me during my time as a MSc student in the research lab of Jan Peters.
Joe and Pascal taught me that successful research involves curiosity and creativity, but also requires perseverance at times. 
Without the immense support from all of these individuals, it would not have been possible for me to pursue my PhD.

Speaking about prerequisites for my PhD, I am grateful to David and Claudia Harding, who funded my studies through the Harding Distinguished Postgraduate Scholars Programme.
It is not an exaggeration to claim that, without their generous philanthropy, it would not have been possible for me to pursue my PhD.
In this context, I also want to thank Queens' College and the corresponding donors, and the Department of Engineering for numerous academic travel grants, which allowed me to attend international conferences to present my work.

My research has been performed using resources provided by the University of Cambridge Research Computing Service, funded by EPSRC Tier-2 capital grants and my supervisor, which I am grateful for.
I thank Stuart Rankin for maintaining these computing facilities.

Last but certainly not least, I thank my parents and my sister for their unwavering support.
\end{acknowledgements}

% *********************** Adding TOC and List of Figures ***********************

\tableofcontents

% \listoffigures

% \listoftables

% \printnomenclature[space] space can be set as 2em between symbol and description
%\printnomenclature[3em]

% \printnomenclature

% ******************************** Main Matter *********************************
\mainmatter

%!TEX root = ../thesis.tex
%*******************************************************************************
%*********************************** First Chapter *****************************
%*******************************************************************************

\chapter{Introduction}
\label{chap:intro}

\ifpdf
    \graphicspath{{Chapter1/Figs/Raster/}{Chapter1/Figs/PDF/}{Chapter1/Figs/}}
\else
    \graphicspath{{Chapter1/Figs/Vector/}{Chapter1/Figs/}}
\fi
Recently, machine learning undoubtedly transformed the world, improving capabilities in areas such as computer vision or natural language understanding at an unprecedented pace.
Powered by advances in hardware for massively-parallel computation, learning from large amounts of data has become the norm.
Nonetheless, common failure modes include erratic behaviour and hallucinations, often in combination with overconfidence, obstructing the use of machine learning models in safety-critical environments, where calibrated uncertainty quantification is crucial to prevent serious consequences.

Providing data-driven models with the ability to \emph{know when they don't know} has historically been the domain of probabilistic modelling and Bayesian inference, which leverage principled mathematics to incorporate prior knowledge and quantify uncertainty. 
Gaussian processes are a prominent member in this category, because they can be used for flexible uncertainty-aware function approximation.
However, the classical formulation of Gaussian processes requires the inverse of a large positive-definite matrix, which is typically computed via matrix decomposition.
This does not scale well to large amounts of data and modern hardware.

To improve the scalability of Gaussian processes, researchers have proposed a multitude of different approaches, including sparse and variational methods \citep{candela2005,titsias09,titsias2009report,hensman13,wu2022variational,wenger2022posterior,wenger2024computation}, random feature approximations \citep{rahimi08,sutherland15,li2019towards,tripp2023tanimoto}, approaches exploiting linear algebraic structure \citep{bonilla2007multi,stegle2011efficient,wilson2015kernel,gardner2018product,kapoor2021skiing,lin2024scaling,lin2025scalable}, and iterative methods using linear system solvers \citep{gardner18,WangPGT2019exactgp,lin2023sampling,lin2024stochastic,lin2024warm,lin2024improving,lin2024scaling,lin2025scalable}.
An overview of prevalent approaches will be given in \Cref{sec:scalable_inference}.

This dissertation focuses on the powerful combination of \emph{iterative methods} and \emph{pathwise conditioning} \citep{wilson20,wilson21} to develop novel methodological contributions which facilitate the use of Gaussian processes in modern large-scale settings.
The fundamental idea of iterative methods is to express computationally expensive terms as solutions to systems of linear equations with a positive-definite coefficient matrix.
The solutions to these linear systems are then obtained via optimisation of corresponding convex quadratic objectives using iterative gradient-based optimisers, circumventing explicit matrix decompositions.

Iterative methods have several intrinsic advantages.
Their iterative nature implies that computations are amenable to mini-batching, which substantially reduces memory requirements.
Additionally, iterative methods rely on matrix multiplications instead of matrix decompositions, which is more suitable for modern hardware.
Furthermore, iterative methods generally perform approximate inference in the exact Gaussian process model up to a specified numerical tolerance, whereas other methods for scalable inference often perform inference in an approximate model.
Moreover, iterative methods can be combined with other approaches for scalable inference, resulting in even greater scalability.
For example, \Cref{chap:sgd} combines them with sparse methods and \Cref{chap:lkgp} combines them with structured linear algebra.

Pathwise conditioning \citep{wilson20,wilson21} is a method to generate samples from a Gaussian process posterior.
However, to understand and appreciate the benefits of pathwise conditioning, it is necessary to introduce some context about posterior inference first.

In standard Gaussian process regression, which will be discussed in \Cref{sec:gp_regression}, posterior inference consists of updating the Gaussian process prior based on observed data to obtain the posterior.
Although the posterior of a Gaussian process is technically a distribution over functions, the conventional approach considers its evaluation at a finite number of locations, resulting in a finite-dimensional posterior distribution which is explicitly represented by its mean and covariance matrix.
The process of computing the posterior mean and covariance matrix at selected locations is typically regarded as \emph{making predictions}, especially if the selected locations are not included in the observed data.

While the posterior mean and covariance matrix are often useful by themselves, there are certain situations where samples from the posterior are required.
For example, many acquisition functions in Bayesian optimisation \citep{Snoek2012,garnett2023bobook} can be expressed as the expected value of some function with respect to an input which follows a Gaussian process posterior.
However, it is generally intractable to derive analytical expressions for such expected values.
Instead, an unbiased estimate can be obtained via Monte Carlo integration, which computes an average over function values evaluated at samples from the Gaussian process posterior, requiring actual samples rather than its mean and covariance matrix.

The conventional way of generating a sample from a Gaussian process posterior consists of evaluating the posterior mean and covariance matrix, followed by computing the Cholesky decomposition of the latter, which is used to apply an affine transformation to a sample from a standard normal distribution.
Details of this procedure will be discussed in \Cref{sec:generating_posterior_samples}.
Since the posterior mean and covariance matrix depend on the predetermined locations, the whole procedure must be repeated to evaluate the sample at other locations, which can be computationally expensive.
In the same example as before, acquisition functions are typically maximised using iterative optimisation, leading to the evaluation of posterior samples at many sequentially dependent locations and the repetition of the whole sampling procedure for each new set of locations.

In contrast to the conventional way of generating a posterior sample, pathwise conditioning expresses a sample from the Gaussian process posterior directly as a random function by applying a data-dependent update to a sample from the Gaussian process prior.
In particular, pathwise conditioning does not explicitly manifest the posterior mean or covariance matrix.
In terms of required computations, the standard way of drawing posterior samples requires the solution to one system of linear equations per location at which the sample shall be evaluated, whereas pathwise conditioning requires one linear system solution per sample, independent of the number of locations.
This comparison illustrates why pathwise conditioning is particularly useful if evaluations are necessary at many different locations, such as during acquisition function maximisation in Bayesian optimisation.
However, a subtle yet substantial caveat of pathwise conditioning is the fact that it requires a sample from the Gaussian process prior, which can be computationally demanding on its own.
Nonetheless, effective approximations such as random features exist, which will be discussed in \Cref{sec:random_features}.
Furthermore, a more detailed discussion of pathwise conditioning is provided in \Cref{sec:generating_posterior_samples}.

While iterative methods and pathwise conditioning already have their own individual merits, their combination is particularly powerful, and also the overall topic of this dissertation.
In particular, iterative methods can be used to solve a single large system of linear equations, and the resulting posterior sample can then be evaluated at many different input locations via pathwise conditioning.
This synergy allows posterior inference in Gaussian processes to scale to massive amounts of observed data and large numbers of locations for prediction.
Throughout this dissertation, all original methodological contributions leverage this powerful and synergistic combination in different ways to improve the scalability of Gaussian processes.
The following section provides brief summaries of these contributions and an outline of the remaining dissertation.

\section{Outline and Contributions}
This section provides an outline of this dissertation and its contributions.
In the following, and throughout this dissertation, I write \emph{we} in a communal way, to refer to myself, but also to refer to my co-authors, the reader, and sometimes the scientific community in general.

\begin{itemize}
    \item \Cref{chap:background} provides the necessary background and concepts for the following chapters.
    It is organised into two sections, where the first section formally introduces Gaussian processes, including posterior inference, pathwise conditioning, commonly used covariance functions, and the basics of model selection via marginal likelihood optimisation.
    The second section discusses related work on scalable inference in Gaussian processes which is relevant for the contributions of this dissertation.
    \item \Cref{chap:sgd} introduces stochastic gradient descent as a novel iterative linear system solver for Gaussian processes which, in combination with pathwise conditioning, facilitates approximate inference with asymptotically linear time and memory requirements, in contrast to respectively cubic and quadratic requirements of conventional approaches.
    We develop low-variance optimisation objectives to draw samples from the posterior and further extend these to inducing points for even greater scalability.
    Additionally, we observe that stochastic gradient descent often produces accurate predictions, even in cases where it does not quickly converge to the optimum.
    We explain this through a spectral characterisation of its implicit bias, showing that our algorithm produces predictive distributions close to the true posterior both in regions with sufficient data coverage, and in regions sufficiently far away from the data.
    Furthermore, we demonstrate empirically that stochastic gradient descent achieves state-of-the-art performance on sufficiently large-scale or ill-conditioned regression tasks, and a large-scale Bayesian optimisation problem.
    This chapter is based on \citet{lin2023sampling}.
    \item \Cref{chap:sdd} builds on the previous chapter to develop stochastic \emph{dual} descent for Gaussian processes.
    We introduce a dual optimisation objective, which shares the same unique minimiser, but exhibits more favourable curvature properties, allowing for significantly larger step sizes, faster convergence, and reduced computational costs.
    Additionally, we compare and analyse different stochastic gradient estimation, momentum acceleration, and iterate averaging techniques, leading to recommendations with lower variance and superior convergence properties.
    Empirically, we improve upon the state-of-the-art from the previous chapter, and further demonstrate that stochastic dual descent places Gaussian process regression on par with state-of-the-art graph neural networks on a molecular binding affinity prediction task.
    This chapter is based on \citet{lin2024stochastic}.
    \item \Cref{chap:solvers} makes generic contributions for iterative linear system solvers in the context of marginal likelihood optimisation for Gaussian processes, which are applicable to any iterative linear system solver.
    We introduce a pathwise gradient estimator, which reduces the required number of solver iterations and amortises the computational costs of drawing posterior samples by leveraging a connection to pathwise conditioning.
    Additionally, we propose to initialise linear system solvers with intermediate results, leading to faster convergence at the cost of introducing negligible bias.
    Furthermore, we investigate the behaviour of linear system solvers on a limited computational budget, stopping them before reaching convergence, which is common practice in large-scale settings.
    Empirically, we show that our techniques provide speed-ups of up to $72\times$ when solving until convergence, and decrease the average residual norm by up to $7\times$ when stopping early.
    This chapter is based on \citet{lin2024improving} and \citet{lin2024warm}.
    \item \Cref{chap:lkgp} considers Gaussian process regression with Kronecker product-structured kernel matrices, which enable scalable inference via factorised matrix decompositions, but usually require fully gridded data from a Cartesian product space to be applicable.
    To lift this limitation, we propose \emph{latent} Kronecker structure, expressing the covariance matrix of observed values as the projection of a latent Kronecker product.
    However, due to the introduced projections, the structure can no longer be exploited via factorised matrix decompositions.
    Since the structure still permits fast matrix multiplication, we apply iterative linear system solvers and pathwise conditioning for scalable inference, requiring substantially fewer computational resources than standard iterative methods without latent Kronecker structure.
    Additionally, we derive an asymptotic break-even point, which quantifies the level of sparsity up to which latent Kronecker structure will be efficient, and demonstrate empirically that the formula is accurate.
    Furthermore, we conduct experiments on real-world datasets with up to five million examples, including robotics, automated machine learning, and climate modelling applications, and show that our method outperforms state-of-the-art sparse and variational approaches.
    This chapter is based on \citet{lin2024scaling} and \citet{lin2025scalable}.
    \item \Cref{chap:conclusion} concludes this dissertation by summarising original contributions from the previous chapters, discussing their key insights, notable takeaways, and limitations.
    Finally, I provide an overview and brief discussion of potential directions for future research, including further improvements for stochastic gradient descent and generic improvements for iterative linear system solvers.
\end{itemize}

\section{List of Publications}
This section enumerates, in chronological order, all peer-reviewed publications which I have written or contributed to while working towards this dissertation.
Titles of publications whose content is included in this dissertation, and my own name are boldfaced.
Superscript asterisks denote equal contribution and shared first authorship.

\begin{enumerate}
    \item \textbf{J. A. Lin}, J. Antorán, and J. M. Hernández-Lobato. Online Laplace Model Selection Revisited. In \emph{Advances in Approximate Bayesian Inference}, 2023. (\textbf{Contributed Talk})
    \item K. Tazi, \textbf{J. A. Lin}, R. Viljoen, A. Gardner, T. John, H. Ge, and R. E. Turner. Beyond Intuition, a Framework for Applying GPs to Real-World Data. In \emph{ICML Structured Probabilistic Inference \& Generative Modeling Workshop}, 2023.
    \item \textbf{J. A. Lin}, G. Flamich, and J. M. Hernández-Lobato. Minimal Random Code Learning with Mean-KL Parameterization. In \emph{ICML Neural Compression Workshop}, 2023.
    \item \textbf{J. A. Lin}*, J. Antorán*, S. Padhy*, D. Janz, J. M. Hernández-Lobato, and A. Terenin. \textbf{Sampling from Gaussian Process Posteriors using Stochastic Gradient Descent}. In \emph{Advances in Neural Information Processing Systems}, 2023. (\textbf{Oral Presentation})
    \item \textbf{J. A. Lin}*, S. Padhy*, J. Antorán*, A. Tripp, A. Terenin, C. Szepesvári, J. M. Hernández-Lobato, and D. Janz. \textbf{Stochastic Gradient Descent for Gaussian Processes Done Right}. In \emph{International Conference on Learning Representations}, 2024.
    \item \textbf{J. A. Lin}, S. Padhy, B. Mlodozeniec, and J. M. Hernández-Lobato. \textbf{Warm Start Marginal Likelihood Optimisation for Iterative Gaussian Processes}. In \emph{Advances in Approximate Bayesian Inference}, 2024.
    \item \textbf{J. A. Lin}, S. Padhy, B. Mlodozeniec, J. Antorán, and J. M. Hernández-Lobato. \textbf{Improving Linear System Solvers for Hyperparameter Optimisation in Iterative Gaussian Processes}. In \emph{Advances in Neural Information Processing Systems}, 2024.
    \item \textbf{J. A. Lin}, S. Ament, M. Balandat, and E. Bakshy. \textbf{Scaling Gaussian Processes for Learning Curve Prediction via Latent Kronecker Structure}. In \emph{NeurIPS Bayesian Decision-making and Uncertainty Workshop}, 2024.
    \item \textbf{J. A. Lin}, S. Ament, M. Balandat, D. Eriksson, J. M. Hernández-Lobato, and E. Bakshy. \textbf{Scalable Gaussian Processes with Latent Kronecker Structure}. In \emph{International Conference on Machine Learning}, 2025. 
\end{enumerate}

%!TEX root = ../thesis.tex
%*******************************************************************************
%****************************** Second Chapter *********************************
%*******************************************************************************
\chapter{Background and Related Work}
\label{chap:background}

\ifpdf
    \graphicspath{{Chapter2/Figs/Raster/}{Chapter2/Figs/PDF/}{Chapter2/Figs/}}
\else
    \graphicspath{{Chapter2/Figs/Vector/}{Chapter2/Figs/}}
\fi

This chapter formally introduces Gaussian processes and related work on scalable inference.
To successfully follow the discussion of these topics, a basic understanding of linear algebra, multivariable calculus, and probability is required.
In particular, positive-definite matrices and the multivariate normal distribution will be very important.
In the following, the first section will introduce Gaussian processes and how to use them for regression, covariance functions and examples of popular choices thereof, and the marginal likelihood and its role in model selection.
We refer to \citet{rasmussen2006} for a more comprehensive introduction to Gaussian processes.
The second section will discuss relevant related work on performing scalable inference in Gaussian processes, including sparse and variational methods, random features, approaches involving structured kernel matrices, and applications of iterative linear system solvers.
Many of these methods for scalable inference will reappear in later chapters, either used by contributed methods or as baselines for comparison.

This chapter includes content which is adapted from the following publications:
\begin{itemize}
    \item J. A. Lin, J. Antorán, S. Padhy, D. Janz, J. M. Hernández-Lobato, and A. Terenin. Sampling from Gaussian Process Posteriors using Stochastic Gradient Descent. In \emph{Advances in Neural Information Processing Systems}, 2023.
    \item J. A. Lin, S. Padhy, J. Antorán, A. Tripp, A. Terenin, C. Szepesvári, J. M. Hernández-Lobato, and D. Janz. Stochastic Gradient Descent for Gaussian Processes Done Right. In \emph{International Conference on Learning Representations}, 2024.
    \item J. A. Lin, S. Ament, M. Balandat, D. Eriksson, J. M. Hernández-Lobato, and E. Bakshy. Scalable Gaussian Processes with Latent Kronecker Structure. In \emph{International Conference on Machine Learning}, 2025. 
\end{itemize}

\section{Gaussian Processes}
Intuitively, a Gaussian process is a generalisation of the multivariate normal distribution over random vectors to a distribution over random functions.
While it may seem formidable to think about distributions over infinite-dimensional objects, the main strategy is to only ever reason with and about finite-dimensional quantities, but to ensure that the reasoning applies to any finite subset.
This leads us to the formal definition of a Gaussian process.

Let $f: \c{X} \rightarrow \R$ be a stochastic process which maps elements from an index set $\c{X}$ into $\R$.
The stochastic process $f$ is a Gaussian process if and only if for any finite subset $\{ \v{x}_i \}_{i=1}^n \subset \c{X}$, the collection of random variables $\{ f(\v{x}_i) \}_{i=1}^n$ follows a multivariate normal distribution.
Alternatively, a Gaussian process $f$ can be uniquely identified by its mean function $\mu$ and its kernel or covariance function $k$, which are defined as
\begin{equation}
    \mu(\v{x}) = \E [ f(\v{x}) ] \quad \text{and} \quad
    k(\v{x}, \v{x}') = \E [ (f(\v{x}) - \mu(\v{x}))(f(\v{x}') - \mu(\v{x}')) ],
\end{equation}
for $\v{x}, \v{x}' \in \c{X}$.
We write $f \~[GP](\mu, k)$ to indicate that $f$ is a Gaussian process with mean function $\mu$ and covariance function $k$.
The mean function $\mu$ defines the expected value of $f(\v{x})$ for any given input $\v{x}$.
It is often assumed to be zero to simplify derivations, which is reasonable because data can usually be centred around zero by subtracting its empirical mean.
The covariance function $k$ defines the concept of similarity or closeness, which is important for modelling purposes because we expect similar inputs to have similar outputs.
Formally, $k$ must be positive semi-definite and it defines the covariance in output space as a function of the input space, namely $k(\v{x}, \v{x}') = \Cov (f(\v{x}), f(\v{x}'))$.
When Gaussian processes are applied to model data, the covariance function typically carries the main properties of the associated model and is able to encode properties such as stationarity or periodicity.

\subsection{Gaussian Process Regression}
\label{sec:gp_regression}
So far, we have discussed Gaussian processes as an abstract mathematical object.
However, in machine learning, we are interested in leveraging Gaussian processes as a tool to model data.
To this end, let $\c{D} = \{ (\v{x}_i, y_i) \}_{i=1}^n$ be a dataset consisting of $n$ pairs of inputs $\v{x}_i \in \c{X}$ and outputs $y_i \in \R$, which we denote as $\m{X}$ and $\v{y}$.
Unless otherwise specified, we assume $\c{X} = \R^d$ throughout this dissertation.
Our goal is to learn a function $f$ which can describe their relationship as
$y_i = f(\v{x}_i) + \eps_i$, where each $\eps_i \~[N](0, \sigma^2)$ is independent and identically distributed observation noise.
We can approach this task by introducing a Gaussian process prior over the latent function $f$ and performing Bayesian inference.

In Bayesian inference, we first express our \emph{prior} belief over a random quantity and later, after observing data, update this belief to obtain the \emph{posterior} according to Bayes' rule.
In Gaussian process regression, our prior belief consists of assuming that $f$ follows a Gaussian process with a particular mean function $\mu$ and covariance function $k$, which implies that
\begin{equation}
\label{eq:gp_prior}
    \ubr{
        \begin{bmatrix}
            f(\v{x}_1) \\ \vdots \\ f(\v{x}_n)
        \end{bmatrix}
    }_{{\v{f}_\m{X}}}
    \~[N]\Bigg(
    \ubr{
        \begin{bmatrix}
            \mu(\v{x}_1) \\ \vdots \\ \mu(\v{x}_n)
        \end{bmatrix}
    }_{\v\mu_\m{X}},
    \ubr{
        \begin{bmatrix}
            k(\v{x}_1, \v{x}_1) & \dots & k(\v{x}_1, \v{x}_n) \\
            \vdots & \ddots & \vdots \\
            k(\v{x}_n, \v{x}_1) & \dots & k(\v{x}_n, \v{x}_n)
        \end{bmatrix}
    }_{\m{K}_\m{XX}}
    \Bigg),
\end{equation}
where $\v{f}_\m{X} \in \R^n$ and $\v\mu_\m{X} \in \R^n$ refer to vectors consisting of $f(\v{x}_i)$ and $\mu(\v{x}_i)$ for $i \in [1, ..., n]$, respectively, and $\m{K}_\m{XX} \in \R^{n \times n}$ is the \emph{kernel matrix}, containing pairwise evaluations $k(\v{x}_i,\v{x}_j)$ for $i,j \in [1,\dots,n]$.
In other words, under the Gaussian process prior, that is, before actually observing any data, the latent function evaluated at the location of our observed inputs $\v{f}_\m{X}$ is assumed to follow a multivariate normal distribution with mean $\v\mu_\m{X}$ and covariance matrix $\m{K}_\m{XX}$.
This may sound confusing because we are trying to evaluate $f$ at input locations $\v{x}_i$ before actually observing $\v{x}_i$.
However, the prior distribution over $\v{f}_\m{X}$ follows directly from the definition of a Gaussian process, namely that, for any finite subset $\{ \v{x}_i \}_{i=1}^n \subset \c{X}$, the collection of random variables $\{ f(\v{x}_i) \}_{i=1}^n$ follows a multivariate normal distribution.

To find the \emph{likelihood} of latent function values $\v{f}_\m{X}$ given observed noisy outputs $\v{y}$, we recall the relationship $y_i = f(\v{x}_i) + \eps_i$ with $\eps_i \~[N](0, \sigma^2)$ independent and identically distributed, which defines the conditional distribution of $\v{y} \given \v{f}_\m{X}$ as
\begin{equation}
\label{eq:gp_likelihood}
    \ubr{
        \begin{bmatrix}
            y_1 \given f(\v{x}_1) \\ \vdots \\ y_n \given f(\v{x}_n)
        \end{bmatrix}
    }_{\v{y} \given \v{f}_\m{X}}
    \~[N]\Bigg(
    \ubr{
        \begin{bmatrix}
            f(\v{x}_1) \\ \vdots \\ f(\v{x}_n)
        \end{bmatrix}
    }_{\v{f}_\m{X}},
    \ubr{
        \begin{bmatrix}
            \sigma^2 & \dots & 0 \\
            \vdots & \ddots & \vdots \\
            0 & \dots & \sigma^2
        \end{bmatrix}
    }_{\sigma^2 \m{I}}
    \Bigg).
\end{equation}
Similarly, we can use the same relationship to find the \emph{marginal likelihood} to be
\begin{equation}
\label{eq:gp_marginal_likelihood}
    \ubr{
        \begin{bmatrix}
            y_1 \\ \vdots \\ y_n
        \end{bmatrix}
    }_{\v{y}}
    \~[N]\Bigg(
    \ubr{
        \begin{bmatrix}
            \mu(\v{x}_1) \\ \vdots \\ \mu(\v{x}_n)
        \end{bmatrix}
    }_{\v{\mu}_\m{X}},
    \ubr{
        \begin{bmatrix}
            k(\v{x}_1, \v{x}_1) + \sigma^2 & \dots & k(\v{x}_1, \v{x}_n) \\
            \vdots & \ddots & \vdots \\
            k(\v{x}_n, \v{x}_1) & \dots & k(\v{x}_n, \v{x}_n) + \sigma^2
        \end{bmatrix}
    }_{\m{K}_\m{XX} + \sigma^2 \m{I}}
    \Bigg),
\end{equation}
which, in the context of Bayesian inference, is also sometimes called the \emph{evidence}.
It plays an important role in model selection and will be discussed in \Cref{sec:model_selection}.

Finally, to be able to make predictions at new input locations $\m{X}_* \subset \c{X}$, we want to find the \emph{posterior} distribution over $\v{f}_{\m{X}_*}$, the latent function values at $\m{X}_*$, given noisy observations $\v{y}$.
To this end, we first consider the \emph{joint} distribution over $\v{y}$ and $\v{f}_{\m{X}_*}$, which can be written as
\begin{equation}
    \begin{bmatrix}
        \v{y} \\ \v{f}_{\m{X}_*}
    \end{bmatrix}
    \~[N]\Bigg(
    \begin{bmatrix}
        \v\mu_{\m{X}} \\ \v\mu_{\m{X}_*}
    \end{bmatrix},
    \begin{bmatrix}
       \m{K}_\m{XX} + \sigma^2 \m{I} & \m{K}_{\m{X}\m{X}_*} \\
        \m{K}_{\m{X}_*\m{X}} & \m{K}_{\m{X}_*\m{X}_*}
    \end{bmatrix}
    \Bigg).
\end{equation}
The posterior is simply the conditional distribution of $\v{f}_{\m{X}_*}$ given $\v{y}$, which can be computed using properties of the multivariate normal distribution, resulting in updates to the prior,
\begin{align}
\label{eq:posterior}
    \v{f}_{\m{X}_* \given \v{y}}
    & \~[N]\del{
        \v\mu_{\m{X}_* \given \v{y}},
        \m{K}_{\m{X}_*\m{X}_* \given \v{y}}
    }, \\
\label{eq:posterior_mean}
    \v\mu_{\m{X}_* \given \v{y}}
    &= \v\mu_{\m{X}_*} + \m{K}_{\m{X}_*\m{X}} (\m{K}_\m{XX} + \sigma^2 \m{I})\inv (\v{y} - \v\mu_{\m{X}}), \\
\label{eq:posterior_cov}
    \m{K}_{\m{X}_*\m{X}_* \given \v{y}}
    &= \m{K}_{\m{X}_*\m{X}_*} - \m{K}_{\m{X}_*\m{X}} (\m{K}_\m{XX} + \sigma^2 \m{I})\inv \m{K}_{\m{X}\m{X}_*}.
\end{align}
With these equations, we are now able to make predictions for new inputs $\m{X}_*$ while taking our observed data and assumptions about the latent function $f$ into consideration.
However, the inverse matrices in the expressions above generally imply $\c{O}(n^3)$ time and $\c{O}(n^2)$ space complexities when using direct methods, which does not scale gracefully.
Dealing with this issue and improving scalability is the overall topic of this dissertation.

In contrast to many other methods, Gaussian process regression predicts a \emph{joint distribution} over latent function values instead of a single prediction per input.
This provides uncertainty quantification, which is essential for downstream applications such as Bayesian optimisation.
Although some deep learning methods, for example, also model predictive uncertainties, they are typically still limited to predicting \emph{marginal} variances for each output and not able to predict \emph{joint} covariances over a set of multiple outputs.
Therefore, the ability to predict joint distributions differentiates Gaussian process regression from most other regression methods.

\subsection{Generating Posterior Samples}
\label{sec:generating_posterior_samples}
The standard way of drawing a sample from a Gaussian process posterior predictive distribution consists of evaluating \Cref{eq:posterior_mean,eq:posterior_cov} at $n_*$ desired input locations $\m{X}_*$ to obtain the posterior mean $\v\mu_{\m{X}_* \given \v{y}}$ and covariance matrix $\m{K}_{\m{X}_*\m{X}_* \given \v{y}}$, computing the Cholesky decomposition of $\m{K}_{\m{X}_*\m{X}_* \given \v{y}}$, and using the resulting Cholesky factor $\m{L}$ to perform an affine transformation of a sample $\v{w}$ from a standard normal distribution,
\begin{equation}
\label{eq:loc_scale_sample}
    \v{f}_{\m{X}_* \given \v{y}} = \v\mu_{\m{X}_* \given \v{y}} + \m{L} \v{w}
    \quad \text{with} \quad \m{K}_{\m{X}_*\m{X}_* \given \v{y}} = \m{L} \m{L}\T
    \quad \text{and} \quad \v{w} \~[N](\v{0}, \m{I}),
\end{equation}
such that each unique posterior sample $\v{f}_{\m{X}_* \given \v{y}}$ corresponds to a unique $\v{w}$, and 
$\v{f}_{\m{X}_* \given \v{y}}$ satisfies
\begin{align}
    \E \sbr{\v{f}_{\m{X}_* \given \v{y}}}
    &= \E \sbr{\v\mu_{\m{X}_* \given \v{y}} + \m{L} \v{w}}
    = \v\mu_{\m{X}_* \given \v{y}} + \m{L} \E \sbr{\v{w}}
    = \v\mu_{\m{X}_* \given \v{y}} + \m{L} \v{0}
    = \v\mu_{\m{X}_* \given \v{y}}, \\
    \Var \del{\v{f}_{\m{X}_* \given \v{y}}}
    &= \Var \del{\v\mu_{\m{X}_* \given \v{y}} + \m{L} \v{w}}
    = \m{L} \Var \del{\v{w}} \m{L}\T
    = \m{L} \m{I} \m{L}\T
    = \m{K}_{\m{X}_*\m{X}_* \given \v{y}}.
\end{align}
There are several computational bottlenecks associated with drawing a posterior sample in this way.
First, we need to compute the posterior mean $\v\mu_{\m{X}_* \given \v{y}}$ and covariance matrix $\m{K}_{\m{X}_*\m{X}_* \given \v{y}}$, which is asymptotically dominated by the latter.
In particular, to compute $\m{K}_{\m{X}_*\m{X}_* \given \v{y}}$, we need to compute $\m{K}_{\m{X}_*\m{X}_*}$ and $\m{K}_{\m{X}_*\m{X}} (\m{K}_\m{XX} + \sigma^2 \m{I})\inv \m{K}_{\m{X}\m{X}_*}$.
Assuming a constant cost of evaluating the covariance function $k$, the first term requires $\c{O}(n_*^2)$ time and $\c{O}(n_*^2)$ space.
For the second term, calculation of the explicit inverse is typically avoided to improve numerical stability.
Instead, since $\m{K}_\m{XX} + \sigma^2 \m{I}$ is positive-definite, a Cholesky decomposition of $\m{K}_\m{XX} + \sigma^2 \m{I}$ followed by a batch of $n_*$ triangular $n \times n $ linear system solves against $\m{K}_{\m{X}\m{X}_*}$ can be used.
The Cholesky decomposition requires $\c{O}(n^3)$ time and $\c{O}(n^2)$ space, and the triangular linear system solves take $\c{O}(n^2 n_*)$ time and $\c{O}(n^2 + nn_*)$ space.

Once we computed $\m{K}_{\m{X}_*\m{X}_* \given \v{y}}$, we need to perform the affine transformation, whose computational complexity is dominated by the Cholesky decomposition of $\m{K}_{\m{X}_*\m{X}_* \given \v{y}}$ taking $\c{O}(n_*^3)$ time and $\c{O}(n_*^2)$ space.
Combining all the terms results in an overall asymptotic time complexity of $\c{O}(n^3 + n_*^3)$ and space complexity of $\c{O}(n^2 + n_*^2)$ to draw a single posterior sample evaluated at $n_*$ locations.
If posterior samples are drawn repeatedly without changing the observed training data, the Cholesky decomposition of $\m{K}_\m{XX} + \sigma^2 \m{I}$ can be cached and reused, lowering the time complexity to $\c{O}(n^2n_* + n_*^3)$.

\subsubsection{Pathwise Conditioning}
An alternative way of generating samples from the posterior was proposed by \citet{wilson20,wilson21}, who express a posterior sample as a transformed sample from the prior,
\begin{align}
\label{eq:pathwise_conditioning}
\v{f}_{\m{X}_* \given \v{y}}
= \ubr{\v{f}_{\m{X}_*} \vphantom{\v{f}_{\m{X}_*} + \m{K}_{\m{X}_*\m{X}} (\m{K}_{\m{XX}} + \sigma^2 \m{I})\inv (\v{y} - (\v{f}_\m{X} + \v\eps))} }_{\text{prior}} + \ubr{\m{K}_{\m{X}_*\m{X}} (\m{K}_{\m{XX}} + \sigma^2 \m{I})\inv (\v{y} - (\v{f}_\m{X} + \v\eps)) \vphantom{\v{f}_{\m{X}_*} + \m{K}_{\m{X}_*\m{X}} (\m{K}_{\m{XX}} + \sigma^2 \m{I})\inv (\v{y} - (\v{f}_\m{X} + \v\eps))} }_{\text{data-dependent update term}},
\end{align}
where $\v{f}_{\m{X}_* \given \v{y}}$ is the posterior sample evaluated at test inputs $\m{X}_*$, $\v{f}_{\m{X}_*}$ and $\v{f}_{\m{X}}$ represent the prior sample evaluated at test inputs $\m{X}_*$ and train inputs $\m{X}$ respectively, and $\v\eps \~[N](\v{0}, \sigma^2 \m{I})$ is a random normal vector.
In this formulation, each unique posterior sample $\v{f}_{\m{X}_* \given \v{y}}$ corresponds to a unique set of $\v{f}_{\m{X}_*}$, $\v{f}_\m{X}$, and $\v\eps$.
Furthermore, under a zero-mean prior, the update term can be interpreted as the posterior mean $\m{K}_{\m{X}_*\m{X}} (\m{K}_{\m{XX}} + \sigma^2 \m{I})\inv \v{y}$ (see \Cref{eq:posterior_mean}) minus an uncertainty reduction term $\m{K}_{\m{X}_*\m{X}} (\m{K}_{\m{XX}} + \sigma^2 \m{I})\inv (\v{f}_\m{X} + \v\eps)$, which becomes apparent by calculating the expected value and covariance matrix of $\v{f}_{\m{X}_* \given \v{y}}$,
\begin{align}
    \E \sbr{\v{f}_{\m{X}_* \given \v{y}}}
    &= \E \sbr{\v{f}_{\m{X}_*}} + \m{K}_{\m{X}_*\m{X}} (\m{K}_{\m{XX}} + \sigma^2 \m{I})\inv (\v{y} - (\E \sbr{\v{f}_\m{X}} + \E \sbr{\v\eps})), \\
    &= \v{\mu}_{\m{X}_*} + \m{K}_{\m{X}_*\m{X}} (\m{K}_{\m{XX}} + \sigma^2 \m{I})\inv (\v{y} - (\v{\mu}_\m{X} + \v{0})), \\
    &= \v{\mu}_{\m{X}_* \given \v{y}}, \\
    \Var \del{\v{f}_{\m{X}_* \given \v{y}}}
    &= \Var \del{\v{f}_{\m{X}_*}} + \Var \del{\m{K}_{\m{X}_*\m{X}} (\m{K}_{\m{XX}} + \sigma^2 \m{I})\inv (\v{f}_\m{X} + \v\eps)} \notag \\
    &\quad  - \Cov \del{\v{f}_{\m{X}_*}, \m{K}_{\m{X}_*\m{X}} (\m{K}_{\m{XX}} + \sigma^2 \m{I})\inv (\v{f}_\m{X} + \v\eps)}, \notag \\
    &\quad  - \Cov \del{\m{K}_{\m{X}_*\m{X}} (\m{K}_{\m{XX}} + \sigma^2 \m{I})\inv (\v{f}_\m{X} + \v\eps), \v{f}_{\m{X}_*}}, \\
    &= \m{K}_{\m{X}_*\m{X}_*} + \m{K}_{\m{X}_*\m{X}} (\m{K}_{\m{XX}} + \sigma^2 \m{I})\inv \Var \del{\v{f}_\m{X} + \v\eps} (\m{K}_{\m{XX}} + \sigma^2 \m{I})\inv \m{K}_{\m{X}\m{X}_*} \notag \\
    &\quad  - \Cov \del{\v{f}_{\m{X}_*}, \v{f}_\m{X} + \v\eps} (\m{K}_{\m{XX}} + \sigma^2 \m{I})\inv \m{K}_{\m{X}\m{X}_*}, \notag \\
    &\quad  - \m{K}_{\m{X}_*\m{X}} (\m{K}_{\m{XX}} + \sigma^2 \m{I})\inv \Cov \del{\v{f}_\m{X} + \v\eps, \v{f}_{\m{X}_*}}, \\
    &= \m{K}_{\m{X}_*\m{X}_*} + \m{K}_{\m{X}_*\m{X}} (\m{K}_{\m{XX}} + \sigma^2 \m{I})\inv \m{K}_{\m{X}\m{X}_*} \notag \\
    &\quad  - 2\m{K}_{\m{X}_* \m{X}} (\m{K}_{\m{XX}} + \sigma^2 \m{I})\inv \m{K}_{\m{X}\m{X}_*}, \\
    &= \m{K}_{\m{X}_*\m{X}_*} - \m{K}_{\m{X}_*\m{X}} (\m{K}_\m{XX} + \sigma^2 \m{I})\inv \m{K}_{\m{X}\m{X}_*}, \\
    &= \m{K}_{\m{X}_*\m{X}_* \given \v{y}},
\end{align}
revealing that subtracting the uncertainty reduction term results in subtracting a positive-definite matrix from the prior covariance matrix $\m{K}_{\m{X}_*\m{X}_*}$, reducing the posterior covariance in a Loewner order sense.
This calculation also confirms that $\v{f}_{\m{X}_* \given \v{y}}$ has the correct distribution.

Compared to using an affine transformation to generate posterior samples via \Cref{eq:loc_scale_sample}, pathwise conditioning has the property that the term $(\m{K}_{\m{XX}} + \sigma^2 \m{I})\inv (\v{y} - (\v{f}_\m{X} + \v\eps))$, which involves the computationally expensive matrix inverse, does not depend on the test inputs $\m{X}_*$.
Therefore, this term can be computed once, cached, and reused when drawing posterior samples at different test inputs $\m{X}_*$ to save computational costs.
This is particularly useful for acquisition function optimisation in the context of Bayesian optimisation, because the former may require the sequential evaluation of posterior samples at a really large number of different input locations.
Some later chapters of this dissertation will focus on methods to effectively compute this computationally expensive term with large amounts of training data.

An important detail is the method used to obtain $\v{f}_\m{X}$ and $\v{f}_{\m{X}_*}$, the prior sample evaluated at the train and test inputs respectively.
Even if $\m{X}$ and $\m{X}_*$ are disjoint, $\v{f}_\m{X}$ and $\v{f}_{\m{X}_*}$ are still correlated because they follow the underlying Gaussian process prior,
\begin{equation}
    \begin{bmatrix}
        \v{f}_\m{X} \\ \v{f}_{\m{X}_*}
    \end{bmatrix}
    \~[N]\Bigg(
    \begin{bmatrix}
        \v\mu_{\m{X}} \\ \v\mu_{\m{X}_*}
    \end{bmatrix},
    \begin{bmatrix}
       \m{K}_\m{XX} & \m{K}_{\m{X}\m{X}_*} \\
        \m{K}_{\m{X}_*\m{X}} & \m{K}_{\m{X}_*\m{X}_*}
    \end{bmatrix}
    \Bigg).
\end{equation}
Generating $\v{f}_\m{X}$ and $\v{f}_{\m{X}_*}$ jointly using an affine transformation of a standard normal random variable, akin to \Cref{eq:loc_scale_sample}, requires the Cholesky decomposition of the joint covariance matrix, whose direct computation would require $\c{O}((n + n_*)^3)$ time and $\c{O}((n + n_*)^2)$ space.
Therefore, \citet{wilson20,wilson21} suggest to use scalable approximations instead, such as random features, which will be discussed in \Cref{sec:random_features}.

If $\m{X}$ is fixed but $\m{X}_*$ changes repeatedly, for example, when maximising acquisition functions in Bayesian optimisation, the Cholesky decomposition of the joint covariance matrix can be calculated more efficiently by caching the Cholesky factor of $\m{K}_\m{XX}$, which is equivalent to conditional sampling \citep{jiang20}.
To derive the update equations, we first set the joint covariance matrix equal to its Cholesky decomposition,
\begin{equation}
    \begin{bmatrix}
       \m{K}_\m{XX} & \m{K}_{\m{X}\m{X}_*} \\
        \m{K}_{\m{X}_*\m{X}} & \m{K}_{\m{X}_*\m{X}_*}
    \end{bmatrix}
    = \m{L} \m{L}\T
    =
    \begin{bmatrix}
       \m{L}_{11} & \v{0} \\
        \m{L}_{21} & \m{L}_{22}
    \end{bmatrix}
    \begin{bmatrix}
       \m{L}_{11}\T & \m{L}_{21}\T \\
        \v{0} & \m{L}_{22}\T
    \end{bmatrix},
\end{equation}
where $\m{L}_{11}$ and $\m{L}_{22}$ are lower triangular matrices and $\m{L}_{21}$ is a dense matrix.
Performing the matrix multiplication on the right yields a block matrix equation,
\begin{equation}
    \begin{bmatrix}
       \m{K}_\m{XX} & \m{K}_{\m{X}\m{X}_*} \\
        \m{K}_{\m{X}_*\m{X}} & \m{K}_{\m{X}_*\m{X}_*}
    \end{bmatrix}
    =
    \begin{bmatrix}
       \m{L}_{11}\m{L}_{11}\T & \m{L}_{11}\m{L}_{21}\T \\
        \m{L}_{21}\m{L}_{11}\T & \m{L}_{21}\m{L}_{21}\T + \m{L}_{22}\m{L}_{22}\T
    \end{bmatrix},
\end{equation}
from which we can derive a system of three matrix equations,
\begin{align}
    \m{K}_\m{XX} &= \m{L}_{11}\m{L}_{11}\T, \\
    \m{K}_{\m{X}\m{X}_*} &= \m{L}_{11}\m{L}_{21}\T, \\
    \m{K}_{\m{X}_*\m{X}_*} &= \m{L}_{21}\m{L}_{21}\T + \m{L}_{22}\m{L}_{22}\T.
\end{align}
Upon inspection of the first equation, we identify $\m{L}_{11}$ as the lower triangular Cholesky factor of the train covariance matrix $\m{K}_\m{XX}$.
Calculating $\m{L}_{11}$ takes $\c{O}(n^3)$ time, but this cost can be amortised by caching and reusing $\m{L}_{11}$ as long as $\m{X}$ stays fixed.
Solving the remaining equations for $\m{L}_{21}$ and $\m{L}_{22}$ results in
\begin{align}
    \m{L}_{21} &= \del{\m{L}_{11}\inv \m{K}_{\m{X}\m{X}_*}}\T, \\
    \m{L}_{22} &= \del{\m{K}_{\m{X}_*\m{X}_*} - \m{L}_{21}\m{L}_{21}\T}^{\frac{1}{2}},
\end{align}
where calculating $\m{L}_{21}$ consists of solving $n_*$ triangular linear systems in $\c{O}(n^2n_*)$ time, and 
computing $\m{L}_{22}$ requires $\c{O}(nn_*^2 + n_*^3)$ time to perform some matrix operations and another Cholesky decomposition.
Thus, repeated cubic costs in $n$ can be avoided by caching $\m{L}_{11}$.

\subsection{Covariance Functions}
\label{sec:covariance_functions}
In the previous sections, we assumed the existence of a positive semi-definite covariance function $k$ which intuitively captures similarity by expressing the covariance in output space as a function of the input space, that is $k(\v{x}, \v{x}') = \Cov (f(\v{x}), f(\v{x}'))$.
In general, given two valid covariance functions $k_1$ and $k_2$, we can combine them to create a new valid covariance function $k_3$, for example, using multiplication $k_3(\v{x}, \v{x}') = k_1(\v{x}, \v{x}')k_2(\v{x}, \v{x}')$ or addition $k_3(\v{x}, \v{x}') = k_1(\v{x}, \v{x}') + k_2(\v{x}, \v{x}')$.
Choosing any particular covariance function will have a strong influence on the properties of the resulting Gaussian process regression model.
For example, some covariance functions are \emph{stationary}, which means that $k(\v{x}, \v{x}')$ only depends on $\v{x} - \v{x}'$ rather than $\v{x}$ and $\v{x}'$ individually.
This property makes stationary covariance functions invariant to translations in the input space.
In the following, we will discuss and illustrate some commonly used covariance functions for the case of $\c{X} = \R^d$.

\subsubsection{Squared Exponential}
The squared exponential covariance function is arguably the most commonly used covariance function in the context of Gaussian processes.
It is stationary and sometimes also called the Gaussian radial basis function kernel, because it can be expressed as
\begin{equation}
    k_{\mathrm{SE}}(\v{x}, \v{x}'; \ell) = \exp \del{-\frac{\norm{\v{x} - \v{x}'}_2^2}{2\ell^2}},
\end{equation}
which is proportional to the density of a Gaussian distribution.
In the expression above, $\ell$ is a length scale parameter which influences how quickly functions from the Gaussian process will change.
Smaller length scales lead to faster changes and larger length scales produce smoother functions (see \Cref{fig:squared_exponential} for an illustration).
\begin{figure}[hb]
    \centering
    \includegraphics[width=6in]{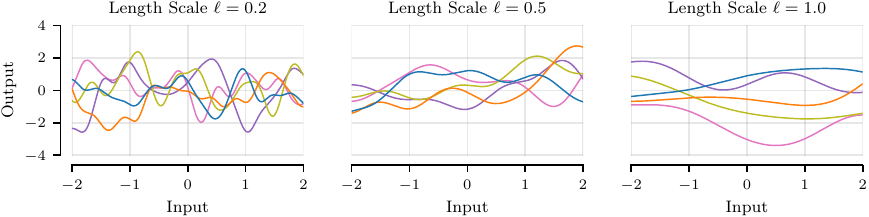}
    \caption{Samples from a Gaussian process prior using squared exponential covariance function and different length scales $\ell$. Smaller length scales lead to faster changes and larger length scales produce smoother functions.}
    \label{fig:squared_exponential}
\end{figure}

For multi-dimensional inputs, that is $d > 1$, it is common to use a separate length scale per input dimension, which is also known as automatic relevance determination.
It is also common to include a signal scale parameter which multiplies the whole expression to adapt the scale in the output space.
In general, functions from a Gaussian process prior with squared exponential covariance function will be infinitely differentiable.
If somewhat less smooth functions are desired, the next type of covariance function might be more appropriate.

\subsubsection{Matérn}
The general expression of the Matérn covariance function is given by
\begin{equation}
    k_{\mathrm{Mat}}(\v{x}, \v{x}'; \ell, \nu) = \frac{2^{1 - \nu}}{\Gamma(\nu)} \del{\sqrt{2 \nu} \frac{\norm{\v{x} - \v{x}'}_2}{\ell}}^\nu K_\nu \del{\sqrt{2 \nu} \frac{\norm{\v{x} - \v{x}'}_2}{\ell}}
\end{equation}
where $\ell$ is a length scale parameter akin to the one used by the squared exponential covariance function, $\nu$ is a smoothness parameter which controls the degree of differentiability of the resulting functions, $\Gamma$ is the gamma function which generalises the factorial function, and $K_\nu$ is the modified Bessel function of the second kind of order $\nu$.
\begin{figure}[hb]
    \centering
    \includegraphics[width=6in]{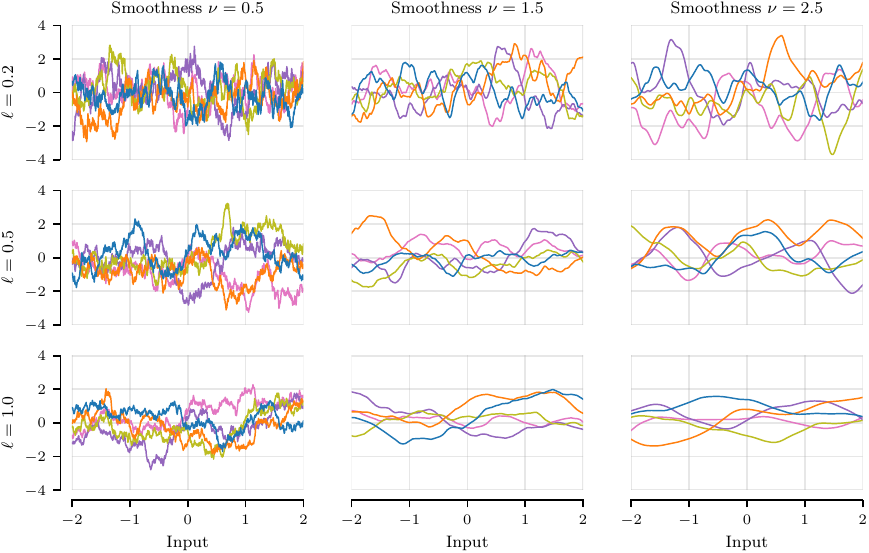}
    \caption{Samples from a Gaussian process prior with Matérn covariance function using different length scales $\ell$ and smoothness $\nu$. Increasing $\nu$ leads to smoother functions.}
    \label{fig:matern}
\end{figure}

The Matérn covariance function is a good choice to model continuous functions which are somewhat smooth but not infinitely differentiable.
It can also be decorated with an amplitude or signal scale parameter, and its length scale parameter can also be extended to length scales per input dimension if $d > 1$.
For specific values of $\nu$, its analytic expression can be substantially simplified.
Commonly used values of $\nu$ are $\nicefrac{1}{2}$, $\nicefrac{3}{2}$, and $\nicefrac{5}{2}$, for which
\begin{align}
    k_{\mathrm{Mat}}(\v{x}, \v{x}'; \ell, \nicefrac{1}{2})
    &= \exp \del{-\frac{\norm{\v{x} - \v{x}'}_2}{\ell}}, \\
    k_{\mathrm{Mat}}(\v{x}, \v{x}'; \ell, \nicefrac{3}{2})
    &= \del{1 {\,+\,} \frac{\sqrt{3}\norm{\v{x} - \v{x}'}_2}{\ell}} \exp \del{-\frac{\sqrt{3}\norm{\v{x} - \v{x}'}_2}{\ell}}, \\
    k_{\mathrm{Mat}}(\v{x}, \v{x}'; \ell, \nicefrac{5}{2})
    &= \del{1 {\,+\,} \frac{\sqrt{5}\norm{\v{x} - \v{x}'}_2}{\ell} + \frac{5\norm{\v{x} - \v{x}'}_2^2}{3\ell^2}} \exp \del{-\frac{\sqrt{5}\norm{\v{x} - \v{x}'}_2}{\ell}}.
\end{align}
As $\nu$ increases, the resulting functions become smoother (see \Cref{fig:matern}), and in the limit of $\nu \to \infty$, the Matérn covariance function converges to the squared exponential covariance function.
Like the squared exponential, the Matérn covariance function is also stationary.

\begin{figure}[hb]
    \centering
    \includegraphics[width=6in]{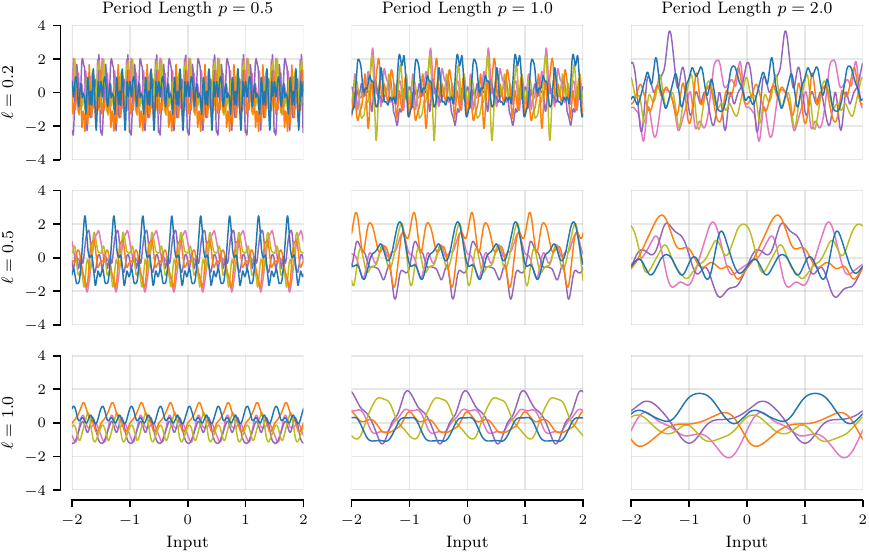}
    \caption{Samples from a Gaussian process prior with periodic covariance function using different length scales $\ell$ and period lengths $p$. Smaller period lengths $p$ result in functions which repeat themselves quicker, while larger $p$ increases the interval between repetitions.}
    \label{fig:periodic}
\end{figure}
\subsubsection{Periodic}
As the name suggests, the periodic covariance function models periodically repeating functions.
It can be written as
\begin{equation}
    k_{\mathrm{Per}}(\v{x}, \v{x}'; \ell, p) = \exp \del{-\frac{2 \sin^2 \del{\frac{\pi}{p} \norm{\v{x} - \v{x}'}_2}}{\ell^2}},
\end{equation}
where $\ell$ is a length scale and $p$ is the period length after which the function repeats itself (see \Cref{fig:periodic} for some examples).
In this form, the periodic covariance function is strictly periodic, allowing no room for deviations or changes over time, which might not be a realistic assumption for certain real-world settings.
Therefore, it is often combined with a squared exponential or Matérn covariance function.
Multiplying both covariance functions together results in functions which are locally periodic but still allowed to change slightly for different regions in the input space.

\subsection{Model Selection}
\label{sec:model_selection}
Most covariance functions, including the ones discussed in the previous section, require the selection of some parameters, such as length scales or the period length.
In rare circumstances, these parameters can be selected based on ground truth information or expert knowledge.
However, in most circumstances, we would prefer an automatic, rigorous, and data-informed decision.
The gold standard for Gaussian process regression, and arguably Bayesian inference in general, is marginal likelihood optimisation.

In Bayesian inference, the marginal likelihood represents the probability of observing the given training data under a particular model while considering all possible parameter configurations of the given model.
The marginal likelihood is also called the \emph{evidence}, and the procedure of maximising it is often referred to as \emph{type-II maximum likelihood} or \emph{empirical Bayes}.
For a thorough introduction to Bayesian inference and the marginal likelihood in the context of machine learning, we refer to \citet{bishop2006}.

For Gaussian process regression, a particular model refers to a particular covariance function and its parameters, all possible parameter configurations refers to all possible latent function values, and considering all possible latent function values refers to calculating the integral
\begin{equation}
    \ubr{p(\v{y} \given \m{X}, k, \v{\vartheta}, \sigma^2)}_{\text{marginal likelihood}}
    = \int
    \ubr{p(\v{y} \given \v{f}_\m{X}, \sigma^2)}_{\text{likelihood}} \,
    \ubr{p(\v{f}_\m{X} \given \m{X}, k, \v{\vartheta})}_{\text{prior}} \d \v{f}_\m{X},
\end{equation}
where $\m{X}$ and $\v{y}$ are respectively the observed inputs and outputs, $\v{f}_\m{X}$ are the corresponding latent function values, $\v{\vartheta}$ is a vector which contains all parameters used by the covariance function $k$, and $\sigma^2$ is the amount of observation noise.
Here, prior refers to the distribution over latent function values under the Gaussian process prior (see \Cref{eq:gp_prior}) and likelihood refers to the conditional distribution of noisy observations given latent function values (see \Cref{eq:gp_likelihood}).
In the context of Gaussian process regression, an analytical solution of this integral exists (see \Cref{eq:gp_marginal_likelihood}), although similar integrals are often intractable.

In the framework of Bayesian inference, it is desirable to maximise the marginal likelihood because it promotes data fit while penalising overly complex models.
Therefore, the resulting model will, at least in theory, achieve an optimal balance between data fit and complexity.
For the actual optimisation, it is typically preferred to consider the logarithm of the marginal likelihood as a function of all quantities which we want to select while keeping everything else constant.
In the context of Gaussian process regression, the observed data $\m{X}$ and $\v{y}$ are considered fixed.
Assuming that the covariance function $k$ is also fixed, the logarithm of the marginal likelihood as a function of $\v{\theta} = \{ \v{\vartheta}, \sigma^2 \}$ is given by
\begin{equation}
\label{eq:mll_background}
    \c{L}(\v{\theta})
    =
    \ubr{-\frac{1}{2} \del{\v{y} - \v{\mu}_\m{X}}\T \del{\m{K}_\m{XX} + \sigma^2 \m{I}}\inv \del{\v{y} - \v{\mu}_\m{X}}}_{\text{data fit term}}
    \ubr{-\frac{1}{2} \log \det \del{\m{K}_\m{XX} + \sigma^2 \m{I}}}_{\text{model complexity term}}
    \ubr{-\frac{n}{2} \log \del{2 \pi}}_{\text{constant}},
\end{equation}
where the fact that $\m{K}_\m{XX}$ depends on $\v{\vartheta}$ is suppressed to simplify the notation.
Technically, the mean function $\mu$ could also depend on further parameters, but, for simplicity, this case is also not considered here.
Typically, $\c{L}$ is optimised using gradient-based optimisation, which requires computation of the partial derivatives
\begin{align}
\label{eq:mll_grad_background}
    \frac{\partial}{\partial \theta_i}\c{L}(\v{\theta})
    &=
    \frac{1}{2} \del{\del{\m{K}_\m{XX} + \sigma^2 \m{I}}\inv \del{\v{y} - \v{\mu}_\m{X}}}\T
    \frac{\partial \del{\m{K}_\m{XX} + \sigma^2 \m{I}}}{\partial \theta_i}
    \del{\del{\m{K}_\m{XX} + \sigma^2 \m{I}}\inv \del{\v{y} - \v{\mu}_\m{X}}} \notag \\
    &\quad -\frac{1}{2} \tr \del{\del{\m{K}_\m{XX} + \sigma^2 \m{I}}\inv \frac{\partial \del{\m{K}_\m{XX} + \sigma^2 \m{I}}}{\partial \theta_i}}.
\end{align}
Due to the inverse matrices, calculating these partial derivatives tends to be computationally expensive or even intractable without further approximations.

\Cref{chap:solvers} discusses methodological contributions to make marginal likelihood optimisation in the context of Gaussian process regression more efficient.
An extensive amount of research literature on Gaussian processes introduces methods and techniques to improve the scalability of marginal likelihood optimisation or posterior inference.
Some of these existing approaches, which are particularly relevant for this dissertation, will be discussed in the next section.

\section{Scalable Inference}
\label{sec:scalable_inference}
In the previous sections, we have established some background about Gaussian processes and discussed relevant computations which are necessary to perform inference.
Unfortunately, in most cases, the time and space requirements of standard inference methods respectively obey cubic and quadratic scaling relative to the amount of training data, which quickly becomes intractable.
The overall goal of this dissertation is to contribute methods for scalable inference in Gaussian processes.
This section provides an overview of existing methods and related work on this topic.
We start by discussing sparse and variational methods, which generally represent a large dataset with a smaller amount of data or inducing points.
Afterwards, we explain random features, which construct unbiased estimators by leveraging the property that covariance functions correspond to implicit inner products.
Subsequently, we cover methods based on structured kernel matrices, which exploit certain linear algebraic structure for efficient computations.
Finally, we conclude with approaches involving iterative linear system solvers, which circumvent explicit matrix inverses or decompositions by performing iterative optimisation to obtain the same result.
Iterative linear system solvers are particularly important for this dissertation because they are a core component of all original contributions throughout this dissertation.

\subsection{Sparse and Variational Methods}
\label{sec:sparse_and_variational}
To deal with the $\c{O}(n^3)$ and $\c{O}(n^2)$ computational time and space complexities of matrix inversion via Cholesky decomposition, researchers developed \emph{sparse} methods which use a set of $m \ll n$ \emph{inducing points} to represent the observed training data.
While effective at improving inference speed, this approach requires a principled way to select inducing points.

A naive approach to choose inducing points is to select a subset of the data uniformly at random.
Conditioning the Gaussian process on a subset of size $m$ reduces the computational complexity to $\c{O}(m^3)$.
However, \citet{candela2005} state that they \emph{`would not generally expect [subset of data] to be a competitive method, since it would seem impossible (even with fairly redundant data and a good choice of the subset) to get a realistic picture of the uncertainties, when only a part of the training data is even considered'}.

The Nyström approximation \citep{williams2000using} uses a subset $\m{Z}$ of $m \ll n$ examples from the whole dataset $\m{X}$ to approximate $\m{K}_{\m{X}\m{X}}$ as $\m{Q}_{\m{X}\m{X}} = \m{K}_{\m{X}\m{Z}} \m{K}_{\m{Z}\m{Z}}\inv \m{K}_{\m{Z}\m{X}}$.
This approach is motivated by considering the eigendecomposition of a joint matrix $\m{K} = \m{U} \m{\Lambda} \m{U}\T$, which is assumed to be of rank $m$,
\begin{equation}
    \m{K}
    = \begin{bmatrix} \m{U}_{\m{Z}} \\ \m{U}_{\m{X}} \end{bmatrix} \m{\Lambda} \begin{bmatrix} \m{U}_{\m{Z}} \\ \m{U}_{\m{X}} \end{bmatrix}\T
    = \begin{bmatrix} \m{U}_{\m{Z}} \m{\Lambda} \m{U}_{\m{Z}}\T & \m{U}_{\m{Z}} \m{\Lambda} \m{U}_{\m{X}}\T \\ \m{U}_{\m{X}} \m{\Lambda} \m{U}_{\m{Z}}\T & \m{U}_{\m{X}} \m{\Lambda} \m{U}_{\m{X}}\T \end{bmatrix}
    = \begin{bmatrix} \m{K}_{\m{Z}\m{Z}} & \m{K}_{\m{Z}\m{X}} \\ \m{K}_{\m{X}\m{Z}} & \m{K}_{\m{X}\m{X}} \end{bmatrix}.
\end{equation}
Using $\m{K}_{\m{X}\m{Z}} = \m{U}_{\m{X}} \m{\Lambda} \m{U}_{\m{Z}}\T$, we can solve for $\m{U}_{\m{X}}$, obtaining $\m{U}_{\m{X}} = \m{K}_{\m{X}\m{Z}} \m{U}_{\m{Z}} \m{\Lambda}\inv$.
Finally, we can substitute this expression for $\m{U}_{\m{X}}$ to calculate $\m{K}_{\m{X}\m{X}}$ as
\begin{equation}
    \m{K}_{\m{X}\m{X}} = \m{U}_{\m{X}} \m{\Lambda} \m{U}_{\m{X}}\T
    = \m{K}_{\m{X}\m{Z}} \m{U}_{\m{Z}} \m{\Lambda}\inv \m{\Lambda} \m{\Lambda}\inv \m{U}_{\m{Z}}\T \m{K}_{\m{Z}\m{X}}
    = \m{K}_{\m{X}\m{Z}} \m{K}_{\m{Z}\m{Z}}\inv \m{K}_{\m{Z}\m{X}}
    = \m{Q}_{\m{X}\m{X}},
\end{equation}
obtaining the Nyström approximation.
The equation above requires $\m{K}$ to be full rank, but, in practice, $\m{K}$ will be low rank if $\m{Z}$ is a subset of $\m{X}$.
However, even if $\m{Q}_{\m{X}\m{X}}$ is low rank due to $m \ll n$, the matrix $\m{Q}_{\m{X}\m{X}} + \sigma^2 \m{I}$ still has full rank because $\sigma^2 \m{I}$ has full rank.

\citet{candela2005} created a taxonomy of sparse Gaussian processes, distinguishing between the Subset of Regressors (SoR) approximation \citep{silverman1985,wahba1999,smola2001}, the Deterministic Training Conditional (DTC) approximation \citep{csato2002,seeger2003}, the Fully Independent Training Conditional (FITC) approximation \citep{snelson2005} and the Partially Independent Training Conditional (PITC) approximation \citep{schwaighofer2002}.
Each category in this taxonomy corresponds to a different joint prior distribution over latent function values $\v{f}_\m{X}$ and $\v{f}_{\m{X}_*}$ of the observed data $\m{X}$ and unobserved data $\m{X}_*$.

With $\m{Q}_{\m{X}\m{X}_*} = \m{K}_{\m{X}\m{Z}} \m{K}_{\m{Z}\m{Z}}\inv \m{K}_{\m{Z}\m{X}_*}$, $\m{Q}_{\m{X}_*\m{X}} = \m{K}_{\m{X}_*\m{Z}} \m{K}_{\m{Z}\m{Z}}\inv \m{K}_{\m{Z}\m{X}}$ and $\m{Q}_{\m{X}_*\m{X}_*} = \m{K}_{\m{X}_*\m{Z}} \m{K}_{\m{Z}\m{Z}}\inv \m{K}_{\m{Z}\m{X}_*}$, the joint prior distributions of all approximations share the same mean but differ in their covariance matrices.
In particular,
\begin{align}
    \Var_{\mathrm{SoR}}\del{\begin{bmatrix} \v{f}_\m{X} \\ \v{f}_{\m{X}_*} \end{bmatrix}}
    &= \begin{bmatrix} \m{Q}_{\m{X}\m{X}} & \m{Q}_{\m{X} \m{X}_*} \\ \m{Q}_{\m{X}_*\m{X}} & \m{Q}_{\m{X}_*\m{X}_*} \end{bmatrix} \\
    \Var_{\mathrm{DTC}}\del{\begin{bmatrix} \v{f}_\m{X} \\ \v{f}_{\m{X}_*} \end{bmatrix}}
    &= \begin{bmatrix} \m{Q}_{\m{X}\m{X}} & \m{Q}_{\m{X} \m{X}_*} \\ \m{Q}_{\m{X}_*\m{X}} & \m{K}_{\m{X}_*\m{X}_*} \end{bmatrix} \\
    \Var_{\mathrm{FITC}}\del{\begin{bmatrix} \v{f}_\m{X} \\ \v{f}_{\m{X}_*} \end{bmatrix}}
    &= \begin{bmatrix} \m{Q}_{\m{X}\m{X}} - \mathrm{diag}(\m{Q}_{\m{X}\m{X}} - \m{K}_{\m{X}\m{X}}) & \m{Q}_{\m{X} \m{X}_*} \\ \m{Q}_{\m{X}_*\m{X}} & \m{K}_{\m{X}_*\m{X}_*} \end{bmatrix} \\
    \Var_{\mathrm{PITC}}\del{\begin{bmatrix} \v{f}_\m{X} \\ \v{f}_{\m{X}_*} \end{bmatrix}}
    &= \begin{bmatrix} \m{Q}_{\m{X}\m{X}} - \mathrm{blockdiag}(\m{Q}_{\m{X}\m{X}} - \m{K}_{\m{X}\m{X}}) & \m{Q}_{\m{X} \m{X}_*} \\ \m{Q}_{\m{X}_*\m{X}} & \m{K}_{\m{X}_*\m{X}_*} \end{bmatrix},
\end{align}
where $\m{Q}$ corresponds to using a low-rank covariance function and $\m{K}$ corresponds to using the true covariance function.

In this framework, the Nyström approximation \citep{williams2000using} corresponds to
\begin{equation}
    \Var_{\mathrm{Nyst}}\del{\begin{bmatrix} \v{f}_\m{X} \\ \v{f}_{\m{X}_*} \end{bmatrix}}
    = \begin{bmatrix} \m{Q}_{\m{X}\m{X}} & \m{K}_{\m{X} \m{X}_*} \\ \m{K}_{\m{X}_*\m{X}} & \m{K}_{\m{X}_*\m{X}_*} \end{bmatrix}.
\end{equation}
In other words, the Nyström approximation uses the low-rank approximation only for the training data and the true covariance function otherwise.

The methods discussed so far either use heuristics to select specific data examples as inducing points or perform continuous optimisation to select the inducing points by maximising the marginal likelihood of a modified Gaussian process model with approximate prior.
However, the former might be suboptimal and the latter does not guarantee that the modified model will be faithful to the exact model.

\citet{titsias09,titsias2009report} proposed a principled variational framework for sparse Gaussian processes which approximates the exact posterior with a variational posterior.
To this end, $m$ inducing points $\m{Z} = \{ \v{z}_j \}_{j=1}^m$ with corresponding latent function values $\v{f}_\m{Z}$ are introduced, such that
\begin{equation}
    \v{f}_\m{X} \given \v{f}_\m{Z} \~[N](\m{K}_{\m{X}\m{Z}} \m{K}_{\m{Z}\m{Z}}\inv \v{f}_\m{Z}, \m{K}_{\m{X}\m{X}} - \m{K}_{\m{X}\m{Z}} \m{K}_{\m{Z}\m{Z}}\inv \m{K}_{\m{Z}\m{X}})
    \quad \text{and} \quad
    \v{f}_\m{Z} \~[N](\v{0}, \m{K}_{\m{Z}\m{Z}}),
\end{equation}
where we identify $\m{K}_{\m{X}\m{Z}} \m{K}_{\m{Z}\m{Z}}\inv \m{K}_{\m{Z}\m{X}} = \m{Q}_{\m{X}\m{X}}$ as the Nyström approximation.
Subsequently, a variational distribution $q(\v{f}_{\m{Z}})$ is introduced and a lower bound on the true marginal likelihood,
\begin{equation}
    \log p(\v{y}) \geq \E_{q(\v{f}_\m{Z})} \sbr{\log \frac{p(\v{y}, \v{f}_{\m{Z}})}{q(\v{f}_\m{Z})}} = \c{L}(\m{Z}, q),
\end{equation}
is formed using Jensen's inequality.
\citet{titsias09,titsias2009report} derives the optimal variational distribution $q$, such that the lower bound becomes
\begin{equation}
    \c{L}_{\mathrm{SGPR}}(\m{Z}) = \log \c{N}\del{\v{y} \mid \v{0}, \m{Q}_{\m{X}\m{X}} + \sigma^2\m{I}}
    - \frac{1}{2\sigma^2}\mathrm{Tr}\del{\m{K}_{\m{X}\m{X}} - \m{Q}_{\m{X}\m{X}}},
\end{equation}
which can be maximised to select the inducing points $\m{Z}$ in a principled way.
Furthermore, the corresponding posterior predictive distribution becomes
\begin{align}
    \v{f}_{\m{X}_* \given \v{y}}^{[\m{Z}]}
    & \~[N]\del{
        \v\mu_{\m{X}_* \given \v{y}}^{[\m{Z}]},
        \m{K}_{\m{X}_*\m{X}_* \given \v{y}}^{[\m{Z}]}
    }, \\
    \label{eq:posterior_mean_titsias}
    \v\mu_{\m{X}_* \given \v{y}}^{[\m{Z}]}
    &= \sigma^{-2}\m{K}_{\m{X}_*\m{Z}} (\m{K}_{\m{Z}\m{Z}} + \sigma^{-2}\m{K}_{\m{Z}\m{X}} \m{K}_{\m{X}\m{Z}})\inv \m{K}_{\m{Z}\m{X}}\v{y}, \\
    \label{eq:posterior_cov_titsias}
    \m{K}_{\m{X}_*\m{X}_* \given \v{y}}^{[\m{Z}]}
    &= \m{K}_{\m{X}_*\m{X}_*} - \m{K}_{\m{X}_*\m{Z}} \del{\m{K}_{\m{Z}\m{Z}}\inv - \del{\m{K}_{\m{Z}\m{Z}} + \sigma^{-2} \m{K}_{\m{Z}\m{X}} \m{K}_{\m{X}\m{Z}}} \inv} \m{K}_{\m{Z}\m{X}_*},
\end{align}
resulting in an overall time complexity of $\c{O}(nm^2)$ and space complexity of $\c{O}(nm)$.

While a computational complexity of $\c{O}(nm^2)$ with $m \ll n$ is a great improvement over $\c{O}(n^3)$, inference can still be intractable if $n$ is really large.
Therefore, \citet{hensman13} developed a method to perform \emph{stochastic} variational inference in sparse Gaussian processes which is applicable to larger datasets because optimisation can be performed in mini-batches.

By replacing the \emph{collapsed} variational posterior with an \emph{explicitly parametrised} variational posterior with global mean and covariance parameters $\v{m} \in \R^m$ and $\m{S} \in \R^{m \times m}$, \citet{hensman13} introduce yet another lower bound on the lower bound of \citet{titsias09,titsias2009report}.
The new lower bound can then be written as
\begin{align}
    \c{L}_{\mathrm{SVGP}}(\m{Z}, \v{m}, \m{S}) &= \sum_{i=1}^n
    \del{ 
        \log \c{N}(y_i \mid \m{K}_{\v{x}_i \m{Z}}\m{K}_{\m{Z}\m{Z}}\inv \v{m}, \sigma^2)
        - \frac{c_i}{2\sigma^2}
        - \frac{1}{2}\tr \del{\m{S} \m{\Lambda}_i}
    } \notag \\
    &\quad - D_{\mathrm{KL}}\del{\c{N}(\v{m}, \m{S}) \from \c{N}(\v{0}, \m{K}_{\m{Z}\m{Z}})},
\end{align}
where $\m{K}_{\v{x}_i \m{Z}}$ is the $i^{\mathrm{th}}$ row of $\m{K}_{\m{X}\m{Z}}$, $c_i$ is the $i^{\mathrm{th}}$ diagonal element of $\m{K}_{\m{X}\m{X}} - \m{K}_{\m{X}\m{Z}}\m{K}_{\m{Z}\m{Z}}\inv \m{K}_{\m{Z}\m{X}}$, and  $\m{\Lambda}_i = \sigma^{-2} \m{K}_{\m{Z}\m{Z}}\inv \m{K}_{\m{Z} \v{x}_i} \m{K}_{\v{x}_i \m{Z}} \m{K}_{\m{Z}\m{Z}}\inv$.
Importantly, the sum over individual data pairs in $\c{L}_{\mathrm{SVGP}}$ enables stochastic gradient methods.
The gradients with respect to $\v{m}$ and $\m{S}$ are
\begin{equation}
    \frac{\partial \c{L}_{\mathrm{SVGP}}}{\partial \v{m}} = \sigma^{-2}\m{K}_{\m{Z}\m{Z}}\inv \m{K}_{\m{Z}\m{X}} \v{y} - \m{\Lambda} \v{m} \qquad \text{and} \qquad
    \frac{\partial \c{L}_{\mathrm{SVGP}}}{\partial \m{S}} = \frac{1}{2}\m{S}\inv - \frac{1}{2}\m{\Lambda},
\end{equation}
where $\m{\Lambda} = \sigma^{-2} \m{K}_{\m{Z}\m{Z}}\inv \m{K}_{\m{Z} \m{X}} \m{K}_{\m{X} \m{Z}} \m{K}_{\m{Z}\m{Z}}\inv + \m{K}_{\m{Z} \m{Z}}^{-1}$.
Furthermore, \citet{hensman13} proposed to use \emph{natural} gradients in the canonical parameters to improve convergence speed.
The natural parameters are $\v{\theta}_1 = \m{S}\inv \v{m}$ and $\v{\theta}_2 = -\frac{1}{2}\m{S}\inv$ and the corresponding natural gradient steps of length $\eta$ are
\begin{align}
    \v{\theta}_1^{(t+1)} &= \v{\theta}_1^{(t)} + \eta \del{\sigma^{-2}\m{K}_{\m{Z}\m{Z}}\inv \m{K}_{\m{Z}\m{X}} \v{y} - \v{\theta}_1^{(t)}}, \\
    \v{\theta}_2^{(t+1)} &= \v{\theta}_2^{(t)} + \eta \del{\v{\theta}_2^{(t)} - \frac{1}{2}\m{\Lambda}}.
\end{align}
Stochastic approximations of the natural gradients can be obtained using mini-batches, such that each stochastic gradient evaluation has a computational time complexity of $\c{O}(m^3)$.

More recent work on sparse Gaussian processes include \citet{wu2022variational}, who introduce a prior which only retains correlations for $\kappa$ nearest neighbours to produce a sparse kernel matrix, and \citet{wenger2022posterior,wenger2024computation}, who propose \emph{actions} to create low-rank approximations which guarantee that the approximate predictive variance is always larger than the exact predictive variance.
The additional variance is interpreted as \emph{computational uncertainty} and helps the method achieve well-calibrated predictions while scaling to large datasets.

\subsection{Random Features}
\label{sec:random_features}
An alternative way to construct low-rank kernel matrix approximations can be derived by considering \emph{Mercer's theorem} \citep{mercer1909}, which informally states that a symmetric positive semi-definite kernel function $k(\v{x}, \v{x}')$ can be expressed as an inner product $\innerprod{\v{\varphi}(\v{x})}{\v{\varphi}(\v{x}')}_{\c{F}}$ in some potentially infinite-dimensional feature space $\c{F}$, where $\v{\varphi}: \c{X} \to \c{F}$ is a feature projection, which maps elements $\v{x} \in \c{X}$ into $\c{F}$.
Random features \citep{rahimi08} leverage this property by constructing a particular feature projection $\v{\phi}_{\v{\omega}}: \R^d \to \R^m$, which depends on random parameters $\v{\omega}$ and maps $\v{x} \in \R^d$ into $\R^m$, such that the dot product in $\R^m$ becomes an unbiased approximation of the kernel function,
\begin{equation}
    k(\v{x}, \v{x}')
    = \innerprod{\v{\varphi}(\v{x})}{\v{\varphi}(\v{x}')}_{\c{F}}
    \approx \v{\phi}_{\v{\omega}}(\v{x})\T \v{\phi}_{\v{\omega}}(\v{x}').
\end{equation}
Random Fourier features \citep{rahimi08,sutherland15} are a prominent example which applies to stationary covariance functions.
To construct random Fourier features, \citet{rahimi08} leverage \emph{Bochner's theorem}, which informally states that a covariance function is stationary if and only if it is the Fourier transform of a unique probability measure, leading to the relationship
\begin{equation}
    k(\v{x}, \v{x}')
    = \int p(\v{\omega}) \exp \del{i \v{\omega}\T(\v{x} - \v{x}')} \d \v{\omega}
    = \E_{\v{\omega}} \sbr{\exp \del{i \v{\omega}\T(\v{x} - \v{x}')}}.
\end{equation}
In this setting, \citet{rahimi08} show that choosing $\phi_{\v{\omega}}(\v{x}) = \sqrt{2} \cos \del{\v{\omega}\T \v{x} + b}$, where $b$ is sampled uniformly at random from $[0, 2\pi]$, satisfies $k(\v{x}, \v{x}') = \E_{\v{\omega}} \sbr{\phi_{\v{\omega}}(\v{x}) \phi_{\v{\omega}}(\v{x}')}$.
The variance can be reduced by considering a sample average, which can be expressed as
\begin{equation}
    k(\v{x}, \v{x}')
    \approx \frac{1}{m} \sum_{j=1}^m \phi_{\v{\omega}_j}(\v{x}) \phi_{\v{\omega}_j}(\v{x}')
    = \v{\phi}_{\v{\omega}}(\v{x})\T \v{\phi}_{\v{\omega}}(\v{x}'),
\end{equation}
where $\v{\phi}_{\v{\omega}}(\v{x})$ consists of stacking all $\phi_{\v{\omega}_j}(\v{x})$ into a vector and introducing a rescaling factor,
\begin{equation}
    \v{\phi}_{\v{\omega}}(\v{x}) = \sqrt{\frac{2}{m}}
    \begin{bmatrix}
        \cos \del{\v{\omega}_1\T \v{x} + b_1} &
        \hdots &
        \cos \del{\v{\omega}_m\T \v{x} + b_m}
    \end{bmatrix}\T,
\end{equation}
where each $\v{\omega}_j$ and $b_j$ are independent and identically distributed.
It is also possible to define
\begin{equation}
    \v{\phi}_{\v{\omega}}(\v{x}) = \frac{1}{\sqrt{m}}
    \begin{bmatrix}
        \sin \del{\v{\omega}_1\T \v{x}} &
        \cos \del{\v{\omega}_1\T \v{x}} &
        \hdots &
        \sin \del{\v{\omega}_m\T \v{x}} &
        \cos \del{\v{\omega}_m\T \v{x}}
    \end{bmatrix}\T,
\end{equation}
which does not depend on $b_j$ and also has lower variance \citep{sutherland15}.

For both variants of $\v{\phi}_{\v{\omega}}(\v{x})$, the probability distribution of $\v{\omega}_j$, also known as the \emph{spectral density}, determines which covariance function is approximated.
For example, the squared exponential covariance function $k_{\mathrm{SE}}$ corresponds to the multivariate normal distribution and the Matérn covariance function $k_{\mathrm{Mat}}$ corresponds to the multivariate Student's $t$-distribution.

Random features can be used to efficiently draw samples from a Gaussian process.
Suppose we want to draw latent function values $\v{f}_\m{X}$ at $n$ input locations $\m{X}$ from a Gaussian process with mean $\mu$ and covariance function $k$, such that $\v{f}_\m{X} \~[N](\v{\mu}_\m{X}, \m{K}_{\m{XX}})$.
We can write $\v{f}_\m{X}$ as
\begin{equation}
\label{eq:random_feature_sampling}
    \v{f}_\m{X} = \v{\mu}_\m{X} +  \m{\Phi}_\m{X} \v{w}
    \qquad \text{with} \qquad
    \v{w} \in \R^m \~[N](\v{0}, \m{I}),
\end{equation}
where $\m{\Phi}_\m{X} \in \R^{n \times m}$ is a matrix which collects $m$ random features for each input,
\begin{equation}
    \m{\Phi}_\m{X} = 
    \frac{1}{\sqrt{m}}
    \begin{bmatrix}
        \phi_{\v{\omega}_1}(\v{x}_1) &
        \dots &
        \phi_{\v{\omega}_m}(\v{x}_1) \\
        \vdots & \ddots & \vdots \\
        \phi_{\v{\omega}_1}(\v{x}_n) &
        \dots &
        \phi_{\v{\omega}_m}(\v{x}_n)
    \end{bmatrix}
    =
    \begin{bmatrix}
        \v{\phi}_{\v{\omega}}(\v{x}_1)\T \\
        \vdots \\
        \v{\phi}_{\v{\omega}}(\v{x}_n)\T
    \end{bmatrix}.
\end{equation}
To verify that \Cref{eq:random_feature_sampling} yields samples with the desired distribution, we calculate
\begin{align}
    \E \sbr{\v{f}_\m{X}}
    &= \E \sbr{ \v{\mu}_\m{X} + \m{\Phi}_\m{X} \v{w} }
    = \v{\mu}_\m{X} + \m{\Phi}_\m{X} \E \sbr{\v{w}}
    = \v{\mu}_\m{X} + \m{\Phi}_\m{X} \v{0}
    = \v{\mu}_\m{X}, \\
    \Var \del{\v{f}_\m{X}}
    &= \Var \del{\v{\mu}_\m{X} + \m{\Phi}_\m{X} \v{w}}
    = \m{\Phi}_\m{X} \Var \del{\v{w}} \m{\Phi}_\m{X}\T
    = \m{\Phi}_\m{X} \m{I} \m{\Phi}_\m{X}\T
    = \m{\Phi}_\m{X} \m{\Phi}_\m{X}\T,
\end{align}
confirming that $\v{f}_\m{X}$ has the correct mean $\v{\mu}_\m{X}$.
For the variance, we observe that each entry in $\m{\Phi}_\m{X} \m{\Phi}_\m{X}\T$ has the form $\v{\phi}_{\v{\omega}}(\v{x})\T \v{\phi}_{\v{\omega}}(\v{x}') \approx k(\v{x}, \v{x}')$ and thus conclude that $\m{\Phi}_\m{X} \m{\Phi}_\m{X}\T \approx \m{K}_\m{XX}$.

Drawing samples in this way has a time complexity of $\c{O}(nm)$, which is considerably faster than the usual $\c{O}(n^3)$, especially if $m \ll n$.
If we are interested in joint samples $\v{f}_\m{X}$ and $\v{f}_{\m{X}_*}$ at $n$ inputs $\m{X}$ and $n_*$ inputs $\m{X}_*$, the time complexity becomes $\c{O}((n + n_*)m)$ instead of $\c{O}((n + n_*)^3)$.
Furthermore, if $\m{X}$ is fixed and $\m{X}_*$ changes, such that we want to keep $\v{f}_\m{X}$ but update $\v{f}_{\m{X}_*}$, expensive conditional sampling, as discussed in \Cref{sec:generating_posterior_samples}, is not necessary.
Instead, reusing the same random feature parameters $\v{\omega}$ and standard normal sample $\v{w}$ results in samples $\v{f}_{\m{X}_*}$ which are jointly distributed with $\v{f}_\m{X}$ in $\c{O}(n_*m)$ instead of $\c{O}(nn_*^2 + n_*^3)$ time.
To verify the claim that $\v{f}_\m{X}$ and $\v{f}_{\m{X}_*}$ are jointly distributed, we calculate
\begin{align}
    \Cov \del{\v{f}_\m{X}, \v{f}_{\m{X}_*}}
    &= \Cov \del{\v{\mu}_\m{X} + \m{\Phi}_\m{X} \v{w}, \v{\mu}_{\m{X}_*} + \m{\Phi}_{\m{X}_*} \v{w}}, \\
    &= \m{\Phi}_\m{X} \E \sbr{\v{w} \v{w}\T} \m{\Phi}_{\m{X}_*} - \del{\m{\Phi}_\m{X} \E \sbr{\v{w}}} \del{\m{\Phi}_{\m{X}_*} \E \sbr{\v{w}}}\T, \\
    &= \m{\Phi}_\m{X} \m{\Phi}_{\m{X}_*} \approx \m{K}_{\m{X} \m{X}_*},
\end{align}
confirming that $\v{f}_\m{X}$ and $\v{f}_{\m{X}_*}$ indeed have the correct covariance.

\subsection{Structured Kernel Matrices}
\label{sec:structured_kernel_matrices}
In certain special cases, kernel matrices will be structured, facilitating more efficient linear algebra operations.
A prominent example are product kernels of the form
\begin{equation}
    k(\v{x}, \v{x}') = \prod_{j=1}^m k_j(\v{x}_j, \v{x}_j')
    \quad \text{with} \quad
    \v{x} = \sbr{\v{x}_1, \dots, \v{x}_m} \quad \text{and} \quad
    \v{x}' = \sbr{ \v{x}_1', \dots, \v{x}_m'},
\end{equation}
where $k_j$ are individual covariance functions, and $\v{x}, \v{x}'$ are the concatenation of all $\v{x}_j, \v{x}_j'$ respectively.
They are commonly used in multi-task or spatio-temporal settings and induce a kernel matrix $\m{K}_\m{XX}$ which can be factorised as a Kronecker product,
\begin{equation}
    \m{K}_\m{XX} = \bigotimes_{j=1}^m \m{K}_j
    \quad \text{with} \quad
    \m{K}_j = k_j(\m{X}_j, \m{X}_j)
    \quad \text{and} \quad
    \m{X} = \m{X}_1 \times \dots \times \m{X}_m,
\end{equation}
where the inputs $\m{X}$ factorise as a Cartesian product $\m{X}_1 \times \dots \times \m{X}_m$ and each Kronecker factor $\m{K}_j$ consists of the pairwise kernel distances over rows in $\m{X}_j$.
To facilitate efficient inference, we can represent the eigendecomposition of $\m{K}_\m{XX}$ as
\begin{equation}
    \m{K}_\m{XX}
    = \bigotimes_{j=1}^m \m{K}_j
    = \bigotimes_{j=1}^m \m{Q}_j \m{\Lambda}_j \m{Q}_j\T
    = \del{\bigotimes_{j=1}^m \m{Q}_j} \del{\bigotimes_{j=1}^m \m{\Lambda}_j} \del{\bigotimes_{j=1}^m\m{Q}_j}\T,
\end{equation}
where $\m{Q}_j \m{\Lambda}_j \m{Q}_j\T$ is the eigendecomposition of Kronecker factor $\m{K}_j$ with orthogonal matrix $\m{Q}_j$ and diagonal matrix $\m{\Lambda}_j$.
This allows us to compute the inverse of $\m{K}_\m{XX} + \sigma^2 \m{I}$ as
\begin{align}
    \del{\m{K}_\m{XX} + \sigma^2 \m{I}}\inv
    &= \del{\del{\bigotimes_{j=1}^m \m{Q}_j} \del{\bigotimes_{j=1}^m \m{\Lambda}_j} \del{\bigotimes_{j=1}^m\m{Q}_j}\T + \sigma^2 \m{I}}\inv, \\
    &= \del{\del{\bigotimes_{j=1}^m \m{Q}_j} \del{\bigotimes_{j=1}^m \m{\Lambda}_j + \sigma^2 \m{I}} \del{\bigotimes_{j=1}^m\m{Q}_j}\T}\inv, \\
    &= \del{\bigotimes_{j=1}^m \m{Q}_j} \del{\bigotimes_{j=1}^m \m{\Lambda}_j + \sigma^2 \m{I}}\inv \del{\bigotimes_{j=1}^m\m{Q}_j}\T,
\end{align}
where $\bigotimes_{j=1}^m \m{\Lambda}_j + \sigma^2 \m{I}$ is a diagonal matrix whose inverse is easily computed by taking the reciprocal of each entry on the diagonal \citep{bonilla2007multi,stegle2011efficient}.
Importantly, the factorisation implies that the $n \times n$ kernel matrix $\m{K}_\m{XX}$ factorises as a Kronecker product of $m$ factors of size $n_j \times n_j$ and $n = \prod_{j=1}^m n_j$.
Therefore, computing the eigendecomposition of all Kronecker factors has a total time complexity of $\c{O}(\sum_{j=1}^m n_j^3)$ compared to $\c{O}(\prod_{j=1}^m n_j^3)$ without the factorisation, turning multiplicative scaling into additive scaling.
Similarly, the space complexity becomes $\c{O}(\sum_{j=1}^m n_j^2)$ instead of $\c{O}(\prod_{j=1}^m n_j^2)$.

To generate samples, we can express the Cholesky decomposition of $\m{K}_\m{XX}$ as
\begin{equation}
    \m{K}_\m{XX}
    = \bigotimes_{j=1}^m \m{K}_j
    = \bigotimes_{j=1}^m \m{L}_j \m{L}_j\T
    = \del{\bigotimes_{j=1}^m \m{L}_j} \del{\bigotimes_{j=1}^m \m{L}_j}\T,
\end{equation}
such that $\bigotimes_{j=1}^m \m{L}_j$ can be used in an efficient way.

In addition to the product kernel assumption, the main limitation of this method is that it must be possible to express the observed inputs $\m{X}$ as a Cartesian product $\m{X}_1 \times \dots \times \m{X}_m$, because otherwise the joint covariance matrix $\m{K}_\m{XX}$ no longer exhibits Kronecker structure.
In \Cref{chap:lkgp}, we propose a method which overcomes this limitation.

\citet{wilson2015kernel} approximate the original kernel matrix $\m{K}_\m{XX}$ as 
\begin{equation}
    \m{K}_\m{XX}
    \approx \m{K}_\m{XZ} \m{K}_\m{ZZ}\inv \m{K}_\m{ZX}
    \approx \m{W} \m{K}_\m{ZZ} \m{K}_\m{ZZ}\inv \m{K}_\m{ZZ} \m{W}\T
    = \m{W} \m{K}_\m{ZZ} \m{W}\T,
\end{equation}
where $\m{K}_\m{XZ} \approx \m{W} \m{K}_\m{ZZ}$ corresponds to the interpolation of kernel values at inducing points $\m{Z}$ using interpolation weights $\m{W}$ to compute kernel values for inputs $\m{X}$.
Since $\m{Z}$ are inducing points, their locations can be chosen arbitrarily.
\citet{wilson2015kernel} leverage this property by placing $\m{Z}$ on a grid such that $\m{K}_\m{ZZ}$ exhibits Kronecker or Toeplitz structure, facilitating fast matrix multiplication which can be exploited by iterative methods to perform inference (see \Cref{sec:iterative_methods}).
Since iterative methods do not require explicit matrix inverses or decompositions, the number of inducing points used for interpolation can be substantially larger than the number of inducing points used by variational methods.
However, scalability in the number of input dimensions is strongly limited by the fact that the required number of inducing points to create a dense grid increases exponentially with the number of dimensions.

To address these scalability issues with respect to the number of input dimensions, \citet{gardner2018product} approximate the original kernel with a Hadamard product of one-dimensional interpolated kernels.
Every interpolated kernel is further approximated as a low-rank Lanczos decomposition,
\begin{equation}
    \m{K}_\m{XX}
    \approx \m{Q}_1 \m{T}_1 \m{Q}_1\T \odot \dots \odot \m{Q}_d \m{T}_d \m{Q}_d\T,
\end{equation}
which gives rise to an algebraic structure that allows for efficient matrix multiplication.
While this addresses the scalability issues, it is usually not competitive with other methods due to the errors introduced by the approximations~\citep{kapoor2021skiing}.

\subsection{Iterative Linear System Solvers}
\label{sec:iterative_methods}
In the context of Gaussian processes, the biggest computational challenge typically consists of evaluating expressions such as $(\m{K}_\m{XX} + \sigma^2 \m{I})\inv \v{y}$ or alike, which feature the inverse of a large positive-definite matrix multiplied with a vector or another matrix.
In general, we can equivalently view such an expression as the solution of a system of linear equations or the minimiser of a convex quadratic objective,
\begin{equation}
    \v{v} = \m{A}\inv \v{b}
    \quad \iff \quad
    \m{A} \v{v} = \v{b}
    \quad \iff \quad
    \v{v} = \argmin_{\v{u}} \; \frac{1}{2} \v{u}\T \m{A} \v{u} - \v{u}\T \v{b},
\end{equation}
where setting $\m{A} = \m{K}_\m{XX} + \sigma^2 \m{I}$ and $\v{b} = \v{y}$ results in the previously mentioned expression.
The benefits of taking this view become apparent by observing that the matrix $\m{A}$ no longer needs to be inverted.
Instead, the solution $\v{v}$ can be computed via iterative gradient-based optimisation of the quadratic objective, because the gradient is simply the residual $\m{A}\v{u} - \v{b}$.
Furthermore, the matrix $\m{A}$ does not have to be fully instantiated in memory, as long as each entry in $\m{A}$ can be evaluated independently.
Therefore, by iterating over the rows of $\m{A}$, the product $\m{A} \v{u}$ can be computed with $\c{O}(n)$ space instead of $\c{O}(n^2)$.

The most widely used algorithm to solve these kinds of linear systems or their equivalent quadratic optimisation problems is the method of conjugate gradients \citep{hestenes1952}.
The algorithm leverages the fact that the solution $\v{v}$ can be written as a set of mutually conjugate vectors $\v{p}_1, \dots, \v{p}_n$, where $\v{p}_i$ is conjugate to $\v{p}_j$ if $\v{p}_i\T \m{A} \v{p}_j = 0$.
Since these vectors form a basis of $\R^n$, we can express $\v{v}$ as a linear combination $\v{v} = \sum_{i=1}^n c_i \v{p}_i$, which implies that $\m{A}\v{v} = \sum_{i=1}^n c_i \m{A}\v{p}_i$.
For a particular $\v{p}_k$, we thus have
\begin{equation}
    \v{p}_k\T \v{b} = \v{p}_k\T \m{A}\v{v} = \sum_{i=1}^n c_i \v{p}_k\T \m{A} \v{p}_i = c_k \v{p}_k\T \m{A} \v{p}_k
    \quad \implies \quad c_k = \frac{\v{p}_k\T \v{b}}{\v{p}_k\T \m{A} \v{p}_k},
\end{equation}
where the last equality on the left follows from $\v{p}_k\T \m{A} \v{p}_i = 0$ for $i \neq k$.
Therefore, we can compute $\v{v}$ by finding a sequence of conjugate vectors $\v{p}_1, \dots, \v{p}_n$ and calculating the corresponding coefficients $c_1, \dots, c_n$.
The sequence of conjugate vectors is typically computed using an iterative Gram-Schmidt process whose necessary computations can be amortised.
The overall algorithm can be expressed as
\begin{equation}
\label{eq:cg}
\begin{aligned}
\v{v}_0 &\gets \v{0} \text{ (typically)}, \\
\v{r}_0 &\gets \v{b} - \m{A} \v{v}_0, \\
\v{d}_0 &\gets \v{r}_0,
\end{aligned}
\qquad
\qquad
\begin{aligned}
\alpha_t &\gets \frac{\v{r}_t\T \v{r}_t}{\v{d}_t \m{A} \v{d}_t}, \\
\beta_{t} &\gets \frac{\v{r}_{t + 1}\T \v{r}_{t + 1}}{\v{r}_t\T \v{r}_t},
\end{aligned}
\qquad
\qquad
\begin{aligned}
\v{v}_{t + 1} &\gets \v{v}_t + \alpha_t \v{d}_t, \\
\v{r}_{t + 1} &\gets \v{r}_t - \alpha_t \m{A} \v{d}_t, \\
\v{d}_{t + 1} &\gets \v{r}_{t + 1} + \beta_t \v{d}_t.
\end{aligned}
\end{equation}

Assuming numerically exact computations, the method of conjugate gradients converges after at most $n$ iterations, where the time complexity of each iteration is $\c{O}(n^2)$.
However, the algorithm is typically stopped after a constant number of iterations or when the residual norm $\norm{\v{r}_t}_2$ reaches a certain tolerance, and its convergence properties depend on the condition number of $\m{A}$.
In practice, accumulation of rounding errors due to floating point arithmetic can cause problems with keeping track of the current residual and maintaining the property that all $\v{p}_j$ are mutually conjugate.
The current residual can be recomputed at the expense of some additional computations, but the latter is difficult to address.
For a more elaborate discussion on the method of conjugate gradients, we refer to \citet{shewchuk1994}.

In the context of Gaussian processes, \citet{gardner18,WangPGT2019exactgp} used conjugate gradients to estimate the derivative of the marginal likelihood (see \Cref{eq:mll_grad_background}), which is necessary for model selection.
In particular, the derivative contains two terms which involve an inverse matrix.
The first term is $(\m{K}_\m{XX} + \sigma^2 \m{I})\inv (\v{y} - \v\mu_{\m{X}})$ and it is also used when calculating the posterior mean (see \Cref{eq:posterior_mean}).
The second term is a trace term, which is approximated using \emph{Hutchinson's trace estimator} \citep{Hutchinson1990},
\begin{equation}
    \tr \del{\m{H}_{\v{\theta}}\inv \frac{\partial \m{H}_{\v{\theta}}}{\partial \theta_i}}
    = \tr \del{ \E \sbr{\v{z} \v{z}\T} \m{H}_{\v{\theta}}\inv \frac{\partial \m{H}_{\v{\theta}}}{\partial \theta_i}}
    = \E \sbr{\v{z}\T \m{H}_{\v{\theta}}\inv \frac{\partial \m{H}_{\v{\theta}}}{\partial \theta_i} \v{z}}
    \approx \frac{1}{s} \sum_{j=1}^s \v{z}_j\T \m{H}_{\v{\theta}}\inv \frac{\partial \m{H}_{\v{\theta}}}{\partial \theta_i} \v{z}_j,
\end{equation}
where $\m{H}_{\v{\theta}} = \m{K}_{\m{X}\m{X}} + \sigma^2 \m{I}$, $\v{\theta}$ refers to $\sigma^2$ and kernel hyperparameters $\v{\vartheta}$, and $\v{z}_j$ are random probe vectors for which $\E \sbr{\v{z}_j \v{z}_j\T} = \m{I}$.
Both terms can be addressed at the same time by solving a batch of linear systems,
\begin{equation}
    (\m{K}_\m{XX} + \sigma^2 \m{I}) \; \sbr{\v{v}_{\v{y}}, \v{v}_1, \dots, \v{v}_s}
    = \sbr{\v{y} - \v{\mu}_\m{X}, \v{z}_1, \dots, \v{z}_s},
\end{equation}
using the method of conjugate gradients.
Since these linear systems share the same coefficient matrix $\m{K}_\m{XX} + \sigma^2 \m{I}$, efficient matrix multiplication is applicable.
To calculate the posterior covariance matrix or to generate posterior samples, \citet{WangPGT2019exactgp} rely on approximate kernel matrix decompositions using the Lanczos algorithm \citep{pleiss2018}.

\citet{wilson20,wilson21} used conjugate gradients to compute $(\m{K}_{\m{XX}} + \sigma^2 \m{I})\inv (\v{y} - (\v{f}_\m{X} + \v\eps))$ to generate a posterior sample
% from the Gaussian process posterior
via pathwise conditioning (see \Cref{eq:pathwise_conditioning}).
% To generate multiple samples in parallel, a batch of linear systems can also be considered in this setting.
As discussed in \Cref{sec:generating_posterior_samples}, the result can be cached and reused when evaluating the posterior sample at another location.
Therefore, using an iterative linear system solver to compute this expensive term once is a powerful concept which allows Gaussian processes to scale to large amounts of training data.
In the following chapters, all original contributions leverage this powerful combination of iterative linear system solvers and pathwise conditioning in various ways.

% \begin{algorithm}[t]
%     \caption{Conjugate gradients for solving $\m{A} \v{v} = \v{b}$}
%     \label{alg:cg}
% \begin{algorithmic}[1]
%     \Require Coefficient matrix $\m{A}$, targets $\v{b}$, tolerance $\tau$, maximum number of iterations $t_{\mathrm{max}}$
%     \State $\v{v} \gets \v{0}$
%     \State $\v{r} \gets \v{b} - \m{A} \v{v}$
%     \State $\v{d} \gets \v{r}$
%     \State $\gamma \gets \v{r}\T \v{r}$
%     \State $t \gets 0$
%     \While{$t < t_{\mathrm{max}}$ \textbf{ and } $\Vert \v{r} \Vert_2 / \Vert \v{b} \Vert_2 > \tau$}
%         \State $\alpha \gets \gamma / \v{d}\T \m{A} \v{d}$
%         \State $\v{v} \gets \v{v} + \alpha \v{d}$ 
%         \State $\v{r} \gets \v{r} - \alpha \m{A} \v{d}$
%         \State $\beta \gets \v{r}\T \v{r} / \gamma$
%         \State $\v{d} \gets \v{r} + \beta \v{d}$
%         \State $\gamma \gets \v{r}\T \v{r}$
%         \State $t \gets t + 1$
%     \EndWhile
%     \State \Return $\v{v}$
% \end{algorithmic}
% \end{algorithm}

%!TEX root = ../thesis.tex
%*******************************************************************************
%****************************** Third Chapter **********************************
%*******************************************************************************
\chapter[Stochastic Gradient Descent for Gaussian Processes]{Stochastic Gradient Descent\\for Gaussian Processes}
\label{chap:sgd}

% **************************** Define Graphics Path **************************
\ifpdf
    \graphicspath{{Chapter3/Figs/Raster/}{Chapter3/Figs/PDF/}{Chapter3/Figs/}}
\else
    \graphicspath{{Chapter3/Figs/Vector/}{Chapter3/Figs/}}
\fi

After reviewing Gaussian processes and some of the existing literature on scalable inference, this chapter introduces the first methodological contribution of this dissertation.
In particular, we introduce a novel iterative linear system solver for Gaussian processes, namely stochastic gradient descent, which facilitates approximate inference with asymptotically linear time and memory requirements.
Additionally, we develop low-variance optimisation objectives and extend these to inducing points for even greater scalability.
Furthermore, we observe that stochastic gradient descent often produces accurate predictions, even in cases where it does not quickly converge to the optimum.
We explain this through a spectral characterisation of its implicit bias, showing that our algorithm produces predictive distributions close to the true posterior both in regions with sufficient data coverage, and in regions sufficiently far away from the data.
Furthermore, we demonstrate empirically that stochastic gradient descent achieves state-of-the-art performance on sufficiently large-scale or ill-conditioned regression tasks, and a large-scale Bayesian optimisation problem.

This chapter includes content which is adapted from the following publication:
\begin{itemize}
    \item J. A. Lin, J. Antorán, S. Padhy, D. Janz, J. M. Hernández-Lobato, and A. Terenin. Sampling from Gaussian Process Posteriors using Stochastic Gradient Descent. In \emph{Advances in Neural Information Processing Systems}, 2023.
\end{itemize}
My contributions to this project consist of developing major parts of the software implementation, conducting a large amount of experiments, and contributing to writing the manuscript and some of the theoretical results.

\section{Introduction}
Gaussian processes (GPs) provide a comprehensive framework for learning unknown functions in an uncertainty-aware manner.
This often makes Gaussian processes the model of choice for sequential decision-making, achieving state-of-the-art performance in tasks such as optimising molecules in computational chemistry settings \citep{Gomez-Bombarelli18} and automated hyperparameter tuning \citep{Snoek2012,Hernandez-Lobato14}.

The main limitation of Gaussian processes is that their computational cost is cubic in the size of the training dataset.
Significant research efforts have been directed at addressing this limitation, resulting in two key classes of scalable inference methods: (i) \emph{inducing point} methods \citep{titsias09,hensman13}, which approximate the GP posterior, and (ii) \emph{conjugate gradient} methods \citep{gardner18,WangPGT2019exactgp}, which approximate the computation needed to obtain the GP posterior.
Note that in structured settings, such as geospatial learning in low dimensions, specialised techniques are available \citep{wilson2015kernel,gardner2018product,Wilkinson20}.
Throughout this work, we focus on the generic setting, where scalability limitations remain unresolved.

In recent years, stochastic gradient descent (SGD) has emerged as the leading technique for training machine learning models at scale, in both deep learning \citep{Zhang2023sgd} and in related settings such as kernel methods \citep{dai2014doubly} and Bayesian modelling \citep{Mandt17Descent}.
While the principles behind the effectiveness of SGD are not yet fully understood, empirically, SGD often leads to good predictive performance, even when it does not fully converge.
The latter is the default regime in deep learning, and has motivated researchers to study \emph{implicit biases} and related properties of SGD \citep{Belkin2019,ZouWBGK2021benign}.

In the context of GPs, SGD is commonly used to learn kernel hyperparameters by optimising the marginal likelihood \citep{Cunningham2016preconditioning,gardner18,chen20,chen22} or closely related variational objectives \citep{titsias09,hensman13}. 
In this work, we explore applying SGD to the complementary problem of approximating GP posterior samples given fixed kernel hyperparameters.
In one of his seminal books on statistical learning theory, \citet{vapnik95} famously said: \emph{"When solving a given problem, try to avoid solving a more general problem as an intermediate step."}
Motivated by this viewpoint, as well as the aforementioned property of good performance often not requiring full convergence when using SGD, we ask: \emph{Do the linear systems arising in GP computations necessarily need to be solved to a small error tolerance? If not, can SGD help accelerate these computations?}

We answer the latter question affirmatively, with specific contributions as follows:
\begin{itemize}
    \item We develop a scheme for drawing GP posterior samples by applying SGD to a quadratic problem.
    In particular, we recast the pathwise conditioning technique of \citet{wilson20,wilson21} as an optimisation problem to which we apply the low-variance SGD sampling estimator of \citet{antoran2023sampling}. Further, we extend the proposed method to inducing point Gaussian processes.
    \item We characterise the implicit bias in SGD-approximated GP posteriors showing that despite optimisation not fully converging, these match the true posterior in regions both near to and far from the data.
    \item Experimentally, we show that SGD produces strong results, compared to variational and conjugate gradient methods, on both large-scale and ill-conditioned regression tasks, and on a parallel Thompson sampling task, where error bar calibration is paramount.
\end{itemize}

\section{Posterior Samples via Stochastic Optimisation}
Assuming a zero-mean prior, we now develop and analyse techniques for drawing samples from GP posteriors using SGD by rewriting the pathwise conditioning formula \eqref{eq:pathwise_conditioning} in terms of two stochastic optimisation problems.
As a preview, \Cref{fig:scalable-learning-comparison} illustrates SGD on a pair of toy problems, designed to capture complementary computational difficulties.

\subsection{Computing the Posterior Mean}
\label{sec:mean-optim}
We begin by deriving a stochastic objective for the posterior mean.
In particular, the GP posterior mean $\mu_{(\.) \given \v{y}}$ can be expressed in terms of the quadratic optimisation problem
\begin{align}
\mu_{(\.) \given \v{y}} &= \m{K}_{(\.) \m{X}}\v{v}^* = \sum_{i=1}^n k(\., \v{x}_i) v_i^*, \\
\label{eqn:mean-optim}
\v{v}^* &= \argmin_{\v{v} \in \R^n} \; \frac{1}{2} \norm{\v{y} - \m{K}_{\m{X} \m{X}} \v{v}}_2^2 + \frac{\sigma^2}{2} \norm{\v{v}}_{\m{K}_{\m{X}\m{X}}}^2
,
\end{align}
where $\v{v}^* = (\m{K}_{\m{X}\m{X}} + \sigma^2 \m{I})^{-1}\v{y}$.
We say that $k(\., \v{x}_i)$ are the \emph{canonical basis functions}, $v_i$ are the \emph{representer weights} \citep{Scholkopf2001representer}, and that $\norm{\v{v}}_{\m{K}_{\m{X}\m{X}}}^2 = \v{v}\T\m{K}_{\m{X}\m{X}}\v{v}$ is the \emph{regulariser}.
The respective optimisation problem for obtaining posterior samples is similar, but involves a stochastic objective, which will be given and analysed in \Cref{sec:posterior-samples}.
\begin{figure}[ht]
\centering
\includegraphics[width=6in]{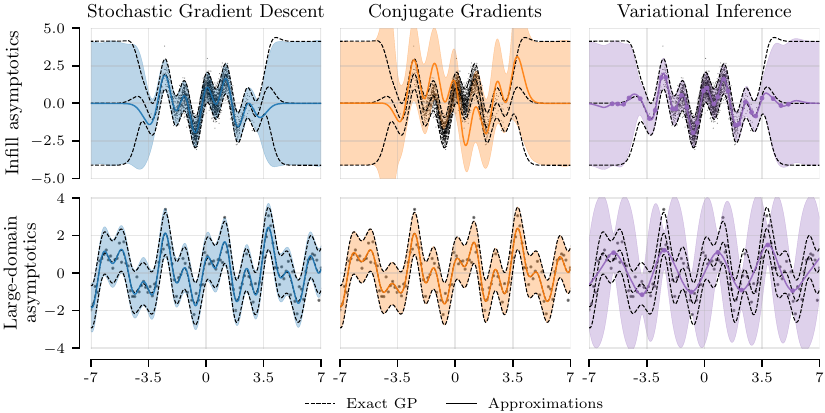}
\caption{Comparison of SGD, CG, and SVGP for GP inference with a squared exponential kernel on $10$k data points from $\sin(2x)+\cos(5x)$ with observation noise scale $0.5$. We draw $2000$ function samples with all methods by running them for $10$ minutes on an RTX $2070$ GPU.
\emph{Infill asymptotics} considers $x_i\~[N](0, 1)$. A large number of points near zero result in a very ill-conditioned kernel matrix, preventing CG from converging. SGD converges in all of input space except at the edges of the data. SVGP can summarise the data with only 20 inducing points.
\emph{Large domain asymptotics} considers data on a regular grid with fixed spacing.
This problem is better conditioned, allowing SGD and CG to recover the exact solution. However, 1024 inducing points are not enough for SVGP to summarise the data.
}
\label{fig:scalable-learning-comparison}
\end{figure}

The optimisation problem \eqref{eqn:mean-optim} with optimal representer weight solution $\v{v}^*$
requires $\c{O}(n^2)$ operations to compute both its square error and regulariser terms exactly.
The square error loss term is amenable to mini-batching, which gives an unbiased estimate in $\c{O}(n)$ operations. 
Assuming that $k$ is stationary, we can stochastically estimate the regulariser with random Fourier features \citep{rahimi08} using the approximation $\norm{\v{v}}_{\m{K}_{\m{X}\m{X}}}^2 \approx \v{v}\T \m{\Phi}_\m{X} \m{\Phi}_\m{X}\T \v{v}$, where $\m{\Phi}_\m{X} = \sbr{\v{\phi}_{\v{\omega}_1}(\m{X}), \dots, \v{\phi}_{\v{\omega}_q}(\m{X})}$ is a random feature matrix such that $\m{K}_\m{XX} \approx \m{\Phi}_\m{X} \m{\Phi}_\m{X}\T$ (see \Cref{sec:random_features} for details).
Combining both estimators gives our SGD objective
\begin{equation}
\label{eqn:stochastic_mean-objective}
 \c{L}(\v{v}) = \frac{n}{2p} \sum_{i=1}^p (y_i - \m{K}_{\v{x}_i \m{X}}\v{v})^2
 + \frac{\sigma^2}{2} \sum_{j=1}^q (\v{\phi}_{\v{\omega}_j}(\m{X})\T \v{v})^2
\end{equation}
where $p$ is the mini-batch size and $q$ is the number of random features.
The regulariser term is unbiased even when drawing a single random feature because the number of features only controls the variance.
\Cref{eqn:stochastic_mean-objective} presents $\c{O}(n)$ time complexity, in contrast to $\c{O}(n^2)$ for exact evaluation.

\subsection{Computing Posterior Samples}
\label{sec:posterior-samples}
We now frame GP posterior samples in a manner amenable to SGD computation similarly to \Cref{eqn:stochastic_mean-objective}.
First, we rewrite the pathwise conditioning expression \eqref{eq:pathwise_conditioning} as
\begin{equation}
\label{eqn:pathwise-zero-mean}
f_{(\.) \given \v{y}}
= \ubr{f_{(\.)} \vphantom{\v{\mu}_{(\.) \given \v{y}}}}_{\text{prior}}
+ \ubr{\m{K}_{(\.) \m{X}} \v{v}^* \vphantom{\v{\mu}_{(\.) \given \v{y}}}}_{\text{posterior mean}}
- \ubr{\m{K}_{(\.)\m{X}} (\m{K}_{\m{XX}} + \sigma^2 \m{I})\inv (\v{f}_\m{X} + \v\eps) \vphantom{\v{\mu}_{(\.) \given \v{y}}}}_{\text{uncertainty reduction term}},
\end{equation}
where $\v\eps \~[N](\v{0}, \sigma^2 \m{I})$, $f \~[GP](0, k)$, and $\v{f}_\m{X}$ is $f$ evaluated at $\m{X}$.
We approximate the prior function sample using random Fourier features (see \Cref{sec:random_features}) and the posterior mean with the minimiser $\v{v}^*$ of \Cref{eqn:stochastic_mean-objective} obtained by SGD. 
Each posterior sample's uncertainty reduction term is parametrised by a set of representer weights given by a linear solve against a noisy prior sample evaluated at the observed inputs, namely $\v{\alpha}^* = (\m{K}_{\m{X}\m{X}} + \sigma^2 \m{I})^{-1}(\v{f}_\m{X} + \v\eps)$. 
We construct an optimisation objective targeting a sample's optimal representer weights as
\begin{equation}
\label{eqn:samples-optim}
\v{\alpha}^* =\argmin_{\v{\alpha}\in\R^n} \; \frac{1}{2} \norm{\v{f}_{\m{X}} + \v\eps - \m{K}_{\m{X} \m{X}}\v{\alpha}}_2^2 + \frac{\sigma^2}{2} \norm{\v{\alpha}}_{\m{K}_{\m{X}\m{X}}}^2,
\quad
\begin{aligned}
\v{f}_{\m{X}} &\~[N](\v{0}, \m{K}_{\m{X}\m{X}}),
\\
\v\eps &\~[N](\v{0}, \sigma^2 \m{I}).
\end{aligned}
\end{equation}
Applying mini-batch estimation to this objective results in high gradient variance, since the presence of $\eps_i$ makes the targets noisy. 
To avoid this, we modify \Cref{eqn:samples-optim} as
\begin{equation}
\label{eqn:samples-optim-reduced}
\v{\alpha}^* =\argmin_{\v{\alpha}\in\R^n} \; \frac{1}{2} \norm{\v{f}_{\m{X}} - \m{K}_{\m{X} \m{X}}\v{\alpha}}_2^2 + \frac{\sigma^2}{2} \norm{\v{\alpha} - \v\delta}_{\m{K}_{\m{X}\m{X}}}^2,
\quad
\begin{aligned}
\v{f}_{\m{X}} &\~[N](\v{0}, \m{K}_{\m{X}\m{X}}),
\\
\v\delta &\~[N](\v{0}, \sigma^{-2} \m{I}),
\end{aligned}
\end{equation}
moving the noise term into the regulariser.
This modification preserves the optimal representer weights since both objectives have the same gradient and are thus equal up to a constant.
\begin{proofbox}
\begin{proof}
Let $\v{w} \~[N](\v{0}, \m{I})$, such that $\v\eps = \sigma \v{w}$ and $\v{\delta} = \sigma^{-1} \v{w}$.
\begin{align}
    \nabla_{\v{\alpha}} \; \eqref{eqn:samples-optim}
    &= \nabla_{\v{\alpha}} \sbr{{\textstyle \frac{1}{2}} \norm{\v{f}_{\m{X}} + \v\eps - \m{K}_{\m{X} \m{X}}\v{\alpha}}_2^2 + \smash{\textstyle \frac{\sigma^2}{2}} \norm{\v{\alpha}}_{\m{K}_{\m{X}\m{X}}}^2} \\
    &= - \m{K}_{\m{X} \m{X}} \del{\v{f}_{\m{X}} + \sigma \v{w} - \m{K}_{\m{X} \m{X}}\v{\alpha}} + \sigma^2 \m{K}_{\m{X}\m{X}} \v{\alpha} \\
    &= - \m{K}_{\m{X} \m{X}} \del{\v{f}_{\m{X}} - \m{K}_{\m{X} \m{X}}\v{\alpha} + \sigma \v{w} - \sigma^2 \v{\alpha}} \\
    \nabla_{\v{\alpha}} \; \eqref{eqn:samples-optim-reduced}
    &= \nabla_{\v{\alpha}} \sbr{{\textstyle \frac{1}{2}} \norm{\v{f}_{\m{X}} - \m{K}_{\m{X} \m{X}}\v{\alpha}}_2^2 + \smash{\textstyle \frac{\sigma^2}{2}} \norm{\v{\alpha} - \v{\delta}}_{\m{K}_{\m{X}\m{X}}}^2} \\
    &= - \m{K}_{\m{X} \m{X}} \del{\v{f}_{\m{X}} - \m{K}_{\m{X} \m{X}}\v{\alpha}} + \sigma^2 \m{K}_{\m{X}\m{X}} \del{\v{\alpha} - \sigma^{-1} \v{w}} \\
    &= - \m{K}_{\m{X} \m{X}} \del{\v{f}_{\m{X}} - \m{K}_{\m{X} \m{X}}\v{\alpha} + \sigma \v{w} - \sigma^2 \v{\alpha}}
\end{align}
The claim follows from the fact that these two expressions are identical.
\end{proof}
\end{proofbox}
This generalises the variance reduction technique of \citet{antoran2023sampling} to the GP setting.
\Cref{fig:grad_var_inducing_sgd} illustrates mini-batch gradient variance for these objectives.
Applying the mini-batch and random feature estimators of \Cref{eqn:stochastic_mean-objective}, each evaluation of the objective takes $\c{O}(ns)$ time, where $s$ is the number of posterior samples. 
\begin{figure}
\centering
\includegraphics[width=6in]{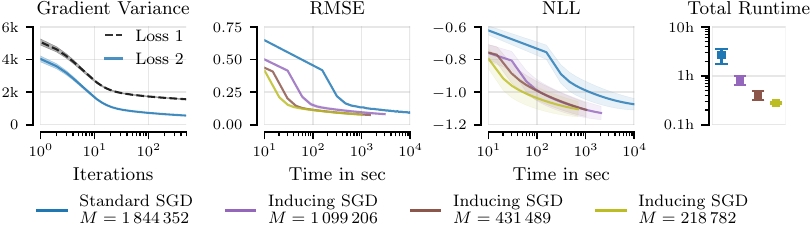}
\caption{Left: gradient variance throughout optimisation for a single-sample mini-batch estimator of \Cref{eqn:samples-optim} (Loss $1$), and the proposed sampling objective from \Cref{eqn:samples-optim-reduced} (Loss $2$), on the \textsc{elevators} dataset ($n \approx 16$k). Middle: RMSE and negative log-likelihood (NLL) obtained by SGD and its inducing point variants on the \textsc{houseelectric} dataset ($n \approx 2$M) for different numbers of inducing points. Right: total runtimes on an A100 GPU.}
\label{fig:grad_var_inducing_sgd}
\end{figure}

\vspace{-.1cm}
\subsection{Inducing Points}
\label{subsec:inducing_points}
So far, our sampling objectives have presented linear cost in dataset size.
In the large-scale setting, algorithms with costs independent of the dataset size are often preferable, which can be achieved through \emph{inducing point posteriors} \citep{titsias09,hensman13}.

Let $\m{Z}$ be a set of $m$ inducing points. 
Applying pathwise conditioning to the inducing point approximation of \citet{titsias09} (see \Cref{sec:sparse_and_variational} for details) gives the expression
\begin{equation}
\label{eqn:pathwise-zero-mean-inducing}
f_{(\.) \given \v{y}}^{[\m{Z}]}
= f_{(\.)}
+ \m{K}_{(\.) \m{Z}} \m{K}_{\m{Z} \m{Z}}\inv \m{K}_{\m{Z}\m{X}} (\m{K}_{\m{X}\m{Z}} \m{K}_{\m{Z}\m{Z}}\inv \m{K}_{\m{Z}\m{X}} + \sigma^2 \m{I})\inv (\v{y} - (\v{f}_\m{X}^{[\m{Z}]} + \v\eps)),
\end{equation}
where $f \~[GP](0, k)$ is a prior function sample, $\v{f}_\m{X}^{[\m{Z}]} = \m{K}_{\m{X}\m{Z}}\m{K}_{\m{Z}\m{Z}}^{-1}\v{f}_\m{Z}$, and $\v\eps\~[N](\v{0},\sigma^2 \m{I})$.
\begin{proofbox}
\begin{proof}
Let $\m{A} = \m{K}_{\m{Z} \m{Z}}\inv \m{K}_{\m{Z}\m{X}} (\m{K}_{\m{X}\m{Z}} \m{K}_{\m{Z}\m{Z}}\inv \m{K}_{\m{Z}\m{X}} + \sigma^2 \m{I})\inv$.
\begin{align}
\E \sbr{f_{(\.) \given \v{y}}^{[\m{Z}]}}
&= \E \sbr{f_{(\.)}} + \m{K}_{(\.) \m{Z}} \m{A} (\v{y}  - (\E [\v{f}_\m{X}^{[\m{Z}]}] + \E \sbr{\v\eps})) \\
&= \sigma^{-2}\m{K}_{(\.) \m{Z}}(\m{K}_{\m{Z}\m{Z}} + \sigma^{-2}\m{K}_{\m{Z} \m{X}} \m{K}_{\m{X} \m{Z}})^{-1} \m{K}_{\m{Z} \m{X}} \v{y} \\
&= \m{K}_{(\.) \m{Z}} \v{v}^* = \mu_{(\.) \given \v{y}}^{[\m{Z}]} \qquad \text{(see \Cref{eq:posterior_mean_titsias})}
\end{align}
\begin{align}
\Var \del{f_{(\.) \given \v{y}}^{[\m{Z}]}}
&= \Var \del{f_{(\.)}} + \Var \del{\m{K}_{(\.) \m{Z}} \m{A}(\v{f}_\m{X}^{[\m{Z}]} + \v\eps)} \notag \\
&\quad - \Cov \del{f_{(\.)}, \m{K}_{(\.) \m{Z}} \m{A}(\v{f}_\m{X}^{[\m{Z}]} + \v\eps)} \notag \\
&\quad - \Cov \del{\m{K}_{(\.) \m{Z}} \m{A}(\v{f}_\m{X}^{[\m{Z}]} + \v\eps), f_{(\.)}} \\
&= \m{K}_{(\.) (\.)} + \m{K}_{(\.) \m{Z}} \m{A}
\del{\m{K}_{\m{X} \m{Z}} \m{K}_{\m{Z} \m{Z}}\inv \m{K}_{\m{Z} \m{X}} + \sigma^2 \m{I}}
\m{A}\T \m{K}_{\m{Z} (\.)} \notag \\
&\quad -  \m{K}_{(\.) \m{Z}} \m{K}_{\m{Z} \m{Z}}\inv \m{K}_{\m{Z} \m{X}} \m{A}\T \m{K}_{\m{Z} (\.)} - \m{K}_{(\.) \m{Z}} \m{A} \m{K}_{\m{X} \m{Z}} \m{K}_{\m{Z} \m{Z}}\inv \m{K}_{\m{Z} (\.)} \\
&= \m{K}_{(\.) (\.)} - \m{K}_{(\.) \m{Z}} \m{A} \m{K}_{\m{X} \m{Z}} \m{K}_{\m{Z} \m{Z}}\inv \m{K}_{\m{Z} (\.)} \quad \text{(\emph{Woodbury identity})} \\
&= \m{K}_{(\.)(\.)} - \m{K}_{(\.)\m{Z}} \del{\m{K}_{\m{Z}\m{Z}}\inv - \del{\m{K}_{\m{Z}\m{Z}} + \sigma^{-2} \m{K}_{\m{Z}\m{X}} \m{K}_{\m{X}\m{Z}}} \inv} \m{K}_{\m{Z}(\.)} \\
&= \m{K}_{(\.)(\.) \given \v{y}}^{[\m{Z}]} \qquad \text{(see \Cref{eq:posterior_cov_titsias})}
\end{align}
This confirms the claim that $f_{(\.) \given \v{y}}^{[\m{Z}]}$ has the desired expected value and variance.
\end{proof}
\end{proofbox}

Following \citet[Theorem 5]{wild2023connections}, the optimal inducing point posterior mean $\mu_{(\.) \given \v{y}}^{[\m{Z}]}$ can therefore be written as
\begin{align}
\mu_{(\.) \given \v{y}}^{[\m{Z}]}
&= \m{K}_{(\.) \m{Z}}\v{v}^* = \sum_{j=1}^m k(\.,\v{z}_j) v_j^*, \\
\label{eqn:inducing-optim}
\v{v}^* &= \argmin_{\v{v}\in\R^m} \; \frac{1}{2} \norm{\v{y} - \m{K}_{\m{X} \m{Z}} \v{v}}_2^2 + \frac{\sigma^2}{2} \norm{\v{v}}_{\m{K}_{\m{Z}\m{Z}}}^2,
\end{align}
and we can again parametrise the uncertainty reduction term as $\m{K}_{(\.) \m{Z}} \v{\alpha}^*$ with 
\begin{gather}
\label{eqn:inducing-samples-ht-optim}
\v{\alpha}^* = \argmin_{\v\alpha\in\R^m} \; \frac{1}{2} \Vert \v{f}_\m{X}^{[\m{Z}]} + \v\eps - \m{K}_{\m{X} \m{Z}}\v\alpha \Vert_2^2 + \frac{\sigma^2}{2} \norm{\v\alpha }_{\m{K}_{\m{Z}\m{Z}}}^2,
\quad
\begin{aligned}
\v{f}_\m{X}^{[\m{Z}]} &\~[N](\v{0}, \m{K}_{\m{X}\m{Z}}\m{K}_{\m{Z}\m{Z}}^{-1}\m{K}_{\m{Z}\m{X}}),
\\
\v\eps &\~[N](\v{0}, \sigma^2 \m{I}).
\end{aligned}
\raisetag{0.925\baselineskip}
\end{gather}
\begin{proofbox}
\begin{proof}
We can show that these optimisation objectives result in the desired $\v{v}^*$ and $\v{\alpha}^*$ by deriving their minimisers analytically.
To this end, we first calculate their gradients.
\begin{align}
    \nabla_{\v{v}} \; \eqref{eqn:inducing-optim}
    &=
    \nabla_{\v{v}} \Big[ {\textstyle \frac{1}{2}} \norm{\v{y} - \m{K}_{\m{X}\m{Z}} \v{v}}_2^2 + {\textstyle \frac{\sigma^2}{2}} \norm{\v{v}}_{\m{K}_{\m{Z}\m{Z}}}^2 \Big] \\
    &=
    - \m{K}_{\m{Z}\m{X}} \del{\v{y} - \m{K}_{\m{X}\m{Z}} \v{v}}
    + \sigma^2 \m{K}_{\m{Z}\m{Z}} \v{v} \\
    \nabla_{\v{\alpha}} \; \eqref{eqn:inducing-samples-ht-optim}
    &=
    \nabla_{\v{\alpha}} \Big[ {\textstyle \frac{1}{2}} \Vert \v{f}_{\m{X}}^{[\m{Z}]} + \v\eps - \m{K}_{\m{X}\m{Z}} \v{\alpha} \Vert_2^2 + {\textstyle \frac{\sigma^2}{2}} \norm{\v{\alpha}}_{\m{K}_{\m{Z}\m{Z}}}^2 \Big] \\
    &=
    - \m{K}_{\m{Z}\m{X}} \del{\v{f}_{\m{X}}^{[\m{Z}]} + \v\eps - \m{K}_{\m{X}\m{Z}} \v{\alpha}}
    + \sigma^2 \m{K}_{\m{Z}\m{Z}} \v{\alpha}
\end{align}
Settings these gradients to zero and solving for $\v{v}$ and $\v{\alpha}$ yields $\v{v}^*$ and $\v{\alpha}^*$.
\begin{align}
    \v{v}^*
    &= \del{\sigma^2 \m{K}_{\m{Z}\m{Z}} + \m{K}_{\m{Z} \m{X}} \m{K}_{\m{X} \m{Z}}}\inv \m{K}_{\m{Z} \m{X}} \v{y} \\
    \v{\alpha}^*
    &= \del{\sigma^2 \m{K}_{\m{Z}\m{Z}} + \m{K}_{\m{Z} \m{X}} \m{K}_{\m{X} \m{Z}}}\inv \m{K}_{\m{Z} \m{X}} \del{\v{f}_{\m{X}}^{[\m{Z}]} + \v\eps}
\end{align}
We can rearrange $\v{v}^*$ and $\v{\alpha}^*$ by applying the Woodbury identity twice.
\begin{align}
    \v{v}^*
    &= \del{\sigma^2 \m{K}_{\m{Z}\m{Z}} + \m{K}_{\m{Z} \m{X}} \m{K}_{\m{X} \m{Z}}}\inv \m{K}_{\m{Z} \m{X}} \v{y} \\
    &= \sigma^{-2} \del{\m{K}_{\m{Z}\m{Z}} + \sigma^{-2} \m{K}_{\m{Z} \m{X}} \m{K}_{\m{X} \m{Z}}}\inv \m{K}_{\m{Z} \m{X}} \v{y} \\
    &= \sbr{\m{K}_{\m{Z} \m{Z}}\inv - \m{K}_{\m{Z} \m{Z}}\inv \m{K}_{\m{Z} \m{X}} \del{\sigma^2 \m{I} + \m{K}_{\m{X} \m{Z}} \m{K}_{\m{Z} \m{Z}}\inv \m{K}_{\m{Z} \m{X}}}\inv \m{K}_{\m{X} \m{Z}} \m{K}_{\m{Z} \m{Z}}\inv } \sigma^{-2} \m{K}_{\m{Z} \m{X}} \v{y} \\
    &= \m{K}_{\m{Z} \m{Z}}\inv \m{K}_{\m{Z} \m{X}} \sbr{ \m{I} - \del{\sigma^2 \m{I} + \m{K}_{\m{X} \m{Z}} \m{K}_{\m{Z} \m{Z}}\inv \m{K}_{\m{Z} \m{X}}}\inv \m{K}_{\m{X} \m{Z}} \m{K}_{\m{Z} \m{Z}}\inv \m{K}_{\m{Z} \m{X}}} \sigma^{-2} \v{y} \\
    &= \m{K}_{\m{Z} \m{Z}}\inv \m{K}_{\m{Z} \m{X}} \del{\m{K}_{\m{X} \m{Z}} \m{K}_{\m{Z} \m{Z}}\inv \m{K}_{\m{Z} \m{X}} + \sigma^2 \m{I}}\inv \v{y}
\end{align}
The case for $\v{\alpha}^*$ is analogous.
Therefore, \Cref{eqn:pathwise-zero-mean-inducing} can be expressed as
\begin{equation}
f_{(\.) \given \v{y}}^{[\m{Z}]}
= f_{(\.)}
+ \m{K}_{(\.) \m{Z}} (\v{v}^* - \v{\alpha}^*),
\end{equation}
which demonstrates that optimising the objectives from \Cref{eqn:inducing-optim,eqn:inducing-samples-ht-optim} results in the desired $\v{v}^*$ and $\v{\alpha}^*$ to draw samples from the inducing point posterior.
\end{proof}
\end{proofbox}
Exact implementation of \Cref{eqn:inducing-samples-ht-optim} is precluded by the need to sample $\v{f}_\m{X}^{[\m{Z}]}$ from a Gaussian with covariance $\m{K}_{\m{X}\m{Z}}\m{K}_{\m{Z}\m{Z}}^{-1}\m{K}_{\m{Z}\m{X}}$. 
However, we identify this matrix as a Nyström approximation to $\m{K}_{\m{X}\m{X}}$.
Thus, we can approximate \Cref{eqn:inducing-samples-ht-optim} by replacing $\v{f}_\m{X}^{[\m{Z}]}$ with $\v{f}_\m{X}$, which can be sampled with random Fourier features (see \Cref{sec:random_features}).
The approximation error is small when $m$ is large and the inducing points are close enough to the data, which is effectively required for inducing point methods to perform well in general.
Therefore, we use $\v{f}_\m{X}$ instead of $\v{f}_\m{X}^{[\m{Z}]}$ and apply the stochastic estimator from \Cref{eqn:stochastic_mean-objective} to the inducing point sampling objective.

The inducing point objectives differ from those presented in previous sections in that there are $\c{O}(m)$ instead of $\c{O}(n)$ learnable representer weights $\v{v}$ and $\v{\alpha}$, and we may choose the value of $m$ and locations $\m{Z}$ freely. 
The cost of inducing point representer weight updates is thus $\c{O}(ms)$, where $s$ is the number of samples. 
This contrasts with the $\c{O}(m^3)$ update cost of stochastic gradient variational Gaussian processes \citep{hensman13}.
\Cref{fig:grad_var_inducing_sgd} shows that SGD with $m \approx 100$k inducing points matches the performance of regular SGD on \textsc{houseelectric} with $n \approx 2$M, but is an order of magnitude faster.

\begin{figure}[ht]
\centering
\includegraphics[width=6in]{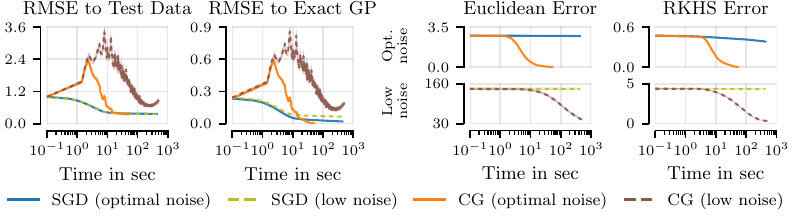}
\caption{
Convergence of GP posterior mean with SGD and CG as a function of time on the \textsc{elevators} dataset ($n \approx 16$k) while setting the noise scale to (i) maximise exact GP marginal likelihood and (ii) to $0.001$, labelled \emph{low noise}. We plot from left to right: test RMSE, RMSE to the exact GP mean at the test inputs, representer weight error $\norm{\v{v} - \v{v}^*}_{2}$, and RKHS error $\norm{\v{v} - \v{v}^*}_{\m{K}_{\m{X} \m{X}}}$. In the latter two plots, the low noise setting is shown on the bottom.
}
\label{fig:exact_vs_low_noise}
\end{figure}

\subsection{Implicit Bias and Posterior Geometry}
\label{sec:implicit_bias}
We have detailed an SGD-based scheme for obtaining approximate samples from a posterior Gaussian process.
Despite SGD's significantly lower cost per iteration than CG, its convergence to the true optima, shown in \Cref{fig:exact_vs_low_noise}, is much slower with respect to both Euclidean representer weight space, and the reproducing kernel Hilbert space (RKHS) induced by the kernel.
Nonetheless, the predictions obtained by SGD are very close to those of the exact GP, and effectively achieve the same test RMSE.
Moreover, \Cref{fig:error-and-eigenfunctions} shows the SGD posterior on a 1D toy task exhibits error bars of the correct width close to the data, and which revert smoothly to the prior far away from the data. 
Empirically, differences between the SGD and exact posteriors concentrate at the borders of data-dense regions.

We now argue the behaviour seen in \Cref{fig:error-and-eigenfunctions} is a general feature of SGD: one can expect it to obtain good performance even in situations where it does not converge to the exact solution.
Consider posterior function samples in pathwise form, namely $f_{(\.) \given \v{y}} = f_{(\.)} + \m{K}_{(\.)\m{X}}\v{v}$, where $f\~[GP](0, k)$ is a prior function sample and $\v{v}$ are the learnable representer weights.
We characterise the behaviour of SGD-computed approximate posteriors by splitting the input space $\c{X}$ into three regions, which we call the \emph{prior}, \emph{interpolation}, and \emph{extrapolation} regions.
The prior region consists of points which are sufficiently distant from the observed training data, the interpolation region refers to points which are close to the observed training data, and the extrapolation region corresponds to the remaining space after eliminating the prior and interpolation regions.
In the following, we will discuss and theoretically analyse these regions separate from each other.

\subsubsection{The Prior Region}
This corresponds to points sufficiently distant from the observed data. 
Here, for kernels that decay over space, the canonical basis functions $k(\., \v{x}_i)$ go to zero. 
Thus, both the true posterior and any approximations formulated pathwise revert to the prior.
More precisely, let $\c{X} = \R^d$, let $k$ satisfy $\lim_{c \to \infty} k(c\,\v{x}, \v{x}') = 0$ for all $\v{x}'$ and $\v{x}$ in $\c{X}$, and let $f_{(\.) \given \v{y}}$ be given by $f_{(\.) \given \v{y}} = f_{(\.)} + \m{K}_{(\.)\m{X}}\v{v}$, with $\v{v} \in \R^n$. Then, by passing the limit through the sum, for any fixed $\v{v}$, it follows immediately that $ \lim_{c \to \infty} f_{c\,\v{x} \given \v{y}} = f_{c\,\v{x}}$. Therefore, SGD cannot incur error in regions which are sufficiently far away from the data.
This effect is depicted in \Cref{fig:error-and-eigenfunctions}.

\subsubsection{The Interpolation Region}
This includes points close to the training data. 
We characterise this region via subspaces of the RKHS, where we show that SGD incurs small error.

Let $\m{K}_{\m{X}\m{X}} = \m{U}\m\Lambda\m{U}^\mathsf{T}$ be the eigendecomposition of the kernel matrix.
We index the eigenvalues $\m\Lambda = \mathrm{diag}(\lambda_1,\dotsc, \lambda_n)$ in descending order.
Define the \emph{spectral basis functions} as eigenvalue-weighted and eigenvector-rotated linear combinations of canonical basis functions
\begin{equation}
u^{(i)}(\.) = \sum_{j=1}^n \frac{U_{ji}}{\sqrt{\lambda_i}} k(\., \v{x}_j).
\end{equation}
These functions, which also appear in kernel principal component analysis, are orthonormal with respect to the RKHS inner product.
To characterise them further, we lift the Courant-Fischer characterisation of eigenvalues and eigenvectors to the RKHS $\c{H}_k$ induced by $k$, obtaining the expression
\begin{equation}
u^{(i)}(\.) = \argmax_{u \in \c{H}_k} \cbr{\sum_{i=1}^n u(\v{x}_i)^2 : \norm{u}_{\c{H}_k} = 1, \innerprod{u}{u^{(j)}}_{\c{H}_k} = 0, \forall j < i}.
\end{equation}
This means in particular that the top spectral basis function $u^{(1)}$ is a function of fixed RKHS norm, that is, of fixed degree of smoothness, as defined by the kernel $k$, which takes maximal values at the observations $\v{x}_1, ..., \v{x}_n$. 
Thus, $u^{(1)}$ will be large near clusters of observations.
The same will be true for the subsequent spectral basis functions, which also take maximal values at the observations, but are constrained to be RKHS-orthogonal to previous spectral basis functions. 
The proof is provided in Appendix \ref{apd:courant_in_input_space}.
\Cref{fig:error-and-eigenfunctions} confirms that the top spectral basis functions are indeed centred on the observed data.

\begin{figure}[ht]
\centering
\includegraphics[width=6in]{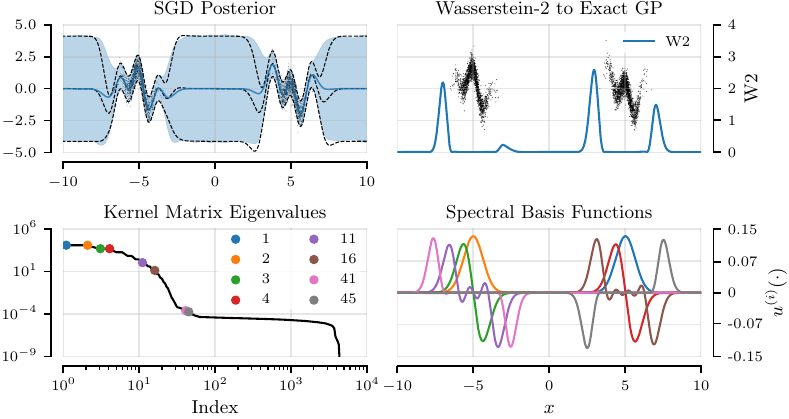}
\caption{SGD error and spectral basis functions. Top-left: SGD (blue) and exact GP (black, dashed) fit to a toy regression dataset ($n = 10$k). Top-right: Wasserstein-2 distance (W2) between both processes' marginals. The Wasserstein-2 distances are low near the data (interpolation region) and far away from the training data. The error concentrates at the edges of the data (extrapolation region). Bottom: The low-index spectral basis functions lie on the interpolation region, where the Wasserstein-2 distances are low, while functions of index $10$ and larger lie on the extrapolation region where the error is large.}
\label{fig:error-and-eigenfunctions}
\end{figure}

Empirically, SGD matches the true posterior in this region.
We formalise this observation by showing that SGD converges quickly in the directions spanned by spectral basis functions with large eigenvalues.

\begin{restatable}{proposition}{PropSGD}
\label{thm:sgd_convergence}
Let $\delta > 0$ and $\sigma^2 > 0$.
Let $\proj_{u^{(i)}}(\.)$ be the orthogonal projection onto the subspace spanned by $u^{(i)}$.
Let $\mu_{(\.)\given \v{y}}$ be the exact posterior predictive mean, and let $\hat{\mu}_{(\.) \given \v{y}}$ be the posterior predictive mean obtained by Polyak-averaged SGD after $t$ steps, starting from an initial set of representer weights equal to zero, and using a sufficiently small learning rate of $0 < \eta < \frac{1}{\lambda_1(\lambda_1 + \sigma^2)}$.
Assume the stochastic estimate of the gradient is $G$-sub-Gaussian.
Then, with probability $1-\delta$, we have for $i=1,...,n$ that
\begin{equation}
\norm[1]{\proj_{u^{(i)}} \mu_{(\.)\given\v{y}} - \proj_{u^{(i)}} \hat{\mu}_{(\.) \given \v{y}}}_{\c{H}_k} \leq \frac{1}{\sqrt{\lambda_i^3}} \del{\frac{\norm{\v{y}}_2}{\eta\sigma^2t} + G\sqrt{\frac{2}{t} \log\frac{n}{\delta}}}.
\end{equation}
\end{restatable}

We expect $G$ to be at most $\c{O}(\lambda_1^2 \norm{\v{y}}_\infty)$ with high probability.
The proof for \Cref{thm:sgd_convergence}, an additional pointwise convergence bound, and a variant which handles projections onto general subspaces spanned by basis functions are provided in Appendix \ref{apd:implicit_bias}.

The result extends immediately from the posterior mean to posterior samples by replacing $\v{y}$ with $\v{f}_{\m{X}} + \v\eps$.
Therefore, \emph{SGD quickly converges to the posterior GP in the data-dense region}, where the spectral basis functions corresponding to large eigenvalues are located.
Since convergence speed on the span of each basis function is independent of the magnitude of the other basis functions' eigenvalues, SGD can perform well even when the kernel matrix is ill-conditioned.
This is shown in \Cref{fig:exact_vs_low_noise}.

\subsubsection{The Extrapolation Region}
This can be found by elimination from the input space of the prior and interpolation regions, in both of which SGD incurs low error. 
Consider the spectral basis functions $u^{(i)}(\.)$ with small eigenvalues.
By orthogonality of $u^{(1)},...,u^{(n)}$, such functions cannot be large near the observations while retaining a prescribed norm. 
Their mass is therefore placed away from the observations. 
SGD converges slowly in this region, resulting in a large error in its solution in both a Euclidean and RKHS sense, as seen in \Cref{fig:exact_vs_low_noise}. 
Fortunately, due to the lack of data in the extrapolation region, the excess test error incurred due to SGD non-convergence may be low, resulting in \emph{benign non-convergence} \citep{ZouWBGK2021benign}.
Similar phenomena have been observed in the inverse problems literature, where this is called \emph{iterative regularisation} \citep{hansen98,jin23}.
\Cref{fig:error-and-eigenfunctions} shows that the Wasserstein-2 distance to the exact GP predictions is large in this region, as, when initialised at zero, SGD tends to return small representer weights, thus reverting to the prior.

\section{Experiments}
\label{sec:experiments}
We now turn to empirical evaluation of SGD GPs, focusing on their predictive and decision-making properties. 
We compare SGD GPs with the two most popular scalable GP techniques: preconditioned conjugate gradient (CG) \citep{WangPGT2019exactgp} and sparse stochastic variational inference (SVGP) \citep{hensman13}.
For all methods, we use $2000$ random Fourier features to draw each prior function used for computing posterior function samples via pathwise conditioning.
Following \citet{WangPGT2019exactgp}, we use a pivoted Cholesky preconditioner of size $100$ for CG, except in cases where this slows down convergence, where we report results without preconditioning instead.
For SVGP, we use $m = 4096$ inducing points for all datasets, initialising their locations with the $k$-means algorithm.
In all SGD experiments, we use a Nesterov momentum value of $0.9$ and Polyak averaging \citep{antoran2023sampling}.
Additionally, we draw $100$ new random features at each step to estimate the regulariser term.
Furthermore, we perform gradient clipping using a maximum gradient norm of $0.1$.

\subsection{Regression Baselines}
\label{subsec:regression}
We first compare SGD-based predictions with baselines in terms of predictive performance, scaling of computational cost with problem size, and robustness to the ill-conditioning of kernel matrices and corresponding linear systems.
Following \citet{WangPGT2019exactgp}, we consider 9 datasets from the UCI repository \citep{Dua2019UCI} ranging in size from $n = 15$k to $n \approx 2$M data points and dimensionality from $d = 3$ to $d = 90$. 
We report mean and standard deviation over five $90\%$-train $10\%$-test splits for the small and medium datasets, and three splits for the largest dataset.
For all methods, we use a 
Matérn-$\nicefrac{3}{2}$ kernel with a fixed set of hyperparameters obtained using maximum marginal likelihood.
For each dataset, we select a homoscedastic Gaussian noise scale, signal scale and a separate length scale per input.
For datasets with less than $50$k observations, we tune these hyperparameters to maximise the exact GP marginal likelihood. 
The cubic cost of this procedure makes it intractable at a larger scale.
Thus, for datasets with more than $50$k observations, we use the following procedure:
\begin{enumerate}
\item From the training data, select a centroid data point uniformly at random.
\item Select a subset of $10$k data points with the smallest Euclidean distance to the centroid.
\item Find hyperparameters by maximising the exact GP marginal likelihood on this subset.
\item Repeat the preceding steps with $10$ different centroids and average the hyperparameters.
\end{enumerate}
This approach avoids aliasing bias due to data subsampling and is tractable for large datasets.

We run SGD for $100$k steps using a learning rate of $0.5$ to estimate the representer weights of the mean function, a learning rate of $0.1$ to draw samples, and a fixed batch size of $512$ for both the mean function and samples.
For CG, we run up to $1000$ steps for datasets with $n \leq 500$k, and a relative residual norm tolerance of $0.01$.
On the four largest datasets, CG's cost per step is too large to run $1000$ steps.
Instead, we run $100$ steps, which takes roughly 9 hours per function sample on a TPUv2 core.
For SVGP, we learn the variational parameters for $m = 4096$ inducing points by maximising the ELBO with Adam until convergence.
For all methods, we estimate predictive variances for log-likelihood computations from 64 function samples drawn using pathwise conditioning.

\begin{figure}[t]
\centering
\includegraphics[width=6in]{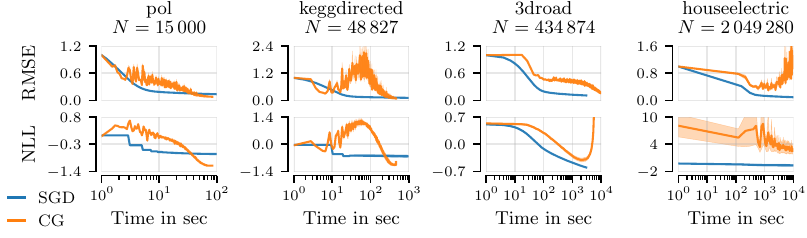}
\caption{Test RMSE and NLL as a function of compute time for CG and SGD. Step-like discontinuities are due to the fact that the metrics are evaluated at discrete intervals.}
\label{fig:rmse_llh_trace}
\end{figure}

\subsubsection{Predictive Performance and Scalability with Training Dataset Size}
Our complete set of results is provided in \Cref{tab:regression}, including test RMSE, test negative log-likelihood (NLL) and compute time needed to obtain the predictive mean on a single core of a TPUv2 device.
Drawing multiple samples requires repeating this computation, which we perform in parallel.
In the small setting ($n \leq 20$k), taking $100$k steps of SGD presents a compute cost comparable to running CG to tolerance, which usually takes around $500$ - $800$ steps. 
Here, CG eventually converges to the exact solution, while SGD tends to present increased test error due to non-convergence.
In the large setting ($n \geq 100$k), where neither method converges within the provided compute budget, SGD achieves better RMSE and NLL.
SVGP always converges faster than CG and SGD, but it only performs best on the \textsc{buzz} dataset, which can likely be summarised well by $m = 4096$ inducing points.

From \Cref{fig:rmse_llh_trace}, we see that SGD makes the vast majority of its progress in prediction space in its first few iterations, improving roughly monotonically with the number of steps. 
Thus, early stopping after $100$k iterations incurs only moderate errors. 
In contrast, CG's initial steps actually increase test error, resulting in very poor performance if stopped too early.
This interacts poorly with the number of CG steps needed and the cost per step, which generally grow with increased amounts of data \citep{terenin23}.

\begin{table}
\caption{Regression task mean and standard error for GP predictive mean RMSE, low-noise RMSE ($\dagger$), TPUv2 node hours used to obtain the predictive mean, and negative log-likelihood (NLL) computed with variances estimated from 64 function samples. SVGP is omitted for the low noise setting, where it fails to run. Metrics are reported for the datasets standardised to zero mean and unit variance.}
\label{tab:regression}
\scriptsize
\setlength{\tabcolsep}{2.5pt}
\renewcommand{\arraystretch}{1.1}
\begin{tabular}{l c c c c c c c c c c}
\toprule
\multicolumn{2}{c}{Dataset} & \textsc{pol} & \textsc{elevators} & \textsc{bike} & \textsc{protein} & \textsc{keggdir} & \textsc{3droad} & \textsc{song} & \textsc{buzz} & \textsc{houseelec} \\
\multicolumn{2}{c}{Size} & 15000 & 16599 & 17379 & 45730 & 48827 & 434874 & 515345 & 583250 & 2049280 \\
\midrule
\multirow{3}{*}{\rotatebox[origin=c]{90}{RMSE}}
 & SGD
 & 0.13\,$\pm$\,0.00 & 0.38\,$\pm$\,0.00 & 0.11\,$\pm$\,0.00 & \textbf{0.51\,$\pm$\,0.00} & 0.12\,$\pm$\,0.00 & \textbf{0.11\,$\pm$\,0.00} & \textbf{0.80\,$\pm$\,0.00} & 0.42\,$\pm$\,0.01 & \textbf{0.09\,$\pm$\,0.00} \\
 & CG
 & \textbf{0.08\,$\pm$\,0.00} & \textbf{0.35\,$\pm$\,0.00} & \textbf{0.04\,$\pm$\,0.00} & \textbf{0.50\,$\pm$\,0.00} & \textbf{0.08\,$\pm$\,0.00} & 0.15\,$\pm$\,0.01 & 0.85\,$\pm$\,0.03 & 1.41\,$\pm$\,0.08 & 0.87\,$\pm$\,0.14 \\
 & SVGP
 & 0.10\,$\pm$\,0.00 & 0.37\,$\pm$\,0.00 & 0.07\,$\pm$\,0.00 & 0.57\,$\pm$\,0.00 & 0.09\,$\pm$\,0.00 & 0.49\,$\pm$\,0.01 & 0.81\,$\pm$\,0.00 & \textbf{0.33\,$\pm$\,0.00} & 0.11\,$\pm$\,0.01 \\
\midrule
\multirow{3}{*}{\rotatebox[origin=c]{90}{RMSE $\dagger$}}
 & SGD
 & \textbf{0.13\,$\pm$\,0.00} & \textbf{0.38\,$\pm$\,0.00} & 0.11\,$\pm$\,0.00 & \textbf{0.51\,$\pm$\,0.00} & \textbf{0.12\,$\pm$\,0.00} & \textbf{0.11\,$\pm$\,0.00} & \textbf{0.80\,$\pm$\,0.00} & \textbf{0.42\,$\pm$\,0.01} & \textbf{0.09\,$\pm$\,0.00} \\
 & CG
 & 0.16\,$\pm$\,0.01 & 0.68\,$\pm$\,0.09 & \textbf{0.05\,$\pm$\,0.01} & 3.03\,$\pm$\,0.23 & 9.79\,$\pm$\,1.06 & 0.34\,$\pm$\,0.02 & 0.83\,$\pm$\,0.02 & 5.66\,$\pm$\,1.14 & 0.93\,$\pm$\,0.19 \\
 & SVGP
 & --- & --- & --- & --- & --- & --- & --- & --- & --- \\
\midrule
\multirow{3}{*}{\rotatebox[origin=c]{90}{Minutes}}
 & SGD
 & 3.51\,$\pm$\,0.01 & 3.51\,$\pm$\,0.01 & 5.70\,$\pm$\,0.02 & \textbf{7.10\,$\pm$\,0.01} & 15.2\,$\pm$\,0.02 & 27.6\,$\pm$\,11.4 & 220\,$\pm$\,14.5 & 347\,$\pm$\,61.5 & 162\,$\pm$\,54.3 \\
 & CG
 & \textbf{2.18\,$\pm$\,0.32} & \textbf{1.72\,$\pm$\,0.60} & \textbf{2.81\,$\pm$\,0.22} & 9.07\,$\pm$\,1.68 & \textbf{12.5\,$\pm$\,1.99} & 85.2\,$\pm$\,36.0 & 195\,$\pm$\,2.31 & 351\,$\pm$\,48.3 & 157\,$\pm$\,0.41 \\
 & SVGP
 & 21.2\,$\pm$\,0.27 & 21.3\,$\pm$\,0.12 & 20.5\,$\pm$\,0.02 & 20.8\,$\pm$\,0.04 & 20.8\,$\pm$\,0.05 & \textbf{21.0\,$\pm$\,0.12} & \textbf{24.7\,$\pm$\,0.05} & \textbf{25.6\,$\pm$\,0.05} & \textbf{20.0\,$\pm$\,0.03} \\
\midrule
\multirow{3}{*}{\rotatebox[origin=c]{90}{NLL}}
 & SGD
 & -0.70\,$\pm$\,0.02 & 0.47\,$\pm$\,0.00 & -0.48\,$\pm$\,0.08 & 0.64\,$\pm$\,0.01 & -0.62\,$\pm$\,0.07 & \textbf{-0.60\,$\pm$\,0.00} & \textbf{1.21\,$\pm$\,0.00} & 0.83\,$\pm$\,0.07 & \textbf{-1.09\,$\pm$\,0.04} \\
 & CG
 & \textbf{-1.17\,$\pm$\,0.01} & \textbf{0.38\,$\pm$\,0.00} & \textbf{-2.62\,$\pm$\,0.06} & \textbf{0.62\,$\pm$\,0.01} & \textbf{-0.92\,$\pm$\,0.10} & 16.27\,$\pm$\,0.45 & 1.36\,$\pm$\,0.07 & 2.38\,$\pm$\,0.08 & 2.07\,$\pm$\,0.58 \\
 & SVGP
 & -0.71\,$\pm$\,0.01 & 0.43,$\pm$\,0.00 & -1.27\,$\pm$\,0.02 & 0.86\,$\pm$\,0.01 & -0.70\,$\pm$\,0.02 & 0.67\,$\pm$\,0.02 & 1.22\,$\pm$\,0.00 & \textbf{0.25\,$\pm$\,0.04} & -0.89\,$\pm$\,0.10 \\
\bottomrule
\end{tabular}
\end{table}

\subsubsection{Robustness to Kernel Matrix Ill-Conditioning}
GP models are known to be sensitive to kernel matrix conditioning.
We explore how this affects the algorithms under consideration by fixing the noise to a low value of $\sigma = 0.001$ and running them on our regression datasets. 
\Cref{tab:regression} shows that the performance of CG severely degrades on all datasets and, for SVGP, optimisation diverges for all datasets.
SGD's results remain essentially unchanged. This is because the noise only changes the smallest kernel matrix eigenvalues substantially and these do not affect convergence for the top spectral basis functions. 
This mirrors results presented previously for the \textsc{elevators} dataset in \Cref{fig:exact_vs_low_noise}. 

\subsubsection{Regression with Large Numbers of Inducing Points}
We demonstrate the inducing point variant of our method, presented in \Cref{subsec:inducing_points}, on \textsc{houseelectric} ($n \approx 2$M). 
We select varying numbers of inducing points as a subset of the training dataset using Annoy, an approximate nearest neighbour algorithm \citep{Annoy}.
We run Annoy on the \textsc{houseelectric} dataset with \texttt{num\_trees} set to $50$, retrieve $100$ nearest neighbours for each point in the dataset, and keep close neighbours whose Euclidean distance to the original point is less than a specified \texttt{max\_dist} parameter.
For each point with more than one close neighbour, we eliminate the original point and any other points which are close to both the original point and any of its close neighbours.
The remaining points are kept, and their number is adjusted by modifying \texttt{max\_dist}.
\Cref{fig:grad_var_inducing_sgd} shows that the time required for $100$k SGD steps scales roughly linearly with inducing points. It takes $68$min for full SGD and $50$min, $25$min, and $17$min for $m = 1099$k, $728$k, and $218$k, respectively.
Performance in terms of RMSE and NLL degrades less than $10\%$, even when using $218$k inducing points.

\subsection{Large-Scale Parallel Thompson Sampling}
\label{subsec:bo}
A fundamental goal of scalable Gaussian processes is to produce uncertainty estimates useful for sequential decision-making. 
Motivated by problems in large-scale recommender systems, where both the initial dataset and the total number of users queried are simultaneously large \citep{rubens2015active, elahi2016survey}, we benchmark SGD on a large-scale Bayesian optimisation task.

We draw a target function from a GP prior $g \~[GP](0, k)$ and optimize it on $\c{X} = [0, 1]^d$ with parallel Thompson sampling \citep{hernandezlobato2017Parallel}.
In particular, we choose $\v{x}_{\mathrm{new}} = \argmax_{\v{x} \in \c{X}} f_{\v{x} \given \v{y}}$ for a set of posterior samples drawn in parallel. 
We compute these samples using pathwise conditioning with each respective scalable GP method.
For each function sample maximum, we evaluate $y_\mathrm{new} = g(\v{x}_{\mathrm{new}}) + \eps$, where $\eps\~[N](0, \sigma^2)$ with $\sigma = 0.001$, and we add the pair $(\v{x}_{\mathrm{new}}, y_{\mathrm{new}})$ to the training data.
We use an acquisition batch size of $1000$ samples, and maximise them with a multi-start gradient-based approach:
\begin{enumerate}
\item Evaluate the posterior function sample at a large number of nearby input locations.
We find nearby locations using a combination of exploration and exploitation.
For exploration, we sample locations uniformly at random from $[0, 1]^d$.
For exploitation, we first subsample the training data with probabilities proportional to the observed objective function values and then add Gaussian noise $\eps_{\t{nearby}} \~[N](0, \sigma_{\t{nearby}}^2)$, where $\sigma_{\t{nearby}} = \ell / 2$ and $\ell$ is the kernel length scale.
We find 10\% of nearby locations using the uniform exploration strategy and 90\% using the exploitation strategy.
\item Select the nearby locations which have the highest acquisition function values.
To find them, we first try $50$k nearby locations and identify the location with the highest acquisition function value.
We repeat this process $30$ times, finding and evaluating a total of $1.5$M nearby locations and obtaining a total of $30$ top nearby locations.
\item Maximise the acquisition function with gradient-based optimisation, using the top nearby locations as initialisation.
After optimisation, the best location becomes $\v{x}_\mathrm{new}$, the maximiser at which the target function $g$ will be evaluated in the next acquisition step.
Initialising at the top nearby locations, we perform $100$ steps of Adam on the sampled acquisition function with a learning rate of $0.001$ to find the maximiser.
\end{enumerate}

\begin{figure}[ht]
\centering
\includegraphics[width=6in]{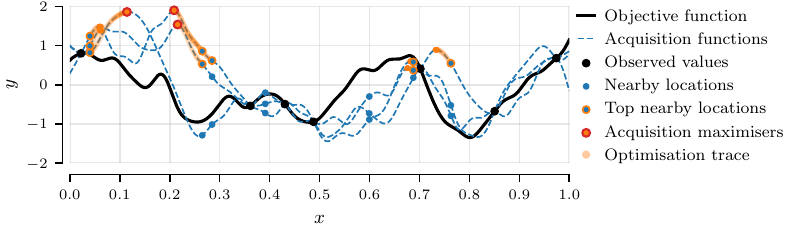}
\caption{Illustration of a single Thompson sampling acquisition step on a 1D problem.}
\label{fig:thompson_1D}
\end{figure}

In every Thompson step, we perform this process in parallel for $1000$ random acquisition functions sampled from the GP posterior, resulting in a total of $1000$ $\v{x}_\mathrm{new}$, which will be added to the training data and evaluated at the objective function.
Although we share the initial nearby locations between sampled acquisition functions, each acquisition function will, in general, produce distinct top nearby locations and maximisers.
\Cref{fig:thompson_1D} illustrates a single Thompson step on a 1D problem using $3$ acquisition functions, $7$ nearby locations and $3$ top nearby locations.

For the large-scale experiment, we set the search space dimensionality to $d = 8$, the largest considered by \citet{wilson20}, and initialise all methods with a dataset of $50$k observations sampled uniformly at random from $\c{X}$.
To prevent model misspecification as a confounding factor, we use a Matérn-$\nicefrac{3}{2}$ kernel and consider length scales of $(0.1, 0.2, 0.3, 0.4, 0.5)$ for both the target function and our models. For each length scale, we repeat the experiment with $10$ different random seeds.

In large-scale Bayesian optimisation, training and posterior function optimisation costs can become significant, and predictions may be needed on demand.
For this reason, we consider two variants of our experiment with different compute budgets.
In the small-compute setting, SGD is run for $15$k steps using a learning rate of $0.3$ for the mean and $0.0003$ for the samples, SVGP is given $m = 1024$ inducing points and $20$k steps to fit the variational parameters, and CG is run for $10$ steps. 
In the large-compute setting, all methods are run for 5 times as many optimisation steps.

Mean results and standard errors across length scales and seeds are presented in \Cref{fig:thompson}.
We plot the maximum function value achieved by each method.
In the small-compute setting, the time required for $30$ Thompson steps with SVGP and SGD are dominated by the algorithm used to maximise the models' posterior samples.
In contrast, CG takes roughly twice the time, requiring almost $3$h of wall-clock time. 
Despite this, SGD makes the largest progress per acquisition step, finding a target function value that improves upon the initial training set maximum twice as much as the other inference methods. 
VI and CG perform comparably, with the latter providing slightly more progress both per acquisition step and unit of time.
Despite their limited compute budget, all methods outperform random search by a significant margin.
In the large-compute setting, all inference methods achieve a similar maximum target value by the end of the experiment. 
CG and SGD make similar progress per acquisition step but SGD is faster per unit of time.
SVGP is slightly faster than SGD but achieves a lower maximum target value.
In summary, our results suggest that SGD can be an appealing uncertainty quantification technique for large-scale sequential decision-making.

\begin{figure}[t]
\centering
\includegraphics[width=6in]{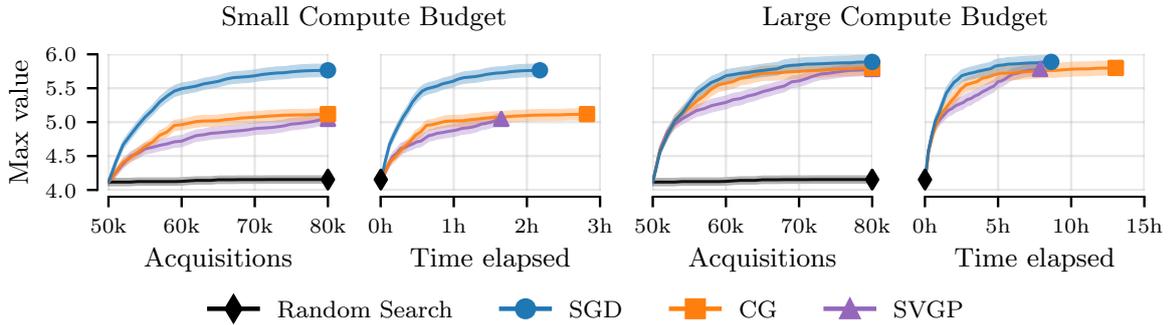}
\caption{Maximum function values (mean and standard error) obtained by Thompson sampling with our approximate inference methods as a function of acquisition steps and of compute time on an A100 GPU. The latter includes time spent drawing function samples and finding their maxima. All methods share a starting dataset of $50$k points and we take $30$ Thompson steps, acquiring $1000$ points per step.}
\label{fig:thompson}
\end{figure}

\section{Discussion}
In this chapter, we explored using stochastic gradient algorithms to approximately compute Gaussian process posterior means and function samples. 
We derived optimisation objectives with linear and sublinear cost for both and showed that SGD can produce accurate predictions, even when it does not converge to the optimum.
Additionally, we developed a spectral characterisation of the effects of non-convergence, showing that it manifests itself mainly through error in an extrapolation region located away, but not too far away, from the observations.
Furthermore, we benchmarked SGD on regression tasks of various scales, achieving state-of-the-art performance for sufficiently large or ill-conditioned settings.
On a Thompson sampling benchmark, where well-calibrated uncertainty is paramount, SGD matches the performance of more expensive baselines at a fraction of the computational cost.
In the next chapter, we will build upon and improve our SGD approach using insights from the literature on optimisation and kernel methods.

%!TEX root = ../thesis.tex
%*******************************************************************************
%****************************** Fourth Chapter *********************************
%*******************************************************************************
\chapter[Stochastic Dual Descent for Gaussian Processes]{Stochastic Dual Descent\\for Gaussian Processes}
\label{chap:sdd}

% **************************** Define Graphics Path **************************
\ifpdf
    \graphicspath{{Chapter4/Figs/Raster/}{Chapter4/Figs/PDF/}{Chapter4/Figs/}}
\else
    \graphicspath{{Chapter4/Figs/Vector/}{Chapter4/Figs/}}
\fi

The previous chapter introduced stochastic gradient descent to draw samples from a Gaussian process posterior.
In this chapter, we build upon this algorithm to develop stochastic \emph{dual} descent for Gaussian processes, a substantially improved version of the former.
To this end, we propose a dual optimisation objective, which facilitates larger step sizes and much faster convergence due to more favourable curvature properties.
Additionally, we discuss different stochastic gradient estimators, focusing on the important difference between multiplicative noise versus additive noise.
Furthermore, we analyse and compare momentum acceleration and iterative averaging techniques, leading to recommendations which improve the convergence properties of the proposed algorithm.
Empirically, we surpass our own state-of-the-art performance from the previous chapter, and further demonstrate that our improved algorithm matches the performance of state-of-the-art graph neural networks on a molecular binding affinity prediction task.

This chapter includes content which is adapted from the following publication:
\begin{itemize}
    \item J. A. Lin, S. Padhy, J. Antorán, A. Tripp, A. Terenin, C. Szepesvári, J. M. Hernández-Lobato, and D. Janz. Stochastic Gradient Descent for Gaussian Processes Done Right. In \emph{International Conference on Learning Representations}, 2024.
\end{itemize}
My contributions to this project consist of proposing the idea, initiating the collaboration, developing major parts of the software implementation, conducting most of the experiments, and contributing to writing the manuscript.

\section{Introduction}
Gaussian process regression is a probabilistic framework for learning unknown functions.
It is the gold standard modelling choice in Bayesian optimisation, where uncertainty-aware decision-making is required to gather data efficiently in a sequential way.
The main limitation of Gaussian process models is that their fitting requires solving a large quadratic optimisation problem, minimising the kernel ridge regression objective, which, using direct methods, has a cost cubic in the number of observations.

Standard approaches to reducing the cost of fitting Gaussian process models either apply approximations to reduce the complexity of the quadratic optimisation problem, such as the Nyström and related variational approximations \citep{williams2000using,titsias09,hensman13}, or use carefully engineered iterative solvers, such as preconditioned conjugate gradients \citep{gardner18,WangPGT2019exactgp}, or a combination of both \citep{rudi2017falkon,lin2023sampling}.
An alternative approach, which we focus on in this chapter, is the use of stochastic gradient descent (SGD) to minimise the kernel ridge regression objective.

Existing work, including the previous chapter, has previously investigated the use of SGD in the context of Gaussian processes and related kernel methods \citep{lin2023sampling,dai2014doubly,kivinen2004online}.
In particular, in the previous chapter, we pointed out that SGD may be competitive with conjugate gradients (CG) in terms of both the mean and uncertainty predictions that it produces when the compute budget is limited.
In this chapter, we go a step further, and demonstrate that, when done right, SGD can broadly outperform CG.

To this end, we propose a simple SGD-based algorithm, called \emph{stochastic dual descent} (SDD).
Our algorithm is an adaptation of ideas from the stochastic dual coordinate ascent (SDCA) algorithm of \citet{shalev2013stochastic} to the large-scale deep-learning-type gradient descent setting, combined with insights on stochastic approximation from \citet{dieuleveut2017harder} and \citet{varre2021last}.
We demonstrate the following evidence in favour of SDD:
\begin{itemize}
    \item On standard UCI regression benchmarks with up to 2 million observations, stochastic dual descent either matches or improves upon the performance of conjugate gradients.
    \item On the large-scale Bayesian optimisation task considered in the previous chapter \citep{lin2023sampling}, stochastic dual descent outperforms stochastic gradient descent and other baselines, both against the number of iterations and against wall-clock time.
    \item On a molecular binding affinity prediction task, the performance of Gaussian process regression via stochastic dual descent matches state-of-the-art graph neural networks.
\end{itemize}

\section{Stochastic Dual Descent for Regression and Sampling}
We show our proposed algorithm, \emph{stochastic dual descent}, in \Cref{alg:compute-mean}.
It can be used for regression, where the goal is to produce a good approximation to the posterior mean function $\mu_{(\.) \given \v{y}}$, and sampling, that is, generating a sample from the Gaussian process posterior.
The algorithm takes as input a kernel matrix $\m{K}_{\m{X}\m{X}}$, whose entries can be computed individually using the kernel function $k$ and $\v{x}_1, \dots, \v{x}_n$ as needed, a likelihood variance $\sigma^2 > 0$, and a target vector $\v{b} \in \R^n$.
It produces a vector of coefficients $\widebar{\v{\v{\alpha}}}_{t_\mathrm{max}} \in \R^n$, 
which approximates
\begin{equation}
\label{eq:alpha-star}
\v{\v{\alpha}}^* (\v{b}) = (\m{K}_{\m{X}\m{X}} + \sigma^2 \m{I})\inv \v{b}.
\end{equation}
Now, given an $\v{\v{\alpha}} \in \R^n$, let
\begin{equation}
    h_{\v{\alpha}} (\.) = \sum_{i=1}^n k(\., \v{x}_i) \alpha_{i},
\end{equation}
such that $h_{\v{\alpha}^*(\v{y})}$ gives the mean function $\mu_{(\.) \given \v{y}}$.
Furthermore, given a sample $f \~[GP](0, k)$ from the Gaussian process prior and a noise vector sample $\v\eps \~[N](\v{0}, \sigma^2 \m{I})$, we have that
\begin{equation}
    f + h_{\v{\alpha}^*(\v{y} - (\v{f}_\m{X} + \v\eps))}
\end{equation}  
is a sample from the Gaussian process posterior via pathwise conditioning \citep{wilson20,wilson21}. 
In practice, one might approximate $f$ using random features, as done by \citet{wilson20,wilson21,lin2023sampling} (see \Cref{sec:random_features} for details).

\begin{algorithm}[t]
\caption{Stochastic dual descent for approximating $\v{\alpha}^*(\v{b}) = (\m{K}_\m{XX} + \sigma^2 \m{I})\inv \v{b}$}
\label{alg:compute-mean}
\begin{algorithmic}[1]
    \Require Kernel matrix $\m{K}_\m{XX}$ with rows $\v{k}_1, \dots, \v{k}_n \in \R^n$, targets $\v{b} \in \R^n$,\\
    likelihood variance $\sigma^2 > 0$,
    number of steps $t_\mathrm{max} \in \N^+$, 
    batch size $b \in \{1, \dots, n\}$, 
    step size $\beta > 0$,
    momentum parameter $\rho \in [0,1)$, 
    averaging parameter $r \in (0,1]$
    \State $\v{v}_0 = \v{0}$; $\v{\alpha}_0 = \v{0}$; $\widebar{\v{\alpha}}_0 = \v{0}$
    \For{$t \in \{1, \dots, t_\mathrm{max}\}$}
        \State Sample $\c{I}_t = \{i_1, \dots, i_b\} \~[U](\{1,\dotsc, n\})$ independently \Comment{random coordinates}
        \State $\v{g}_t = \frac{n}{b} \sum_{i \in \c{I}_t} (( \v{k}_i + \sigma^2 \v{e}_i)\T (\v{\alpha}_{t-1} + \rho \v{v}_{t-1}) - b_i)\v{e}_i$ \Comment{gradient estimate}
        \State $\v{v}_t = \rho \v{v}_{t-1} - \beta \v{g}_t$ \Comment{velocity update}
        \State $\v{\alpha}_t = \v{\alpha}_{t-1} + \v{v}_t$ \Comment{parameter update}
        \State $\widebar{\v{\alpha}}_t = r\v{\alpha}_t + (1-r)\widebar{\v{\alpha}}_{t-1}$ \Comment{iterate averaging}
    \EndFor
    \State \Return $\widebar{\v{\alpha}}_{t_\mathrm{max}}$
\end{algorithmic}
\end{algorithm}
The SDD algorithm, in contrast with previous SGD implementations, uses (i) a dual objective instead of the (primal) kernel ridge regression objective, (ii) stochastic approximation entirely via random subsampling of the data instead of random features, (iii) Nesterov's momentum, and (iv) geometric, rather than arithmetic, iterate averaging.
In the following subsections, we examine and justify each of the choices behind the algorithm design, and illustrate these on the UCI dataset \textsc{pol} \citep{Dua2019UCI}, chosen for its small size, which helps us to compare against less effective alternatives.

\subsection{Primal and Dual Objectives}
\label{subsec:primal_vs_dual}
The vector $\v{\alpha}^*(\v{b})$ of \Cref{eq:alpha-star} is the minimiser of the kernel ridge regression objective,
\begin{equation}
\label{eq:primal_objective}
L(\v{\alpha}) = \frac{1}{2}\norm{\v{b} - \m{K}_\m{XX} \v{\alpha}}_2^2 + \frac{\sigma^2}{2} \norm{\v{\alpha}}^2_{\m{K}_\m{XX}},
\end{equation}
over $\v{\alpha} \in \R^n$ \citep{scholkopf2002learning}.
We will refer to $L$ as the \emph{primal} objective, to contrast with our proposed dual objective, which we shall introduce shortly.

Consider using gradient descent to minimise $L(\v{\alpha})$. 
This entails constructing a sequence $(\v{\alpha}_t)_{t=1}^{t_{\mathrm{max}}}$ of elements in $\R^n$, which we initialise at the standard but otherwise arbitrary choice of $\v{\alpha}_0 = \v{0}$, and setting
\begin{equation}
    \v{\alpha}_{t+1} = \v{\alpha}_t - \beta \nabla L(\v{\alpha}_t),
\end{equation}
where $\beta > 0$ is a step size and $\nabla L$ is the gradient of $L$.
The speed at which $\v{\alpha}_t$ approaches $\v{\alpha}^*(\v{b})$ is determined by the condition number of the Hessian: the larger the condition number, the slower the convergence speed.
The intuitive reason for this correspondence is that, to guarantee convergence, the step size needs to scale inversely with the largest eigenvalue of the Hessian, while progress in the direction of an eigenvector underlying an eigenvalue is governed by the step size multiplied with the corresponding eigenvalue.
The \emph{primal} gradient and Hessian are
\begin{equation}
\label{eq:primal_gradient}
    \nabla L(\v{\alpha}) = \m{K}_\m{XX} (\m{K}_{\m{XX}}\v{\alpha} + \sigma^2 \v{\alpha} - \v{b})
    \quad \text{and} \quad
    \nabla^2 L(\v{\alpha}) = \m{K}_\m{XX}(\m{K}_\m{XX} + \sigma^2 \m{I})
\end{equation}
respectively, and therefore we have the following (tight) bounds on the relevant eigenvalues,
\begin{equation}
    0 \leq \lambda_n(\m{K}_\m{XX}(\m{K}_\m{XX} + \sigma^2 \m{I})) \leq \lambda_1(\m{K}_\m{XX} (\m{K}_\m{XX} + \sigma^2 \m{I})) \leq \kappa n(\kappa n + \sigma^2),
\end{equation}
where $\lambda_i(\.)$ returns the $i$th largest eigenvalue, and $\kappa = \sup_{\v{x}\in \c{X}} k(\v{x}, \v{x})$ is finite by assumption.
These bounds only allow for a step size $\beta$ on the order of $(\kappa n(\kappa n + \sigma^2))^{-1}$. Moreover, since the minimum eigenvalue is not bounded away from zero, we do not have a priori guarantees for the performance of gradient descent using $L$. 

\begin{figure}[t]
    \centering
    \includegraphics[width=6in]{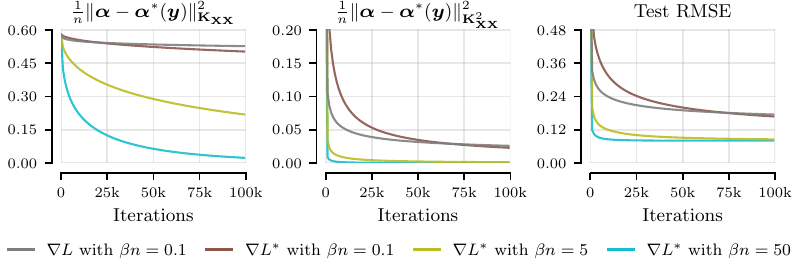}
    \caption{Comparison of full-batch primal and dual gradient descent on \textsc{pol} with varying step sizes. Primal gradient descent becomes unstable and diverges for (normalised) step sizes of $\beta n > 0.1$. Dual gradient descent is stable with larger step sizes, allowing for markedly faster convergence than the primal.
    For $\beta n = 0.1$, the dual gradient makes more progress in the $\m{K}_\m{XX}$-norm, whereas the primal gradient makes more progress in the $\m{K}_\m{XX}^2$-norm.}
    \label{fig:toy-gradient-primal-vs-dual}
\end{figure}

Instead, we consider minimisation of the \emph{dual objective}
\begin{equation}
    L^*(\v{\alpha}) = \frac12 \norm{\v{\alpha}}_{\m{K}_\m{XX} + \sigma^2 \m{I}}^2 - \v{\alpha}\T \v{b}.
\end{equation}
The dual objective $L^*$ has the same unique minimiser as $L$, namely $\v{\alpha}^*(\v{b})$, and the two are, up to rescaling, a strong dual pair.

\begin{proposition}
Let $L$ be the primal objective and $L^*$ be the dual objective.
We have that
\begin{equation}
    \min_{\v{\alpha} \in \R^n} L(\v{\alpha}) = -\sigma^2 \min_{\v{\alpha} \in \R^n} L^*(\v{\alpha}),
\end{equation}
with $\v{\alpha}^\star(\v{b})$ minimising both $L$ and $L^*$.
\end{proposition}

\begin{proofbox}
\begin{proof}
We can write $\min_{\v{\alpha} \in \R^n} L(\v{\alpha})$ as the constrained optimisation problem
\begin{equation}
    \min_{\v{u} \in \R^n} \min_{\v{\alpha} \in \R^n} \frac{1}{2} \norm{\v{u}}_2^2 + \frac{\sigma^2}{2}\norm{\v{\alpha}}_{\m{K}_\m{XX}}^2 \quad \text{subject to} \quad \v{u} = \v{b} - \m{K}_\m{XX} \v{\alpha},
\end{equation}
which is quadratic in both $\v{u}$ and $\v{\alpha}$.
Introducing Lagrange multipliers $\v{\lambda} \in \R^n$, in the form $\sigma^2 \v{\lambda}$, where we recall that $\sigma^2 > 0$, the solution of the above is equal to that of
\begin{equation}
     \min_{\v{u} \in \R^n} \min_{\v{\alpha} \in \R^n} \sup_{\v{\lambda} \in \R^n} \frac{1}{2} \norm{\v{u}}_2^2 + \frac{\sigma^2}{2} \norm{\v{\alpha}}_{\m{K}_\m{XX}}^2 + \sigma^2 \v{\lambda}\T (\v{b} - \m{K}_\m{XX} \v{\alpha} - \v{u}),
\end{equation}
which is a finite-dimensional quadratic problem.
Therefore, we have strong duality \citep{boyd2004convex} and can thus exchange the order of the minimum operators and the supremum, yielding the equivalent problem
\begin{equation}
    \sup_{\v{\lambda} \in \R^n} \del{\min_{\v{u} \in \R^n} \frac{1}{2} \norm{\v{u}}_2^2 - \sigma^2 \v{\lambda}\T \v{u}} + \del{\min_{\v{\alpha} \in \R^n} \frac{\sigma^2}{2}\norm{\v{\alpha}}_{\m{K}_\m{XX}}^2 - \sigma^2 \v{\lambda}\T \m{K}_\m{XX} \v{\alpha}} + \sigma^2 \v{\lambda}\T \v{b}.
\end{equation}
Noting that the two inner minimisation problems are quadratic, we solve these analytically using the first-order optimality conditions, that is $\v{u} = \sigma^2 \v{\lambda}$ and $\v{\alpha} = \v{\lambda}$, to obtain that the above is equivalent to
\begin{equation}
    \sup_{\v{\lambda} \in \R^n} - \sigma^2 \del{\frac{1}{2} \norm{\v{\lambda}}_{\m{K}_\m{XX} + \sigma^2 \m{I}}^2 - \v{\lambda}\T \v{b}} = -\sigma^2 \min_{\v{\lambda} \in \R^n} L^*(\v{\lambda}).
\end{equation}
The claim follows by chaining the above equalities.
The fact that $\v{\alpha}^*(\v{b})$ minimises both $L$ and $L^*$ can also be established from the first-order optimality conditions. 
\end{proof}
\end{proofbox}

The gradient and Hessian of the dual objective are given by
\begin{equation}
    \label{eq:dual_gradient}
    \nabla L^*(\v{\alpha}) = \m{K}_{\m{XX}}\v{\alpha}  + \sigma^2\v{\alpha} - \v{b}
    \quad \text{and} \quad
    \nabla^2L^*(\v{\alpha}) = \m{K}_\m{XX} + \sigma^2 \m{I}.
\end{equation}
Examining the eigenvalues of the above hessian, we see that gradient descent on $L^*$ may use a step size of up to an order of $\kappa n$ higher than on $L$.
Furthermore, since the condition number of the dual Hessian satisfies $\mathrm{cond}(\m{K}_\m{XX} +\sigma^2 \m{I}) \leq 1+ \kappa n/ \sigma^2 $ and $\sigma^2$ is positive, we have faster convergence, and can provide an a priori bound on the number of iterations required for any fixed error level for any length $n$ sequence of observations.

\Cref{fig:toy-gradient-primal-vs-dual} contains an experimental comparison of gradient descent with primal and dual objectives on the UCI \textsc{pol} regression task. 
Gradient descent with the primal objective is only stable up to $\beta n = 0.1$ and diverges for larger step sizes.
In contrast, gradient descent with the dual objective is stable with a up to $500\times$ larger step size, and converges faster and to a better solution.
We show this on three evaluation metrics: distance to $\v{\alpha}^*(\v{y})$ measured in $\norm{\.}_{\m{K}_\m{XX}}^2$, the $\m{K}_\m{XX}$-norm (squared) and in $\norm{.}_{\m{K}_{\m{XX}}^2}^2$, the $\m{K}_\m{XX}^2$-norm (squared), and test set root-mean-square error (RMSE).
To understand the difference between the two norms, note that the $\m{K}_\m{XX}$-norm error bounds the error of approximating $h_{\v{\alpha}^*(\m{b})}$ uniformly.
\begin{proposition}
Let $\kappa = \sup_{\v{x} \in \c{X}} k(\v{x}, \v{x})$.
For any $\v{\alpha} \in \R^n$, we have 
\begin{equation}
    \label{eq:pointwise-error}
    \norm{h_{\v{\alpha}} - h_{\v{\alpha}^*(\v{b})}}_\infty \leq \sqrt{\kappa} \norm{\v{\alpha} - \v{\alpha}^*(\v{b})}_{\m{K}_\m{XX}}.
\end{equation}
\end{proposition}
\begin{proofbox}
\begin{proof}
Let $\c{H}$ be a Hilbert space of functions $\c{X} \to \R$ with inner product $\innerprod{\.}{\.}_\c{H}$ and corresponding norm $\norm{\.}_\c{H}$.
Let $k(\v{x}, \.)$ be the evaluation functional on $\c{H}$ at $\v{x} \in \c{X}$, such that
\begin{equation}
    \forall \v{x} \in \c{X}, \, \forall h \in \c{H}: \innerprod{k(\v{x}, \.)}{h}_\c{H} = h(\v{x}).
\end{equation}
To prove the claim, we first observe that, for any $\v{\alpha}, \v{\alpha}' \in \R^n$,
\begin{align}
    \norm{h_{\v{\alpha}} - h_{\v{\alpha}'}}_\infty 
    &= \sup_{\v{x} \in \c{X}} \; \abs{h_{\v{\alpha}}(\v{x}) - h_{\v{\alpha}'}(\v{x})}, \\
    &= \sup_{\v{x} \in \c{X}} \; \abs{\innerprod{k(\v{x}, \.)}{h_{\v{\alpha}} - h_{\v{\alpha}'}}_\c{H}}, \\
    &\leq \sup_{\v{x} \in \c{X}} \; \norm{k(\v{x}, \.)}_\c{H} \norm{h_{\v{\alpha}} - h_{\v{\alpha}'}}_\c{H}, \\
    &\leq \sqrt{\kappa} \norm{h_{\v{\alpha}} - h_{\v{\alpha}'}}_\c{H}.
\end{align}
For observations $\m{X}$, let $\m{\Phi}: \c{H} \to \R^n$ be the linear operator mapping $h \mapsto h(\m{X})$, where $h(\m{X}) = [h(\v{x}_1), \dots, h(\v{x}_n)]$.
We write $\m{\Phi^*}$ for the adjoint of $\m{\Phi}$.
Now, observe that $h_{\v{\alpha}} = \m{\Phi}^*\v{\alpha}$ and $h_{\v{\alpha}'} = \m{\Phi}^*\v{\alpha}'$, and thus we have the equalities
\begin{align}
    \norm{h_{\v{\alpha}} - h_{\v{\alpha}'}}_\c{H}^2
    &= \innerprod{\m{\Phi}^*(\v{\alpha} - \v{\alpha}')}{\m{\Phi}^* (\v{\alpha} - \v{\alpha}')}_\c{H}, \\
    &= \innerprod{\v{\alpha} - \v{\alpha}'}{\m{\Phi}\m{\Phi}^* (\v{\alpha} - \v{\alpha}')}_\c{H}, \\
    &= \norm{\v{\alpha} - \v{\alpha}'}_{\m{K}_\m{XX}}^2,
\end{align} 
where for the final equality, we note that $\m{K}_\m{XX}$ is the matrix of the operator $\m{\Phi}\m{\Phi}^*$ with respect to the standard basis.
Combining the above two displays yields the claim.
\end{proof}
\end{proofbox}
Uniform norm guarantees of this type are crucial for sequential decision-making tasks, such as Bayesian optimisation, where test input locations may be arbitrary.
The $\m{K}_\m{XX}^2$-norm metric, on the other hand, reflects the training error. 

Examining the two gradients shows the primal gradient optimises for the $\m{K}_\m{XX}^2$-norm, while the dual optimises for the $\m{K}_\m{XX}$-norm.
Indeed, we see in \Cref{fig:toy-gradient-primal-vs-dual} that, up to $70$k iterations, the dual method is better on the $\m{K}_\m{XX}$-norm metric while the primal method is better on the $\m{K}_\m{XX}^2$-norm metric, when both methods use $\beta n = 0.1$.
Later, the dual gradient performs better on all metrics.
This is also expected because the minimum eigenvalue of the Hessian of the dual objective is greater than the minimum eigenvalue of the primal objective.

\subsection{Random Features versus Random Coordinates}
\label{subsec:stochastic_approximation}
\begin{figure}[t]
    \centering
    \includegraphics[width=6in]{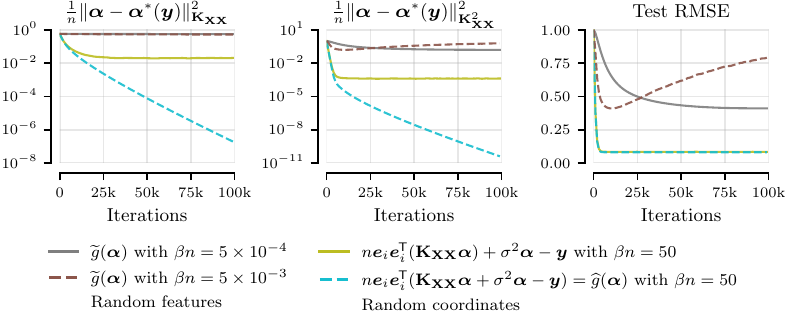}
    \caption{Comparison of stochastic dual descent on the \textsc{pol} dataset with either $\widetilde g$ (random Fourier features) or $\widehat g$ (random coordinates), using batch size $b = 512$, momentum $\rho = 0.9$ and averaging parameter $r = 0.001$ (see \Cref{subsec:acceleration_and_averaging} for an explanation of the latter two). Random features converge with $\beta n=5\times 10^{-4}$ but perform poorly, and diverge with a higher step size. Random coordinates are stable with $\beta n = 50$ and show a much stronger performance on all metrics. We include a version of the random coordinate estimator where only the $\m{K}_\m{XX} \v{\alpha}$ term is subsampled: this breaks the multiplicative noise property, and results in an estimate which is worse on both the $\m{K}_\m{XX}$-norm and the $\m{K}_\m{XX}^2$-norm metric.}
    \label{fig:toy-features-vs-coordinates}
\end{figure}

To compute either the primal or the dual gradients, presented in \Cref{eq:primal_gradient,eq:dual_gradient}, we need to compute matrix-vector products of the form $\m{K}_\m{XX} \v{\alpha}$, which require order $n^2$ computations.
We now introduce and discuss two types of stochastic approximations for our dual gradient $g(\v{\alpha}) = \nabla L^*(\v{\alpha})$, which reduce the cost to order $n$ per step, and carefully examine the noise they introduce into the gradient.

For the sake of exposition (and exposition only, this is not an assumption of our algorithm), assume that we are in an $m$-dimensional (finite) linear model setting, such that $\m{K}_\m{XX}$ is of the form $\sum_{j=1}^m \v{z}_j \v{z}_j\T$, where $\v{z}_j\in \R^n$ collects the values of the $j$th feature of the observations $\v{x}_1, \dots, \v{x}_n$.
Then, for $j \~[U]({\{1, \dots, m\}})$, we have that $\E[m \v{z}_j \v{z}_j\T] = \m{K}_\m{XX}$, and therefore 
\begin{equation}
    \widetilde g(\v{\alpha}) = m\v{z}_j \v{z}_j\T \v{\alpha} + \sigma^2\v{\alpha} - \v{b}
\end{equation}
is an unbiased estimate of $g(\v{\alpha})$, and we call $\widetilde g$ the \emph{random feature} estimator.
The alternative we consider are \emph{random coordinates}, where we sample $i \~[U](\{1, \dotsc, n\})$ and calculate   
\begin{equation}
    \widehat g(\v{\alpha}) =  n \v{e}_i \v{e}_i\T g(\v{\alpha})
    = n \v{e}_i (\v{k}_i\T \v{\alpha} + \sigma^2 \alpha_i - b_i),
\end{equation}
with $\v{k}_i$ being the $i$th row of $\m{K}_\m{XX}$ as a column vector.
Observe that $\widehat g(\v{\alpha})$ zeros all but the $i$th coordinate of $g(\v{\alpha})$, and then scales the result by $n$.
Since $\E[n \v{e}_i \v{e}_i\T] = \m{I}$, $\widehat g(\v{\alpha})$ is also unbiased. Note that the cost of calculating either $\widetilde g(\v{\alpha})$ or $\widehat g(\v{\alpha})$ is linear in $n$, and therefore both achieve our goal of reduced computation time. 

However, while $\widetilde g$ and $\widehat g$ may appear similarly appealing, the nature of the noise introduced by these, and thus their qualities, are quite different.
In particular, one can show that
\begin{align}
\norm{ \widehat g(\v{\alpha}) - g(\v{\alpha}) } \leq \norm{ (n \v{e}_i \v{e}_i\T - \m{I}) (\m{K}_\m{XX} + \sigma^2 \m{I})} \norm{\v{\alpha} - \v{\alpha}^*(\v{b})},
\end{align}
where $n \v{e}_i \v{e}_i\T$ is the random coordinate approximation to the identity matrix.
As such, the noise introduced by $\widehat g(\v{\alpha})$ is proportional to the distance between the current iterate $\v{\alpha}$ and the solution $\v{\alpha}^*$, making it a \emph{multiplicative noise} gradient oracle \citep{dieuleveut2017harder}.
Intuitively, multiplicative noise oracles automatically reduce the amount of noise injected as the iterates get closer to their target.
On the other hand, we have that
\begin{align}
\norm{ \widetilde g(\v{\alpha}) - g(\v{\alpha}) } = \norm{(m \v{z}_j \v{z}_j\T - \m{K}_\m{XX}) \v{\alpha}}\,,
\end{align}
where the error in $\widetilde g(\v{\alpha})$ is \emph{not} reduced as $\v{\alpha}$ approaches $\v{\alpha}^*(\v{b})$.
The behaviour of $\widetilde g$ is that of an \emph{additive noise} gradient oracle. Algorithms using multiplicative noise oracles, when convergent, often yield better performance \citep{varre2021last}.

Another consideration is that when the number of features $m$ is larger than $n$, individual features $\v{z}_j$ might also be less informative than individual gradient coordinates $\v{e}_i \v{e}_i\T g(\v{\alpha})$.
Thus, selecting features uniformly at random, as in $\widetilde g$, may yield a poorer estimate.
While this can be addressed by introducing an importance sampling distribution for the features $\v{z}_j$ \citep{li2019towards}, this adds implementation complexity, and could likewise be used to improve the random coordinate estimate.

In \Cref{fig:toy-features-vs-coordinates}, we compare variants of stochastic dual descent with either random Fourier features or random coordinates.
We see that random features, which produce additive noise, can only be used with very small step sizes and have poor asymptotic performance. 
Further, we test two versions of random coordinates: $\widehat g(\v{\alpha})$, where, as presented before, we subsample the whole gradient, and an alternative, $n \v{e}_i \v{e}_i\T (\m{K}_\m{XX} \v{\alpha}) + \sigma^2 \v{\alpha} - \v{b}$, where only the $\m{K}_\m{XX} \v{\alpha}$ term is subsampled.
While both are stable with much higher step sizes than random features, the latter has worse asymptotic performance.
This is a kind of \emph{Rao-Blackwellisation trap}: introducing the known value of $\sigma^2 \v{\alpha} - \v{b}$ in place of its estimate $n \v{e}_i \v{e}_i\T (\sigma^2 \v{\alpha} - \v{b})$ destroys the multiplicative property of the noise, leading to worse rather than better performance.

There are many other possibilities for constructing stochastic gradient estimates. 
For example, \citet{dai2014doubly} applied the random coordinate method on top of the random feature method, in an effort to further reduce the order $n$ cost per iteration of gradient computation. However, based on the above discussion, we generally recommend estimates that produce multiplicative noise, and our randomised coordinate gradient estimates as the simplest of such estimates. 

\subsection{Momentum and Iterate Averaging}
\label{subsec:acceleration_and_averaging}

\begin{figure}[t]
    \centering
    \includegraphics[width=6in]{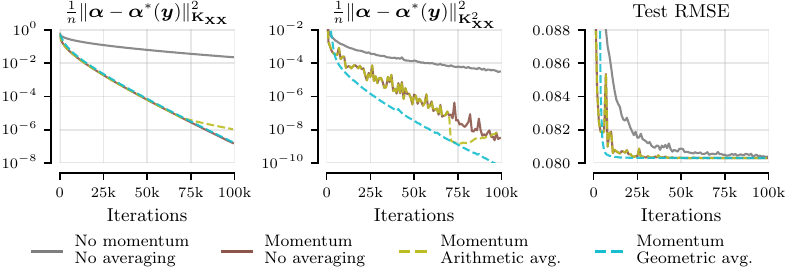}
    \caption{Comparison of optimisation strategies for random coordinate estimator of the dual objective on the \textsc{pol} dataset, using momentum $\rho = 0.9$, averaging parameter $r = 0.001$, batch size $b = 128$, and step size $\beta n = 50$. Nesterov's momentum significantly improves convergence speed across all metrics. The dashed olive line, marked \emph{arithmetic averaging}, shows the regular iterate up until $70$k steps, at which point averaging commences and the averaged iterate is shown. Once enabled, arithmetic iterate averaging slows down convergence in the $\m{K}_\m{XX}$-norm metric. Geometric iterate averaging outperforms arithmetic averaging and unaveraged iterates throughout optimisation.}
    \label{fig:toy-iterate-averaging}
\end{figure}

Momentum, or acceleration, is a range of modifications to the usual gradient descent updates that aim to improve the rate of convergence, in particular, with respect to its dependence on the curvature of the optimisation problem \citep{polyak1964some}.
Many schemes for momentum exist, with RMSProp, AdaGrad and Adam being popular in the deep learning literature.
However, as our objective has a constant curvature, we recommend the use of the simpler Nesterov's momentum \citep{nesterov1983method}, as adapted by \citet{sutskever2013importance} (see \Cref{alg:compute-mean} for the form of these updates).
Examining \Cref{fig:toy-iterate-averaging}, we see that momentum is vital on this problem, independent of iterate averaging.
We set the momentum parameter to $\rho = 0.9$.

Polyak-Ruppert averaging returns an average over (the tail of) stochastic gradient iterates, rather than the final iterate, to reduce noise in the estimate \citep{polyak1990new,ruppert1988efficient,polyak1992acceleration}.
While Polyak-Ruppert averaging is necessary with a constant step size and additive noise, it is not necessary under multiplicative noise \citep{varre2021last}, and can potentially slow down convergence.
For our problem, we recommend using \emph{geometric averaging} instead, where we let $\overline{\v{\alpha}}_0 = \v{\alpha}_0$ and, at each step, compute
\begin{equation}
    \overline{\v{\alpha}}_t = r\v{\alpha}_t + (1-r) \overline{\v{\alpha}}_{t-1} \quad \text{using an averaging parameter} \quad r \in (0, 1],
\end{equation}
and return $\overline{\v{\alpha}}_{t_\mathrm{max}}$.
Geometric averaging is an anytime approach, because it does not rely on a fixed tail window size.
Therefore, it can be used in combination with early stopping, and the value of $r$ can be tuned adaptively.
\Cref{fig:toy-iterate-averaging} shows that geometric averaging outperforms both arithmetic averaging and the final iterate $\v{\alpha}_{t_\mathrm{max}}$.
We set $r = 100/t_\mathrm{max}$.

\subsection{Connections to the Literature}
\label{sec:related-work}
The dual formulation for the kernel ridge regression objective was first pointed out in the kernel literature by \citet{saunders1998ridge}.
It features in textbooks on the topic \citep{scholkopf2002learning}, albeit with no mention of its better conditioning. 
Gradient descent on the dual objective is equivalent to applying the \emph{kernel adatron} update rule of \citet{frie1998kernel}, stated there for the hinge loss, and to the \emph{kernel least-mean-square} algorithm of \citet{liu2008kernel}, analysed theoretically in \citet{dieuleveut2017harder}.
\citet{shalev2013stochastic} introduce \emph{stochastic dual coordinate ascent} (SDCA), which uses the dual objective with random coordinates and analytic line search.
\citet{tu2016large} implement a block version of SDCA.
\citet{bo2008greedy} propose a method similar to SDCA, but with coordinates chosen by active learning, later reimplemented by \citet{wu2024large}.

The algorithm from the previous chapter, based on \citet{lin2023sampling}, which we will refer to as SGD in our upcoming experiments, uses the gradient estimate 
\begin{equation}
\label{eq:sgd-loss}
    \nabla L(\v{\alpha}) \approx
        n \v{k}_i( \v{k}_i\T \v{\alpha} - b_i) + \sigma^2 \smash{\sum_{j=1}^m} \v{z}_j \v{z}_j\T \v{\alpha},
\end{equation}
where $i \~[U](\{1, \dotsc, n\})$ is a uniformly random index and $\sum_{j=1}^m \v{z}_j \v{z}_j\T$ is a random Fourier feature approximation of $\m{K}_\m{XX}$.
This can be understood as targeting the gradient of the primal objective with a mixed multiplicative-additive estimator. Stochastic gradient descent was also used, amongst others, by \citet{dai2014doubly} for kernel regression with the primal objective and random Fourier features, by \citet{kivinen2004online} for online learning with the primal objective and random sampling of training data, and by \citet{antoran2023sampling} for large-scale Bayesian linear models with the dual objective. 

\section{Experiments}
We present three experiments which confirm the strength of our SDD algorithm on standard benchmarks.
The first two experiments, on UCI regression and large-scale Bayesian optimisation, replicate those from the previous chapter based on \citet{lin2023sampling}, and compare against SGD \citep{lin2023sampling}, CG \citep{WangPGT2019exactgp}, and SVGP \citep{hensman13}.
Unless indicated otherwise, we use the code, setup and hyperparameters from the previous chapter based on \citet{lin2023sampling}.
Our third experiment tests stochastic dual descent on five molecule-protein binding affinity prediction benchmarks of \citet{Ortegon2022dockstring}.

\subsection{Regression Baselines}
We benchmark on UCI regression datasets \citep{Dua2019UCI} from the previous chapter based on \citet{lin2023sampling}, which are also considered by \citet{WangPGT2019exactgp}.
We run SDD for $100$k iterations, the same number used by \citet{lin2023sampling} for SGD, but with step sizes $100\times$ larger than \citet{lin2023sampling}, except for \textsc{elevators}, \textsc{keggdirected}, and \textsc{buzz}, where this causes divergence, and we use $10\times$ larger step sizes instead.
We run CG to a tolerance of $0.01$, except for the four largest data sets, where we stop CG after $100$ iterations.
This still provides CG with a larger compute budget than SGD and SDD. 
CG uses a pivoted Cholesky preconditioner of rank $100$.
For SVGP, we use $3000$ inducing points for the smaller five datasets and $9000$ for the larger four to match the runtime of the other methods. 

\begin{table}[t]
\caption{Root-mean-square-error (RMSE), compute time (on an A100 GPU), and negative log-likelihood (NLL), for UCI regression tasks for all methods considered. We report mean values and standard error across five $90\%$-train $10\%$-test splits for all data sets, except the largest, where three splits are used. Targets are normalised to zero mean and unit variance.}
\label{tab:UCI_regression}
\centering
\scriptsize
\setlength{\tabcolsep}{2.5pt}
\renewcommand{\arraystretch}{1.1}
\begin{tabular}{l l c c c c c c c c c}
\toprule
\multicolumn{2}{c}{Dataset} & \textsc{pol} & \textsc{elevators} & \textsc{bike} & \textsc{protein} & \textsc{keggdir} & \textsc{3droad} & \textsc{song} & \textsc{buzz} & \textsc{houseelec} \\
\multicolumn{2}{c}{Size} & 15000 & 16599 & 17379 & 45730 & 48827 & 434874 & 515345 & 583250 & 2049280 \\
\midrule
\multirow{4}{*}{\rotatebox[origin=c]{90}{RMSE}}
 & SDD
 & \textbf{0.08\,$\pm$\,0.00} & \textbf{0.35\,$\pm$\,0.00} & \textbf{0.04\,$\pm$\,0.00} & \textbf{0.50\,$\pm$\,0.01} & \textbf{0.08\,$\pm$\,0.00} & \textbf{0.04\,$\pm$\,0.00} & \textbf{0.75\,$\pm$\,0.00} & \textbf{0.28\,$\pm$\,0.00} & \textbf{0.04\,$\pm$\,0.00} \\
 & SGD
 & 0.13\,$\pm$\,0.00 & 0.38\,$\pm$\,0.00 & 0.11\,$\pm$\,0.00 & 0.51\,$\pm$\,0.00 & 0.12\,$\pm$\,0.00 & 0.11\,$\pm$\,0.00 & 0.80\,$\pm$\,0.00 & 0.42\,$\pm$\,0.01 & 0.09\,$\pm$\,0.00 \\
 & CG
 & \textbf{0.08\,$\pm$\,0.00} & \textbf{0.35\,$\pm$\,0.00} & \textbf{0.04\,$\pm$\,0.00} & \textbf{0.50\,$\pm$\,0.00} & \textbf{0.08\,$\pm$\,0.00} & 0.18\,$\pm$\,0.02 & 0.87\,$\pm$\,0.05 & 1.88\,$\pm$\,0.19 & 0.87\,$\pm$\,0.14 \\
 & SVGP
 & 0.10\,$\pm$\,0.00 & 0.37\,$\pm$\,0.00 & 0.08\,$\pm$\,0.00 & 0.57\,$\pm$\,0.00 & 0.10\,$\pm$\,0.00 & 0.47\,$\pm$\,0.01 & 0.80\,$\pm$\,0.00 & 0.32\,$\pm$\,0.00 & 0.12\,$\pm$\,0.00 \\
\midrule
\multirow{4}{*}{\rotatebox[origin=c]{90}{Time (min)}}
 & SDD
 & 1.88\,$\pm$\,0.01 & 1.13\,$\pm$\,0.02 & 1.15\,$\pm$\,0.02 & 1.36\,$\pm$\,0.01 & 1.70\,$\pm$\,0.00 & \textbf{3.32\,$\pm$\,0.01} & \textbf{185\,$\pm$\,0.56} & \textbf{207\,$\pm$\,0.10} & \textbf{47.8\,$\pm$\,0.02} \\
 & SGD
 & 2.80\,$\pm$\,0.01 & 2.07\,$\pm$\,0.03 & 2.12\,$\pm$\,0.04 & 2.87\,$\pm$\,0.01 & 3.30\,$\pm$\,0.12 & 6.68\,$\pm$\,0.02 & 190\,$\pm$\,0.61 & 212\,$\pm$\,0.15 & 69.5\,$\pm$\,0.06 \\
 & CG
 & \textbf{0.17\,$\pm$\,0.00} & \textbf{0.04\,$\pm$\,0.00} & \textbf{0.11\,$\pm$\,0.01} & \textbf{0.16\,$\pm$\,0.01} & \textbf{0.17\,$\pm$\,0.00} & 13.4\,$\pm$\,0.01 & 192\,$\pm$\,0.77 & 244\,$\pm$\,0.04 & 157\,$\pm$\,0.01 \\
 & SVGP
 & 11.5\,$\pm$\,0.01 & 11.3\,$\pm$\,0.06 & 11.1\,$\pm$\,0.02 & 11.1\,$\pm$\,0.02 & 11.5\,$\pm$\,0.04 & 152\,$\pm$\,0.15 & 213\,$\pm$\,0.13 & 209\,$\pm$\,0.37 & 154\,$\pm$\,0.12 \\
\midrule
\multirow{4}{*}{\rotatebox[origin=c]{90}{NLL}}
 & SDD
 & \textbf{-1.18\,$\pm$\,0.01} & \textbf{0.38\,$\pm$\,0.01} & -2.49\,$\pm$\,0.09 & \textbf{0.63\,$\pm$\,0.02} & \textbf{-0.92\,$\pm$\,0.11} & \textbf{-1.70\,$\pm$\,0.01} & \textbf{1.13\,$\pm$\,0.01} & \textbf{0.17\,$\pm$\,0.06} & \textbf{-1.46\,$\pm$\,0.10} \\
 & SGD
 & -0.70\,$\pm$\,0.02 & 0.47\,$\pm$\,0.00 & -0.48\,$\pm$\,0.08 & 0.64\,$\pm$\,0.01 & -0.62\,$\pm$\,0.07 & -0.60\,$\pm$\,0.00 & 1.21\,$\pm$\,0.00 & 0.83\,$\pm$\,0.07 & -1.09\,$\pm$\,0.04 \\
 & CG
 & \textbf{-1.17\,$\pm$\,0.01} & \textbf{0.38\,$\pm$\,0.00} & \textbf{-2.62\,$\pm$\,0.06} & \textbf{0.62\,$\pm$\,0.01} & \textbf{-0.92\,$\pm$\,0.10} & 16.3\,$\pm$\,0.45 & 1.36\,$\pm$\,0.07 & 2.38\,$\pm$\,0.08 & 2.07\,$\pm$\,0.58 \\
 & SVGP
 & -0.67\,$\pm$\,0.01 & 0.43\,$\pm$\,0.00 & -1.21\,$\pm$\,0.01 & 0.85\,$\pm$\,0.01 & -0.54\,$\pm$\,0.02 & 0.60\,$\pm$\,0.00 & 1.21\,$\pm$\,0.00 & 0.22\,$\pm$\,0.03 & -0.61\,$\pm$\,0.01 \\
\bottomrule
\end{tabular}

\end{table}

The results, reported in \Cref{tab:UCI_regression}, show that SDD matches or outperforms all baselines on all UCI datasets in terms of RMSE of the mean prediction across test data.
SDD strictly outperforms SGD on all datasets and metrics, matches CG on the five smaller data sets, where the latter reaches tolerance, and outperforms CG on the four larger datasets.
The same holds for the negative log-likelihood (NLL), computed using $64$ posterior function samples, except on \textsc{bike}, where CG marginally outperforms SDD. 
Since SDD requires only one matrix-vector multiplication per step, as opposed to two for SGD, it provides about $30\%$ wall-clock time speed-up relative to SGD.
We run SDD for $100$k iterations to match the SGD baseline, SDD often converges earlier than that.

\begin{figure}[t]
\centering
\includegraphics[width=6in]{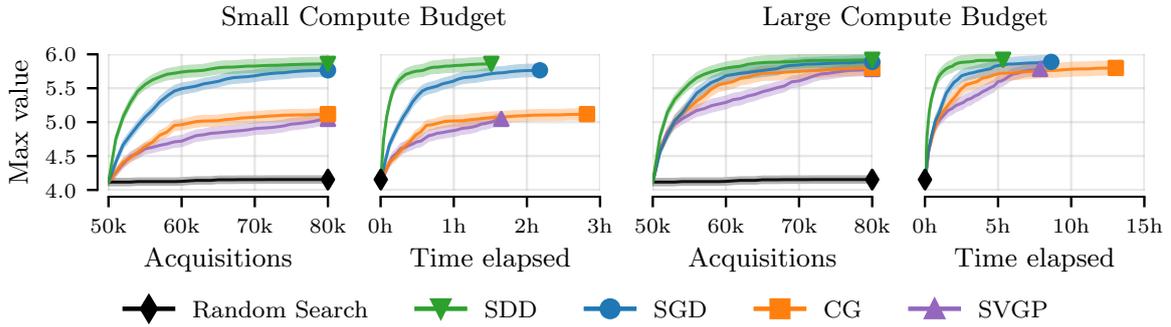}
\caption{Results for the parallel Thompson sampling task. Plots show mean and standard error of the maximum function values identified, across $5$ length scales and $10$ seeds, against both the number of observations acquired and the corresponding compute time on an A100 GPU. Reported compute time includes drawing posterior function samples and finding their maxima. All methods share an initial dataset of $50$k random points, and perform $30$ steps of parallel Thompson sampling, acquiring $1000$ points at each step.}
\label{fig:bayesopt-task}
\end{figure}

\subsection{Large-Scale Parallel Thompson Sampling}
Next, we replicate the synthetic large-scale black-box function optimisation experiment from the previous chapter, which is based on \citet{lin2023sampling}.
The experiment consists of finding the maxima of functions mapping $[0,1]^8 \to \R$ sampled from Matérn-$\nicefrac{3}{2}$ Gaussian process priors with length scales $\ell = 0.1, 0.2, 0.3, 0.4, 0.5$, and $10$ random functions per length scale, using parallel Thompson sampling \citep{hernandezlobato2017Parallel}.
Each method uses the same kernel as the target function, and the latter is drawn using $2000$ random Fourier features.
All methods are initialised with the same $50$k points chosen uniformly at random on the domain. 
We run $30$ iterations of parallel Thompson sampling, acquiring $1000$ points at each iteration.
The experiment considers two different compute budgets.
For the small compute budget, SGD and SDD are run for $15$k steps, SVGP is given $20$k steps and CG is run for $10$ steps.
For the large compute budget, all methods are run for $5\times$ as many steps.
Here, SGD uses a step size of $\beta n = 0.3$ for the mean and $\beta n = 0.0003$ for the samples \citep{lin2023sampling}.
SDD uses step sizes that are $10\times$ larger, namely $\beta n = 3$ for the mean and $\beta n = 0.003$ for the samples.
We present the results on this task in \Cref{fig:bayesopt-task}, averaged over both length scales and random seeds.
In both small and large compute settings, SDD makes the most progress in terms of maximum value found while also using the least compute. The performances of SDD and SGD degrade gracefully when the compute budget is limited.

\subsection{Molecule-Protein Binding Affinity Prediction}
\label{expt:molecules}
The binding affinity between a molecule and certain proteins is a widely used preliminary filter in drug discovery \citep{pinza2019docking}, and machine learning is increasingly used to estimate this quantity \citep{yang2021efficient}.
In this experiment, we show that Gaussian processes with SDD are competitive with graph neural networks for binding affinity prediction.

In this experiment, we use the \textsc{dockstring} regression benchmark \citep{Ortegon2022dockstring}, which contains five tasks, corresponding to five different proteins. 
The inputs are the graph structures of $250$k candidate molecules and the targets are real-valued affinity scores from the docking simulator \emph{AutoDock Vina} \citep{trott2010autodock}. 
For each protein, we use standard train-test splits of $210$k and $40$k molecules, respectively, which were produced by structure-based clustering to avoid similar molecules from occurring both in the train and test set \citep{Ortegon2022dockstring}.
We follow all preprocessing steps of \citet{Ortegon2022dockstring} for this benchmark, including limiting the maximum docking score to $5$.

Furthermore, we use Morgan fingerprints of dimension $1024$ \citep{rogers2010extended} to represent the molecules.
They encode subgraphs up to a certain radius around each atom in a molecule and can be interpreted as sparse vectors of counts, analogous to a bag of words representation of a document.
Accordingly, the Tanimoto coefficient $T(\v{x}, \v{x}')$, also known as the Jaccard index, is a way to measure similarity between fingerprints \citep{ralaivola2005graph}.
It is defined as
\begin{equation}
    T(\v{x}, \v{x}') = \frac{\sum_i\min(x_i,x'_i)}{\sum_i{\max(x_i,x'_i)}},
\end{equation}
which is a valid kernel function with a known random feature expansion based on random hashes \citep{tripp2023tanimoto}.
The feature expansion builds upon prior work for fast retrieval of documents using random hashes which approximate the Tanimoto coefficient, that is, a distribution $\P_h$ over hash functions $h$ such that
\begin{equation*}
    \P_h(h(\v{x}) = h(\v{x}')) = T(\v{x}, \v{x}').
\end{equation*}
Following \citet{tripp2023tanimoto}, we extend hashes into random features by using the former to index a random tensor whose entries are independent Rademacher random variables.
We use the random hash of \citet{ioffe2010improved}.
We consider a Gaussian process with Tanimoto kernel and the hyperparameters of \citet{tripp2023tanimoto}, namely a scalar kernel amplitude, the noise variance $\sigma^2$, and a constant prior mean.
These are chosen using an exact Gaussian process on a randomly selected subset of the data.
The same values are used for all Gaussian process models to ensure that the performance differences are solely due to the quality of the posterior approximation.
The SGD method uses $100$-dimensional random features for the regulariser.

\begin{table}[t]
    \caption{Test set $R^2$ scores obtained for each target protein on the \textsc{dockstring} molecular binding affinity prediction task. Results with (\ensuremath{\cdot})\textsuperscript{\ensuremath{\dagger}} are from \citet{Ortegon2022dockstring} and results with (\ensuremath{\cdot})\textsuperscript{\ensuremath{\ddagger}} are from \citet{tripp2023tanimoto}. SVGP uses $1000$ inducing points.}
    \label{tab:molecule_regression}
    \centering
\footnotesize
\setlength{\tabcolsep}{2.5pt}
\renewcommand{\arraystretch}{1.1}
\hfill
\begin{tabular}{l c c c c c}
\toprule
Method & \textsc{ESR2} & \textsc{F2} & \textsc{KIT} & \textsc{PARP1} & \textsc{PGR} \\
\midrule
Attentive FP\textsuperscript{\ensuremath{\dagger}} & \textbf{0.627} & \textbf{0.880} & \textbf{0.806} & \textbf{0.910} & \textbf{0.678} \\
MPNN\textsuperscript{\ensuremath{\dagger}} & 0.506 & 0.798 & 0.755 & 0.815 & 0.324 \\
XGBoost\textsuperscript{\ensuremath{\dagger}} & 0.497 & 0.688 & 0.674 & 0.723 & 0.345 \\
\bottomrule
\end{tabular}
\hfill
\begin{tabular}{l c c c c c}
\toprule
Method & \textsc{ESR2} & \textsc{F2} & \textsc{KIT} & \textsc{PARP1} & \textsc{PGR} \\
\midrule
SDD & \textbf{0.627} & \textbf{0.880} & 0.790 & 0.907 & 0.626 \\
SGD & 0.526 & 0.832 & 0.697 & 0.857 & 0.408 \\
SVGP\textsuperscript{\ensuremath{\ddagger}} & 0.533 & 0.839 & 0.696 & 0.872 & 0.477 \\ 
\bottomrule
\end{tabular}
\hfill
\end{table}

In \Cref{tab:molecule_regression}, following \citet{Ortegon2022dockstring}, we report $R^2$ values. Alongside results for SDD and SGD, we include results from \citet{Ortegon2022dockstring} for XGBoost and two graph neural networks, MPNN \citep{gilmer2017neural} and the state-of-the-art Attentive FP \citep{xiong2019pushing}.
We also include the results for SVGP reported by \citet{tripp2023tanimoto}. 
Our results show that SDD matches the performance of Attentive FP on the ESR2 and FP2 proteins, and comes close on the others. 
To the best of our knowledge, this is the first time that Gaussian processes are competitive on a large-scale molecular prediction task.

\section{Discussion}
In this chapter, we introduced stochastic dual descent, a specialised first-order stochastic optimisation algorithm for computing Gaussian process predictions and posterior samples.
To design this algorithm, we adapted various ideas from the optimisation literature for Gaussian processes, arriving at an algorithm which is simultaneously simple and matches or exceeds the performance of relevant baselines.
We showed that stochastic dual descent yields strong performance on standard regression benchmarks and a large-scale Bayesian optimisation benchmark, and matches the performance of state-of-the-art graph neural networks on a molecular binding affinity prediction task.
The next chapter proposes generic improvements for iterative linear system solvers, which are also applicable to stochastic dual descent.

%!TEX root = ../thesis.tex
%*******************************************************************************
%****************************** Fifth Chapter **********************************
%*******************************************************************************
\chapter[Improving Linear System Solvers for Hyperparameter Optimisation]{Improving Linear System Solvers\\for Hyperparameter Optimisation}
\label{chap:solvers}

% **************************** Define Graphics Path **************************
\ifpdf
    \graphicspath{{Chapter5/Figs/Raster/}{Chapter5/Figs/PDF/}{Chapter5/Figs/}}
\else
    \graphicspath{{Chapter5/Figs/Vector/}{Chapter5/Figs/}}
\fi

The previous two chapters considered stochastic gradient and dual descent as scalable linear system solvers for posterior inference in Gaussian processes.
In particular, we assumed fixed kernel and noise hyperparameters, which were obtained via heuristics or set to constants.
In this chapter, we consider hyperparameter optimisation with iterative linear system solvers, proposing generic improvements which are applicable to any solver.
We introduce a pathwise gradient estimator and warm start linear system solvers, which results in significantly faster convergence.
Additionally, we investigate iterative linear system solvers on a limited compute budget, stopping them before reaching convergence, which is common practice in large-scale settings.
Empirically, our techniques provide speed-ups of up to $72\times$ when solving until convergence, and decrease the average residual norm by up to $7\times$ when stopping early.

This chapter includes content which is adapted from the following publications:
\begin{itemize}
    \item J. A. Lin, S. Padhy, B. Mlodozeniec, and J. M. Hernández-Lobato. Warm Start Marginal Likelihood Optimisation for Iterative Gaussian Processes. In \emph{Advances in Approximate Bayesian Inference}, 2024.
    \item J. A. Lin, S. Padhy, B. Mlodozeniec, J. Antorán, and J. M. Hernández-Lobato. Improving Linear System Solvers for Hyperparameter Optimisation in Iterative Gaussian Processes. In \emph{Advances in Neural Information Processing Systems}, 2024.
\end{itemize}
My contributions to these projects consist of developing the ideas, initiating and managing the collaborations, implementing most of the source code, conducting most of the experiments, contributing some parts of the theory, and writing the manuscripts.

\begin{figure}[t]
    \centering
    \includegraphics[width=6in]{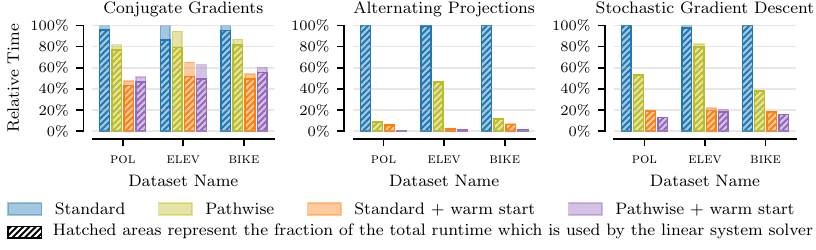}
    \caption{Comparison of relative runtimes for different methods, linear system solvers, and datasets. The linear system solver (hatched areas) dominates the total training time (coloured patches). The pathwise gradient estimator requires less time than the standard estimator. Initialising at the previous solution (warm start) further reduces the runtime of the linear system solver for both estimators.}
    \label{fig:train_solver_time}
\end{figure}

\section{Introduction}
Gaussian processes \citep{rasmussen2006} (GPs) are a versatile class of probabilistic machine learning models which are used widely for Bayesian optimisation of black-box functions \citep{Snoek2012}, climate and earth sciences \citep{ghasemi2021application,tazi2023beyond}, and data-efficient learning in robotics and control \citep{Deisenroth2015}.
However, their effectiveness depends on good estimates of hyperparameters, such as kernel length scales and observation noise variance.
These quantities are typically learned by maximising the marginal likelihood, which balances model complexity with training data fit.
In general, the marginal likelihood is a non-convex function of the hyperparameters and evaluating its gradient requires inverting the kernel matrix.
Using direct methods, this requires compute and memory resources which are respectively cubic and quadratic in the number of training examples.
This is intractable when dealing with large datasets of modern interest.

Approaches to improve the scalability of GPs can roughly be grouped into two categories.
Sparse methods \citep{candela2005,titsias09,titsias2009report,hensman13} approximate the kernel matrix with a low-rank surrogate, which is cheaper to invert.
However, this reduced flexibility may result in failure to properly fit increasingly large or sufficiently complex data \citep{lin2023sampling}.
On the other hand, iterative methods express GP computations in terms of systems of linear equations.
The solution to these linear systems is approximated up to a specified numerical precision with linear system solvers, such as conjugate gradients (CG) \citep{gardner18,WangPGT2019exactgp}, alternating projections (AP) \citep{shalev2013stochastic,tu2016large,wu2024large}, or stochastic gradient descent (SGD) \citep{lin2023sampling, lin2024stochastic}.
These iterative methods allow for a trade-off between compute time and accuracy.
However, convergence can be slow in the large data regime, where system conditioning is often poor.

In this chapter, we focus on iterative methods and identify techniques, which were important to the success of previously proposed methods, but did not receive special attention in the literature.
Many of these amount to amortisations which leverage previous computations to accelerate subsequent ones.
We analyse and adapt these techniques, and show that they can be applied to accelerate different linear solvers, obtaining speed-ups of up to $72\times$ without sacrificing predictive performance (see \Cref{fig:train_solver_time}).

In the following, we summarise our contributions:
\begin{itemize}
\item We introduce a pathwise estimator of the marginal likelihood gradient and demonstrate that, under real-world conditions, the solutions to the linear systems required by this estimator are closer to the origin than those of the standard gradient estimator, allowing our solvers to converge faster. Additionally, these solutions transform into samples from the GP posterior without further matrix inversions, amortising the computational costs of predictive posterior inference.

\item We propose to warm start linear system solvers throughout marginal likelihood optimisation by reusing linear system solutions to initialise the solver in the subsequent step. This results in faster convergence. Although this technically introduces bias into the optimisation, we show that, theoretically and empirically, the optimisation quality does not suffer.

\item We investigate the behaviour of linear system solvers on a limited compute budget, such that reaching the specified tolerance is not guaranteed. Here, warm starting allows the linear system solver to accumulate solver progress across marginal likelihood steps, progressively improving the solution quality of the linear system solver despite early stopping.

\item We demonstrate empirically that the methods above either reduce the required number of iterations until convergence without sacrificing performance or improve the performance if a limited compute budget hinders convergence. Across different UCI regression datasets and linear system solvers, we observe average speed-ups of up to $72\times$ when solving until the tolerance is reached, and increased performance when the compute budget is limited.
\end{itemize}

\subsection{Marginal Likelihood Optimisation in Iterative Gaussian Processes}
\label{sec:marginal_likelihood_iterative}
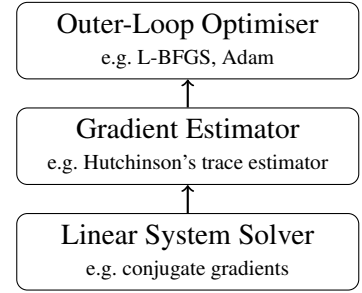
\begin{wrapfigure}{r}{0.35\textwidth}
\vspace{-0.4cm}
\centering
\begin{tikzpicture}[
  node/.style={rectangle, rounded corners, minimum width=4.5cm, minimum height=1cm, draw=black},
  arrow/.style={thick,<-},
  node distance=1.4cm
]

\node (optimiser) [node, align=center] {\small{Outer-Loop Optimiser}\\[-.2cm]\scriptsize{e.g.\ L-BFGS, Adam}};
\node (gradient) [node, below of=optimiser, align=center] {\small{Gradient Estimator}\\[-.2cm]\scriptsize{e.g.\ Hutchinson's trace estimator}};
\node (solver) [node, below of=gradient, align=center] {\small{Linear System Solver}\\[-.2cm]\scriptsize{e.g.\ conjugate gradients}};

\draw [arrow] (optimiser) -- (gradient);
\draw [arrow] (gradient) -- (solver);

\end{tikzpicture}
\caption{Marginal likelihood optimisation for iterative GPs.}
\label{fig:diagram}
\vspace{-.5cm}
\end{wrapfigure}
Marginal likelihood optimisation for iterative GPs consists of bi-level optimisation, where the outer loop maximises the marginal likelihood \eqref{eq:mll} using stochastic estimates of its gradient \eqref{eq:mll_grad}.
Computing these gradient estimates requires the solution to systems of linear equations.
These solutions are obtained using an iterative solver in the inner loop.
\Cref{fig:diagram} illustrates this three-level hierarchy.
In this chapter, we assume a prior mean of zero and define $\m{H}_{\v{\theta}} = \m{K}_\m{XX} + \sigma^2 \m{I}$, where $\v{\theta}$ includes all kernel hyperparameters $\v{\vartheta}$ and observation noise $\sigma^2$.
Throughout this chapter, we use a Matérn-\nicefrac{3}{2} kernel, parametrised by a scalar signal scale and a length scale per input dimension.
The marginal likelihood $\c{L}$ as a function of $\v{\theta}$ and its gradient $\nabla_{\theta_i} \c{L}$ with respect to $\theta_i$ can be expressed as
\begin{align}
    \label{eq:mll}
    \c{L}(\v{\theta}) &= -\frac{1}{2} \v{y}\T \m{H}_{\v{\theta}}\inv \v{y} - \frac{1}{2} \log \det \m{H}_{\v{\theta}} - \frac{n}{2} \log 2 \pi, \\
    \label{eq:mll_grad}
    \nabla_{\theta_i} \c{L}(\v{\theta}) &= \frac{1}{2} \del{\m{H}_{\v{\theta}}\inv \v{y}}\T \frac{\partial \m{H}_{\v{\theta}}}{\partial \theta_i} \del{\m{H}_{\v{\theta}}\inv \v{y}} - \frac{1}{2} \tr \del{\m{H}_{\v{\theta}}\inv \frac{\partial \m{H}_{\v{\theta}}}{\partial \theta_i}},
\end{align}
where the partial derivative of $\m{H}_{\v{\theta}}$ with respect to $\theta_i$ is a $n \times n$ matrix (see \Cref{sec:model_selection} for reference).
We assume $n$ is too large to compute the inverse or log-determinant of $\m{H}_{\v{\theta}}$ and iterative methods are used instead.

\subsubsection{Outer-Loop Optimiser}
The outer-loop optimiser maximises the marginal likelihood $\c{L}$ using its gradient.
Common choices are L-BFGS, when exact gradients are available \citep{artemev2021tighter}, and Adam in the large data setting, when stochastic approximation is required \citep{WangPGT2019exactgp, wu2024large}.
We consider the case where gradients are stochastic and use Adam.

\subsubsection{Gradient Estimator}
The gradient of the marginal likelihood involves two computationally expensive components: linear solves against the targets $\m{H}_{\v{\theta}}\inv \v{y}$ and the trace term $\tr \del{\m{H}_{\v{\theta}}\inv \partial \m{H}_{\v{\theta}} / \partial \theta_i}$.
An unbiased estimate of the latter can be obtained using $s$ probe vectors and Hutchinson's trace estimator \citep{Hutchinson1990} (see \Cref{sec:iterative_methods} for reference),
\begin{equation}
    \label{eq:trace}
    \tr \del{\m{H}_{\v{\theta}}\inv \frac{\partial \m{H}_{\v{\theta}}}{\partial \theta_i}}
    = \E_{\v{z}} \sbr{\v{z}\T \m{H}_{\v{\theta}}\inv \frac{\partial \m{H}_{\v{\theta}}}{\partial \theta_i} \v{z}}
    \approx \frac{1}{s} \sum_{j=1}^s \v{z}_j\T \m{H}_{\v{\theta}}\inv \frac{\partial \m{H}_{\v{\theta}}}{\partial \theta_i} \v{z}_j,
\end{equation}
where the probe vectors $\v{z}_j \in \R^n$ satisfy $\E [\v{z}_j \v{z}_j\T] = \m{I}$, and $\v{z}_j\T \m{H}_{\v{\theta}}\inv$ is obtained using a linear solve.
We refer to this as the \emph{standard estimator} and set $s = 64$, unless otherwise specified.

\subsubsection{Linear System Solver}
Substituting the trace estimator above back into the gradient, we obtain an unbiased gradient estimate in terms of the solution to a batch of linear systems,
\begin{equation}
\label{eq:standard_system}
    \m{H}_{\v{\theta}} \, \sbr{\, \v{v}_{\v{y}}, \v{v}_1, \dots, \v{v}_s \,}
    = \sbr{\, \v{y}, \v{z}_1, \dots, \v{z}_s \,},
\end{equation}
which share the coefficient matrix $\m{H}_{\v{\theta}}$.
Since $\m{H}_{\v{\theta}}$ is positive-definite, the solution $\v{v} = \m{H}_{\v{\theta}}\inv \v{b}$ to the system $\m{H}_{\v{\theta}} \, \v{v} = \v{b}$ can be obtained by finding the unique minimiser of a corresponding convex quadratic objective,
\begin{equation}
    \label{eq:quadratic}
    \v{v} = \underset{\v{u}}{\arg \min} \;\; \frac{1}{2} \v{u}\T \m{H}_{\v{\theta}} \, \v{u} - \v{u}\T \v{b},
\end{equation}
facilitating the use of iterative solvers (see \Cref{sec:iterative_methods}).
Most popular in the GP literature are conjugate gradients (CG) \citep{gardner18,WangPGT2019exactgp,wilson20,wilson21}, alternating projections (AP) \citep{shalev2013stochastic,tu2016large,wu2024large} and stochastic gradient descent (SGD) \citep{lin2023sampling,lin2024stochastic}.
For details about CG, we refer to \Cref{sec:iterative_methods}.
AP optimises the original quadratic objective by iteratively solving smaller quadratic objectives, where each smaller objective corresponds to a block matrix along the diagonal of $\m{H}_{\v{\theta}}$.
The block size is chosen such that the smaller objectives can be solved directly using Cholesky factorisation.
The solution to the original quadratic objective is obtained by accumulating lower-dimensional updates and keeping track of the residual.
Pseudocode for AP is provided in \Cref{alg:ap}.
For more information about SGD, we refer to \Cref{chap:sdd}, because we use the improved stochastic dual descent variant.

Solvers are often run until the relative residual norm $\norm{\v{b} - \m{H}_{\v{\theta}} \v{u}}_2 / \norm{\v{b}}_2$ reaches a certain tolerance $\tau$ \citep{WangPGT2019exactgp,maddox2021iterative,wu2024large}.
Following \citet{maddox2021iterative}, we set $\tau = 0.01$.
The linear system solver in the inner loop dominates the computational costs of marginal likelihood optimisation for iterative GPs, as illustrated in \Cref{fig:train_solver_time}.
Therefore, improving linear system solvers is the main focus of this chapter.

\begin{algorithm}[t]
    \caption{Alternating projections for solving $\m{H}_{\v{\theta}} \, \v{v} = \v{b}$}
    \label{alg:ap}
\begin{algorithmic}[1]
    \Require Coefficient matrix $\m{H}_{\v{\theta}}$, targets $\v{b}$, tolerance $\tau$, maximum number of iterations $t_{\mathrm{max}}$, block size $b$, block partitions $[1], [2], \dots, [\lceil \frac{n}{b} \rceil]$
    \State $\v{v} \gets \v{0}$ (or previous solution if warm start); $\bm{r} \gets \v{b} - \m{H}_{\v{\theta}} \v{v}$; $t \gets 0$
    \While{$t < t_{\mathrm{max}}$ \textbf{ and } $\Vert \bm{r} \Vert_2 / \Vert \v{b} \Vert_2 > \tau$}
        \State $[i] \gets \texttt{arg\_max}(\Vert \v{r}[1] \Vert_2, \dots, \Vert \v{r}[\lceil \frac{n}{b} \rceil] \Vert_2)$ \Comment{Select block with largest residual}
        \State $\v{d} \gets \texttt{cholesky\_solve}(\m{H}_{\v{\theta}}[i, i], \bm{r}[i])$ \Comment{Solve lower-dimensional quadratic}
        \State $\v{v}[i] \gets \v{v}[i] + \v{d}$ \Comment{Update global solution vector}
        \State $\v{r} \gets \v{r} - \m{H}_{\v{\theta}}[:, i] \, \v{d}$ \Comment{Update running residual}
        \State $t \gets t + 1$ \Comment{Update iteration counter}
    \EndWhile
    \State \Return $\v{v}$
\end{algorithmic}
\end{algorithm}

\section{Pathwise Estimation of Marginal Likelihood Gradients}
\label{sec:pathwise_estimator}
We introduce the \emph{pathwise estimator}, an alternative to the standard estimator which reduces the required number of solver iterations until convergence (see \Cref{fig:initial_distance}).
Additionally, the estimator also provides us with posterior function samples via pathwise conditioning, hence the name \emph{pathwise} estimator.
This facilitates predictions without further linear solves.

We modify the standard estimator to absorb $\m{H}_{\v{\theta}}\inv$ into the distribution of the random probe vectors \citep{antoran2023sampling},
\begin{equation}
    \label{eq:pathwise_estimator}
    \tr \del{\m{H}_{\v{\theta}}\inv \frac{\partial \m{H}_{\v{\theta}}}{\partial \theta_i}}
    = \tr \del{\E [\v{\hat{z}}_j \v{\hat{z}}_j\T] \frac{\partial \m{H}_{\v{\theta}}}{\partial \theta_i}}
    = \E_{\v{\hat{z}}} \sbr{\v{\hat{z}}\T \frac{\partial \m{H}_{\v{\theta}}}{\partial \theta_i} \v{\hat{z}}}
    \approx \frac{1}{s} \sum_{j=1}^s \v{\hat{z}}_j\T \frac{\partial \m{H}_{\v{\theta}}}{\partial \theta_i} \v{\hat{z}}_j,
\end{equation}
where $\E [\v{\hat{z}}_j \v{\hat{z}}_j\T] = \m{H}_{\v{\theta}}\inv$.
Probe vectors $\v{\hat{z}}$ with the desired second moment can be obtained as
\begin{equation}
\begin{aligned}
    \v{f}_\m{X} &\~[N](\v{0}, \m{K}_\m{XX}) \\
    \v{\eps} &\~[N](\v{0}, \sigma^2 \m{I})
\end{aligned}
\implies
\v{f}_\m{X} + \v\eps = \v{\xi} \~[N](\v{0}, \m{H}_{\v{\theta}})
\implies
\v{\hat{z}} = \m{H}_{\v{\theta}}\inv \v{\xi} \sim \mathcal{N}(\v{0}, \m{H}_{\v{\theta}}\inv),
\end{equation}
where $\v{f}_\m{X}$ is a sample from the Gaussian process prior evaluated at the observed inputs $\m{X}$ and $\v\eps$ is a Gaussian noise sample.
Following \citet{wilson20,lin2023sampling}, we use random features \citep{rahimi08,sutherland15} to efficiently draw samples $f \~[GP](0, k)$ from the Gaussian process prior (see \Cref{sec:random_features}).
We use 2000 random features throughout this chapter.
Finally, we obtain $\v{v}_{\v{y}}$ and $\v{\hat{z}}_j$ by solving
\begin{equation}
\label{eq:pathwise_system}
    \m{H}_{\v{\theta}} \, \left[ \, \v{v}_{\v{y}}, \v{\hat{z}}_1, \dots, \v{\hat{z}}_s \, \right]
    = \left[ \, \v{y}, \v{\xi}_1, \dots, \v{\xi}_s \, \right],
\end{equation}
akin to the standard estimator, where the $\v{\xi}_j$ are resampled in each outer-loop step.

The name \emph{pathwise} estimator comes from the fact that solving the linear systems provides us with all the necessary terms to construct a set of $s$ posterior samples via pathwise conditioning (see \Cref{eq:pathwise_conditioning} for reference). Each of these is given by
\begin{equation}
    \v{f}_{(\.) \given \v{y}} = \v{f}_{(\.)} + \m{K}_{(\.)\m{X}} \m{H}_{\v{\theta}}\inv (\v{y} - \v{\xi})
    =  \v{f}_{(\.)} + \m{K}_{(\.)\m{X}}(\v{v}_{\v{y}} - \v{\hat{z}}).
\end{equation}
We can use these to make predictions without requiring any additional linear system solves.

\begin{figure}[t]
    \centering
    \includegraphics[width=6in]{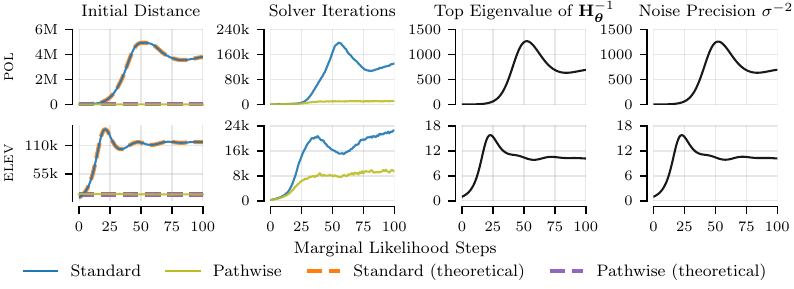}
    \caption{On the \textsc{pol} and \textsc{elevators} datasets, the pathwise estimator results in a lower RKHS distance \eqref{eq:RKHS_distance} between solver initialisation and solution, as predicted by theory (left) (see \Cref{eq:standard_dist,eq:pathwise_dist} for reference). This results in fewer AP iterations until reaching the tolerance (left middle). When using the standard estimator, the initial distance follows the top eigenvalue of $\m{H}_{\v{\theta}}\inv$ (right middle), which is strongly related to the noise precision (right). The latter tends to increase during marginal likelihood optimisation when fitting the data. The effects are greater on \textsc{pol} due to the higher noise precision.}
    \label{fig:initial_distance}
\end{figure}

\subsection{Initial Distance to the Linear System Solution}
Under realistic conditions, the pathwise estimator moves the solution of the linear system closer to the origin.
To show this, we consider the linear system $\m{H}_{\v{\theta}} \v{u} = \v{b}$ and measure the RKHS distance between the initialisation $\v{u}_{\mathrm{init}}$ and the solution $\v{u} = \m{H}_{\v{\theta}}\inv \, \v{b}$ as $\lVert \v{u}_{\mathrm{init}} - \v{u} \rVert_{\m{H}_{\v{\theta}}}^2$.
With $\v{u}_{\mathrm{init}} = \v{0}$, which is standard \citep{WangPGT2019exactgp,lin2023sampling,wu2024large},
\begin{equation}
\label{eq:RKHS_distance}
    \norm{\v{u}_{\mathrm{init}} - \v{u}}_{\m{H}_{\v{\theta}}}^2
    = \norm{\v{u}}_{\m{H}_{\v{\theta}}}^2
    = \v{u}\T \m{H}_{\v{\theta}} \v{u}
    = \v{b}\T \m{H}_{\v{\theta}}\inv \m{H}_{\v{\theta}} \m{H}_{\v{\theta}}\inv \v{b}
    = \v{b}\T \m{H}_{\v{\theta}}\inv \v{b}.
\end{equation}
Since $\v{b}$ is a random vector in our case, we analyse the expected squared distance,
\begin{equation}
    \E \sbr{\v{b}\T \m{H}_{\v{\theta}}\inv \v{b}}
    = \E \sbr{\tr \del{\v{b}\T \m{H}_{\v{\theta}}\inv \v{b}}}
    = \E \sbr{\tr \del{\v{b} \v{b}\T \m{H}_{\v{\theta}}\inv}}
    = \tr \del{\E \sbr{\v{b} \v{b}\T} \m{H}_{\v{\theta}}\inv}.
\end{equation}
For the standard estimator \eqref{eq:trace}, we substitute $\v{b} = \v{z}$ with $\E \sbr{\v{z} \v{z}\T} = \m{I}$, yielding
\begin{equation}
\label{eq:standard_dist}
    \E \sbr{\norm{\v{u}_{\mathrm{init}} - \v{u}}_{\m{H}_{\v{\theta}}}^2}
    = \tr\del{\E \sbr{\v{z} \v{z}\T} \m{H}_{\v{\theta}}\inv}
    = \tr\del{\m{I} \, \m{H}_{\v{\theta}}\inv}
    = \tr\del{\m{H}_{\v{\theta}}\inv}.
\end{equation}
For the pathwise estimator \eqref{eq:pathwise_estimator}, we substitute $\v{b} = \v{\xi}$ with $\E \sbr{\v{\xi} \v{\xi}\T} = \m{H}_{\v{\theta}}$, yielding
\begin{equation}
\label{eq:pathwise_dist}
    \E \sbr{\norm{\v{u}_{\mathrm{init}} - \v{u}}_{\m{H}_{\v{\theta}}}^2}
    = \tr\del{\E \sbr{\v{\xi} \v{\xi}\T} \m{H}_{\v{\theta}}\inv}
    = \tr\del{\m{H}_{\v{\theta}} \, \m{H}_{\v{\theta}}\inv}
    = \tr\del{\m{I}}
    = n.
\end{equation}
Therefore, the initial distance for the standard estimator is equal to the trace of $\m{H}_{\v{\theta}}\inv$, whereas it is constant for the pathwise estimator.
\Cref{fig:initial_distance} illustrates that this trace follows the top eigenvalue, which matches the noise precision.
As the model fits the data, the noise precision increases, increasing the initial distance for the standard but not for the pathwise estimator.
In practice, the latter leads to faster solver convergence (see \Cref{tab:results_main}).

\subsection{How Many Probe Vectors and Posterior Samples Do We Need?}
In the literature, it is common to use $s \leq 16$ probe vectors for marginal likelihood optimisation \citep{gardner18,maddox2021iterative, antoran2023sampling,wu2024large}.
However, a larger number of posterior samples, around $s = 64$, is necessary to make accurate predictions \citep{antoran2023sampling,lin2023sampling,lin2024stochastic} (see \Cref{fig:n_samples}).
Thus, to amortise linear system solves across marginal likelihood optimisation and prediction, we must use the same number of probes for both.
Interestingly, as shown in \Cref{fig:n_samples}, using $64$ instead of $16$ probe vectors only increases the runtime by around $10$\% because the computational costs are dominated by kernel function evaluations, which are shared among probe vectors.
\begin{figure}[hb]
    \centering
    \includegraphics[width=6in]{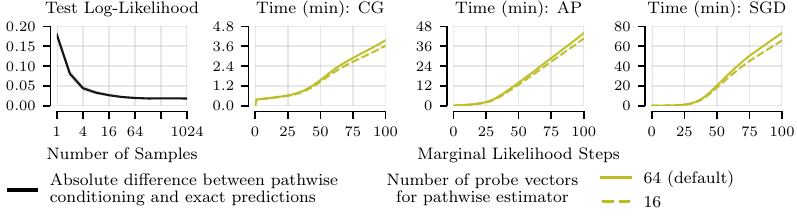}
    \caption{On the \textsc{pol} dataset, increasing the number of posterior samples improves the performance of pathwise conditioning until diminishing returns start to manifest with more than $64$ samples (left). Furthermore, with $4\times$ as many probe vectors, the total cumulative runtime only increases by around $10$\% because the computational costs are dominated by shared kernel function evaluations (right).}
    \label{fig:n_samples}
\end{figure}

\subsection{Estimator Variance}
The variance of the pathwise estimator is less than or equal to the variance of the standard estimator with Gaussian probe vectors, as formally stated by the proposition below.
\begin{proposition}
Let the probe vectors of the standard estimator be $\v{z} \~[N](\v{0}, \m{I})$, and the probe vectors of the pathwise estimator be $\v{\hat{z}} \~[N](\v{0}, \m{H}_{\v{\theta}}\inv)$.
Then we have
\begin{equation}
    \Var \ubr{\del{\v{\hat{z}}\T \frac{\partial \m{H}_{\v{\theta}}}{\partial \theta_i} \v{\hat{z}}}}_{\textnormal{pathwise estimator}}
    \leq
    \Var \ubr{\del{\v{z}\T \m{H}_{\v{\theta}}\inv \frac{\partial \m{H}_{\v{\theta}}}{\partial \theta_i} \v{z}}}_{\textnormal{standard estimator}},
\end{equation}
with equality if and only if $\m{H}_{\v{\theta}}\inv$ and $\partial \m{H}_{\v{\theta}} / \partial \theta_i$ commute with each other.
\end{proposition}
\begin{proofbox}
\begin{proof}
Let $\m{A} = \m{H}_{\v{\theta}}\inv (\partial \m{H}_{\v{\theta}}/\partial \theta_i)$, where $\m{H}_{\v{\theta}}\inv$ is positive definite and $\partial \m{H}_{\v{\theta}}/\partial \theta_i$ is symmetric, but $\m{A}$ is potentially not symmetric.
Using standard identities for multivariate normal random variables \citep{Petersen2006}, we calculate the variance of the standard estimator as
\begin{align}
    \Var \del{\v{z}\T \m{H}_{\v{\theta}}\inv \frac{\partial \m{H}_{\v{\theta}}}{\partial \theta_i} \v{z}}
    &= \tr \del{\m{H}_{\v{\theta}}\inv \frac{\partial \m{H}_{\v{\theta}}}{\partial \theta_i} \del{ \m{H}_{\v{\theta}}\inv \frac{\partial \m{H}_{\v{\theta}}}{\partial \theta_i} + \frac{\partial \m{H}_{\v{\theta}}}{\partial \theta_i} \m{H}_{\v{\theta}}\inv}}, \\
    &= \tr \del{\m{A}^2}
    + \tr \del{\m{A} \m{A}\T},
\end{align}
and the variance of the pathwise estimator is given by
\begin{equation}
    \Var \del{\v{\hat{z}}\T \frac{\partial \m{H}_{\v{\theta}}}{\partial \theta_i} \v{\hat{z}}}
    = 2\,\tr \del{\frac{\partial \m{H}_{\v{\theta}}}{\partial \theta_i}  \m{H}_{\v{\theta}}\inv \frac{\partial \m{H}_{\v{\theta}}}{\partial \theta_i}  \m{H}_{\v{\theta}}\inv}
    = 2 \tr \del{\m{A}^2}.
\end{equation}
Therefore, the difference between both variances depends on the relationship between $\tr \del{\m{A} \m{A}\T}$ and $\tr \del{\m{A}^2}$.
Using the Frobenius norm $\norm{\.}_{\mathrm{F}}$ and the Schur decomposition $\m{A} = \m{Q}\m{T}\m{Q}\inv$, where $\m{Q}$ is unitary and $\m{T}$ is upper triangular, we express $\tr \del{\m{A}\m{A}\T}$ as
\begin{equation}
    \tr \del{\m{A} \m{A}\T}
    = \norm{\m{A}}_{\mathrm{F}}^2
    = \norm{\m{Q}\m{T}\m{Q}\inv}_{\mathrm{F}}^2
    = \norm{\m{T}}_{\mathrm{F}}^2
    = \sum_{j=1}^n \del{\lambda_{j}^2 + \sum_{l=j+1}^n t_{jl}^2},
\end{equation}
where $\lambda_j$ are the eigenvalues of $\m{A}$ and $t_{jl}$ are the upper off-diagonal entries of $\m{T}$.
To establish a relationship with $\tr \del{\m{A}^2}$, we observe that $\m{A}$ is similar to the symmetric matrix $\m{H}_{\v{\theta}}^{-\frac{1}{2}} (\partial \m{H}_{\v{\theta}}/\partial \theta_i) \m{H}_{\v{\theta}}^{-\frac{1}{2}}$.
Therefore, $\m{A}$ is diagonalisable and $\tr \del{\m{A}^2} = \sum_{j=1}^n \lambda_j^2$.
Since each individual $t_{jl}^2 \geq 0$, we conclude that $\sum_{j=1}^n\sum_{l=j+1}^n t_{jl}^2 \geq 0$, and thus
\begin{equation}
    \tr \del{\m{A}^2} \leq \tr \del{\m{A}\m{A}\T}
    \implies
    \Var \del{\v{\hat{z}}\T \frac{\partial \m{H}_{\v{\theta}}}{\partial \theta_i} \v{\hat{z}}}
    \leq
    \Var \del{\v{z}\T \m{H}_{\v{\theta}}\inv \frac{\partial \m{H}_{\v{\theta}}}{\partial \theta_i} \v{z}},
\end{equation}
with equality if and only if $\m{H}_{\v{\theta}}\inv$ and $\partial \m{H}_{\v{\theta}} / \partial \theta_i$ commute \citep{Horn1991}.
\end{proof}
\end{proofbox}

For example, consider $\partial \m{H}_{\v{\theta}} / \partial \sigma = 2 \sigma \m{I}$.
In this case, $\m{H}_{\v{\theta}}\inv$ and $\partial \m{H}_{\v{\theta}} / \partial \sigma$ commute, such that the two variances are equal.
In general, a sufficient condition for matrix multiplication to be commutative is simultaneous diagonalisability of two matrices.
Related work develops trace estimators with lower variance \citep{Hutch++,XTrace}, but we did not pursue these as we find the variance to be sufficiently low, even when relying on only $s = 16$ probe vectors.

\subsection{Approximate Prior Function Samples Using Random Features}
In practice, the pathwise estimator requires samples from the prior $f \~[GP](0, k)$, which is computationally expensive for large datasets without the use of random features.
In \Cref{fig:pathwise_exact_prior}, we show that, despite using random features, the marginal likelihood optimisation trajectory of the pathwise estimator matches the trajectory of exact optimisation using Cholesky factorisation and backpropagation most of the time.
Further, we confirm that deviations of the pathwise estimator are indeed due to the use of random features by demonstrating that we can remove these deviations using exact samples from the prior instead.
\begin{figure}[ht]
    \centering
    \includegraphics[width=6in]{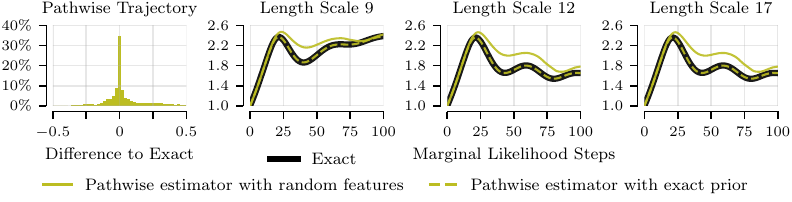}
    \caption{Across all datasets and marginal likelihood steps, most hyperparameter trajectories of the pathwise estimator rarely differ from exact optimisation, as shown by the histogram illustrating the differences between hyperparameters (left). On selected length scales of the \textsc{elevators} dataset, the pathwise estimator deviates due to the use of random features to approximate prior function samples. With exact samples from the prior, the pathwise estimator matches exact optimisation again (right).}
    \label{fig:pathwise_exact_prior}
\end{figure}

\begin{figure}[ht]
    \centering
    \includegraphics[width=6in]{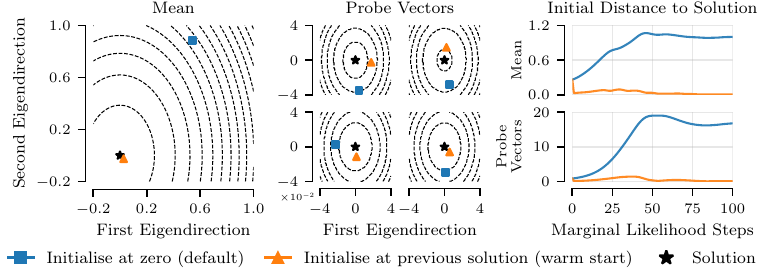}
    \caption{Two-dimensional cross-sections of top eigendirections of the inner-loop quadratic objective after $20$ marginal likelihood steps on the \textsc{pol} dataset, centred on the current solution placed at the origin (left and middle). Warm starting significantly reduces the initial RKHS distance to the solution throughout marginal likelihood optimisation (right).}
    \label{fig:2d_cross_section}
\end{figure}
\section{Warm Starting Linear System Solvers}
\label{sec:warm_start}
Iterative linear system solvers are typically initialised at zero \citep{gardner18,WangPGT2019exactgp,lin2023sampling,wu2024large,lin2024stochastic}.
However, because the outer-loop marginal likelihood optimisation does not change the hyperparameters much between consecutive steps, we expect that the solution to inner-loop linear systems also does not change much between consecutive steps.
Therefore, we suggest to \emph{warm start} linear system solvers by initialising them at the solution of the previous step \citep{lin2024warm}.
This requires that the targets of the linear systems, $\v{z}_j$ or $\v{\xi}_j$, are not resampled throughout optimisation, which can introduce bias \citep{chen20}.
However, we find that warm starting consistently provides gains across all linear system solvers for both the standard and the pathwise estimator, and that the bias is negligible.
\Cref{fig:2d_cross_section} visualises the two top eigendirections of the inner-loop quadratic objective on \textsc{pol}.
Throughout training, warm starting at the solution to the previous linear system results in a substantially smaller initial distance to the current solution.

\subsection{Effects on Linear System Solver Convergence}
Reducing the initial RKHS distance to the solution reduces the required number of solver iterations until convergence for all solvers, as shown in \Cref{fig:solver_iterations_main} and \Cref{tab:results_main}.
However, the effectiveness depends on the solver type.
Due to using line search, CG is more sensitive to the search direction rather than the distance to the solution.
It only obtains a $2.1\times$ speed-up on average.
AP and SGD benefit more, with average speed-ups of $18.9\times$ and $5.1\times$, respectively.

\begin{figure}[ht]
    \centering
    \includegraphics[width=6in]{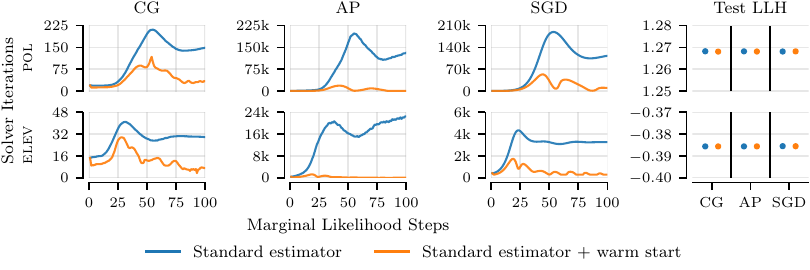}
    \caption{Required number of linear system solver iterations to reach the tolerance $\tau = 0.01$ during marginal likelihood optimisation on the \textsc{pol} and \textsc{elevators} datasets. Warm starting with the previous solution reduces the required number of iterations to reach the tolerance without sacrificing predictive performance.}
    \label{fig:solver_iterations_main}
\end{figure}

\subsection{Does Warm Starting Introduce Bias?}
A potential concern when warm starting is that the latter introduces bias into the optimisation trajectory because the linear system targets are not resampled throughout optimisation.
Although individual gradient estimates are unbiased, estimates are correlated along the optimisation trajectory.
In fact, after fixing the targets, gradients become deterministic and it is unclear whether the induced optimum converges to the true optimum.
The concern might be likened to how pointwise convergence of integrable functions does not always imply convergence of the integrals of those functions, potentially biasing the optima of the limit of the integrals.
Fortunately, one can show that the marginal likelihood at the optimum implied by these gradients will converge in probability to the marginal likelihood of the true optimum.
\begin{proposition}
    \label{thm:bound-on-biased-optimum}
    (informal) Under reasonable assumptions, the marginal likelihood $\c{L}$ of the hyperparameters obtained by maximising the objective implied by the warm-started gradients $\tilde{\v{\theta}}^*$ will converge in probability to the marginal likelihood of a true maximum $\v{\theta}^*$,
    \begin{equation}
    \c{L}(\tilde{\v{\theta}}^*) \overset{p}{\to} \c{L}(\v{\theta}^*) \quad \text{as} \quad s \to \infty.
    \end{equation}
\end{proposition}
See \citet[Appendix A]{lin2024improving} for a formal proof and further details.
In practice, a small number of samples seems to be sufficient.

\begin{figure}[ht]
\centering
\includegraphics[width=6in]{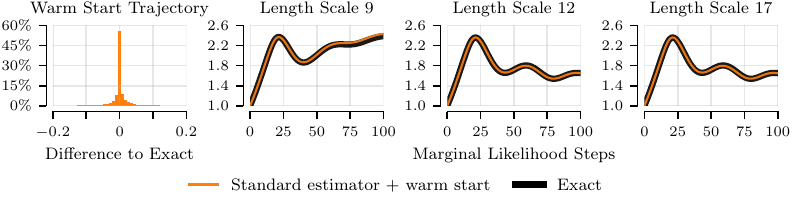}
\caption{Across marginal likelihood steps and datasets, warm starting results in hyperparameter trajectories which barely differ from exact optimisation, as shown by the histogram (left). On the length scales from \Cref{fig:pathwise_exact_prior}, warm starting matches exact optimisation (right).}
\label{fig:warm_start_trajectory}
\end{figure}

\subsection{Warm Starting the Pathwise Estimator}
One advantage of the pathwise estimator from \Cref{sec:pathwise_estimator} is the reduced RKHS distance between the origin and the solution of the quadratic objective.
However, when warm starting, the inner-loop solver no longer initialises at the origin.
Therefore, one may be concerned that we lose this advantage, but this is not the case empirically.
As shown in \Cref{tab:results_main}, combining both techniques further accelerates AP and SGD, reaching $72.1\times$ and $7.2\times$ average speed-ups across our datasets relative to the standard estimator without warm starting.

\begin{table}[h]
\caption{Test log-likelihoods, total training times, and average speed-up among datasets for CG, AP, and SGD after $100$ outer-loop marginal likelihood steps using Adam with a learning rate of $0.1$. We consider five UCI regression datasets with $n<50$k, which allows us to solve to tolerance, and report the mean over $10$ data splits. All hyperparameters are initialised at $1$.}
\label{tab:results_main}
\vspace{0.2cm}
\centering
\small
\setlength{\tabcolsep}{4pt}
\begin{tabular}{l c c | c c c c c | c c c c c | r}
\toprule
& \scriptsize{path} & \scriptsize{warm} & \multicolumn{5}{c|}{Test Log-Likelihood} & \multicolumn{5}{c|}{Total Time (min)} & \multicolumn{1}{c}{Average} \\
& \scriptsize{wise} & \scriptsize{start} & \textsc{pol} & \textsc{elev} & \textsc{bike} & \textsc{prot} & \textsc{kegg} & \textsc{pol} & \textsc{elev} & \textsc{bike} & \textsc{prot} & \textsc{kegg} & \multicolumn{1}{c}{Speed-Up} \\
\midrule
\multirow{4}{*}{\rotatebox[origin=c]{90}{CG}} & & & 1.27 & -0.39 & 2.15 & -0.59 & 1.08 & 4.83 & 1.58 & 5.08 & 29.9 & 28.0 & --- \\
 & \checkmark & & 1.27 & -0.39 & 2.07 & -0.62 & 1.08 & 3.96 & 1.49 & 4.41 & 20.0 & 26.4 & \textbf{1.2} $\times$ \\
 & & \checkmark & 1.27 & -0.39 & 2.15 & -0.59 & 1.08 & 2.28 & 1.03 & 2.74 & 11.5 & 12.8 & \textbf{2.1} $\times$ \\
 & \checkmark & \checkmark & 1.27 & -0.39 & 2.06 & -0.62 & 1.08 & 2.47 & 1.00 & 3.07 & 13.7 & 13.0 & \textbf{1.9} $\times$ \\
\midrule
\multirow{4}{*}{\rotatebox[origin=c]{90}{AP}} & & & 1.27 & -0.39 & 2.15 & -0.59 & --- & 493. & 77.8 & 302. & 131. & > 24 h & --- \\
 & \checkmark & & 1.27 & -0.39 & 2.07 & -0.62 & 1.08 & 27.9 & 1.67 & 19.9 & 16.4 & 211. & > \textbf{5.4} $\times$ \\
 & & \checkmark & 1.27 & -0.39 & 2.15 & -0.59 & 1.08 & 44.0 & 36.4 & 35.1 & 55.8 & 491. & > \textbf{18.9} $\times$ \\
 & \checkmark & \checkmark & 1.27 & -0.39 & 2.06 & -0.62 & 1.08 & 3.90 & 1.21 & 5.40 & 12.3 & 14.0 & > \textbf{72.1} $\times$ \\
\midrule
\multirow{4}{*}{\rotatebox[origin=c]{90}{SGD}} & & & 1.27 & -0.39 & 2.15 & -0.59 & 1.08 & 139. & 5.54 & 412. & 75.2 & 620. & --- \\
 & \checkmark & & 1.27 & -0.39 & 2.07 & -0.63 & 1.08 & 73.6 & 4.58 & 156. & 24.0 & 412. & \textbf{2.1} $\times$ \\
 & & \checkmark & 1.27 & -0.39 & 2.15 & -0.59 & 1.08 & 26.5 & 1.22 & 74.3 & 11.2 & 168. & \textbf{5.1} $\times$ \\
 & \checkmark & \checkmark & 1.27 & -0.39 & 2.06 & -0.62 & 1.07 & 17.9 & 1.14 & 64.2 & 11.9 & 58.7 & \textbf{7.2} $\times$ \\
\bottomrule
\end{tabular}
\end{table}

When warm starting, the right-hand sides of the linear system must not be resampled.
In this case, $\v{f}_\m{X}$ and $\v{\eps}$ are sampled once and fixed afterwards for each $\v{\xi}_j$.
However, $\v{f}_\m{X}$ depends on kernel hyperparameters $\v{\vartheta}$ and $\v{\eps}$ depends on noise scale $\sigma$, which change during each outer-loop step.
Therefore, what does it mean to sample $\v{f}_\m{X}$ and $\v{\eps}$ once and keep them fixed afterwards?
For $\v{f}_\m{X}$, we keep the parameters of the random features fixed.
For $\v{\eps}$, we apply the reparametrisation $\v{\eps} = \sigma \v{w}$, where $\v{w} \~[N](\v{0}, \m{I})$ is sampled once and fixed afterwards, such that $\v{\eps}$ becomes deterministic.
This corresponds to a particular instance of a prior sample, although the distribution of the sample can change due to changes in the hyperparameters.
In each outer-loop step, the random features are recomputed using the fixed random feature parameters and the updated kernel hyperparameters, and the prior function sample is then evaluated at the training data using the updated random features.
Both of these operations take $\mathcal{O}(nm)$ time and are efficient as long as the number of random features $m$ is reasonable.

\section{Solving Linear Systems on a Limited Compute Budget}
\label{sec:compute_budget}
Our experiments so far have only considered relatively small datasets with $n {\,<\,} 50$k, such that inner-loop solvers can reach the tolerance in a reasonable amount of time.
However, on large datasets, reaching a low relative residual norm can become computationally infeasible.
Instead, linear system solvers are commonly given a limited compute budget.
\citet{gardner18} limit the number of CG iterations to $20$, \citet{wu2024large} use $11$ epochs of AP, \citet{antoran2023sampling} run SGD for $50$k iterations, and \citet{lin2023sampling,lin2024stochastic} run SGD for $100$k iterations.
While effective for managing computational costs, it is not well understood how early stopping affects different solvers and marginal likelihood optimisation.
Furthermore, it is also unclear whether a certain tolerance is required to obtain good predictions. 

\begin{figure}[t]
    \centering
    \includegraphics[width=6in]{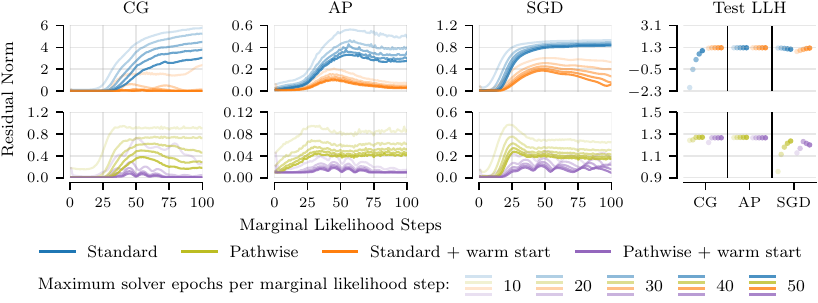}
    \caption{Relative residual norms of the probe vector linear systems at each marginal likelihood step on the \textsc{pol} dataset when solving until the tolerance or a maximum number of solver epochs is reached. Increasing the compute budget generally reduces the residual norm. Given the same compute budget, the pathwise estimator reaches lower residual norms than the standard estimator. Adding warm starts further reduces the residual norm for both estimators. However, the final test log-likelihood does not always match the residual norm. Surprisingly, good predictive performance can be obtained even if the residual norm is much higher than the tolerance $\tau = 0.01$.}
    \label{fig:residual_norm_main}
\end{figure}

\subsection{The Effects of Early Stopping}
We repeat the experiments from \Cref{tab:results_main} but introduce limited compute budgets: $10$, $20$, $30$, $40$ or $50$ solver epochs, where one epoch refers to computing each value in $\m{H}_{\v{\theta}}$ once.
This means that, for CG, one epoch corresponds to one iteration, whereas, for AP and SGD, it depends on the block or batch size, respectively.
Since kernel function evaluations dominate the computational costs of linear system solvers, this results in similar time budgets across methods while preventing compute wastage as a result of time-based stopping.
In this setting, linear system solvers terminate upon either reaching the relative residual norm tolerance or when the compute budget is exhausted, whichever occurs first.

In \Cref{fig:residual_norm_main}, we illustrate the relative residual norms reached for each compute budget on the \textsc{POL} dataset.
In general, the residual norms increase as $\m{H}_{\v{\theta}}$ becomes more ill-conditioned during optimisation, and as the compute budget is decreased.
The increase in residual norms is much larger for CG than the other solvers, which is consistent with previous reports of CG not being amenable to early stopping \citep{lin2023sampling}.
AP seems to behave slightly better than SGD under a limited compute budget.
Both the pathwise estimator and warm starting combine well with early stopping, reaching lower residual norms when using a budget of 10 solver epochs than the standard estimator without warm starting using a budget of 50 epochs. 

In terms of predictive performance, we see that CG with the standard estimator and no warm starting suffers the most from early stopping.
Changing to the pathwise estimator and warm starting recovers good performance most of the time.
SGD also shows some sensitivity to early stopping, but there seems to be a stronger correlation between invested compute and final performance.
Surprisingly, AP generally achieves good predictive performance even on the smallest compute budget, despite not reaching the tolerance of $\tau = 0.01$.
Overall, the relationship between reaching a low residual norm and obtaining good predictive performance seems to be weak.
This is an unexpected yet interesting observation, and future research should investigate the suitability of the relative residual norm as a metric to determine solver convergence.

\subsection{Demonstration on Large Datasets}
After analysing early stopping on small datasets, we now turn to evaluation on larger UCI datasets $391$k $<$ $n$ $<$ $1.8$M \citep{Dua2019UCI}, where solving until reaching the tolerance becomes computationally infeasible.
Therefore, we introduce a compute budget of $10$ solver epochs per marginal likelihood step.
Hyperparameters are initialised with the heuristic from the previous chapters \citep{lin2023sampling} and optimised using a learning rate of $0.03$ for $30$ Adam steps ($15$ for \textsc{houseelectric} due to high computational costs).
We use the pathwise estimator because it accelerates solver convergence (see \Cref{sec:pathwise_estimator}), and because it enables efficient tracking of predictive performance during optimisation.

To enforce positive value constraints during hyperparameter optimisation, we reparametrise each hyperparameter $\theta_i \in \mathbb{R}_{>0}$ as $\theta_i = \log(1 + \exp (\nu_k))$ and apply outer-loop optimiser steps to $\nu_k \in \mathbb{R}$, to facilitate unconstrained optimisation.
Additionally, to improve numerical stability, the relative residual norm tolerance is implemented by solving the rescaled system $\m{H}_{\v{\theta}}\, \v{\tilde{u}} = \v{\tilde{b}}$, where $\v{\tilde{b}} = \v{b} / (\norm{\v{b}}_2 + \epsilon)$, until $\norm{\bm{\tilde{r}}}_2 = \Vert\v{\tilde{b}} - \m{H}_{\v{\theta}} \v{\tilde{u}}\Vert_2 \leq \tau = 0.01$ and then returning $\v{u} = (\norm{\v{b}}_2 + \epsilon) \, \v{\tilde{u}}$, where $\epsilon$ is set to a small constant value to prevent division by zero.
Furthermore, since we are solving batches of systems of linear equations of the form $\m{H}_{\v{\theta}} \, \left[ \, \v{v}_{\v{y}}, \v{v}_1, \dots, \v{v}_s \, \right] = \left[ \, \v{y}, \v{z}_1, \dots, \v{z}_s \, \right]$, we track the residuals of each individual system and calculate separate residual norms for the mean and for the probe vectors, where the residual norm for the mean $\norm{\v{r}_{\v{y}}}_2$ corresponds to the system $\m{H}_{\v{\theta}} \, \v{v}_{\v{y}} = \v{y}$ and the residual norm for the probe vectors $\norm{\v{r}_{\v{z}}}_2$ is defined as the arithmetic average over residual norms corresponding to the systems $\m{H}_{\v{\theta}} \, \left[ \, \v{v}_1, \dots, \v{v}_s \, \right] = \left[ \, \v{z}_1, \dots, \v{z}_s \, \right]$.
Both residual norms must reach the tolerance $\tau$ to satisfy the termination criterion.
We use separate residual norms because $\norm{\v{r}_{\v{y}}}_2$ typically converges faster than $\norm{\v{r}_{\v{z}}}_2$, such that an average other all systems tends to dilute the latter.

For CG, we use a pivoted Cholesky preconditioner of rank $100$ in all experiments, following previous work \citep{WangPGT2019exactgp}.
For AP, we use a block size of $b = 1000$ for all datasets, except \textsc{protein} and \textsc{keggdirected}, where we use $b = 2000$ instead.
Following \citet{wu2024large}, we cache the Cholesky factorisation of every block and select the block with the largest residual norm during each AP iteration.
For SGD, we estimate the current residual by updating it sparsely whenever we compute the gradient on a batch of data, leveraging the property that the negative gradient is equal to the residual.
In practice, we find that this estimates an approximate upper bound on the true residual, which becomes fairly accurate after a few iterations.
Furthermore, we use a batch size of $b = 500$, momentum of $\rho = 0.9$, and no Polyak averaging, because averaging is not strictly necessary \citep{lin2024stochastic} and would interfere with our residual estimation heuristic.
We use learning rates of $30$, $20$, $30$, $20$, and $20$ for the \textsc{pol}, \textsc{elevators}, \textsc{bike}, \textsc{keggdirected} and \textsc{protein} datasets, respectively, picking the largest learning rate from a grid $[5, 10, 20, 30, 50, 60, 70, 80, 90, 100]$ that does not cause the inner linear system solver to diverge on the very first outer marginal likelihood loop.
For the larger datasets, we use learning rates of $10$, $50$, and $50$ for \textsc{buzz}, \textsc{song} and \textsc{houseelectric}, respectively, picking half of the largest learning rate as above.
We find that the larger datasets are more sensitive to diverging when the hyperparameters change, and therefore we choose half of the largest learning rate possible at initialisation.
All solvers are initialised either at zero (no warm start) or at the previous solution (warm start).

\begin{figure}[t]
    \centering
    \includegraphics[width=6in]{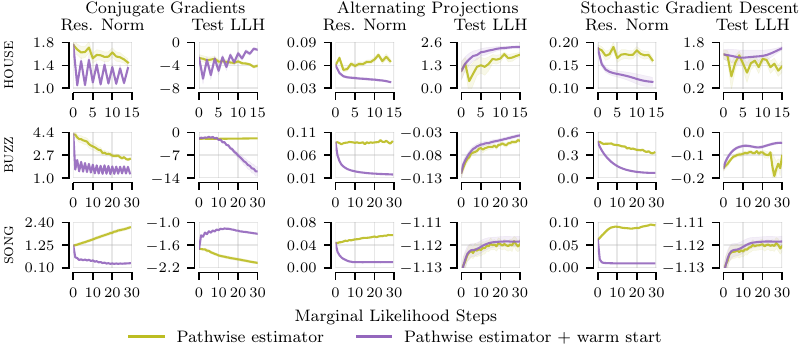}
    \caption{Relative residual norms and test log-likelihoods during marginal likelihood optimisation on large datasets using the pathwise estimator. Warm starting allows solver progress to accumulate over multiple marginal likelihood steps, leading to decreasing residual norms. Without warm starting, residual norms tend to remain similar or increase throughout marginal likelihood steps. Despite reaching significantly lower residual norms, the predictive performance does not always improve, akin to \Cref{fig:residual_norm_main}.}
    \label{fig:large_datasets_main}
\end{figure}

\Cref{fig:large_datasets_main} visualises the evolution of the relative residual norm of the probe vector linear systems and the predictive test log-likelihood during marginal likelihood optimisation.
For all solvers, warm starting leads to lower residual norms throughout outer-loop steps.
This suggests a synergistic behaviour between early stopping and warm starting: the latter allows solver progress to accumulate across marginal likelihood steps, which can be interpreted as amortising the inner-loop linear system solve over multiple outer-loop steps.
Despite the lower residual norm, CG is brittle under early stopping, obtaining significantly worse performance than AP and SGD on \textsc{buzz} and \textsc{houseelectric}. 
AP and SGD seem to be more robust to early stopping.
However, lower residual norms do not always translate to improved predictive performance.
Furthermore, we find that SGD can suffer from the fact that the optimal learning rate changes as the hyperparameters change.

\section{Discussion}
\label{sec:conclusion}
Building on a hierarchical view of marginal likelihood optimisation, this chapter consolidated several iterative GP techniques into a common framework, analysing them and showing their applicability across different linear system solvers.
Overall, these provided speed-ups of up to $72\times$ when solving until a specified tolerance is reached, and decreased the average relative residual norm by up to $7\times$ under a limited compute budget.
Additionally, our analyses led to the following findings:
Firstly, the pathwise gradient estimator accelerates linear system solvers by moving solutions closer to the origin, and also provides amortised predictions as an added benefit by turning probe vectors into posterior samples via pathwise conditioning.
Secondly, warm starting solvers at previous solutions during marginal likelihood optimisation reduces the number of solver iterations to tolerance at the cost of introducing negligible bias into the optimisation trajectory.
Furthermore, warm starting combines well with pathwise gradient estimation.
Finally, stopping iterative linear system solvers after exhausting a limited compute budget generally increases the relative residual norm.
However, when paired with warm starting, solver progress accumulates, amortising inner-loop linear system solves over multiple outer-loop steps.
Nonetheless, we also observed that low relative residual norms are not always necessary to obtain good predictive performance, which presents an interesting avenue for future research.

%!TEX root = ../thesis.tex
%*******************************************************************************
%****************************** Sixth Chapter **********************************
%*******************************************************************************
\chapter[Scalable Gaussian Processes with Latent Kronecker Structure]{Scalable Gaussian Processes\\with Latent Kronecker Structure}
\label{chap:lkgp}

% **************************** Define Graphics Path **************************
\ifpdf
    \graphicspath{{Chapter6/Figs/Raster/}{Chapter6/Figs/PDF/}{Chapter6/Figs/}}
\else
    \graphicspath{{Chapter6/Figs/Vector/}{Chapter6/Figs/}}
\fi

After proposing generic improvements for iterative linear system solvers in the previous chapter, this chapter shifts the focus to the effective combination of iterative linear system solvers with Kronecker product-structured kernel matrices.
The latter enable scalable inference via factorised matrix decompositions, but usually require fully gridded data from a Cartesian product space to be applicable.
To lift this limitation, we propose \emph{latent} Kronecker structure, expressing the covariance matrix of observed values as the projection of a latent Kronecker product.
In combination with iterative linear system solvers and pathwise conditioning, this substantially reduces computational time and memory requirements compared to standard iterative methods without latent Kronecker structure.
Empirically, we demonstrate our method on real-world datasets with up to five million examples.

This chapter includes content which is adapted from the following publications:
\begin{itemize}
    \item J. A. Lin, S. Ament, M. Balandat, and E. Bakshy. Scaling Gaussian Processes for Learning Curve Prediction via Latent Kronecker Structure. In \emph{NeurIPS Bayesian Decision-making and Uncertainty Workshop}, 2024.
    \item J. A. Lin, S. Ament, M. Balandat, D. Eriksson, J. M. Hernández-Lobato, and E. Bakshy. Scalable Gaussian Processes with Latent Kronecker Structure. In \emph{International Conference on Machine Learning}, 2025. 
\end{itemize}

My contributions to these projects consist of developing the ideas, implementing the source code, conducting the experiments, deriving the theoretical results, and contributing to writing the manuscripts.

\section{Introduction}
Gaussian processes (GPs) are probabilistic models prized for their flexibility, data efficiency, and well-calibrated uncertainty estimates.
They play a key role in many applications such as Bayesian optimisation \citep{garnett2023bobook}, reinforcement learning \citep{PILCO}, and active learning \citep{riis2022bayesian}.
However, exact GPs are notoriously challenging to scale to large numbers of training examples $n$.
This primarily stems from having to solve an $n \times n$ linear system involving the kernel matrix to compute the marginal likelihood and the posterior, which has $\c{O}(n^3)$ time complexity using direct methods.

To address these scalability challenges, a plethora of approaches have been proposed, which usually fall into two main categories.
\emph{Sparse} GP approaches reduce the size of the underlying linear system by introducing a set of \emph{inducing points} to approximate the full GP; they include conventional \citep{candela2005} and variational \citep{titsias09,titsias2009report,hensman13} formulations. 
\emph{Iterative} methods \citep{gardner18,WangPGT2019exactgp} employ solvers such as the linear conjugate gradient method, often leveraging hardware parallelism and structured kernel matrices.

Kronecker product structure permits efficient matrix-vector multiplication (MVM).
This structure arises when a product kernel is evaluated on data which can be expressed as a Cartesian product.
An example of this is spatio-temporal data, as found in climate science for instance, where a quantity of interest, such as temperature or humidity, varies across space and is measured at regular time intervals.
In particular, with observations collected at $p$ locations and $q$ times, resulting in $pq$ data points in total, the kernel matrix generated by evaluating a product kernel on this data can be expressed as the Kronecker product of two smaller $p \times p$ and $q \times q$ matrices.
This allows MVM in $\c{O}(p^2q + pq^2)$ instead of $\c{O}(p^2q^2)$ time and $\c{O}(p^2+q^2)$ instead of $\c{O}(p^2q^2)$ space, such that iterative methods require substantially less time and memory. 
However, in real-world data, this structure frequently breaks due to missing observations, typically requiring a reversion to generic sparse or iterative methods that either do not take advantage of special structure or entail approximations to the matrix.

In this chapter, we propose \emph{latent Kronecker structure}, which enables efficient inference despite missing values by representing the joint covariance matrix of observed values as a lazily-evaluated product between projection matrices and a latent Kronecker product.
Combined with iterative methods, which otherwise would not be able to leverage Kronecker structure in the presence of missing values, this reduces the asymptotic time complexity from $\c{O}(p^2 q^2)$ to $\c{O}(p^2q + pq^2)$, and space complexity from $\c{O}(p^2 q^2)$ to $\c{O}(p^2 + q^2)$.
Importantly, unlike sparse approaches, our method does not introduce approximations of the GP prior.
This avoids common problems of sparse GPs, such as underestimating predictive variances \citep{jankowiak20ppgpr}, and overestimating noise \citep{titsias09,bauer2016understanding} due to limited model expressivity with a constant number of inducing points.

In contrast, our method, Latent Kronecker GP (LKGP), facilitates highly scalable inference of \emph{exact} GP models with product kernels.
We empirically demonstrate the superior computational scalability, speed, and modelling performance of our approach on various real-world applications including inverse dynamics prediction for robotics, learning curve prediction, and climate modelling, using datasets with up to five million data points.

\section{Latent Kronecker Structure}
\label{sec:method}
In this chapter, we consider a Gaussian process $f: \c{X} \to \R$ defined on a Cartesian product space $\c{X} = \c{S} \times \c{T}$.
For illustrative purposes, $\c{S}$ could be associated with spatial dimensions and $\c{T}$ could refer to a time dimension or an integer task index, as is common in multi-task GPs \citep{bonilla2007multi}, but we emphasise that $\c{T}$ is not limited to one-dimensional spaces.
A natural way to model this problem is to define a kernel $k_{\c{X}}$ on the product space $\c{X}$, which results in a joint covariance
\begin{equation}
    \mathrm{Cov}(f(\v{x}), f(\v{x}'))
    = k_{\c{X}}(\v{x}, \v{x}')
    = k_{\c{X}}((\v{s}, \v{t}), (\v{s}', \v{t}')),
\end{equation}
where $\v{x}, \v{x}' \in \c{X}$, $\v{s}, \v{s}' \in \c{S}$, and $\v{t}, \v{t}' \in \c{T}$.
However, this results in scalability issues when performing GP regression.
If our training data consists of $n$ outputs $\{ y_i \}_{i=1}^n = \v{y} \in \R^n$ observed at $p$ spatial locations $\{ \v{s}_j \}_{j=1}^p = \m{S} \subset \c{S}$ and $q$ time steps or tasks $\{ \v{t}_k \}_{k=1}^q = \m{T} \subset \c{T}$, such that $n = pq$, then the joint covariance matrix requires $\c{O}(p^2 q^2)$ space and computing its Cholesky factor takes $\c{O}(p^3 q^3)$ time.

\subsection{Ordinary Kronecker Structure}
A common way to reduce the computational complexity is to introduce product kernels and Kronecker structure \citep{bonilla2007multi,stegle2011efficient,zhe19scalable}.
Defining
\begin{equation}
    k_{\c{X}}(\v{x}, \v{x}') = k_{\c{X}}((\v{s}, \v{t}), (\v{s}', \v{t}')) = k_{\c{S}}(\v{s}, \v{s}') \ k_{\c{T}}(\v{t}, \v{t}'), 
\end{equation}
where $k_{\c{S}}$ only considers spatial locations $\v{s}_j$ and $k_{\c{T}}$ solely acts on time steps or tasks $\v{t}_k$, allowing the joint covariance matrix $\m{K}_{\m{X}\m{X}}$ to factorise as
\begin{equation}
    \underset{n \times n}{\m{K}_{\m{X}\m{X}}}
    = k_{\c{S}}(\m{S}, \m{S}) \otimes k_{\c{T}}(\m{T}, \m{T})
    = \underset{p \times p}{\m{K}_{\m{S} \m{S}}}
    \otimes 
    \underset{q \times q}{\m{K}_{\m{T} \m{T}}},
\end{equation}
which can be exploited by expressing decompositions of $\m{K}_{\m{X}\m{X}}$ in terms of decompositions of $\m{K}_{\m{S}\m{S}}$ and $\m{K}_{\m{T}\m{T}}$ instead.
This reduces the asymptotic time complexity to $\c{O}(p^3 + q^3)$ and space complexity to $\c{O}(p^2 + q^2)$ (see \Cref{sec:structured_kernel_matrices} for details).

However, the joint covariance matrix $\m{K}_{\m{X}\m{X}}$ only exhibits the Kronecker structure if observations $y_i$ are available for each time step or task $\v{t}_k$ at every spatial location $\v{s}_j$, and this is often not the case.
For example, suppose we are considering temperatures $y_i$ on different days $\v{t}_k$ measured by various weather stations at locations $\v{s}_j$.
If a single weather station does not record the temperature on any given day, then the resulting observations are no longer fully gridded, and thus ordinary Kronecker methods cannot be applied.
To clarify, these \emph{missing values} refer to missing output observations which correspond to inputs from a Cartesian product space $\c{X} = \c{S} \times \c{T}$.
They do not refer to \emph{missing features} in the input data, and imputing missing values is not the main goal of our contribution.
Instead, we focus on improving scalability by proposing a method which facilitates the use of Kronecker products in the presence of missing values.

\subsection{Dealing with Missing Values}
In the context of gridded data with missing values, such that $n < pq$, a key insight is that the joint covariance matrix over observed values is a submatrix of the joint covariance matrix over the whole Cartesian product grid.
Although the former generally does not have Kronecker structure, the latter does. We leverage this \emph{latent} Kronecker structure in the latter matrix by expressing the joint covariance matrix over observed values as
\begin{equation}
    \underset{n \times n}{\m{K}_{\m{X}\m{X}}}
    = \underset{n \times pq}{\m{P}_{}} \; ( \underset{p \times p}{\m{K}_{\m{S}\m{S}}} \otimes \underset{q \times q}{\m{K}_{\m{T}\m{T}}} ) \; \underset{pq \times n}{{\m{P}_{}}\T},
\end{equation}
where the projection matrix $\m{P}$ is constructed from the identity matrix by removing rows which correspond to missing values.
See \Cref{fig:latent_kronecker} for an illustration.
Crucially, this is not an approximation and facilitates inference in the exact GP model.
In practice, we implement projections efficiently without explicitly instantiating or multiplying by $\m{P}$.

\begin{figure}[t]
\centering
\includegraphics[width=6in]{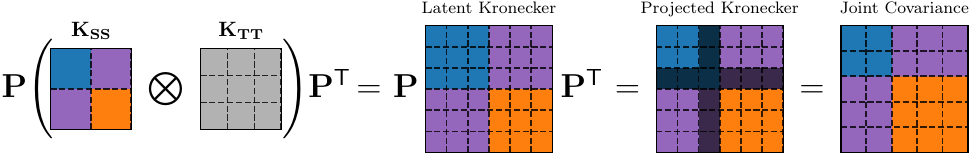}
\caption{The joint covariance matrix over $\{ \textcolor{tab:blue}{(\v{s}_1, \v{t}_1)}, \textcolor{tab:blue}{(\v{s}_1, \v{t}_2)}, \textcolor{tab:orange}{(\v{s}_2, \v{t}_1)}, \textcolor{tab:orange}{(\v{s}_2, \v{t}_2)}, \textcolor{tab:orange}{(\v{s}_2, \v{t}_3)} \}$, namely two out of three time steps at spatial location $\textcolor{tab:blue}{\v{s}_1}$ and three out of three time steps at spatial location $\textcolor{tab:orange}{\v{s}_2}$, can be expressed as the projection of a latent Kronecker product.}
\label{fig:latent_kronecker}
\end{figure}

\subsection{Efficient Inference via Iterative Methods}
As a result of introducing the projections, the eigenvalues and eigenvectors of $\m{K}_{\m{X} \m{X}}$ cannot be expressed in terms of eigenvalues and eigenvectors of $\m{K}_{\m{S} \m{S}}$ and $\m{K}_{\m{T} \m{T}}$ anymore, which prevents the use of Kronecker structure for efficient matrix factorisation.
However, despite the projections, the structure can still be leveraged for fast matrix multiplication.
Therefore, we augment the standard Kronecker product equation,
$(\m{A} \otimes \m{B}) \, \mathrm{vec}(\m{C}) = \mathrm{vec}(\m{B} \m{C} \m{A}\T)$,
with the aforementioned projections, yielding
\begin{equation}
    \m{P} (\m{A} \otimes \m{B}) \m{P}\T \mathrm{vec}(\m{C})
    = \m{P} \mathrm{vec}(\m{B} \mathrm{vec}\inv (\m{P}\T \mathrm{vec}(\m{C})) \m{A}\T).
\end{equation}
In practice, $\mathrm{vec}$ and $\mathrm{vec}\inv$ are implemented as reshaping operations, $\m{P}\T \, \mathrm{vec}(\m{C})$ amounts to zero padding, and left-multiplying by $\m{P}$ is slice indexing, all of which typically do not incur significant computational overheads.
This facilitates efficient inference in the exact GP model via iterative methods, which only rely on 
MVM to compute solutions of linear systems (see \Cref{sec:iterative_methods}).
Leveraging the latent Kronecker structure reduces the asymptotic time complexity of MVM from $\c{O}(n^2)$ to $\c{O}(p^2q + pq^2)$, and the asymptotic space complexity from $\c{O}(n^2)$ to $\c{O}(p^2 + q^2)$, or from $\c{O}(n)$ to $\c{O}(p + q)$ when using lazy kernel evaluations.

\subsection{Posterior Samples via Pathwise Conditioning}
In the context of product kernels and gridded data, \citet{maddox2021highdimout} proposed to draw posterior samples via pathwise conditioning \citep{wilson20,wilson21} to exploit Kronecker structure, reducing the asymptotic time complexity from $\c{O}(p^3q^3)$ to $\c{O}(p^3 + q^3)$.
In our product space notation, the pathwise conditioning equation can be written as
\begin{equation}
    f_{({\.}_{\v{s}}, {\.}_{\v{t}}) 
    \given \v{y}} = f_{({\.}_{\v{s}}, {\.}_{\v{t}})}
    + ( \m{K}_{({\.}_{\v{s}})\m{S}} \otimes \m{K}_{({\.}_{\v{t}})\m{T}})
    (\m{K}_{\m{S}\m{S}} \otimes \m{K}_{\m{T} \m{T}} + \sigma^2 \m{I})\inv(\v{y} - (\v{f}_{\m{S} \times \m{T}} + \v\eps)),
\end{equation}
where $f_{({\.}_{\v{s}}, {\.}_{\v{t}})}$ is a GP prior sample, $\v{f}_{\m{S} \times \m{T}}$ is its evaluation at $\m{X} = \m{S} \times \m{T}$, and $\v\eps \sim \c{N}(\v{0}, \sigma^2 \m{I})$.

To support latent Kronecker structure, we introduce projections $\m{P}$ and $\m{P}\T$, such that
\begin{equation}
    f_{({\.}_{\v{s}}, {\.}_{\v{t}}) 
    \given \v{y}} = f_{({\.}_{\v{s}}, {\.}_{\v{t}})}
    + ( \m{K}_{({\.}_{\v{s}})\m{S}} \otimes \m{K}_{({\.}_{\v{t}})\m{T}}) \m{P}\T
    (\m{P} ( \m{K}_{\m{S}\m{S}} \otimes \m{K}_{\m{T} \m{T}} ) \m{P}\T + \sigma^2 \m{I})\inv(\v{y} - (\m{P} \v{f}_{\m{S} \times \m{T}} + \v\eps)),
\end{equation}
resulting in exact samples from the exact GP posterior.
In combination with iterative linear system solvers, latent Kronecker structure can be leveraged for fast matrix multiplication.

\begin{figure}[t]
\centering
\includegraphics[width=6in]{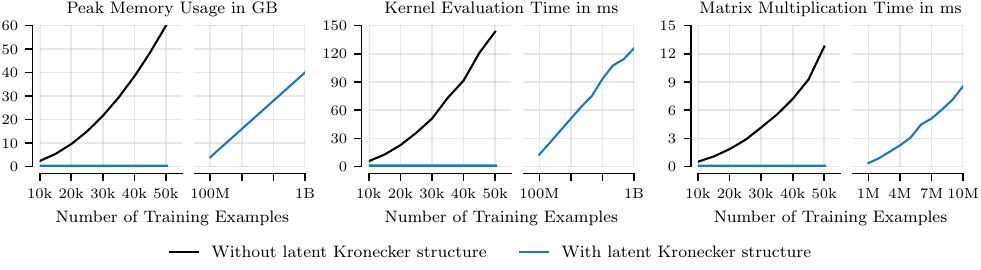}
\caption{Illustration of computational resources used during kernel evaluation and matrix multiplication on ten-dimensional synthetic datasets of different sizes. Without latent Kronecker structure, memory usage escalates quickly and kernel evaluation time dominates matrix multiplication time asymptotically. With latent Kronecker structure, computations can be scaled to several orders of magnitude larger datasets under similar computational resource usage, and matrix multiplication time dominates kernel evaluation time asymptotically. Reported results are the mean over one hundred repetitions using a squared exponential kernel.}
\label{fig:synthetic}
\end{figure}

\subsection{Discussion of Computational Benefits}
The most apparent benefits of using latent Kronecker structure are the improved asymptotic time and space complexities of matrix multiplication, as discussed earlier.
However, there are further, more subtle benefits. In particular, the number of kernel evaluations changes from $\c{O}(n^2)$ evaluations of $k_{\c{X}}$ to $\c{O}(p^2)$ evaluations of $k_{\c{S}}$ plus $\c{O}(q^2)$ evaluations of $k_{\c{T}}$.
This may seem irrelevant if kernel matrices fit into memory, because, in this case, the cost of kernel evaluations is amortised, and the total runtime is dominated by matrix multiplication.
But, if kernel matrices do not fit into memory, their values must be rematerialised during each matrix multiplication, leading to many repeated kernel evaluations which may dominate the total runtime.
Due to much lower memory requirements, latent Kronecker structure also raises the threshold after which kernel values must be rematerialised.
\Cref{fig:synthetic} illustrates memory usage, kernel evaluation times, and matrix multiplication times on ten-dimensional synthetic datasets of various sizes, assuming a balanced factorisation $p = q = \sqrt{n}$.

\subsection{Efficiency of Latent Kronecker Structure}
Latent Kronecker structure can be understood as padding $n \leq pq$ observations to complete a grid of $pq$ values.
Therefore, an important question is: when does the number of missing values $pq - n$ become large enough such that this padding becomes inefficient and dense matrices without structure are preferable?
\begin{proposition}
\label{thm:break_even_point}
    Let $\gamma = 1 - n / pq$ be the missing ratio, that is, the relative amount of padding required to complete a grid with $pq$ values.
    Let $\gamma^*$ be the asymptotic break-even point, that is, the particular missing ratio at which a kernel matrix without factorisation has the same asymptotic performance as a kernel matrix with latent Kronecker structure.
    The asymptotic break-even points for MVM
    time and memory usage are, respectively,
    \begin{equation}
        \gamma_{\mathrm{time}}^* = 1 - \sqrt{\frac{1}{p} + \frac{1}{q}} \quad \text{and} \quad
        \gamma_{\mathrm{mem}}^* = 1 - \sqrt{\frac{1}{p^2} + \frac{1}{q^2}}.
    \end{equation}
\end{proposition}
\begin{proofbox}
\begin{proof}
Let $\gamma = 1 - n / pq$ be the missing ratio, with $n, p, q > 0$, $n \leq pq$, and $\gamma \in [0, 1)$.
Alternatively, we can also write $n = (1 - \gamma) pq$.
The asymptotic time complexity of MVM of a $n \times n$ matrix with a vector of $n$ elements is $\c{O}(n^2)$ for a generic matrix and $\c{O}(p^2q + pq^2)$ for a Kronecker product of two matrices of sizes $p \times p$ and $q \times q$, respectively.
To find the asymptotic break-even point, we equate these asymptotic time complexities, write $n$ in terms of $\gamma$, $p$, and $q$, and solve the quadratic equation for $\gamma$,
\begin{equation}
    n^2 = \del{(1 - \gamma_{\mathrm{time}})pq}^2 = p^2q + pq^2
    \implies
    \gamma_{\mathrm{time}}^* = 1 \pm \sqrt{\frac{1}{p} + \frac{1}{q}}.
\end{equation}
For the asymptotic space complexity, we follow the same approach starting with $\c{O}(n^2)$ and $\c{O}(p^2 + q^2)$.
Again, we set these terms equal to each other and solve for $\gamma$,
\begin{equation}
    n^2 = \del{(1 - \gamma_{\mathrm{mem}})pq}^2 = p^2 + q^2 \implies
    \gamma_{\mathrm{mem}}^* = 1 \pm \sqrt{\frac{1}{p^2} + \frac{1}{q^2}}.
\end{equation}
Since $p, q > 0$ and $\gamma \in [0, 1)$, we conclude that the asymptotic break-even points are
\begin{equation}
    \gamma_{\mathrm{time}}^* = 1 - \sqrt{\frac{1}{p} + \frac{1}{q}}
    \quad \text{ and } \quad
    \gamma_{\mathrm{mem}}^* = 1 - \sqrt{\frac{1}{p^2} + \frac{1}{q^2}},
\end{equation}
which gives the claim.
\end{proof}
\end{proofbox}
In the next section, we validate these asymptotic break-even points empirically.

\section{Experiments}
\label{sec:lkgp_experiments}
We conduct three distinct experiments to empirically evaluate LKGPs: inverse dynamics prediction for robotics, learning curve prediction in an automated machine learning setting, and modelling and prediction of missing values in spatio-temporal climate data. 
In the first experiment, we compare LKGP against iterative methods without latent Kronecker structure, and in the second and third experiment, we compare LKGP to various sparse and variational methods.
For all our experiments, we start with a fully gridded dataset and introduce missing values which are withheld during training and used as test data.
This allows us to have ground truth information for missing values.
We compute predictive means and variances of LKGP using 64 posterior samples obtained via pathwise conditioning, akin to \citet{lin2023sampling}.

\subsection{Inverse Dynamics Prediction}
In this experiment, we predict the inverse dynamics of a SARCOS anthropomorphic robot arm with seven degrees of freedom.
In this problem, we wish to learn the mapping of seven joint positions, velocities, and accelerations to their corresponding seven joint torques.
The main goal of this experiment is to compare LKGPs to standard iterative methods, because the former effectively implements the latter using more efficient matrix algebra.
Additionally, we empirically validate the break-even points predicted by \Cref{thm:break_even_point}, at which both methods should require the same amount of time or memory.

For this experiment, we consider the positions, velocities, and accelerations as $\c{S}$ and the seven torques as $\c{T}$. 
We choose $k_{\c{S}}$ to be a squared exponential kernel and $k_{\c{T}}$ to be a full-rank ICM kernel \citep{bonilla2007multi}, demonstrating that LKGP is compatible with discrete kernels.
We select subsets of the training data split with $p = 5000$ joint positions, velocities, and accelerations, and their corresponding $q = 7$ joint torques, and introduce $10$\%, $20$\%, ..., $90$\% missing values uniformly at random, resulting in $n \leq 35k$. 
The value of $p$ is chosen such that the kernel matrix without factorisation fits into GPU memory.
This results in a fairer comparison since it allows both methods to amortise kernel evaluation costs, such that the comparison is essentially between matrix multiplication costs.
Larger $p$ would force regular iterative methods to recompute kernel values, leading to higher overall compute times.
For both methods, observation noise and kernel hyperparameters are initialised with GPyTorch default values \citep{gardner18}, and inferred by using Adam with a learning rate of $0.1$ to maximise the marginal likelihood for $50$ iterations.
Additionally, both methods use conjugate gradients with a relative residual norm tolerance of $0.01$ as linear system solver.

\begin{figure}[t]
\centering
\raisebox{0.2in}{
\begin{tikzpicture}
\robotArmBaseLink[world width=2.5]
\robotArm[draw annotations=false,config={q1=120,q2=-100,q3=-80},geometry={a1=1.5,a2=1.5,a3=1}]{3}
\end{tikzpicture}}
\hspace{0.1in}
\includegraphics[width=4.75in]{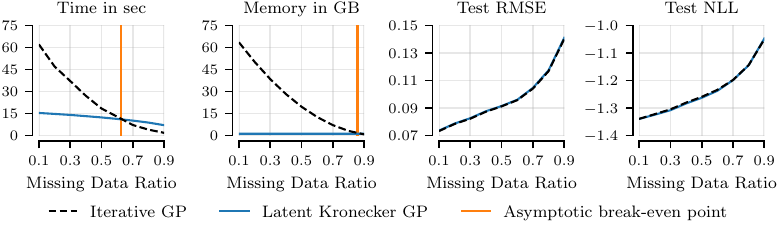}
\caption{Predicting the inverse dynamics of an anthropomorphic robot arm with seven degrees of freedom. Compared to standard iterative methods, leveraging latent Kronecker structure results in significantly lower runtime and memory requirements while maintaining the same predictive performance. The asymptotic break-even point, at which both methods asymptotically require the same amount of time or memory, closely matches the empirical break-even point.
Reported results are the mean over $10$ different random splits of the data.}
\label{fig:sarcos}
\end{figure}

\Cref{fig:sarcos} illustrates the required computational resources and predictive performances for different missing data ratios.
For small missing data ratios, LKGP requires significantly less time and memory.
As the missing ratio increases, the iterative methods eventually become slightly more efficient.
Notably, the empirical break-even points, where time and memory usage is the same for both methods, matches the asymptotic break-even points predicted by \Cref{thm:break_even_point}.
Moreover, the predictive performance of both methods is equivalent across all missing ratios in terms of test root-mean-square-error and test negative log-likelihood.

\subsection{Learning Curve Prediction}
In this experiment, we predict the continuation of partially observed learning curves, given fully observed and other partially observed examples.
We use data from LCBench \citep{ZimLin2021a}, which contains $35$ learning curve datasets, each containing $2000$ learning curves with $52$ steps each, where every step refers to a neural network training epoch.
Each curve within a particular dataset is obtained by training a neural network on the same data but using different hyperparameter configurations.
The main goal of this experiment is to evaluate the performance under a realistic non-uniform pattern of missing values.
In particular, learning curves are observed until a particular time step and missing all remaining values, which simulates neural network hyperparameter optimisation, where learning curve information becomes available as the neural network continues to train.
In this setting, predictions could be leveraged to save computational resources by pruning runs which are predicted to perform poorly \citep{elsken2019nassurvey}. 
We compare our method against SVGP \citep{hensman13}, VNNGP \citep{wu2022variational}, and CaGP \citep{wenger2024computation} (see \Cref{sec:sparse_and_variational} for details).

\begin{figure}[t]
\centering
\includegraphics[width=6in]{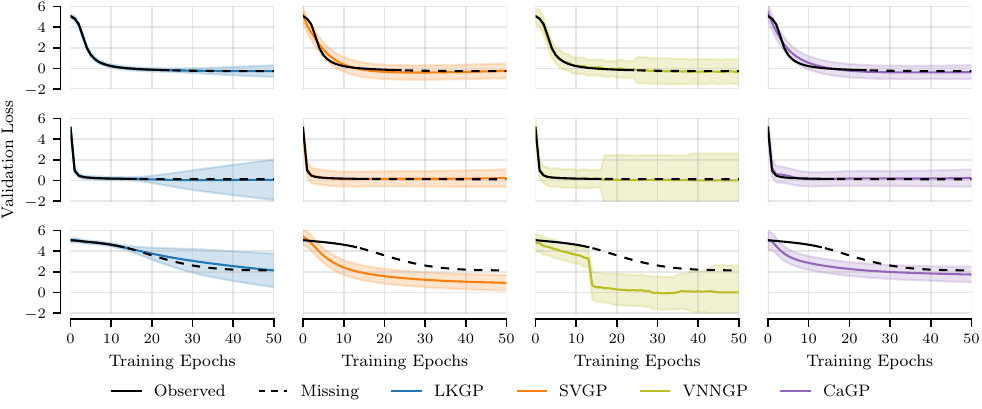}
\caption{Learning curve prediction on the Fashion-MNIST dataset from the LCBench benchmark \citep{ZimLin2021a}. Partially observed learning curves are extrapolated into the future. The predictive mean and two standard deviations of various GP models are visualised. All models are reasonably accurate in the mean predictions, but LKGP produces the most sensible uncertainty estimates, starting with low uncertainty in the observed part of the learning curve and gradually increasing the predicted uncertainty into the missing part. The third row illustrates an outlier with significantly different behaviour than most other learning curves. The sparse methods struggle here because the limited number of inducing points is unlikely to cover such a case. LKGP can adapt because it performs inference in the exact GP.}
\label{fig:lcbench_illustration}
\end{figure}
\begin{table*}[t]
\caption{[Learning Curve Prediction] Predictive performances and total runtimes of learning curve prediction on every fifth LCBench dataset. Compared to the baselines, LKGP produces the best negative log-likelihood on average while also requiring the least time. Reported numbers are the mean over $10$ random seeds. Results of the best model are boldfaced and results of the second best model are underlined.}
\label{tab:lcbench_results_main}
\vspace{0.1in}
\centering
\setlength{\tabcolsep}{3pt}
\small
\begin{adjustbox}{max width=\textwidth}
\begin{tabular}{ l@{\hspace{3.5pt}}l l | c c c c c c c | c}
\toprule
& & Model & APSFailure & MiniBooNE & blood & covertype & higgs & kr-vs-kp & segment & Average Rank \\
\midrule
\multirow{4}{*}{\rotatebox[origin=c]{90}{Train}} & \multirow{4}{*}{\rotatebox[origin=c]{90}{RMSE}} & LKGP & \textbf{1.705} $\pm$ \textbf{0.047} & \textbf{0.047} $\pm$ \textbf{0.001} & \textbf{0.278} $\pm$ \textbf{0.017} & \textbf{0.061} $\pm$ \textbf{0.001} & \textbf{0.039} $\pm$ \textbf{0.001} & \textbf{0.008} $\pm$ \textbf{0.000} & \textbf{0.006} $\pm$ \textbf{0.000} & \textbf{1.171} $\pm$ \textbf{0.560} \\
 & & SVGP & \underline{3.893} $\pm$ \underline{0.149} & 0.201 $\pm$ 0.004 & 0.463 $\pm$ 0.009 & 0.240 $\pm$ 0.003 & 0.200 $\pm$ 0.003 & 0.101 $\pm$ 0.001 & 0.086 $\pm$ 0.001 & 2.771 $\pm$ 0.680 \\
 & & VNNGP & 4.920 $\pm$ 0.194 & \underline{0.113} $\pm$ \underline{0.001} & \underline{0.284} $\pm$ \underline{0.029} & \underline{0.158} $\pm$ \underline{0.003} & \underline{0.130} $\pm$ \underline{0.002} & \underline{0.067} $\pm$ \underline{0.001} & \underline{0.072} $\pm$ \underline{0.001} & \underline{2.314} $\pm$ \underline{0.708} \\
 & & CaGP & 3.972 $\pm$ 0.147 & 0.219 $\pm$ 0.005 & 0.511 $\pm$ 0.011 & 0.241 $\pm$ 0.004 & 0.208 $\pm$ 0.002 & 0.105 $\pm$ 0.002 & 0.091 $\pm$ 0.001 & 3.743 $\pm$ 0.553 \\
\midrule
\multirow{4}{*}{\rotatebox[origin=c]{90}{Test}} & \multirow{4}{*}{\rotatebox[origin=c]{90}{RMSE}} & LKGP & 2.935 $\pm$ 0.150 & 0.335 $\pm$ 0.006 & 0.747 $\pm$ 0.016 & 0.351 $\pm$ 0.008 & 0.394 $\pm$ 0.027 & 0.261 $\pm$ 0.005 & \underline{0.152} $\pm$ \underline{0.003} & 2.600 $\pm$ 0.901 \\
 & & SVGP & \textbf{2.716} $\pm$ \textbf{0.145} & \textbf{0.316} $\pm$ \textbf{0.005} & \textbf{0.540} $\pm$ \textbf{0.013} & \underline{0.294} $\pm$ \underline{0.005} & \textbf{0.285} $\pm$ \textbf{0.005} & \textbf{0.232} $\pm$ \textbf{0.003} & \textbf{0.145} $\pm$ \textbf{0.002} & \textbf{1.657} $\pm$ \textbf{1.068} \\
 & & VNNGP & 3.053 $\pm$ 0.159 & 0.705 $\pm$ 0.013 & 0.915 $\pm$ 0.017 & 0.677 $\pm$ 0.011 & 0.642 $\pm$ 0.009 & 0.591 $\pm$ 0.006 & 0.568 $\pm$ 0.005 & 3.743 $\pm$ 0.731 \\
 & & CaGP & \underline{2.810} $\pm$ \underline{0.154} & \underline{0.325} $\pm$ \underline{0.006} & \underline{0.593} $\pm$ \underline{0.015} & \textbf{0.281} $\pm$ \textbf{0.004} & \underline{0.296} $\pm$ \underline{0.004} & \underline{0.239} $\pm$ \underline{0.004} & 0.153 $\pm$ 0.002 & \underline{2.000} $\pm$ \underline{0.000} \\
\midrule
\multirow{4}{*}{\rotatebox[origin=c]{90}{Train}} & \multirow{4}{*}{\rotatebox[origin=c]{90}{NLL}} & LKGP & \textbf{1.955} $\pm$ \textbf{0.020} & \textbf{-1.573} $\pm$ \textbf{0.022} & \textbf{0.160} $\pm$ \textbf{0.055} & \textbf{-1.311} $\pm$ \textbf{0.017} & \textbf{-1.718} $\pm$ \textbf{0.018} & \textbf{-2.933} $\pm$ \textbf{0.005} & \textbf{-2.943} $\pm$ \textbf{0.005} & \textbf{1.029} $\pm$ \textbf{0.167} \\
 & & SVGP & 3.098 $\pm$ 0.100 & \underline{-0.132} $\pm$ \underline{0.020} & 0.652 $\pm$ 0.021 & \underline{0.025} $\pm$ \underline{0.013} & \underline{-0.150} $\pm$ \underline{0.012} & \underline{-0.800} $\pm$ \underline{0.014} & \underline{-0.977} $\pm$ \underline{0.007} & \underline{2.400} $\pm$ \underline{0.685} \\
 & & VNNGP & 3.005 $\pm$ 0.041 & -0.079 $\pm$ 0.010 & \underline{0.639} $\pm$ \underline{0.045} & 0.104 $\pm$ 0.009 & -0.056 $\pm$ 0.009 & -0.489 $\pm$ 0.008 & -0.456 $\pm$ 0.006 & 3.286 $\pm$ 0.881 \\
 & & CaGP & \underline{2.794} $\pm$ \underline{0.045} & -0.042 $\pm$ 0.022 & 0.798 $\pm$ 0.021 & 0.034 $\pm$ 0.017 & -0.103 $\pm$ 0.012 & -0.718 $\pm$ 0.013 & -0.900 $\pm$ 0.009 & 3.286 $\pm$ 0.564 \\
\midrule
\multirow{4}{*}{\rotatebox[origin=c]{90}{Test}} & \multirow{4}{*}{\rotatebox[origin=c]{90}{NLL}} & LKGP & \underline{2.349} $\pm$ \underline{0.047} & \textbf{0.051} $\pm$ \textbf{0.026} & 1.047 $\pm$ 0.029 & \textbf{-0.057} $\pm$ \textbf{0.019} & \textbf{-0.328} $\pm$ \textbf{0.023} & \textbf{-0.633} $\pm$ \textbf{0.029} & \textbf{-1.214} $\pm$ \textbf{0.020} & \textbf{1.257} $\pm$ \textbf{0.602} \\
 & & SVGP & \textbf{2.327} $\pm$ \textbf{0.065} & \underline{0.319} $\pm$ \underline{0.018} & \textbf{0.781} $\pm$ \textbf{0.020} & 0.199 $\pm$ 0.016 & \underline{0.160} $\pm$ \underline{0.018} & \underline{0.065} $\pm$ \underline{0.032} & \underline{-0.496} $\pm$ \underline{0.014} & \underline{2.543} $\pm$ \underline{1.024} \\
 & & VNNGP & 2.733 $\pm$ 0.031 & 1.046 $\pm$ 0.017 & 1.365 $\pm$ 0.019 & 1.011 $\pm$ 0.016 & 0.891 $\pm$ 0.013 & 0.717 $\pm$ 0.013 & 0.712 $\pm$ 0.009 & 3.200 $\pm$ 1.141 \\
 & & CaGP & 2.451 $\pm$ 0.033 & 0.321 $\pm$ 0.021 & \underline{0.916} $\pm$ \underline{0.023} & \underline{0.142} $\pm$ \underline{0.016} & 0.171 $\pm$ 0.013 & 0.088 $\pm$ 0.031 & -0.463 $\pm$ 0.013 & 3.000 $\pm$ 0.000 \\
\midrule
\multirow{4}{*}{\rotatebox[origin=c]{90}{Time}} & \multirow{4}{*}{\rotatebox[origin=c]{90}{in min}} & LKGP & \textbf{0.371} $\pm$ \textbf{0.008} & \textbf{1.458} $\pm$ \textbf{0.013} & \textbf{0.487} $\pm$ \textbf{0.003} & \textbf{1.728} $\pm$ \textbf{0.010} & \textbf{2.447} $\pm$ \textbf{0.022} & \textbf{1.493} $\pm$ \textbf{0.017} & \textbf{2.012} $\pm$ \textbf{0.018} & \textbf{1.000} $\pm$ \textbf{0.000} \\
 & & SVGP & \underline{6.473} $\pm$ \underline{0.032} & \underline{6.475} $\pm$ \underline{0.032} & \underline{6.479} $\pm$ \underline{0.031} & \underline{6.474} $\pm$ \underline{0.032} & \underline{6.492} $\pm$ \underline{0.033} & \underline{6.500} $\pm$ \underline{0.030} & \underline{6.475} $\pm$ \underline{0.032} & \underline{2.000} $\pm$ \underline{0.000} \\
 & & VNNGP & 26.34 $\pm$ 0.443 & 25.89 $\pm$ 0.161 & 25.78 $\pm$ 0.151 & 25.85 $\pm$ 0.158 & 25.87 $\pm$ 0.181 & 25.85 $\pm$ 0.154 & 26.31 $\pm$ 0.464 & 4.000 $\pm$ 0.000 \\
 & & CaGP & 7.067 $\pm$ 0.028 & 7.036 $\pm$ 0.028 & 7.033 $\pm$ 0.023 & 7.015 $\pm$ 0.020 & 7.046 $\pm$ 0.024 & 7.024 $\pm$ 0.019 & 7.023 $\pm$ 0.020 & 3.000 $\pm$ 0.000 \\
\bottomrule
\end{tabular}
\end{adjustbox}
\end{table*}

For this experiment, we consider hyperparameter configurations as $\c{S}$ and learning curve progression steps as $\c{T}$, such that $p = 2000$, $q = 52$, and thus $n \leq 104$k for each dataset.
We choose both $k_{\c{S}}$ and $k_{\c{T}}$ to be squared exponential kernels.
Out of the $p = 2000$ curves per dataset, $10$\% are provided as fully observed during training and the remaining $90$\% are partially observed. 
The early stopping point is chosen uniformly at random.
For all methods, observation noise and kernel hyperparameters are initialised with GPyTorch default values \citep{gardner18}, and optimised using Adam.
LKGP is trained for $100$ iterations using a learning rate of $0.1$, conjugate gradients with a relative residual norm tolerance of $0.01$ and a pivoted Cholesky preconditioner of rank $100$.
SVGP is trained for $30$ epochs using a learning rate of $0.01$, a batch size of $1000$, and $10$k inducing points, which are initialised at random training data examples.
VNNGP is trained for $1000$ epochs using a learning rate of $0.01$, a batch size of $1000$, inducing points placed at every training example, and $256$ nearest neighbours.
CaGP is trained for $1000$ epochs using a learning rate of $0.1$ and $512$ actions.
See \Cref{fig:lcbench_illustration} for an illustration and qualitative comparison of the GP methods which are considered in this experiment.

\Cref{tab:lcbench_results_main} reports the predictive performance and total runtime for every fifth LCBench dataset and the average rank across all datasets.
On average, SVGP and CaGP achieve better test root-mean-square-error (RMSE), but LKGP produces the best test negative log-likelihood (NLL), providing quantitative evidence for the qualitative observation in \Cref{fig:lcbench_illustration} which suggests that LKGP produces the most sensible uncertainty estimates.
To explain why LKGP performs worse in terms of test RMSE, we suggest that LKGP might be overfitting.
Since missing values are intentionally accumulated at the end of individual learning curves to simulate early stopping, the train (observed) and test (missing) data do not have the same distribution, which makes this setting particularly prone to overfitting.
Our hypothesis is empirically supported by observing that LKGP consistently achieves the best performance on the training data.
The fact that LKGP still achieves the best test NLL is likely due to superior uncertainty quantification of the exact GP model compared to variational approximations, which is a commonly observed phenomenon.
Furthermore, LKGP requires the least amount of time by a large margin, suggesting that LKGP could potentially be scaled to much larger learning curve datasets.

\begin{figure}[t!]
\centering
\includegraphics[width=6in]{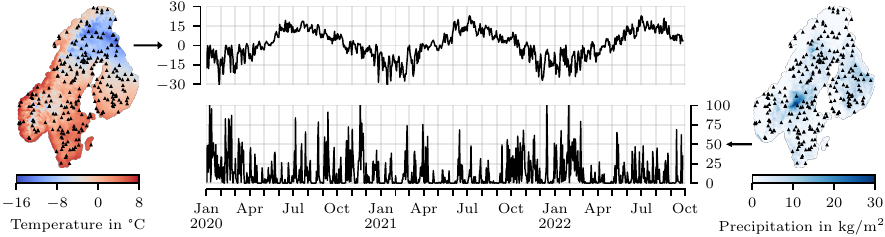}
\caption{Illustration of daily temperature and precipitation data from the Nordic Gridded Climate Dataset \citep{NGCD1,NGCD2}. The heatmaps (left and right) visualise snapshots of a single day and subsampled spatial locations. Every spatial location is associated with its own time series (middle). Temperatures (top) exhibit a seasonal periodic trend. Precipitation (bottom) is noisy but locally correlated.}
\label{fig:ngcd_visual}
\end{figure}

\subsection{Climate Data with Missing Values}
In this experiment, we predict the daily average temperature and precipitation at different spatial locations.
We use the Nordic Gridded Climate Dataset \citep{NGCD1,NGCD2}, which contains observations that are gridded in space and time, namely latitude, longitude and days.
The main goal of this experiment is to demonstrate the scalability of LKGP on large datasets with millions of observations.
Additionally, we investigate how increasing the missing data ratio impacts predictive performances and runtime requirements.
We compare our method to the same sparse and variational methods from the previous experiment.

For this experiment, we consider latitude and longitude jointly as $\c{S}$ and time in days as $\c{T}$.
We select $p = 5000$ spatial locations uniformly at random together with $q = 1000$ days of temperature or precipitation measurements, starting on January 1st, 2020, resulting in two datasets of size $n \leq 5$M.
See \Cref{fig:ngcd_visual} for an illustration.
We choose $k_{\c{S}}$ to be a squared exponential kernel and $k_{\c{T}}$ to be the product of a squared exponential kernel and a periodic kernel to capture seasonal trends.
Missing values are selected uniformly at random with ratios of $10$\%, $20$\%, ..., $50$\%.
For all methods, observation noise and kernel hyperparameters are initialised with GPyTorch default values \citep{gardner18} and optimised using Adam.
LKGP is trained for $100$ iterations using a learning rate of $0.1$, and conjugate gradients with a relative residual norm tolerance of $0.01$ and a pivoted Cholesky preconditioner of rank $100$.
SVGP is trained for $5$ epochs using a learning rate of $0.001$, a batch size of $1000$, and $10$k inducing points, which are initialised at random training data examples.
VNNGP is trained for $50$ epochs using a learning rate of $0.001$, a batch size of $1000$, $500$k inducing points placed at random training data examples, and $256$ nearest neighbours.
CaGP is trained for $50$ epochs using a learning rate of $0.1$ and $256$ actions.

\begin{table*}[ht!]
\caption{[Climate Data with Missing Values] Predictive performances and total runtimes of temperature and precipitation prediction across various missing data ratios. LKGP consistently achieves the best predictive performance and also requires the least time across both datasets and all missing ratios. Reported numbers are the mean over $5$ random seeds. Results of the best model are boldfaced and results of the second best model are underlined.}
\label{tab:ngcd_results}
\vspace{0.1in}
\centering
\setlength{\tabcolsep}{3pt}
\small
\begin{adjustbox}{max width=\textwidth}
\begin{tabular}{ l@{\hspace{3.5pt}}l l | c c c c c | c c c c c}
\toprule
& & & \multicolumn{5}{c|}{Temperature Dataset (with Missing Ratio $10\% - 50\%$)} & \multicolumn{5}{c}{Precipitation Dataset (with Missing Ratio $10\% - 50\%$)} \\
& & Model & $10\%$ & $20\%$ & $30\%$ & $40\%$ & $50\%$ & $10\%$ & $20\%$ & $30\%$ & $40\%$ & $50\%$ \\
\midrule
\multirow{4}{*}{\rotatebox[origin=c]{90}{Train}} & \multirow{4}{*}{\rotatebox[origin=c]{90}{RMSE}} & LKGP & \textbf{0.06} $\pm$ \textbf{0.00} & \textbf{0.06} $\pm$ \textbf{0.00} & \textbf{0.07} $\pm$ \textbf{0.00} & \textbf{0.07} $\pm$ \textbf{0.00} & \textbf{0.07} $\pm$ \textbf{0.00} & \textbf{0.16} $\pm$ \textbf{0.00} & \textbf{0.17} $\pm$ \textbf{0.00} & \textbf{0.17} $\pm$ \textbf{0.00} & \textbf{0.18} $\pm$ \textbf{0.00} & \textbf{0.18} $\pm$ \textbf{0.00} \\
 & & SVGP & 0.21 $\pm$ 0.00 & 0.21 $\pm$ 0.00 & 0.22 $\pm$ 0.00 & 0.22 $\pm$ 0.00 & 0.23 $\pm$ 0.00 & 0.70 $\pm$ 0.00 & 0.71 $\pm$ 0.00 & 0.71 $\pm$ 0.00 & 0.71 $\pm$ 0.00 & 0.72 $\pm$ 0.00 \\
 & & VNNGP & \underline{0.13} $\pm$ \underline{0.00} & \underline{0.13} $\pm$ \underline{0.00} & \underline{0.13} $\pm$ \underline{0.00} & \underline{0.13} $\pm$ \underline{0.00} & \underline{0.13} $\pm$ \underline{0.00} & \underline{0.45} $\pm$ \underline{0.00} & \underline{0.47} $\pm$ \underline{0.00} & \underline{0.49} $\pm$ \underline{0.00} & \underline{0.52} $\pm$ \underline{0.00} & \underline{0.58} $\pm$ \underline{0.00} \\
 & & CaGP & 0.18 $\pm$ 0.00 & 0.19 $\pm$ 0.00 & 0.19 $\pm$ 0.00 & 0.19 $\pm$ 0.00 & 0.19 $\pm$ 0.00 & 0.60 $\pm$ 0.00 & 0.61 $\pm$ 0.00 & 0.61 $\pm$ 0.00 & 0.62 $\pm$ 0.00 & 0.62 $\pm$ 0.00 \\
\midrule
\multirow{4}{*}{\rotatebox[origin=c]{90}{Test}} & \multirow{4}{*}{\rotatebox[origin=c]{90}{RMSE}} & LKGP & \textbf{0.08} $\pm$ \textbf{0.00} & \textbf{0.08} $\pm$ \textbf{0.00} & \textbf{0.08} $\pm$ \textbf{0.00} & \textbf{0.09} $\pm$ \textbf{0.00} & \textbf{0.09} $\pm$ \textbf{0.00} & \textbf{0.24} $\pm$ \textbf{0.00} & \textbf{0.25} $\pm$ \textbf{0.00} & \textbf{0.26} $\pm$ \textbf{0.00} & \textbf{0.27} $\pm$ \textbf{0.00} & \textbf{0.28} $\pm$ \textbf{0.00} \\
 & & SVGP & 0.21 $\pm$ 0.00 & 0.21 $\pm$ 0.00 & 0.22 $\pm$ 0.00 & 0.22 $\pm$ 0.00 & 0.23 $\pm$ 0.00 & 0.70 $\pm$ 0.00 & 0.71 $\pm$ 0.00 & 0.71 $\pm$ 0.00 & 0.72 $\pm$ 0.00 & 0.72 $\pm$ 0.00 \\
 & & VNNGP & \underline{0.13} $\pm$ \underline{0.00} & \underline{0.13} $\pm$ \underline{0.00} & \underline{0.13} $\pm$ \underline{0.00} & \underline{0.13} $\pm$ \underline{0.00} & \underline{0.13} $\pm$ \underline{0.00} & \underline{0.46} $\pm$ \underline{0.00} & \underline{0.48} $\pm$ \underline{0.00} & \underline{0.50} $\pm$ \underline{0.00} & \underline{0.53} $\pm$ \underline{0.00} & \underline{0.58} $\pm$ \underline{0.00} \\
 & & CaGP & 0.18 $\pm$ 0.00 & 0.19 $\pm$ 0.00 & 0.19 $\pm$ 0.00 & 0.19 $\pm$ 0.00 & 0.19 $\pm$ 0.00 & 0.61 $\pm$ 0.00 & 0.61 $\pm$ 0.00 & 0.61 $\pm$ 0.00 & 0.62 $\pm$ 0.00 & 0.63 $\pm$ 0.00 \\
\midrule
\multirow{4}{*}{\rotatebox[origin=c]{90}{Train}} & \multirow{4}{*}{\rotatebox[origin=c]{90}{NLL}} & LKGP & \textbf{-1.33} $\pm$ \textbf{0.02} & \textbf{-1.30} $\pm$ \textbf{0.01} & \textbf{-1.26} $\pm$ \textbf{0.00} & \textbf{-1.21} $\pm$ \textbf{0.00} & \textbf{-1.18} $\pm$ \textbf{0.00} & \textbf{-0.32} $\pm$ \textbf{0.00} & \textbf{-0.29} $\pm$ \textbf{0.00} & \textbf{-0.26} $\pm$ \textbf{0.00} & \textbf{-0.23} $\pm$ \textbf{0.00} & \textbf{-0.18} $\pm$ \textbf{0.00} \\
 & & SVGP & -0.14 $\pm$ 0.00 & -0.12 $\pm$ 0.00 & -0.10 $\pm$ 0.00 & -0.07 $\pm$ 0.00 & -0.05 $\pm$ 0.00 & 1.07 $\pm$ 0.00 & 1.07 $\pm$ 0.00 & 1.08 $\pm$ 0.00 & 1.09 $\pm$ 0.00 & 1.10 $\pm$ 0.00 \\
 & & VNNGP & \underline{-0.58} $\pm$ \underline{0.00} & \underline{-0.58} $\pm$ \underline{0.00} & \underline{-0.57} $\pm$ \underline{0.00} & \underline{-0.57} $\pm$ \underline{0.00} & \underline{-0.57} $\pm$ \underline{0.00} & \underline{0.64} $\pm$ \underline{0.00} & \underline{0.67} $\pm$ \underline{0.00} & \underline{0.71} $\pm$ \underline{0.00} & \underline{0.78} $\pm$ \underline{0.00} & \underline{0.87} $\pm$ \underline{0.00} \\
 & & CaGP & -0.22 $\pm$ 0.00 & -0.21 $\pm$ 0.00 & -0.20 $\pm$ 0.00 & -0.18 $\pm$ 0.00 & -0.17 $\pm$ 0.00 & 0.93 $\pm$ 0.00 & 0.94 $\pm$ 0.00 & 0.94 $\pm$ 0.00 & 0.95 $\pm$ 0.00 & 0.96 $\pm$ 0.00 \\
\midrule
\multirow{4}{*}{\rotatebox[origin=c]{90}{Test}} & \multirow{4}{*}{\rotatebox[origin=c]{90}{NLL}} & LKGP & \textbf{-1.14} $\pm$ \textbf{0.01} & \textbf{-1.11} $\pm$ \textbf{0.01} & \textbf{-1.07} $\pm$ \textbf{0.00} & \textbf{-1.03} $\pm$ \textbf{0.00} & \textbf{-0.99} $\pm$ \textbf{0.00} & \textbf{-0.02} $\pm$ \textbf{0.00} & \textbf{0.01} $\pm$ \textbf{0.00} & \textbf{0.04} $\pm$ \textbf{0.00} & \textbf{0.09} $\pm$ \textbf{0.00} & \textbf{0.14} $\pm$ \textbf{0.00} \\
 & & SVGP & -0.14 $\pm$ 0.00 & -0.12 $\pm$ 0.00 & -0.09 $\pm$ 0.00 & -0.07 $\pm$ 0.00 & -0.05 $\pm$ 0.00 & 1.07 $\pm$ 0.00 & 1.08 $\pm$ 0.00 & 1.08 $\pm$ 0.00 & 1.09 $\pm$ 0.00 & 1.10 $\pm$ 0.00 \\
 & & VNNGP & \underline{-0.57} $\pm$ \underline{0.00} & \underline{-0.56} $\pm$ \underline{0.00} & \underline{-0.56} $\pm$ \underline{0.00} & \underline{-0.56} $\pm$ \underline{0.00} & \underline{-0.55} $\pm$ \underline{0.00} & \underline{0.66} $\pm$ \underline{0.00} & \underline{0.69} $\pm$ \underline{0.00} & \underline{0.73} $\pm$ \underline{0.00} & \underline{0.79} $\pm$ \underline{0.00} & \underline{0.88} $\pm$ \underline{0.00} \\
 & & CaGP & -0.22 $\pm$ 0.00 & -0.20 $\pm$ 0.00 & -0.19 $\pm$ 0.00 & -0.18 $\pm$ 0.00 & -0.17 $\pm$ 0.00 & 0.94 $\pm$ 0.00 & 0.94 $\pm$ 0.00 & 0.94 $\pm$ 0.00 & 0.95 $\pm$ 0.00 & 0.96 $\pm$ 0.00 \\
\midrule
\multirow{4}{*}{\rotatebox[origin=c]{90}{Time}} & \multirow{4}{*}{\rotatebox[origin=c]{90}{in min}} & LKGP & \textbf{28.7} $\pm$ \textbf{0.33} & \textbf{26.9} $\pm$ \textbf{0.41} & \textbf{24.6} $\pm$ \textbf{0.19} & \textbf{23.0} $\pm$ \textbf{0.31} & \textbf{21.2} $\pm$ \textbf{0.16} & \textbf{15.6} $\pm$ \textbf{0.01} & \textbf{14.2} $\pm$ \textbf{0.01} & \textbf{12.7} $\pm$ \textbf{0.01} & \textbf{11.3} $\pm$ \textbf{0.00} & \textbf{9.87} $\pm$ \textbf{0.01} \\
 & & SVGP & \underline{150} $\pm$ \underline{0.08} & \underline{133} $\pm$ \underline{0.07} & \underline{117} $\pm$ \underline{0.04} & \underline{100} $\pm$ \underline{0.03} & \underline{83.3} $\pm$ \underline{0.03} & \underline{150} $\pm$ \underline{0.02} & \underline{133} $\pm$ \underline{0.13} & \underline{116} $\pm$ \underline{0.02} & \underline{99.9} $\pm$ \underline{0.05} & \underline{83.2} $\pm$ \underline{0.03} \\
 & & VNNGP & 158 $\pm$ 0.26 & 140 $\pm$ 0.17 & 123 $\pm$ 0.37 & 105 $\pm$ 0.16 & 87.8 $\pm$ 0.13 & 158 $\pm$ 0.40 & 141 $\pm$ 0.27 & 123 $\pm$ 0.11 & 105 $\pm$ 0.13 & 87.6 $\pm$ 0.22 \\
 & & CaGP & 500 $\pm$ 0.15 & 395 $\pm$ 0.06 & 301 $\pm$ 0.12 & 221 $\pm$ 0.01 & 153 $\pm$ 0.01 & 500 $\pm$ 0.02 & 395 $\pm$ 0.03 & 301 $\pm$ 0.18 & 221 $\pm$ 0.06 & 153 $\pm$ 0.07 \\
\bottomrule
\end{tabular}
\end{adjustbox}
\end{table*}

\Cref{tab:ngcd_results} reports the predictive performances and total runtimes of temperature and precipitation prediction.
On this large dataset with millions of examples, LKGP clearly outperforms the sparse baseline methods in both root-mean-square-error and negative log-likelihood.
This is unsurprising, given that LKGP performs inference in the exact GP model while the other methods are subject to a limited number of inducing points, nearest neighbours, or sparse actions.
However, due to latent Kronecker structure, LKGP also requires the least amount of time by a large margin.
Interestingly, VNNGP consistently outperforms SVGP and CaGP on this problem.
We suspect that this is due to the nearest neighbour mechanism working well on these datasets with actual spatial dimensions.

\section{Discussion}
In this chapter, we proposed a highly scalable exact Gaussian process model for product kernels, which leverages latent Kronecker structure to accelerate computations and reduce memory requirements for data arranged on a partial grid while allowing for missing values.
In contrast to existing Gaussian process models with Kronecker structure, our approach deals with missing values by combining projections and iterative linear algebra methods. 
Empirically, we demonstrated that our method has superior computational scalability compared to standard iterative methods, and substantially outperforms sparse and variational Gaussian processes in terms of prediction quality. 
Future work could investigate specialised kernels, multi-product generalisations, and heteroscedastic noise models.

\subsection{Limitations}
The primary limitation of LKGP is that it employs a product kernel, which assumes that points are only highly correlated if they are close in both $\c{S}$ and $\c{T}$.
While this is a reasonable assumption in many real-world settings, such as correlated time series data, it will not always be appropriate.
The other requirement for LKGP to be applicable is that the data lives on at least a partial Cartesian product grid.
This requirement is less restrictive since our model is highly competitive even if there are a lot of missing values, as our experiments and theory show.
Furthermore, if the data does not even live on a partial grid, it would be possible to generate an artificial grid, for example via local interpolation.

%!TEX root = ../thesis.tex
%*******************************************************************************
%****************************** Seventh Chapter ********************************
%*******************************************************************************
\chapter{Conclusion}
\label{chap:conclusion}

% **************************** Define Graphics Path **************************
\ifpdf
    \graphicspath{{Chapter7/Figs/Raster/}{Chapter7/Figs/PDF/}{Chapter7/Figs/}}
\else
    \graphicspath{{Chapter7/Figs/Vector/}{Chapter7/Figs/}}
\fi

In this dissertation, we approached the topic of scalable Gaussian processes from a modern perspective, involving large datasets and hardware for massively-parallel computation.
We developed methods which are able to scale to such modern circumstances by focusing on the effective combination of iterative methods and pathwise conditioning.
In particular, all of our contributions are generally about reducing asymptotic time and memory requirements, leveraging fast matrix multiplication, improving the speed of convergence, and analysing the behaviour under non-convergence.
I hope that our contributions facilitate the use of Gaussian processes in modern large-scale settings.

In the following, I summarise the original contributions and insights of this dissertation:
\begin{itemize}
    \item \Cref{chap:sgd} introduced stochastic gradient descent as an iterative linear system solver to approximately compute samples from the Gaussian process posterior via pathwise conditioning.
    To this end, we derived custom low-variance optimisation objectives, which we further extended to inducing points for even greater scalability.
    A key insight from this chapter is the observation that \emph{stochastic gradient descent can produce accurate predictions even if it did not converge to the optimum}.
    To explain this phenomenon, we developed a spectral characterisation of the effects of non-convergence, showing that the error under non-convergence mainly manifests at the boundaries of the observed data.
    Empirically, we evaluated our proposed algorithm on various regression tasks, achieving state-of-the-art performance for sufficiently large or ill-conditioned settings.
    Furthermore, we demonstrated that, at a fraction of the computational cost, stochastic gradient descent matches or exceeds the performance of other more computationally expensive baseline methods on a large-scale Bayesian optimisation task.
    \item \Cref{chap:sdd} improved the proposed stochastic gradient descent algorithm from the previous chapter by combining various ideas from the optimisation and kernel communities.
    In particular, we introduced a dual optimisation objective, which facilitates larger step sizes due to more favourable curvature properties, leading to significantly faster convergence and reduced computational costs.
    Additionally, we analysed two different ways of performing stochastic gradient estimation, resulting in the important insight that \emph{the estimator with multiplicative noise, which decreases as the optimisation approaches the optimum, provides substantially improved convergence properties compared to the estimator with constant additive noise}.
    Furthermore, we investigated momentum and iterate averaging techniques, leading to recommendations with lower variance and superior convergence properties.
    Empirically, we improved upon the performance from the previous chapter, and further demonstrated that our method achieves state-of-the-art results on a molecular binding affinity prediction task, matching graph neural networks.
    \item \Cref{chap:solvers} proposed generic contributions towards iterative linear system solvers for marginal likelihood optimisation.
    We introduced a pathwise gradient estimator, which reduces the required number of solver iterations until convergence, and leverages a connection to pathwise conditioning to amortise the computational costs of making predictions.
    Additionally, we used intermediate solutions to warm start iterative linear system solvers, leading to faster convergence at the cost of introducing negligible bias.
    Furthermore, we investigated the behaviour of iterative linear system solvers under a limited computational budget, terminating optimisation before reaching convergence.
    Notably, we observed that, \emph{in certain cases, a small compute budget was sufficient to achieve good predictive performance, despite not reaching convergence}.
    Empirically, we showed that our techniques provide speed-ups of up to $72\times$ when solving until convergence, and decrease the average residual norm by up to $7\times$ when stopping early.
    \item \Cref{chap:lkgp} developed latent Kronecker structure for Gaussian process regression with Kronecker product-structured kernel matrices.
    The latter enables scalable inference via factorised matrix decompositions, but typically requires fully gridded data from a Cartesian product space to be applicable.
    To support incomplete grids, we expressed the kernel matrix as the projection of a latent Kronecker product, and leveraged fast matrix multiplication in combination with iterative methods and pathwise conditioning to perform scalable inference.
    In particular, \emph{the latent factorisation drastically reduces the memory requirements, allowing kernel values to be precomputed, which significantly reduces the overall computational costs}.
    On real-world datasets with up to five million examples, our method outperformed state-of-the-art sparse and variational approaches.
\end{itemize}

\section{Future Directions}
Finally, I will discuss potential directions for future research, based on the current limitations of our contributions.
Personally, I prefer to be optimistic about limitations and think of them more as opportunities for improvement rather than definitive weaknesses.
In this spirit, I will discuss some limitations of the methodological contributions of this dissertation, along with potential avenues for future work.

\subsubsection{Adaptive Learning Rate Selection for Stochastic Gradient Descent}
Stochastic gradient and dual descent (\Cref{chap:sgd,chap:sdd}) in their current form have several free hyperparameters, such as the learning rate, batch size, or momentum and iterate averaging parameters, which is inconvenient for any practitioner who simply wants to use them as a scalable black-box linear system solver.
In particular, the learning rate, also referred to as the step size, has a strong influence on the rate of convergence.
In combination with noise and kernel hyperparameter optimisation, the problem is further exacerbated, because the optimal learning rate depends on their values.
Therefore, future research could consider methods to select the learning rate in an automatic way \citep{defazio2023}.
If successful, this would greatly improve the convenience and practicality of these algorithms.

\subsubsection{Understanding and Leveraging Non-Convergence Properties}
Throughout this dissertation, we have repeatedly observed that terminating iterative linear system solvers before reaching convergence does not necessarily lead to worse performance.
While \Cref{chap:sgd} provided some theoretical arguments in the context of drawing posterior samples with stochastic gradient descent, the relationship between relative residual norm convergence and kernel hyperparameter optimisation in \Cref{chap:solvers} is unclear.
Future work could further investigate how the residual norm relates to the objective of marginal likelihood optimisation.
My impression is that the standard Euclidean norm of the residual, which is commonly used in combination with iterative linear system solvers, is suboptimal, because it does not depend on the spectrum of the kernel matrix.
Since the solutions of the linear systems are used to calculate inner products with respect to the derivative of the kernel matrix, I would expect the latter to play an important role in determining convergence.
The question then becomes how to formulate a convergence criterion which adequately reflects this while also being amenable to tractable, and ideally amortised, computation.
An improved convergence criterion could inform the design of new solvers which directly optimise for it.

\subsubsection{Latent Kronecker Structure of Higher Order and General Purpose}
\Cref{chap:lkgp} proposed the idea of expressing a covariance matrix as the projection of a latent Kronecker product to leverage fast matrix multiplication in the latent space while reducing memory requirements.
Although we only considered latent Kronecker products with two factors, it would be straightforward to generalise this idea to an arbitrary number of factors, as long as the data continues to be a subset of a corresponding Cartesian product space.
For example, in our large-scale experiment on climate data, we factorised the data into space and time.
With a multi-product generalisation, the space coordinates could be further factorised into latitudes and longitudes.
Furthermore, the concept of latent Kronecker structure could be applied to other applications.
Future work could, for example, investigate latent Kronecker structure for second-order optimisation methods.
The latter typically need to represent the Hessian or inverse Hessian matrix of the objective function, which can quickly become intractable due to quadratically increasing memory requirements.
Modern approaches for deep learning already leverage Kronecker factorisation \citep{eschenhagen2023}, such that latent Kronecker structure could be a natural extension.

\subsubsection{Gaussian Processes in the Era of Deep Learning}
Last but not least, I want to share my thoughts about Gaussian processes in the era of deep learning.
At machine learning conferences, I have occasionally been confronted by fellow researchers with the question:
\emph{`If you could scale Gaussian processes to arbitrary amounts of data, would you use them instead of neural networks?'}
Typically, this question was not asked in a derogatory way, but rather out of curiosity, and my answer would be: \emph{`It depends...'}
Aside from uncertainty quantification capabilities, which I mentioned in \Cref{chap:intro}, Gaussian processes and neural networks follow a fundamentally different approach when specifying a hypothesis space of latent functions.
The former is fully specified by explicit mean and covariance functions, whereas the latter explores a potentially much larger space which is implicitly defined by its architecture.
Some may argue that the latter is more powerful since it can adapt to the data, which allows neural networks to learn remarkably complex functions.
However, if the class of latent functions is known, one can explicitly choose, or even design \citep{swersky2014}, a covariance function  which incorporates this information.
I believe that this approach capitalises on the advantages of Gaussian processes and should thus be pursued instead of competing with neural networks to learn complex implicit functions.
Will Gaussian processes ever be able to generate photorealistic images of cute cats or have a fluent conversation in a natural language?
Most likely not.
Should we still cherish them? Yes.

% ********************************** Back Matter *******************************
% Backmatter should be commented out, if you are using appendices after References
%\backmatter

% ********************************** Bibliography ******************************
\begin{spacing}{0.9}

% To use the conventional natbib style referencing
% Bibliography style previews: http://nodonn.tipido.net/bibstyle.php
% Reference styles: http://sites.stat.psu.edu/~surajit/present/bib.htm

% \bibliographystyle{apalike}
%\bibliographystyle{unsrt} % Use for unsorted references  
% \bibliographystyle{plainnat} % use this to have URLs listed in References
\cleardoublepage
% \bibliography{References/references} % Path to your References.bib file

% If you would like to use BibLaTeX for your references, pass `custombib' as
% an option in the document class. The location of 'reference.bib' should be
% specified in the preamble.tex file in the custombib section.
% Comment out the lines related to natbib above and uncomment the following line.

\printbibliography[heading=bibintoc, title={References}]

@inproceedings{tazi2023beyond,
    title = {Beyond Intuition, a Framework for Applying GPs to Real-World Data},
    author = {Kenza Tazi and Jihao Andreas Lin and Ross Viljoen and Alex Gardner and Ti John and Hong Ge and Richard E. Turner},
    booktitle = {ICML Structured Probabilistic Inference \& Generative Modeling Workshop},
    year = {2023}
}

@inproceedings{lin2023sampling,
    title = {Sampling from Gaussian Process Posteriors using Stochastic Gradient Descent},
    author = {Jihao Andreas Lin and Javier Antorán and Shreyas Padhy and David Janz and José Miguel Hernández-Lobato and Alexander Terenin},
    booktitle = {Advances in Neural Information Processing Systems},
    year = {2023}
}

@inproceedings{lin2024stochastic,
    title = {Stochastic Gradient Descent for Gaussian Processes Done Right},
    author = {Jihao Andreas Lin and Shreyas Padhy and Javier Antorán and Austin Tripp and Alexander Terenin and Csaba Szepesvári and José Miguel Hernández-Lobato and David Janz},
    booktitle = {International Conference on Learning Representations},
    year = {2024}
}

@inproceedings{lin2024warm,
    title = {Warm Start Marginal Likelihood Optimisation for Iterative Gaussian Processes},
    author = {Jihao Andreas Lin and Shreyas Padhy and Bruno Mlodozeniec and José Miguel Hernández-Lobato},
    booktitle = {Advances in Approximate Bayesian Inference},
    year = {2024}
}

@inproceedings{lin2024improving,
    title = {Improving Linear System Solvers for Hyperparameter Optimisation in Iterative Gaussian Processes}, 
    author = {Jihao Andreas Lin and Shreyas Padhy and Bruno Mlodozeniec and Javier Antorán and José Miguel Hernández-Lobato},
    booktitle = {Advances in Neural Information Processing Systems},
    year = {2024}
}

@inproceedings{lin2024scaling,
    title = {Scaling Gaussian Processes for Learning Curve Prediction via Latent Kronecker Structure}, 
    author = {Jihao Andreas Lin and Sebastian Ament and Maximilian Balandat and Eytan Bakshy},
    booktitle = {NeurIPS Bayesian Decision-making and Uncertainty Workshop},
    year = {2024}
}

@inproceedings{lin2025scalable,
    title = {Scalable Gaussian Processes with Latent Kronecker Structure}, 
    author = {Jihao Andreas Lin and Sebastian Ament and Maximilian Balandat and David Eriksson and José Miguel Hernández-Lobato and Eytan Bakshy},
    booktitle = {International Conference on Machine Learning},
    year = {2025}
}

@book{rasmussen2006,
    title = {Gaussian Processes for Machine Learning},
    author = {C. E. Rasmussen and C. K. I. Williams},
    year = {2006},
    publisher = {MIT Press}
}

@book{bishop2006,
    title = {Pattern Recognition and Machine Learning},
    author = {Cristopher M. Bishop},
    year = {2006},
    publisher = {Springer}
}

@inproceedings{wilson20,
	author = {James T. Wilson and Viacheslav Borovitskiy and Alexander Terenin and Peter Mostowsky and Marc Peter Deisenroth},
	booktitle = {International Conference on Machine Learning},
	title = {Efficiently Sampling Functions from Gaussian Process Posteriors},
	year = {2020}
}

@article{wilson21,
	author = {James T. Wilson and Viacheslav Borovitskiy and Alexander Terenin and Peter Mostowsky and Marc Peter Deisenroth},
	journal = {Journal of Machine Learning Research},
	volume = {22},
	issue = {1},
	title = {Pathwise Conditioning of Gaussian Processes},
	year = {2021}
}

@inproceedings{jiang20,
	author = {Shali Jiang and Daniel R. Jiang and Maximilian Balandat and Brian Karrer and Jacob R. Gardner and Roman Garnett},
	booktitle = {Advances in Neural Information Processing Systems},
	title = {Efficient Nonmyopic Bayesian Optimization via One-Shot Multi-Step Trees},
	year = {2020}
}

@article{candela2005,
	author = {Joaquin Quiñonero-Candela and Carl Edward Rasmussen},
	journal = {Journal of Machine Learning Research},
	volume = {6},
	issue = {1},
	title = {A Unifying View of Sparse Approximate Gaussian Process Regression},
	year = {2005}
}

@inproceedings{williams2000using,
	author = {Christopher K. I. Williams and Matthias Seeger},
	booktitle = {Advances in Neural Information Processing Systems},
	title = {Using the Nystr\"om Method to Speed Up Kernel Machines},
	year = {2000}
}

@article{silverman1985,
	author = {Bernhard Silverman},
	journal = {Journal of the Royal Statistical Society, Series B},
	volume = {47},
	issue = {1},
	title = {Some Aspects of the Spline Smoothing Approach to Non-Parametric Regression Curve Fitting},
	year = {1985}
}

@inproceedings{wahba1999,
	author = {Grace Wahba and Xiwu Lin and Fangyu Gao and Dong Xiang and Ronald Klein and Barbara Klein},
	booktitle = {Advances in Neural Information Processing Systems},
	title = {The Bias-Variance Tradeoff and the Randomized GACV},
	year = {1999}
}

@inproceedings{smola2001,
	author = {Alexander J. Smola and Peter L. Bartlett},
	booktitle = {Advances in Neural Information Processing Systems},
	title = {Sparse Greedy Gaussian Process Regression},
	year = {2001}
}

@article{csato2002,
	author = {Lehel Csató and Manfred Opper},
	journal = {Neural Computation},
	volume = {14},
	issue = {3},
	title = {Sparse On-Line Gaussian Processes},
	year = {2002}
}

@inproceedings{seeger2003,
	author = {Matthias W. Seeger and Christopher K. I. Williams and Neil D. Lawrence},
	booktitle = {International Workshop on Artificial Intelligence and Statistics},
	title = {Fast Forward Selection to Speed Up Sparse Gaussian Process Regression},
	year = {2003}
}

@inproceedings{snelson2005,
	author = {Edward Snelson and Zoubin Ghahramani},
	booktitle = {Advances in Neural Information Processing Systems},
	title = {Sparse Gaussian Processes using Pseudo-inputs},
	year = {2005}
}

@inproceedings{schwaighofer2002,
	author = {Anton Schwaighofer and Volker Tresp},
	booktitle = {Advances in Neural Information Processing Systems},
	title = {Transductive and Inductive Methods for Approximate Gaussian Process Regression},
	year = {2002}
}

@inproceedings{titsias09,
	author = {Titsias, Michalis K},
	booktitle = {Artificial Intelligence and Statistics},
	title = {Variational learning of inducing variables in sparse Gaussian processes},
	year = {2009}
}

@techreport{titsias2009report,
	author = {Titsias, Michalis K},
	institution = {University of Manchester},
	title = {Variational model selection for sparse Gaussian process regression},
	year = {2009}
}

@inproceedings{hensman13,
	author = {Hensman, James and Fusi, Nicolò and Lawrence, Neil D},
	booktitle = {Uncertainty in Artificial Intelligence},
	title = {Gaussian Processes for Big Data},
	year = {2013}
}

@article{mercer1909,
	author = {James Mercer},
	journal = {Philosophical Transactions of the Royal Society of London, Series A},
	volume = {209},
	title = {Functions of Positive and Negative Type, and their Connection with the Theory of Integral Equations},
	year = {1909}
}

@inproceedings{rahimi08,
	author = {Rahimi, Ali and Recht, Benjamin},
	booktitle = {Advances in Neural Information Processing Systems},
	title = {Random Features for Large-scale Kernel Machines},
	year = {2008}
}

@inproceedings{sutherland15,
	author = {Sutherland, Danica J and Schneider, Jeff},
	booktitle = {Uncertainty in Artificial Intelligence},
	title = {On the Error of Random Fourier Features},
	year = {2015}
}

@inproceedings{bonilla2007multi,
    title = {Multi-task Gaussian Process Prediction},
    author = {Edwin V. Bonilla and Kian Ming A. Chai and Christopher K. I. Williams},
    booktitle = {Advances in Neural Information Processing Systems},
    year = {2007}
}

@inproceedings{stegle2011efficient,
    title = {Efficient inference in matrix-variate Gaussian models with iid observation noise},
    author = {Oliver Stegle and Christoph Lippert and Joris M. Mooij and Neil Lawrence and Karsten Borgwardt},
    booktitle = {Advances in Neural Information Processing Systems},
    year = {2011}
}

@inproceedings{wilson2015kernel,
    title = {Kernel Interpolation for Scalable Structured Gaussian Processes (KISS-GP)},
    author = {Wilson, Andrew and Nickisch, Hannes},
    booktitle = {International Conference on Machine Learning},
    year = {2015}
}

@inproceedings{gardner2018product,
    title = {Product Kernel Interpolation for Scalable Gaussian Processes},
    author = {Gardner, Jacob and Pleiss, Geoff and Wu, Ruihan and Weinberger, Kilian and Wilson, Andrew},
    booktitle = {International Conference on Artificial Intelligence and Statistics},
    year = {2018}
}

@inproceedings{kapoor2021skiing,
    title = {SKIing on Simplices: Kernel Interpolation on the Permutohedral Lattice for Scalable Gaussian Processes},
    author = {Kapoor, Sanyam and Finzi, Marc and Wang, Ke Alexander and Wilson, Andrew Gordon Gordon},
    booktitle = {International Conference on Machine Learning},
    year = {2021}
}

@article{hestenes1952,
    author = {Hestenes, Magnus R. and Stiefel, Eduard},
    title = {Methods of Conjugate Gradients for Solving Linear Systems},
    journal = {Journal of Research of the National Bureau of Standards},
    volume = {49},
    issue = {6},
    year = {1952}
}

@techreport{shewchuk1994,
    author = {Jonathan R. Shewchuk},
    title = {An Introduction to the Conjugate Gradient Method Without the Agonizing Pain},
    institution = {Carnegie Mellon University},
    year = {1994}
}

@inproceedings{gardner18,
	author = {Gardner, Jacob and Pleiss, Geoff and Weinberger, Kilian Q and Bindel, David and Wilson, Andrew G},
	booktitle = {Advances in Neural Information Processing Systems},
	title = {GPyTorch: Blackbox Matrix-Matrix Gaussian Process Inference with GPU Acceleration},
	year = {2018}
}

@inproceedings{WangPGT2019exactgp,
	author = {Ke Alexander Wang and Geoff Pleiss and Jacob R. Gardner and Stephen Tyree and Kilian Q. Weinberger and Andrew Gordon Wilson},
	booktitle = {Advances in Neural Information Processing Systems},
	title = {Exact Gaussian Processes on a Million Data Points},
	year = {2019}
}

@article{Hutchinson1990,
    author = {M.F. Hutchinson},
    title = {A Stochastic Estimator of the Trace of the Influence Matrix for Laplacian Smoothing Splines},
    journal = {Communications in Statistics - Simulation and Computation},
    volume = {19},
    number = {2},
    year = {1990}
}

@inproceedings{pleiss2018,
    author = {Geoff Pleiss and Jacob R. Gardner and Kilian Q. Weinberger and Andrew Gordon Wilson},
    title = {Constant-Time Predictive Distributions for Gaussian Processes},
    booktitle = {International Conference on Machine Learning},
    year = {2018}
}

@article{Gomez-Bombarelli18,
	author = {Gómez-Bombarelli, Rafael and Wei, Jennifer N. and Duvenaud, David and Hernández-Lobato, José Miguel and Sánchez-Lengeling, Benjamín and Sheberla, Dennis and Aguilera-Iparraguirre, Jorge and Hirzel, Timothy D. and Adams, Ryan P. and Aspuru-Guzik, Alán},
	journal = {American Chemical Society Central Science},
	title = {Automatic Chemical Design Using a Data-Driven Continuous Representation of Molecules},
	year = {2018}
}

@inproceedings{Snoek2012,
	author = {Jasper Snoek and Hugo Larochelle and Ryan P. Adams},
	booktitle = {Advances in Neural Information Processing Systems},
	title = {Practical Bayesian Optimization of Machine Learning Algorithms},
	year = {2012}
}

@inproceedings{Hernandez-Lobato14,
	author = {José Miguel Hernández-Lobato and Matthew W. Hoffman and Zoubin Ghahramani},
	booktitle = {Advances in Neural Information Processing Systems},
	title = {Predictive Entropy Search for Efficient Global Optimization of Black-box Functions},
	year = {2014}
}

@inproceedings{Wilkinson20,
	author = {William J. Wilkinson and Paul E. Chang and Michael Riis Andersen and Arno Solin},
	booktitle = {International Conference on Machine Learning},
	title = {State Space Expectation Propagation: Efficient Inference Schemes for Temporal Gaussian Processes},
	year = {2020}
}

@article{Zhang2023sgd,
	author = {Tian, Yingjie and Zhang, Yuqi and Zhang, Haibin},
	journal = {Mathematics},
	volume = {11},
	issue = {3},
	title = {Recent Advances in Stochastic Gradient Descent in Deep Learning},
	year = {2023}
}

@inproceedings{dai2014doubly,
	author = {Dai, Bo and Xie, Bo and He, Niao and Liang, Yingyu and Raj, Anant and Balcan, Maria-Florina F and Song, Le},
	booktitle = {Advances in Neural Information Processing Systems},
	title = {Scalable Kernel Methods via Doubly Stochastic Gradients},
	year = {2014}
}

@article{Mandt17Descent,
	author = {Stephan Mandt and Matthew D. Hoffman and David M. Blei},
	journal = {Journal of Machine Learning Research},
	title = {Stochastic Gradient Descent as Approximate Bayesian Inference},
	year = {2017}
}

@article{Belkin2019,
	author = {Mikhail Belkin and Daniel Hsu and Siyuan Ma and Soumik Mandal},
	journal = {Proceedings of the National Academy of Sciences},
	title = {Reconciling modern machine-learning practice and the classical bias-variance trade-off},
	year = {2019}
}

@inproceedings{ZouWBGK2021benign,
	author = {Difan Zou and Jingfeng Wu and Vladimir Braverman and Quanquan Gu and Sham M. Kakade},
	booktitle = {Conference on Learning Theory},
	title = {Benign Overfitting of Constant-Stepsize SGD for Linear Regression},
	year = {2021}
}

@inproceedings{Cunningham2016preconditioning,
	author = {Cutajar, Kurt and Osborne, Michael and Cunningham, John and Filippone, Maurizio},
	booktitle = {International Conference on Machine Learning},
	title = {Preconditioning Kernel Matrices},
	year = {2016}
}

@inproceedings{chen20,
	author = {Chen, Hao and Zheng, Lili and Al Kontar, Raed and Raskutti, Garvesh},
	booktitle = {Advances in Neural Information Processing Systems},
	title = {Stochastic Gradient Descent in Correlated Settings: A Study on Gaussian Processes},
	year = {2020}
}

@article{chen22,
	author = {Chen, Hao and Zheng, Lili and Al Kontar, Raed and Raskutti, Garvesh},
	journal = {Journal of Machine Learning Research},
    volume = {23},
	title = {Gaussian Process Parameter Estimation Using Mini-batch Stochastic Gradient Descent: Convergence Guarantees and Empirical Benefits},
	year = {2022}
}

@book{vapnik95,
	author = {Vapnik, Vladimir},
	publisher = {Springer},
	title = {The Nature of Statistical Learning},
	year = {1995}
}

@inproceedings{Scholkopf2001representer,
	author = {Schölkopf, Bernhard and Herbrich, Ralf and Smola, Alex J.},
	booktitle = {Computational Learning Theory},
	title = {A Generalized Representer Theorem},
	year = {2001}
}

@article{wild2023connections,
	author = {Veit Wild and Motonobu Kanagawa and Dino Sejdinovic},
	journal = {arXiv:2106.01121},
	title = {Connections and Equivalences between the Nyström Method and Sparse Variational Gaussian Processes},
	year = {2023}
}

@book{hansen98,
	author = {Hansen, Per Christian},
	publisher = {Society for Industrial and Applied Mathematics},
	title = {Rank-Deficient and Discrete Ill-posed Problems: Numerical Aspects of Linear Inversion},
	year = {1998}
}

@article{jin23,
	author = {Jin, Bangti and Kereta, Željko},
	journal = {SIAM Journal on Imaging Sciences},
	title = {On the Convergence of Stochastic Gradient Descent for Linear Inverse Problems in Banach Spaces},
	year = {2023}
}

@misc{Dua2019UCI,
	author = {Dua, Dheeru and Graff, Casey},
	title = {UCI Machine Learning Repository},
	year = {2017}
}

@article{terenin23,
	author = {Alexander Terenin and David R. Burt and Artemev, Artem and Seth Flaxman and Mark van der Wilk and Carl Edward Rasmussen and Hong Ge},
	journal = {Journal of Machine Learning Research},
	title = {Numerically Stable Sparse Gaussian Processes via Minimum Separation using Cover Trees},
	year = {2023}
}

@misc{Annoy,
	author = {Erik Bernhardsson},
	title = {Approximate Nearest Neighbors Oh Yeah},
	year = {2012}
}

@article{rubens2015active,
	author = {Rubens, Neil and Elahi, Mehdi and Sugiyama, Masashi and Kaplan, Dain},
	journal = {Recommender Systems Handbook},
	title = {Active Learning in Recommender Systems},
	year = {2015}
}

@article{elahi2016survey,
	author = {Elahi, Mehdi and Ricci, Francesco and Rubens, Neil},
	journal = {Computer Science Review},
	title = {A survey of active learning in collaborative filtering recommender systems},
	year = {2016}
}

@inproceedings{hernandezlobato2017Parallel,
	author = {José Miguel Hernández-Lobato and James Requeima and Edward O. Pyzer-Knapp and Alán Aspuru-Guzik},
	booktitle = {International Conference on Machine Learning},
	title = {Parallel and Distributed Thompson Sampling for Large-scale Accelerated Exploration of Chemical Space},
	year = {2017}
}

@inproceedings{rudi2017falkon,
	author = {Rudi, Alessandro and Carratino, Luigi and Rosasco, Lorenzo},
	booktitle = {Advances in Neural Information Processing Systems},
	title = {Falkon: An Optimal Large Scale Kernel Method},
	year = {2017}
}

@article{kivinen2004online,
	author = {Kivinen, Jyrki and Smola, Alexander J and Williamson, Robert C},
	journal = {IEEE Transactions on Signal Processing},
	title = {Online Learning with Kernels},
	year = {2004}
}

@article{shalev2013stochastic,
	author = {Shalev-Shwartz, Shai and Zhang, Tong},
	journal = {Journal of Machine Learning Research},
	title = {Stochastic Dual Coordinate Ascent Methods for Regularized Loss Minimization},
	year = {2013}
}

@article{dieuleveut2017harder,
	author = {Dieuleveut, Aymeric and Flammarion, Nicolas and Bach, Francis},
	journal = {Journal of Machine Learning Research},
	title = {Harder, Better, Faster, Stronger Convergence Rates for Least-Squares Regression},
	year = {2017}
}

@inproceedings{varre2021last,
	author = {Varre, Aditya Vardhan and Pillaud-Vivien, Loucas and Flammarion, Nicolas},
	booktitle = {Advances in Neural Information Processing Systems},
	title = {Last iterate convergence of SGD for Least-Squares in the Interpolation regime},
	year = {2021}
}

@book{scholkopf2002learning,
	author = {Schölkopf, Bernhard and Smola, Alexander J},
	publisher = {MIT Press},
	title = {Learning with Kernels: Support Vector Machines, Regularization, Optimization, and Beyond},
	year = {2002}
}

@book{boyd2004convex,
	author = {Boyd, Stephen P and Vandenberghe, Lieven},
	publisher = {Cambridge University Press},
	title = {Convex Optimization},
	year = {2004}
}

@inproceedings{li2019towards,
	author = {Li, Zhu and Ton, Jean-Francois and Oglic, Dino and Sejdinovic, Dino},
	booktitle = {International Conference on Machine Learning},
	title = {Towards a Unified Analysis of Random Fourier Features},
	year = {2019}
}

@article{polyak1964some,
	author = {Polyak, Boris T},
	journal = {USSR Computational Mathematics and Mathematical Physics},
	title = {Some methods of speeding up the convergence of iteration methods},
	year = {1964}
}

@inproceedings{nesterov1983method,
	author = {Nesterov, Yurii},
	booktitle = {Doklady Akademii Nauk SSSR},
	title = {A method for unconstrained convex minimization problem with the rate of convergence $O(1/k^2)$},
	year = {1983}
}

@inproceedings{sutskever2013importance,
	author = {Sutskever, Ilya and Martens, James and Dahl, George and Hinton, Geoffrey},
	booktitle = {International Conference on Machine Learning},
	title = {On the importance of initialization and momentum in deep learning},
	year = {2013}
}

@techreport{ruppert1988efficient,
	author = {Ruppert, David},
	institution = {Cornell University},
	title = {Efficient Estimations from a Slowly Convergent Robbins-Monro Process},
	year = {1988}
}

@article{polyak1990new,
	author = {Polyak, Boris T},
	journal = {Avtomatika i Telemekhanika},
	title = {New stochastic approximation type procedures},
	year = {1990}
}

@article{polyak1992acceleration,
	author = {Polyak, Boris T and Juditsky, Anatoli B},
	journal = {SIAM Journal on Control and Optimization},
	title = {Acceleration of Stochastic Approximation by Averaging},
	year = {1992}
}

@inproceedings{saunders1998ridge,
	author = {Saunders, Craig and Gammerman, Alexander and Vovk, Vladimir},
	booktitle = {International Conference on Machine Learning},
	title = {Ridge Regression Learning Algorithm in Dual Variables},
	year = {1998}
}

@inproceedings{frie1998kernel,
	author = {Frie, Thilo-Thomas and Cristianini, Nello and Campbell, Colin},
	booktitle = {International Conference on Machine Learning},
	title = {The Kernel-Adatron Algorithm: a Fast and Simple Learning Procedure for Support Vector Machines},
	year = {1998}
}

@article{liu2008kernel,
	author = {Liu, Weifeng and Pokharel, Puskal P and Principe, Jose C},
	journal = {IEEE Transactions on Signal Processing},
	title = {The Kernel Least-Mean-Square Algorithm},
	year = {2008}
}

@inproceedings{bo2008greedy,
	author = {Bo, Liefeng and Sminchisescu, Cristian},
	booktitle = {Uncertainty in Artificial Intelligence},
	title = {Greedy Block Coordinate Descent for Large Scale Gaussian Process Regression},
	year = {2008}
}

@article{tu2016large,
	author = {Tu, Stephen and Roelofs, Rebecca and Venkataraman, Shivaram and Recht, Benjamin},
	journal = {arXiv:1602.05310},
	title = {Large Scale Kernel Learning using Block Coordinate Descent},
	year = {2016}
}

@inproceedings{wu2024large,
    title = {Large-Scale Gaussian Processes via Alternating Projection}, 
    author = {Kaiwen Wu and Jonathan Wenger and Haydn Jones and Geoff Pleiss and Jacob R. Gardner},
    year = {2024},
    booktitle = {International Conference on Artificial Intelligence and Statistics}
}

@inproceedings{antoran2023sampling,
	author = {Javier Antorán and Shreyas Padhy and Riccardo Barbano and Eric T. Nalisnick and David Janz and José Miguel Hernández-Lobato},
	booktitle = {International Conference on Learning Representations},
	title = {Sampling-based inference for large linear models, with application to linearised Laplace},
	year = {2023}
}

@article{Ortegon2022dockstring,
	author = {Miguel García-Ortegón and Gregor N. C. Simm and Austin J. Tripp and José Miguel Hernández-Lobato and Andreas Bender and Sergio Bacallado},
	journal = {Journal of Chemical Information and Modeling},
	title = {DOCKSTRING: Easy Molecular Docking Yields Better Benchmarks for Ligand Design},
	year = {2022}
}

@article{pinza2019docking,
	author = {Pinzi, Luca and Rastelli, Giulio},
	journal = {International Journal of Molecular Sciences},
	title = {Molecular Docking: Shifting Paradigms in Drug Discovery},
	year = {2019}
}

@article{yang2021efficient,
	author = {Yang, Ying and Yao, Kun and Repasky, Matthew P and Leswing, Karl and Abel, Robert and Shoichet, Brian K and Jerome, Steven V},
	journal = {Journal of Chemical Theory and Computation},
	title = {Efficient Exploration of Chemical Space with Docking and Deep Learning},
	year = {2021}
}

@article{trott2010autodock,
	author = {Trott, Oleg and Olson, Arthur J},
	journal = {Journal of Computational Chemistry},
	title = {AutoDock Vina: Improving the Speed and Accuracy of Docking with a New Scoring Function, Efficient Optimization, and Multithreading},
	year = {2010}
}

@article{rogers2010extended,
	author = {Rogers, David and Hahn, Mathew},
	journal = {Journal of Chemical Information and Modeling},
	title = {Extended-Connectivity Fingerprints},
	year = {2010}
}

@inproceedings{ioffe2010improved,
	author = {Ioffe, Sergey},
	booktitle = {International Conference on Data Mining},
	title = {Improved Consistent Sampling, Weighted Minhash and L1 Sketching},
	year = {2010}
}

@inproceedings{tripp2023tanimoto,
	author = {Austin Tripp and Sergio Bacallado and Sukriti Singh and José Miguel Hernández-Lobato},
	booktitle = {Advances in Neural Information Processing Systems},
	title = {Tanimoto Random Features for Scalable Molecular Machine Learning},
	year = {2023}
}

@article{ralaivola2005graph,
	author = {Ralaivola, Liva and Swamidass, Sanjay J and Saigo, Hiroto and Baldi, Pierre},
	journal = {Neural Networks},
	title = {Graph kernels for chemical informatics},
	year = {2005}
}

@inproceedings{gilmer2017neural,
	author = {Gilmer, Justin and Schoenholz, Samuel S and Riley, Patrick F and Vinyals, Oriol and Dahl, George E},
	booktitle = {International Conference on Machine Learning},
	title = {Neural Message Passing for Quantum Chemistry},
	year = {2017}
}

@article{xiong2019pushing,
	author = {Xiong, Zhaoping and Wang, Dingyan and Liu, Xiaohong and Zhong, Feisheng and Wan, Xiaozhe and Li, Xutong and Li, Zhaojun and Luo, Xiaomin and Chen, Kaixian and Jiang, Hualiang and others},
	journal = {Journal of Medicinal Chemistry},
	title = {Pushing the Boundaries of Molecular Representation for Drug Discovery with the Graph Attention Mechanism},
	year = {2019}
}

@article{ghasemi2021application,
    title = {Application of Gaussian process regression to forecast multi-step ahead SPEI drought index},
    author = {Ghasemi, Porya and Karbasi, Masoud and Nouri, Alireza Zamani and Tabrizi, Mahdi Sarai and Azamathulla, Hazi Mohammad},
    journal = {Alexandria Engineering Journal},
    year = {2021}
}

@article{Deisenroth2015,
    author = {Deisenroth, Marc Peter and Fox, Dieter and Rasmussen, Carl Edward},
    title = {Gaussian Processes for Data-Efficient Learning in Robotics and Control},
    year = {2015},
    journal = {IEEE Trans. Pattern Anal. Mach. Intell.}
}

@inproceedings{artemev2021tighter,
	author = {Artem Artemev and David R. Burt and Mark van der Wilk},
	booktitle = {International Conference on Machine Learning},
	title = {Tighter Bounds on the Log Marginal Likelihood of Gaussian Process Regression Using Conjugate Gradients},
	year = {2021}
}

@inproceedings{maddox2021iterative,
    title = {When are Iterative Gaussian Processes Reliably Accurate?}, 
    author = {Wesley J. Maddox and Sanyam Kapoor and Andrew Gordon Wilson},
    year = {2021},
    booktitle = {ICML OPTML Workshop}
}

@article{Hutch++,
    author = {Meyer, Raphael and Musco, Cameron and Musco, Christopher and Woodruff, David},
    year = {2021},
    title = {Hutch++: Optimal Stochastic Trace Estimation},
    volume = {2021},
    journal = {Symposium on Simplicity in Algorithms}
}

@article{XTrace,
    author = {Epperly, Ethan N. and Tropp, Joel A. and Webber, Robert J.},
    title = {XTrace: Making the Most of Every Sample in Stochastic Trace Estimation},
    journal = {Matrix Analysis and Applications},
    year = {2024}
}

@book{garnett2023bobook,
    author = {Garnett, Roman},
    publisher = {Cambridge University Press},
    title = {Bayesian Optimization},
    year = {2023}
}

@inproceedings{PILCO,
    title = {PILCO: A Model-Based and Data-Efficient Approach to Policy Search},
    author = {Marc Peter Deisenroth and Carl Edward Rasmussen},
    year = {2011},
    booktitle = {International Conference on Machine Learning}
}

@inproceedings{riis2022bayesian,
    title = {Bayesian Active Learning with Fully Bayesian Gaussian Processes},
    author = {Riis, Christoffer and Antunes, Francisco and H{\"u}ttel, Frederik and Lima Azevedo, Carlos and Pereira, Francisco},
    booktitle = {Advances in Neural Information Processing Systems},
    year = {2022}
}

@inproceedings{jankowiak20ppgpr,
    title = {Parametric Gaussian Process Regressors},
    author = {Jankowiak, Martin and Pleiss, Geoff and Gardner, Jacob},
    booktitle = {International Conference on Machine Learning},
    year = {2020}
}

@inproceedings{bauer2016understanding,
    title = {Understanding Probabilistic Sparse Gaussian Process Approximations},
    author = {Bauer, Matthias and Van der Wilk, Mark and Rasmussen, Carl Edward},
    booktitle = {Advances in Neural Information Processing Systems},
    year = {2016}
}

@inproceedings{wu2022variational,
    title = {Variational Nearest Neighbor Gaussian Process},
    author = {Wu, Luhuan and Pleiss, Geoff and Cunningham, John P},
    booktitle = {International Conference on Machine Learning},
    year = {2022}
}

@inproceedings{wenger2022posterior,
    title = {Posterior and Computational Uncertainty in Gaussian processes},
    author = {Wenger, Jonathan and Pleiss, Geoff and Pförtner, Marvin and Hennig and Cunningham, John P},
    booktitle = {Advances in Neural Information Processing Systems},
    year = {2022}
}

@inproceedings{wenger2024computation,
    title = {Computation-Aware Gaussian Processes: Model Selection And Linear-Time Inference},
    author = {Wenger, Jonathan and Wu, Kaiwen and Hennig, Philipp and Gardner, Jacob R and Pleiss, Geoff and Cunningham, John P},
    booktitle = {Advances in Neural Information Processing Systems},
    year = {2024}
}

@inproceedings{zhe19scalable,
    title = {Scalable High-Order Gaussian Process Regression},
    author = {Zhe, Shandian and Xing, Wei and Kirby, Robert M.},
    booktitle = {International Conference on Artificial Intelligence and Statistics},
    year = {2019}
}

@inproceedings{maddox2021highdimout,
    author = {Maddox, Wesley J. and Balandat, Maximilian and Wilson, Andrew G and Bakshy, Eytan},
    booktitle = {Advances in Neural Information Processing Systems},
    title = {Bayesian Optimization with High-Dimensional Outputs},
    year = {2021}
}

@article {ZimLin2021a,
    author = {Lucas Zimmer and Marius Lindauer and Frank Hutter},
    title = {Auto-PyTorch Tabular: Multi-Fidelity MetaLearning for Efficient and Robust AutoDL},
    journal = {IEEE Transactions on Pattern Analysis and Machine Intelligence},
    year = {2021},
    volume = {43},
    issue = {9}
}

@article{elsken2019nassurvey,
    author = {Elsken, Thomas and Metzen, Jan Hendrik and Hutter, Frank},
    journal = {Journal of Machine Learning Research},
    issue = {55},
    title = {Neural Architecture Search: A Survey},
    volume = {20},
    year = {2019}
}

@article{NGCD1,
    title = {Nordic Temperature Maps},
    author = {Tveito, O E and Førland, E J and Heino, R and Hanssen-Bauer, I and Alexandersson, H and Dahlström, B and Drebs, A and Kern-Hansen, C and Jónsson, T and Vaarby-Laursen, E and Westman, E},
    journal = {DNMI Klima 9/00 KLIMA.},
    year = {2000},
    publisher = {Norwegian Meteorological Institute}
}

@article{NGCD2,
    title = {A GIS-based agro-ecoglogical decision system based on gridded climatology},
    author = {Tveito, O E and Bjørdal, I and Skjelvåg, A O and Aune, B},
    journal = {Metoeorl. Appl.},
    volume = {12},
    year = {2005}
}

@inproceedings{defazio2023,
    title = {Learning-Rate-Free Learning by D-Adaptation},
    author = {Defazio, Aaron and Mishchenko, Konstantin},
    booktitle = {International Conference on Machine Learning},
    year = 	 {2023}
}

@inproceedings{eschenhagen2023,
    title= {Kronecker-Factored Approximate Curvature for Modern Neural Network Architectures},
    author= {Runa Eschenhagen and Alexander Immer and Richard E. Turner and Frank Schneider and Philipp Hennig},
    booktitle = {Advances in Neural Information Processing Systems},
    year = {2023}
}

@article{swersky2014,
	title = {Freeze-Thaw Bayesian Optimization},
	author = {Kevin Swersky and Jasper Snoek and Ryan Prescott Adams},
	journal = {arXiv:1406.3896},
	year = {2014}
}

@book{Petersen2006,
    title = {The Matrix Cookbook},
    author = {Petersen, Kaare Brandt and Pedersen, Michael Syskind},
    year = {2006},
    publisher = {Technical University of Denmark}
}

@book{Horn1991,
    title = {Topics in Matrix Analysis},
    publisher = {Cambridge University Press},
    author = {Horn, Roger A. and Johnson, Charles R.},
    year = {1991}
}

@book{horn12,
	author = {Horn, Roger A. and Johnson, Charles R.},
	publisher = {Cambridge University Press},
	title = {Matrix Analysis},
	year = {2012}
}

\end{spacing}

% ********************************** Appendices ********************************

\begin{appendices} % Using appendices environment for more functunality

%!TEX root = ../thesis.tex
% ******************************* Thesis Appendix A ****************************
\chapter{Mathematical Derivations} 
\label{chap:math}
This appendix contains mathematical derivations which are too verbose to be included in the main body of this dissertation.
The following derivations are based on and adapted from:
\begin{itemize}
    \item J. A. Lin, J. Antorán, S. Padhy, D. Janz, J. M. Hernández-Lobato, and A. Terenin. Sampling from Gaussian Process Posteriors using Stochastic Gradient Descent. In \emph{Advances in Neural Information Processing Systems}, 2023.
    % \item J. A. Lin, S. Padhy, B. Mlodozeniec, J. Antorán, and J. M. Hernández-Lobato. Improving Linear System Solvers for Hyperparameter Optimisation in Iterative Gaussian Processes. In \emph{Advances in Neural Information Processing Systems}, 2024.
\end{itemize}

\section{The Implicit Bias of Stochastic Gradient Descent}
\label{apd:implicit_bias}
This section provides the main theoretical results of \Cref{chap:sgd}, including \Cref{thm:sgd_convergence}.

\subsection{Convergence of Stochastic Gradient Descent}
\label{apd:sgd_convergence}
Before studying what happens in Gaussian processes, we first prove a fairly standard result on the convergence of SGD for a quadratic objective appearing in kernel ridge regression, which represents the posterior mean.
For simplicity, we do not analyse Nesterov momentum nor aspects such as gradient clipping.
Recall that a random variable $z$ is called \emph{$G$-sub-Gaussian} if $\E(\exp(\lambda(z - \E(z)))) \leq \exp(\smash{\frac{1}{2}}G\lambda^2)$ holds for all $\lambda$.
Furthermore, a random vector is called $G$-sub-Gaussian if its dot product with any unit vector is $G$-sub-Gaussian.
One can show this condition is equivalent to having tails that are no heavier than those of a Gaussian random vector, formulated appropriately.
Let $\m{K}_{\m{X}\m{X}} = \m{U}\m\Lambda\m{U}^\mathsf{T}$ be the eigenvalue decomposition of the kernel matrix with eigenvectors $\v{u}_i$ and eigenvalues $\lambda_1 \geq ... \geq \lambda_n > 0$.

\begin{lemma}
\label{lem:sgd-error-bound}
Let $\delta > 0$ and noise variance $\sigma^2 > 0$.
Let $\eta$ be a sufficiently small learning rate of $0 < \eta < \frac{1}{\lambda_1(\lambda_1 + \sigma^2)}$.
Let $\v{v}^* \in \R^n$ be the solution of the respective linear system, and let $\overline{\v{v}}_t$ be the Polyak-averaged SGD iterate after $t$ steps, starting from an initial condition of $\v{v}_0 = \v{0}$. 
Assume that the stochastic estimate of the gradient is unbiased and $G$-sub-Gaussian. Then, with probability $1-\delta$, we have for any $\c{I} \subseteq \{1, ..., n\}$ and all $i \in \c{I}$ that
\begin{equation}
|\v{u}_i^\mathsf{T}(\v{v}^* - \overline{\v{v}}_t)| \leq
\frac{1}{\lambda_i^2} \del{\frac{\norm{\v{y}}_2}{\eta \sigma^2 t} + G\sqrt{\frac{2}{t} \log\frac{|\c{I}|}{\delta}}},
\end{equation}
\end{lemma}

\begin{proof}
\label{proof:lem:sgd-error-bound}
Consider the objective
\begin{equation}
L(\v{v}) = \frac{1}{2} \norm{\v{y} - \m{K}_{\m{X} \m{X}} \v{v}}_2^2 + \frac{\sigma^2}{2} \norm{\v{v}}_{\m{K}_{\m{X}\m{X}}}^2,
\end{equation}
and respective gradient
\begin{equation}
\frac{\partial L}{\partial \v{v}}(\v{v}) = \m{K}_{\m{X}\m{X}}^2 \v{v} - \m{K}_{\m{X}\m{X}} \v{y} + \sigma^2 \m{K}_{\m{X}\m{X}}\v{v} = \m{K}_{\m{X}\m{X}}(\m{K}_{\m{X}\m{X}}+ \sigma^2\m{I})\v{v} - \m{K}_{\m{X}\m{X}} \v{y}.
\end{equation}
Let us first look at non-stochastic gradient optimisation of $L$, without Polyak averaging.
The iterate $\v{v}_t$ is given by
\begin{equation}
\v{v}_t = \v{v}_{t-1} - \eta \frac{\partial L}{\partial \v{v}}(\v{v}_{t-1}) = \del{\m{I} - \eta \m{K}_{\m{X}\m{X}}(\m{K}_{\m{X}\m{X}}+\sigma^2\m{I})}\v{v}_{t-1} + \eta \m{K}_{\m{X}\m{X}}\v{y}.
\end{equation}
Writing $\m{M} = \del{\m{I} - \eta \m{K}_{\m{X}\m{X}}(\m{K}_{\m{X}\m{X}}+\sigma^2\m{I})}$, and recalling that $\v{v}_0 = \v{0}$, we thus have that
\begin{equation}
\label{eq:noiseless-iterate-sgd}
\v{v}_t = \eta \sum_{j=0}^{t-1} \m{M}^j \m{K}_{\m{X}\m{X}}\v{y} = \eta (\m{I}-\m{M})^{-1} (\m{I}-\m{M}^t) \m{K}_{\m{X}\m{X}} \v{y} = \v{v}^* - (\m{K}_{\m{X}\m{X}}+\sigma^2\m{I})^{-1} \m{M}^t \v{y},
\end{equation}
where, for our choice of $\eta$, by direct calculation using simultaneous diagonalisability of $\m{M}$ and $\m{K}_{\m{X}\m{X}}$, the eigenvalues of $\m{M}$ are strictly between $0$ and $1$.
Therefore, the geometric series converges and $\m{I} - \m{M}$ is positive definite.
Examining the error in direction $\v{u}_i$, we see that
\begin{equation}
\label{eq:noiseless-iterate-err-eigenval}
|\v{u}_i^\mathsf{T} (\v{v}^* - \v{v}_{t})| = |\v{u}_i^\mathsf{T} (\m{K}_{\m{X}\m{X}}+\sigma^2\m{I})^{-1} \m{M}^t \v{y}|
\leq \frac{\del{1 - \eta \lambda_i (\lambda_i + \sigma^2)}^t}{\lambda_i + \sigma^2}\norm{\v{y}}_2,
\end{equation}
which applies to ordinary gradient descent without stochastic gradients or Polyak averaging.
Next, we consider stochastic gradient optimisation.
We will first consider the case where the gradient is independently perturbed by $\v\zeta_t\~[N](\v{0},\m{I})$ for each step $t > 0$, and then relax this to sub-Gaussian noise in the sequel.

Specifically, we consider
\begin{equation}
\v{v}'_t = \v{v}'_{t-1} -\eta\del{\frac{\partial L}{\partial \v{v}}(\v{v}'_{t-1}) + \v\zeta_t} = \v{v}_t - \eta \sum_{j=0}^{t-1} \m{M}^j \v\zeta_{t-j},
\end{equation}
with $\v{v}'_0 = \v{v}_0 = \v{0}$.
For these iterates, we consider the respective Polyak-averaged iterates denoted by $\widebar{\v{v}}_t = \frac{1}{t} \sum_{j=1}^t \v{v}'_j$.
We have
\begin{equation}
|\v{u}_i^\mathsf{T}(\v{v}^* - \widebar{\v{v}}_t)| = \abs[4]{\frac{1}{t} \sum_{j=1}^t \v{u}_i^\mathsf{T}(\v{v}^* - \v{v}'_j)} \leq \abs[4]{\ubr{\frac{1}{t} \sum_{j=1}^t \v{u}_i^\mathsf{T}(\v{v}^* - \v{v}_j)}_{A_{i,t}}} + \abs[4]{\ubr{\frac{1}{t}\sum_{j=1}^t \v{u}_i^\mathsf{T}\v(\v{v}_j - \v{v}'_j)}_{B_{i,t}}},
\end{equation}
by expanding the Polyak averages and applying the triangle inequality.
Using \Cref{eq:noiseless-iterate-err-eigenval}, the triangle inequality and another geometric series, we can bound the first sum as
\begin{equation}
|A_{i,t}|
= \frac{\norm{\v{y}}_2}{(\lambda_i + \sigma^2) t} \sum_{j=1}^t \del{1 - \eta \lambda_i (\lambda_i + \sigma^2)}^j
\leq \frac{\norm{\v{y}}_2}{\eta \lambda_i (\lambda_i + \sigma^2)^2 t}.
\end{equation}
For the second sum, we observe that we can change the order of summation to count each $\v\zeta_j$ only once, rather than once for each Polyak average.
This gives 
\begin{align}
B_{i,t} &= \frac{\eta}{t} \sum_{j=1}^t \sum_{q=0}^{j-1} \v{u}_i^\mathsf{T} \m{M}^q \v\zeta_{j-q} = \frac{\eta}{t} \sum_{j=1}^t \del{\sum_{q=0}^{t-j} \v{u}_i^\mathsf{T} \m{M}^q} \v\zeta_j, \\
&= \frac{\eta}{t} \sum_{j=1}^t \sum_{q=0}^{t-j} \del{1 - \eta \lambda_i (\lambda_i + \sigma^2)}^q \v{u}_i^\mathsf{T}\v\zeta_j.
\end{align}
where $\v{u}_i^\mathsf{T}\v\zeta_j\~[N](0, 1)$ are independent and identically distributed.

The variance of $B_{i,t}$ can be bounded by using another geometric series
{\allowdisplaybreaks
\begin{align}
\Var(B_{i,t}) &= \Var\del{\frac{\eta}{t} \sum_{j=1}^t \sum_{q=0}^{t-j} \del{1 - \eta \lambda_i (\lambda_i + \sigma^2)}^q \v{u}_i^\mathsf{T}\v\zeta_j}, \\
&= \frac{\eta^2}{t^2} \sum_{j=1}^t \Var\del{\sum_{q=0}^{t-j} \del{1 - \eta \lambda_i (\lambda_i + \sigma^2)}^q \v{u}_i^\mathsf{T}\v\zeta_j}, \\
&= \frac{\eta^2}{t^2} \sum_{j=1}^t \left[ \del{\sum_{q=0}^{t-j} \del{1 - \eta \lambda_i (\lambda_i + \sigma^2)}^q}^2 \ubr{\Var\del{\v{u}_i^\mathsf{T}\v\zeta_j}}_{=1} \right], \\
& \leq \frac{\eta^2}{t^2} \sum_{j=1}^t \del{\frac{1}{\eta \lambda_i (\lambda_i + \sigma^2)}}^2
= \frac{1}{\lambda_i^2 (\lambda_i + \sigma^2)^2 t}.
\end{align}}
Thus, from standard tail inequalities for Gaussian random variables, for any $\delta' > 0$, we have
\begin{equation}
\P\del{|B_{i,t}| \leq \sqrt{\frac{2}{\lambda_i^2 (\lambda_i + \sigma^2)^2 t}\log\frac{1}{\delta'}}} \geq 1 - \delta'.
\end{equation}
We then take $\delta' = \delta/|\c{I}|$ and apply the union bound for all indices in $\c{I}$.
Finally, by comparing moment generating functions, we can relax the Gaussian assumption to $G$-sub-Gaussian random variables.
To complete the claim, we combine the bounds for the two sums as
\begin{align}
|\v{u}_i^\mathsf{T}(\v{v}^* - \widebar{\v{v}}_t)|
&\leq |A_{i,t}| + |B_{i,t}|, \\
&\leq \frac{\norm{\v{y}}_2}{\eta \lambda_i (\lambda_i + \sigma^2)^2 t} + G\sqrt{\frac{2}{\lambda_i^2 (\lambda_i + \sigma^2)^2 t} \log\frac{|\c{I}|}{\delta}}, \\
&= \frac{1}{\lambda_i (\lambda_i + \sigma^2)} \del{\frac{\norm{\v{y}}_2}{\eta (\lambda_i + \sigma^2) t} + G\sqrt{\frac{2}{t} \log\frac{|\c{I}|}{\delta}}}, \\
&\leq \frac{1}{\lambda_i^2} \del{\frac{\norm{\v{y}}_2}{\eta \sigma^2 t} + G\sqrt{\frac{2}{t} \log\frac{|\c{I}|}{\delta}}},
\end{align}
which gives the claim.
\end{proof}

\subsection{Convergence in Reproducing Kernel Hilbert Space}
Using the preceding result, we will show that SGD converges quickly with respect to a certain seminorm.
Let $k$ be the kernel and let $\c{H}_k$ be its associated reproducing kernel Hilbert space.
\begin{definition}
Let $\m{K}_{\m{X}\m{X}} = \m{U}\m\Lambda\m{U}^\mathsf{T}$ be the eigendecomposition of the kernel matrix.
Define the spectral basis functions as
\begin{equation}
u^{(i)}(\.) = \sum_{j=1}^n \frac{U_{ji}}{\sqrt{\lambda_i}} k(\., \v{x}_j).
\end{equation}
\end{definition}
These are precisely the basis functions which arise in kernel principal component analysis.
We also introduce two subspaces of the RKHS: the span of the representer weights, and the span of a subset of spectral basis functions.
These are defined as follows.
\begin{definition}
Define the representer weight space
\begin{equation}
R_{k,\m{X}} = \Span\cbr[1]{k(\., \v{x}_i) : i=1,...,n} \subseteq \c{H}_k,
\end{equation}
equipped with the subspace inner product.
\end{definition}
\begin{definition}
Let $\c{I} \subseteq \{1,...,n\}$ be an arbitrary set of indices.
Define the interpolation subspace
\begin{equation}
R_{k,\m{X}}^{\c{I}} = \Span\cbr[1]{u^{(i)} : i \in \c{I}} \subseteq R_{k,\m{X}}.
\end{equation}
\end{definition}
We can use these subspaces to define a seminorm on the RKHS, which we will use to measure convergence rates of SGD.
\begin{definition}
Define the interpolation seminorm
\begin{equation}
\abs{f}_{R_{k,\m{X}}^{\c{I}}} = \norm[0]{\proj_{R_{k,\m{X}}^{\c{I}}} f}_{\c{H}_k}.
\end{equation}
\end{definition}
Here, $\proj_{\c{S}} f$ denotes orthogonal projection of a function $f$ onto the subspace $\c{S}$ of a Hilbert space.
To ease notation, in cases where we project onto the span of a single vector, we write the vector in place of $\c{S}$.
The first order of business is to understand the relationship between the spectral basis functions and representer weight space.
\begin{lemma}
The functions $u^{(i)}$ with $i=1,...,n$ form an orthonormal basis of $R_{k,\m{X}}$.
\end{lemma}
\begin{proof}
Let $i \neq j$.
Using the reproducing property, the definition of an eigenvector, the fact that $\v{u}_i$ and $\v{u}_j$ are eigenvectors of $\m{K}_{\m{X}\m{X}}$, and orthonormality of eigenvectors, we have
\begin{equation}
\innerprod[0]{u^{(i)}}{u^{(j)}}_{\c{H}_k} = \frac{\v{u}_i^\mathsf{T} \m{K}_{\m{X}\m{X}} \v{u}_j}{\sqrt{\lambda_i\lambda_j}} = \frac{\lambda_j}{\sqrt{\lambda_i\lambda_j}} \v{u}_i^\mathsf{T} \v{u}_j = 0.
\end{equation}
The fact that $u^{(i)}$ form a basis follows from the fact that they are linearly independent, that there are $n$ of them in total, and that this equals the dimensionality of $R_{k,\m{X}}$ by definition.
Finally, we calculate the norm
\begin{equation}
\norm[0]{u^{(i)}}_{\c{H}_k}^2 = \frac{\v{u}_i^\mathsf{T} \m{K}_{\m{X}\m{X}} \v{u}_i}{\lambda_i} = 1,
\end{equation}
which gives the claim.
\end{proof}
Next, we derive a change-of-basis formula between the canonical and spectral basis functions.
\begin{lemma}
\label{lem:change-of-basis}
For some $\v\theta\in\R^n$, let $\theta(\.) = \sum_{i=1}^n \theta_i k(\., \v{x}_i) = \sum_{j=1}^n w_j u^{(j)}(\.)$.
Then we have 
\begin{equation}
\v{w} = \m\Lambda^{\frac{1}{2}} \m{U}^\mathsf{T} \v\theta
\quad \text{ and } \quad
\norm{\theta}_{\c{H}_k}^2 = \v{w}^\mathsf{T}\v{w}.
\end{equation}
\end{lemma}
\begin{proof}
By expanding $u^{(j)}$, we have 
\begin{equation}
\theta(\.) = \sum_{j=1}^n w_j u^{(j)}(\.) = \sum_{i=1}^n \sum_{j=1}^n w_j \frac{U_{ij}}{\sqrt{\lambda_j}} k(\., \v{x}_i).
\end{equation}
Since $k(\., \v{x}_i)$ form a basis, this implies $\theta_i = \sum_{j=1}^n w_j \frac{U_{ij}}{\sqrt{\lambda_j}}$, or equivalently $\v\theta = \m{U} \m\Lambda^{-\frac{1}{2}} \v{w}$ and $\v{w} = \m\Lambda^{\frac{1}{2}} \m{U}^\mathsf{T} \v\theta$.
From orthonormality of $u^{(j)}$, we conclude that
\begin{equation}
\norm{\theta}_{\c{H}_k}^2 = \innerprod{\theta}{\theta}_{\c{H}_k} = \sum_{i=1}^n \sum_{j=1}^n w_i w_j \innerprod[0]{u^{(i)}}{u^{(j)}}_{\c{H}_k} = \sum_{i=1}^n w_i^2 = \v{w}^\mathsf{T}\v{w},
\end{equation}
which gives the second claim.
\end{proof}
\vspace{-.325cm}
This allows us to get an explicit expression for the previously introduced seminorm.
\begin{lemma}
\label{lem:seminorm-formula}
For $\theta\in R_{k,\m{X}}$, letting $\m\Lambda_{\c{I}} = \mathrm{diag}(\lambda_1 \mathbbold{1}_{1\in\c{I}},...,\lambda_i \mathbbold{1}_{i\in\c{I}},...,\lambda_n \mathbbold{1}_{n\in\c{I}})$, we have
\begin{equation}
\abs{\theta}_{R_{k,\m{X}}^{\c{I}}}^2 = \v\theta^\mathsf{T} \m{U} \m\Lambda_{\c{I}} \m{U}^\mathsf{T} \v\theta.
\end{equation}
\end{lemma}
\begin{proof}
First, we decompose $\theta$ in the orthonormal basis above, giving $\theta(\.) = \sum_{j=1}^n w_j u^{(j)}(\.)$.
By orthonormality, the definition of $R_{k,\m{X}}^{\c{I}}$, and properties of projections, we have 
\begin{equation}
(\proj_{R_{k,\m{X}}^{\c{I}}} \theta)(\.) = \sum_{j\in\c{I}} w_j u^{(j)}(\.).
\end{equation}
Therefore, again using orthonormality, we obtain
\begin{equation}
\abs{\theta}_{R_{k,\m{X}}^{\c{I}}}^2 = \sum_{j\in\c{I}} w_j^2 = \v\theta^\mathsf{T} \m{U} \m\Lambda_{\c{I}} \m{U}^\mathsf{T} \v\theta,
\end{equation}
which gives the claim.
\end{proof}
\vspace{-.325cm}
We claim that our variant of SGD converges quickly with respect to this seminorm.
\begin{proposition}
\label{prop:sgd-error-bound-rkhs}
Let $\mu_{(\.)\given \v{y}}$ be the exact posterior predictive mean.
Let $\hat{\mu}_{(\.) \given \v{y}}$ be the posterior predictive mean obtained by SGD under the assumptions of \Cref{lem:sgd-error-bound}.
For any $\c{I}$, we have 
\begin{equation}
\abs{\mu_{(\.) \given \v{y}} - \hat{\mu}_{(\.) \given \v{y}}}_{R_{k,\m{X}}^{\c{I}}} \leq \del{\frac{\norm{\v{y}}_2}{\eta\sigma^2t} + G\sqrt{\frac{2}{t} \log\frac{|\c{I}|}{\delta}}} \sqrt{\sum_{i\in\c{I}} \frac{1}{\lambda_i^3}}.
\end{equation}
\end{proposition}
\begin{proof}
Using \Cref{lem:seminorm-formula} and linearity, we have 
\begin{align}
\abs{\mu_{(\.) \given \v{y}} - \hat{\mu}_{(\.) \given \v{y}}}_{R_{k,\m{X}}^{\c{I}}}
&= \abs{\sum_{i=1}^n (v^*_i - \hat{v}_i) k(\., \v{x}_i)}_{R_{k,\m{X}}^{\c{I}}}, \\
&= \sqrt{(\v{v}^* - \hat{\v{v}})^\mathsf{T} \m{U}\m\Lambda_{\c{I}} \m{U}^\mathsf{T} (\v{v}^* - \hat{\v{v}})}, \\
&= \sqrt{\sum_{i\in\c{I}} \abs{\v{u}_i^\mathsf{T}(\v{v}^* - \hat{\v{v}})}^2 \lambda_i}, \\
&\leq \sqrt{\sum_{i\in\c{I}} \abs{\frac{1}{\lambda_i^2} \del{\frac{\norm{\v{y}}_2}{\eta\sigma^2t} + G\sqrt{\frac{2}{t} \log\frac{|\c{I}|}{\delta}}}}^2 \lambda_i}, \\
&= \del{\frac{\norm{\v{y}}_2}{\eta\sigma^2t} + G\sqrt{\frac{2}{t} \log\frac{|\c{I}|}{\delta}}} \sqrt{\sum_{i\in\c{I}} \frac{1}{\lambda_i^3}},
\end{align}
where the inequality is obtained using the result of \Cref{lem:sgd-error-bound}, yielding the claim.
\end{proof}
Our main claim follows directly from the established framework.
\PropSGD*
\begin{proof}
Choose $\c{I}$ to be a singleton and apply \Cref{prop:sgd-error-bound-rkhs}.
Replacing $|\c{I}|$ with $n$ in the final bound makes the claim hold with probability $1-\delta$ for all $i=1,...,n$ simultaneously.
\end{proof}
Therefore, SGD converges fast with respect to this seminorm.
In particular, taking $\c{I}$ to be the full index set, a position-dependent pointwise convergence bound follows.
\begin{corollary}
Under the assumptions of \Cref{lem:sgd-error-bound}, we have 
\begin{equation}
|\mu_{(\.) \given \v{y}} - \hat{\mu}_{(\.) \given \v{y}}| \leq \del{\frac{\norm{\v{y}}_2}{\eta\sigma^2t} + G\sqrt{\frac{2}{t} \log\frac{|\c{I}|}{\delta}}} \sum_{i=1}^n \frac{1}{\sqrt{\lambda_i^3}} |u^{(i)}(\.)|.
\end{equation}
\end{corollary}
\begin{proof}
Using \Cref{lem:change-of-basis}, we write 
\begin{equation}
\mu_{(\.) \given \v{y}} - \hat{\mu}_{(\.) \given \v{y}}
= \sum_{j=1}^n (v^*_j - \hat{v}_j) k(\., \v{x}_i)
= \sum_{i=1}^n (w^*_i - \hat{w}_i) u^{(i)}(\.),
\end{equation}
where $\v{w}^* - \hat{\v{w}} = \m\Lambda^{\frac{1}{2}} \m{U}^\mathsf{T} (\v{v}^* - \hat{\v{v}})$.
Applying \Cref{lem:sgd-error-bound} with $\c{I} = \{1,...,n\}$, we conclude 
\begin{align}
|\mu_{(\.) \given \v{y}} - \hat{\mu}_{(\.) \given \v{y}}|
&\leq \sum_{i=1}^n |w^*_i - \hat{w}_i| |u^{(i)}(\.)|
= \sum_{i=1}^n \sqrt{\lambda_i}|\v{u}_i^\mathsf{T} (\v{v}^* - \hat{\v{v}})| |u^{(i)}(\.)|, \\
&\leq \sum_{i=1}^n \sqrt{\lambda_i} \abs{\frac{1}{\lambda_i^2} \del{\frac{\norm{\v{y}}_2}{\eta\sigma^2t} + G\sqrt{\frac{2}{t} \log\frac{|\c{I}|}{\delta}}}} |u^{(i)}(\.)|, \\
&= \del{\frac{\norm{\v{y}}_2}{\eta\sigma^2t} + G\sqrt{\frac{2}{t} \log\frac{|\c{I}|}{\delta}}} \sum_{i=1}^n \frac{1}{\sqrt{\lambda_i^3}} |u^{(i)}(\.)|,
\end{align}
which gives the claim.
\end{proof}
Examining the expression, the approximation error will be high at locations $\v{x} \in \c{X}$ where $u^{(i)}(\v{x})$ corresponding to tail eigenfunctions is large.
We proceed to analyse when this occurs.

\subsection{Courant-Fischer in the Reproducing Kernel Hilbert Space}
\label{apd:courant_in_input_space}
Since the preceding seminorm is induced by a projection within the RKHS, we want to understand what kind of functions the corresponding subspace contains.
To get a qualitative view, the basic idea is to lift the Courant-Fischer characterisation of eigenvalues and eigenvectors to the RKHS.
To this end, we will need the following variant of the min-max theorem, which is useful for simultaneously describing multiple eigenvectors as solutions to a sequence of Rayleigh quotient optimisation problems, where the next optimisation problem is performed on a subspace which ensures its solution is orthogonal to all previous solutions.
As before, we consider eigenvalues in decreasing order, namely $\lambda_1 \geq .. \geq \lambda_n \geq 0$.
\begin{corollary}
\label{res:courant-fischer}
Let $\m{A}$ be a positive semi-definite matrix.
Then 
\begin{equation}
\lambda_i = \max_{\substack{\v{u} \in \R^n\takeaway\{\v{0}\}\\\v{u}^\mathsf{T} \v{u}_j = 0, \forall j < i}} \frac{\v{u}^\mathsf{T}\m{A}\v{u}}{\v{u}^\mathsf{T}\v{u}},
\end{equation}
where the eigenvectors $\v{u}_i$ are the respective maximisers.
\end{corollary}
\begin{proof}
\citet[Theorem 4.2.2]{horn12}, modified appropriately to handle eigenvalues in decreasing order, and where we choose $i_1,...,i_k = 1,...,n-i+1$.
\end{proof}
To prove the main claim, we first consider the representer weight space $R_{k,\m{X}}$ and then $\c{H}_k$.
\begin{lemma}
\label{lem:courant-fischer-rkhs}
We have 
\begin{equation}
u^{(i)}(\.) = \argmax_{u \in R_{k,\m{X}}} \cbr{\sum_{i=1}^n u(\v{x}_i)^2 : \norm{u}_{\c{H}_k} = 1, \innerprod[0]{u}{u^{(j)}}_{\c{H}_k} = 0, \forall j < i}.
\end{equation}
\end{lemma}
\begin{proof}
Define $\v{w}_j = \m{U} \m\Lambda^{-\frac{1}{2}} \v{u}_j$, where we recall that $\v{u}_j$ are the eigenvectors of the kernel matrix.
From \Cref{res:courant-fischer}, the reproducing property, and the explicit form of the RKHS inner product on $R_{k,\m{X}}$ in the basis defined by the canonical basis functions, we conclude that
\begin{align}
\lambda_i &= \max_{\substack{\v{u}\in\R^n\takeaway\{\v{0}\}\\\v{u}^\mathsf{T}\v{u}_j = 0, \forall j<i}} \frac{\v{u}^\mathsf{T} \m{K}_{\m{X}\m{X}} \v{u}}{\v{u}^\mathsf{T}\v{u}}
&=& \max_{\substack{\v{w}\in\R^n\takeaway\{\v{0}\}\\\v{w}^\mathsf{T}\m{K}_{\m{X}\m{X}}\v{w}_j = 0,\forall j<i}} \frac{\v{w}^\mathsf{T} \m{K}_{\m{X}\m{X}}^2 \v{w}}{\v{w}^\mathsf{T} \m{K}_{\m{X}\m{X}} \v{w}}, \\
&= \max_{\substack{\v{w}\in\R^n\takeaway\{\v{0}\}\\\v{w}^\mathsf{T}\m{K}_{\m{X}\m{X}}\v{w}_j = 0,\forall j<i}} \frac{\norm{\m{K}_{\m{X}\m{X}} \v{w}}_2^2}{\v{w}^\mathsf{T} \m{K}_{\m{X}\m{X}} \v{w}}
&=& \max_{\substack{\v{w}\in\R^n\takeaway\{\v{0}\}\\\v{w}^\mathsf{T}\m{K}_{\m{X}\m{X}}\v{w}_j = 0,\forall j<i}} \frac{\norm{\sum_{i=1}^n w_i k(\m{X},\v{x}_i)}_2^2}{\norm{\sum_{i=1}^n w_i k(\., \v{x}_i)}_{\c{H}_k}^2}, \\
&= \max_{\substack{u \in R_{k,\m{X}}\takeaway\{0\}\\\innerprod[0]{u}{u^{(j)}}_{\c{H}_k} = 0, \forall j < i}} \frac{\norm{u(\m{X})}_2^2}{\norm{u}_{\c{H}_k}^2}
&=& \max_{\substack{u \in R_{k,\m{X}}\\\norm{u}_{\c{H}_k}=1\\\innerprod[0]{u}{u^{(j)}}_{\c{H}_k} = 0, \forall j < i}} \sum_{i=1}^n u(\v{x}_i)^2,
\end{align}
which gives the claim.
\end{proof}
Next, we use an orthogonality argument, which relies on the projection theorem for Hilbert spaces, to extend the optimisation problem to $\c{H}_k$.
\begin{proposition}
We have 
\begin{equation}
u^{(i)}(\.) = \argmax_{u \in \c{H}_k} \cbr{\sum_{i=1}^n u(\v{x}_i)^2 : \norm{u}_{\c{H}_k} = 1, \innerprod[0]{u}{u^{(j)}}_{\c{H}_k} = 0, \forall j < i}.
\end{equation}
\end{proposition}
\begin{proof}
Let $\c{H}_k^{(i)}$ be the orthogonal complement of the subspace $\Span\{u^{(j)}, j=1,...,i-1\}$ in $\c{H}_k$ and note by feasibility that $u^{(i)} \in \c{H}_k^{(i)}$.
Consider the decomposition of $u^{(i)}$ into its projection onto an orthogonal complement with respect to the finite-dimensional subspace $\c{H}_k^{(i)} \^ R_{k,\m{X}}$, namely
\begin{equation}
u^{(i)}(\.) = u^\parallel(\.) + u^\orth(\.),
\end{equation}
where this notation is defined as $u^\parallel = \proj_{\c{H}_k^{(i)}\^R_{k,\m{X}}} u^{(i)}$ and $u^\orth = u^{(i)} - u^\parallel$.
Observe that
\begin{equation}
u^\orth(\v{x}_i) = \innerprod[0]{u^\orth}{k(\., \v{x}_i)}_{\c{H}_k} = \innerprod[0]{\proj_{\c{H}_k^{(i)}} u^\orth}{k(\., \v{x}_i)}_{\c{H}_k} = \innerprod[0]{u^\orth}{\proj_{\c{H}_k^{(i)}} k(\., \v{x}_i)}_{\c{H}_k} = 0,
\end{equation}
where the first equality follows from the reproducing property, the second equality follows because $u^\orth \in \c{H}_k^{(i)}$, the third equality follows since the projection is orthogonal and thus self-adjoint, and the final equality follows because $u^\orth$ by definition lives in the orthogonal complement of $\c{H}_k^{(i)}\^R_{k,\m{X}}$ within $\c{H}_k^{(i)}$, with the subspace inner product.
Moreover, by the projection theorem and the fact that $\c{H}_k^{(i)}$ inherits its norm from $\c{H}_k$, note that 
\begin{equation}
\norm[0]{u^{(i)}}_{\c{H}_k}^2 = \norm[0]{u^\parallel}_{\c{H}_k}^2 + \norm[0]{u^\orth}_{\c{H}_k}^2.
\end{equation}
Suppose that $u^\orth \neq 0$.
This implies that $\norm[0]{u^\orth}_{\c{H}_k} > 0$.
From the above, we have 
\begin{equation}
\sum_{i=1}^n u^{(i)}(\v{x}_i)^2 = \sum_{i=1}^n u^\parallel(\v{x}_i)^2
\quad \text{ and } \quad
\norm[0]{u^\parallel}_{\c{H}_k} < \norm[0]{u^{(i)}}_{\c{H}_k}.
\end{equation}
Define the function 
\begin{equation}
u'(\.) = \frac{\norm[0]{u^{(i)}}_{\c{H}_k}}{\norm[0]{u^\parallel}_{\c{H}_k}} u^\parallel(\.),
\end{equation}
and note that $\norm{u'} = 1$, and $\innerprod[0]{u'}{u^{(j)}} = 0$ for all $j < i$ since $u^\parallel \in \c{H}_k^{(i)}$, which makes $u'$ a feasible solution to the optimisation problem considered.
At the same time, it satisfies 
\begin{equation}
\sum_{i=1}^n u'(\v{x}_i)^2 = \frac{\norm[0]{u^{(i)}}_{\c{H}_k}^2}{\norm[0]{u^\parallel}_{\c{H}_k}^2}\sum_{i=1}^n u^\parallel(\v{x}_i)^2 > \sum_{i=1}^n u^\parallel(\v{x}_i)^2 = \sum_{i=1}^n u^{(i)}(\v{x}_i)^2,
\end{equation}
which shows that $u^{(i)}$, contrary to its definition, is not actually optimal.
We conclude that $u^\orth = 0$, which means that $u^{(i)} \in \c{H}_k^{(i)} \^ R_{k,\m{X}}$.
The claim follows from \Cref{lem:courant-fischer-rkhs}.
\end{proof}
This means that the top eigenvalues can be understood as the largest possible squared sums of canonical basis functions evaluated at the data, subject to RKHS-norm constraints.
In situations where $k(\., \v{x})$ is continuous, non-negative, and decays as one moves away from the data, it is intuitively clear that such sums will be maximised by placing weight on canonical basis functions which cover as much of the data as possible.
This mirrors the RKHS view of kernel principal component analysis.

\end{appendices}

% *************************************** Index ********************************
% \printthesisindex % If index is present

\end{document}